\documentclass[table]{article} 
\usepackage[final]{neurips_2022} 

\usepackage{times}

\usepackage[utf8]{inputenc} 
\usepackage[T1]{fontenc}    

\usepackage{amsmath,amsfonts,bm}

\def\eqref#1{equation~\ref{#1}}

\def\1{\bm{1}}

\def\vw{{\bm{w}}}
\def\vx{{\bm{x}}}

\DeclareMathAlphabet{\mathsfit}{\encodingdefault}{\sfdefault}{m}{sl}
\SetMathAlphabet{\mathsfit}{bold}{\encodingdefault}{\sfdefault}{bx}{n}

\newcommand{\E}{\mathbb{E}}

\newcommand{\R}{\mathbb{R}}

\newcommand{\Var}{\mathrm{Var}}

\usepackage{xcolor}
\definecolor{halfgray}{gray}{0.55}
\definecolor{webgreen}{rgb}{0,.5,0}
\definecolor{webbrown}{rgb}{.6,0,0}
\definecolor{webblue}{rgb}{0,0,0.930}
\definecolor{RoyalBlue}{cmyk}{1, 0.50, 0, 0}

\usepackage{hyperref}
\usepackage{url}

\usepackage{graphicx}
\usepackage{booktabs}       
\usepackage{nicefrac}       
\usepackage{microtype}
\usepackage{amssymb}
\usepackage{amsmath}
\usepackage{amsthm}
\usepackage{amsfonts}
\usepackage{algorithm, algpseudocode}
\usepackage{wrapfig}
\usepackage{subcaption}
\usepackage{listings}
\usepackage{siunitx}
\usepackage{placeins}
\usepackage{wrapfig} 
\usepackage{xfp}
\usepackage{ifthen}
\usepackage{commath}
\usepackage{doi}

\usepackage{xcolor}
\definecolor{halfgray}{gray}{0.55}
\definecolor{webgreen}{rgb}{0,.5,0}
\definecolor{webbrown}{rgb}{.6,0,0}
\definecolor{webblue}{HTML}{0000ED}

\definecolor{linkcolor}{HTML}{991408}  
\definecolor{citecolor}{HTML}{2E7E2A}  
\definecolor{filecolor}{HTML}{131877}  
\definecolor{menucolor}{HTML}{727500}  
\definecolor{runcolor} {HTML}{137776}  
\definecolor{urlcolor} {HTML}{0a2bbf}  

\hypersetup{%
    colorlinks=true,%
    ,linkcolor=linkcolor%
    ,citecolor=citecolor%
    ,filecolor=filecolor%
    ,menucolor=menucolor%
    ,runcolor=runcolor%
    ,urlcolor=urlcolor%
}

\newtheorem{proposition}{Proposition}[section]

\newtheorem{corollary}{Corollary}[proposition]

\newcommand{\aref}[1]{\hyperref[#1]{Appendix~\ref*{#1}}}
\renewcommand{\aref}[1]{\hyperref[#1]{\S\ref*{#1}}}

\newcommand{\opn}[1]{\operatorname{#1}}
\DeclareMathOperator{\sgn}{sgn}                         
\DeclareMathOperator{\PP}{P\/}  

\newcommand{\NSYS}[1]{\mathbb{#1}}                                  
\newcommand{\Real}{\NSYS{R}}                                           

\newcommand\independent{\protect\mathpalette{\protect\independenT}{\perp}}
\def\independenT#1#2{\mathrel{\rlap{$#1#2$}\mkern2mu{#1#2}}}

\newcommand{\mbf}[1]{\textbf{\underline{#1}}}
\newcommand{\mbs}[1]{\underline{#1}}
\newcommand{\mbns}[1]{\textit{#1}}

\usepackage{transparent}
\newcommand{\wpm}[1]{{\scriptstyle\transparent{0.9} \pm #1}}  

\newcommand{\heatmapcell}[4]{%
	\cellcolor{black!\fpeval{min(100,max(0,100*(#1 - #3)/(#4 - #3)))}}%
	\ifthenelse{\fpeval{round(100*(#1 - #3)/(#4 - #3))}>50}{\color{white}}{}%
	{#2}%
}
\newcommand{\heatmapcellw}[4]{%
	\cellcolor{black!\fpeval{min(100,max(0,100*(#4 - #1)/(#4 - #3)))}}%
	\ifthenelse{\fpeval{round(100*(#4 - #1)/(#4 - #3))}>50}{\color{white}}{}%
	{#2}%
}

\newcommand{\hider}[1]{}

\newcommand{\revision}[1]{{#1}}

\title{Logical Activation Functions: Logit-space equivalents of Probabilistic Boolean Operators}

\author{%
Scott C.~Lowe$^{1,2,*}$, Robert Earle$^{1,2}$, Jason d'Eon$^{1,2}$, Thomas Trappenberg$^{1}$, Sageev Oore$^{1,2}$\\
\AND\\
$^1$Faculty of Computer Science \\
Dalhousie University \\
Halifax, Nova Scotia \\
Canada \\
\\
$^*$Correspondence: scottclowe@gmail.com
\And\\
$^2$Vector Institute for Artificial Intelligence \\
Toronto, Ontario \\
Canada
}

\begin{document}

\maketitle

\begin{abstract}



The choice of activation functions and their motivation is a long-standing issue within the neural network community.
Neuronal representations within artificial neural networks are commonly understood as logits, representing the log-odds score of presence of features within the stimulus.
We derive logit-space operators equivalent to probabilistic Boolean logic-gates AND, OR, and XNOR for independent probabilities.
Such theories are important to formalize more complex dendritic operations in real neurons, and these operations can be used as activation functions within a neural network, introducing probabilistic Boolean-logic as the core operation of the neural network.
Since these functions involve taking multiple exponents and logarithms, they are computationally expensive and not well suited to be directly used within neural networks.
Consequently, we construct efficient approximations named $\text{AND}_\text{AIL}$ (the AND operator Approximate for Independent Logits), $\text{OR}_\text{AIL}$, and $\text{XNOR}_\text{AIL}$, which utilize only comparison and addition operations, have well-behaved gradients, and can be deployed as activation functions in neural networks.
Like MaxOut, $\text{AND}_\text{AIL}$ and $\text{OR}_\text{AIL}$ are generalizations of ReLU to two-dimensions.
While our primary aim is to formalize dendritic computations within a logit-space probabilistic-Boolean framework, we deploy these new activation functions, both in isolation and in conjunction to demonstrate their effectiveness on a variety of tasks including image classification, transfer learning, abstract reasoning, and compositional zero-shot learning.

\end{abstract}


\section{Introduction}

Non-linear activation functions are essential in artificial neural networks \revision{(ANNs)} to form higher-order representations, since otherwise the network would be degeneratively equivalent to a single linear layer. Most activation functions represent a simple non-linearity despite evidence of much more complex non-linear integration and computations in dendrites \citep{Hentschel+Fine2004, London2005, Payeur2019}. For example, \citet{Gidon2020} recently demonstrated that a single biological neuron can compute the XOR of its inputs, a property long known to be lacking from artificial neurons \citep{perceptrons}. \hider{Can we mimic this complexity in an ANN's activation function in a scalable way?} While simple activation functions work in \revision{ANNs
} in the sense that more complex operations can be formed from the combination of several layers of neurons, understanding the function and impact of advanced dendritic operations in networks is important.
\revision{In this work, we add some of this behaviour to neural activations, corresponding to shifting some of the network's complexity from its global structure to the neural level. There is far more complexity in biological neurons than in the abstractions that we consider here, but we make a step in the direction of using more complex neurons in ANNs.} \revision{We also develop a theoretical underpinning for higher-order activation functions (e.g. like MaxOut, \citealp{maxout}) in a probabilistic framework.}
\revision{We hypothesize that such architectures will be more parameter-efficient in situations where their assumptions hold.}


Neuronal representations within ANNs are commonly understood as logits, representing the log-odds score of presence (versus absence) of features within the stimulus.
\revision{From a Bayesian perspective, a ReLU-like operation corresponds to the removal of all evidence for the lack of a feature. Under the logit interpretation of ANN potentiations, this seems unreasonable. 
This can be seen in some of the ways we interact with neural networks: we must apply batch-norm before activations and not after them; when doing transfer learning from an embedding space we must use pre-activation potentiations instead of activations. Can ANNs do better if we design an architecture which treats potentiations as logits?}

We \revision{thus} derive logit-space operators equivalent to probabilistic Boolean logic-gates AND, OR, and XNOR for independent probabilities.
\revision{Networks constructed in this way can be interpreted as \textbf{performing logical operations} using \textbf{point-estimates of probabilities}, in a similar manner to a Bayesian network.
This brings operations from the symbolism framework of AI (which is more similar to deliberative thinking, or System 2 of \citealp{Kahneman2011}) into the connectionist framework of ANNs (which is more like instinctive, System 1, thinking).
}
We \revision{also} construct computationally feasible approximations to these functions with well-behaved gradients.
These new activation functions, which are generalizations of ReLU to two-dimensions, are then applied on benchmark datasets to demonstrate their effectiveness on a diverse range of tasks including image classification, transfer learning, abstract reasoning, and compositional zero-shot learning.
The new principled approach we present introduces new ways to redistribute computation from the network into the neuronal mechanisms, and build more parameter-efficient models.
We demonstrate \revision{the} effective\revision{ness of} activation functions based on these ideas\revision{,} and expect future work to build on this.

\section{Background}

\revision{Early} artificial neural networks featured either logistic-sigmoid or tanh as their activation function, 
\revision{motivated by} the idea that each layer of the network is building another layer of abstraction of the stimulus space from the last layer.
Each neuron in a layer identifies whether certain properties or features are present within the stimulus, and the pre-activation (potentiation) value of the neuron indicates a score or logit for the presence of that feature.
The sigmoid function, $\sigma(x)=\nicefrac{1}{(1+e^{-x})}$, was hence a natural choice of activation function, since as with logistic regression, this will convert the logits of features into probabilities.
There is evidence that this interpretation has merit.
Previous work has been done to identify which features neurons are tuned to. Examples include LSTM neurons tracking quotation marks, line length, and brackets
\citep{karpathy2015visualizing} 
\revision{and} sentiment \citep{radford2017learning}; 
\revision{projecting} features back to the input space to view them \citep{olah2017feature}; and interpretable combinations of neural activities \citep{olah2020zoom}.

Sigmoidal activation functions are no longer commonly used within machine learning between layers of representations, though sigmoid is still widely used for gating operations which scale the magnitude of other features in an attention-like manner.
The primary disadvantage of the sigmoid activation function is its vanishing gradient --- as the potentiation rises, activity converges to a plateau, and hence the gradient goes to zero.
This prevents feedback information propagating back through the network from the loss function to the early layers of the network, which consequently prevents it from learning to complete the task.

The Rectified Linear Unit activation function \citep{neocognitron,Jarrett2009,Nair2010}, $\opn{ReLU}(x)=\opn{max}(0,x)$, does not have this problem, since in its non-zero regime it has a gradient of 1.
Another advantage of ReLU is it has very low computational demands.
Since it is both effective and efficient, it has proven to be a highly popular choice of activation function.
The chief drawback to ReLU is that it has no sensitivity to changes across half of its input domain, which prevents updates on stimuli which trigger its ``off'' state and can even lead to neuronal death\footnote{Though this problem is very rare when using BatchNorm to stabilize feature distributions.}.
Variants of ReLU have emerged, aiming to smooth out its transition between domains and provide a gradient in its inactive regime.
These include ELU ~\citep{elu}, CELU \citep{celu}, SELU \citep{selu}, GELU \citep{gelu}, SiLU \citep{silu,swish}, and Mish \citep{mish}.
However, all these activation functions still bear the general shape of ReLU and truncate negative logits.

Fuzzy logic operators are generalizations of Boolean logic operations to continuous variables, using rules similar to applying logical operators in probability space.
Prior work has explored networks of fuzzy logic operations, including some which use an activation function constituting a learnable interpolation between fuzzy logic operators \citep{Godfrey2017}.
The activation functions we introduce \revision{here} are motivated similarly to fuzzy logic operators, but designed to operate in logit space instead of in probability space, which better reflects the behaviour and space of pre-activation units.

\revision{
Bayesian neural networks (BNNs) are probabilistic models which can represent the uncertainty in their model parameters, and hence uncertainty in their outputs.
This differs from standard ANNs which use only a single point-estimate of each model parameter, and hence also in their neural activations.
By making their priors explicit and modelling their uncertainty, Bayesian networks can be better calibrated and less vulnerable to overfitting than ANNs.
However, BNNs are more challenging to train, and cannot reasonably be scaled up to the deep architectures which are possible with standard ANNs and necessary in order to learn a sufficiently complex model to solve highly complex tasks.
}

In this work we contribute to the theoretical underpinning of neural activation functions by developing activation functions based on the principle that neurons encode logits --- scores that represent the presence of features in the log-odds space.
In \autoref{s:derivation} we derive and define these functions in detail for different logical operators, and then consider their performance on numerous task types including parity (\autoref{s:parity}), image classification (\autoref{s:mnist}~and~\autoref{s:cifar}), transfer learning (\autoref{s:transfer}), abstract reasoning (\aref{s:abstractreasoning}), soft-rule guided classification as exemplified by the Bach chorale dataset (\autoref{s:chorales}), and compositional zero-shot learning~(\aref{s:czsl}).
These tasks were selected to (1) survey the performance of the new activations on existing benchmark tasks, and (2) evaluate their performance on tasks which we suspect in particular may require logical reasoning and hence benefit from activation functions which apply these logical operations to logits.

\section{Derivation}
\label{s:derivation}

Manipulation of probabilities in logit-space is known to be more efficient for many calculations.
For instance, the log-odds form of Bayes' Rule (\autoref{eq:bayesrulelogit}) states that the posterior logit equals the prior logit plus the log of the likelihood ratio for the new evidence (the log of the Bayes factor).
Thus, working in logit-space allows us to perform Bayesian updates on many sources of evidence simultaneously, merely by summing together the log-likelihood ratios for the evidence (see \aref{a:bayesrule}).
A weighted sum may be used if the amount of credence given to the sources differs --- and this is precisely the operation performed by a linear layer in a neural network.

When considering sets of probabilities, a natural operation is \revision{to measure} the joint probability of two events both occurring--- the AND operation.
Suppose our input space is $x\in[0,1]^2$, and the goal is to output $y>0$ if $x_i=1 \,\forall\, i$, and $y<0$ otherwise, using model with a weight vector $w$ and bias term $b$, such that $y=w^T x + b$.
This can be trivially solved with the weight matrix $w=[1,1]$ and bias term $b=-1.5$.
However, since this is only a linear separator, the solution can not generalize to the case $y>0$ iff $x_i>0 \,\forall\, i$.
Similarly, let us consider how the OR function is solved with a linear layer.
Our goal is to output $y>0$ if $\exists \, x_i=1$, and $y<0$ otherwise.
The binary case can be trivially solved with the weight matrix $w=[1,1]$ and bias term $b=-0.5$.
The difference between the solution for OR and the solution for AND is only an offset to our bias term.
In each case, if the input space is expanded beyond binary to $\Real^2$, the output can be violated by changing only one of the arguments.



\subsection{AND}

Suppose we are given $x$ and $y$ as the logits for the presence (vs absence) of two events, $X$ and $Y$.
These logits have equivalent probability values, which can be obtained using the sigmoid function, $\sigma(u) = (1 + e^{-u})^{-1}$.
Let us assume that the events $X$ and $Y$ are independent of each other.
In this case, the probability of both events occurring (the joint probability) is $\PP(X, Y) = \PP(X \land Y) = \PP(X) \, \PP(Y) = \sigma(x) \, \sigma(y)$.
\revision{}
However, we wish to remain in logit-space, and must determine the logit of the joint probability, $\opn{logit}(\PP(X, Y)$.
This is given by
\begin{align}
    \opn{AND_{IL}} :=
    \opn{logit}(\PP(X \land Y)_{x \independent y})
    &= \log \left( \frac{p}{1 - p} \right), \; \text{where} \; p = \sigma(x) \, \sigma(y), \nonumber \\
    &= \log \left( \frac{\sigma(x) \, \sigma(y)}{1 - \sigma(x) \, \sigma(y)} \right), \label{eq:and_il}
\end{align}
which we coin as $\opn{AND_{IL}}$, the AND operator for independent logits (IL).
This 2d function is illustrated as a contour plot (\autoref{fig:AND}, left \revision{panel}).
Across the plane, the order of magnitude of the output is the same as at least one of the two inputs, scaling approximately linearly.

The approximately linear behaviour of the function is suitable for use as an activation function (no vanishing gradient), however taking exponents and logs scales poorly from a computational perspective.
Hence, we developed a computationally efficient approximation as follows.
Observe that we can loosely approximate $\opn{AND_{IL}}$ with the minimum function (\autoref{fig:AND}, right panel).
This is equivalent to assuming the probability of \textbf{both} $X$ and $Y$ being true equals the probability of the \textbf{least likely} of $X$ and $Y$ being true --- a na\"ive approximation which holds well in three quadrants of the plane, but overestimates the probability when both $X$ and $Y$ are unlikely.
In this quadrant, when both $X$ and $Y$ are both unlikely, a better approximation for $\opn{AND_{IL}}$ is the sum of their logits.

We thus propose $\opn{AND_{AIL}}$, a linear-approximate AND function for independent logits (AIL, i.e. approximate IL).
\begin{equation}
    \opn{AND_{AIL}}(x, y) :=
        \begin{cases}
            x + y, & x < 0, \, y < 0 \\
            \opn{min}(x, y), & \text{otherwise}
        \end{cases}
    \label{eq:and_ail}
\end{equation}
As shown in \autoref{fig:AND} (left, middle), we observe that their output values and shape are very similar.

\begin{figure}[h]
    \centering
    \includegraphics[scale=0.47]{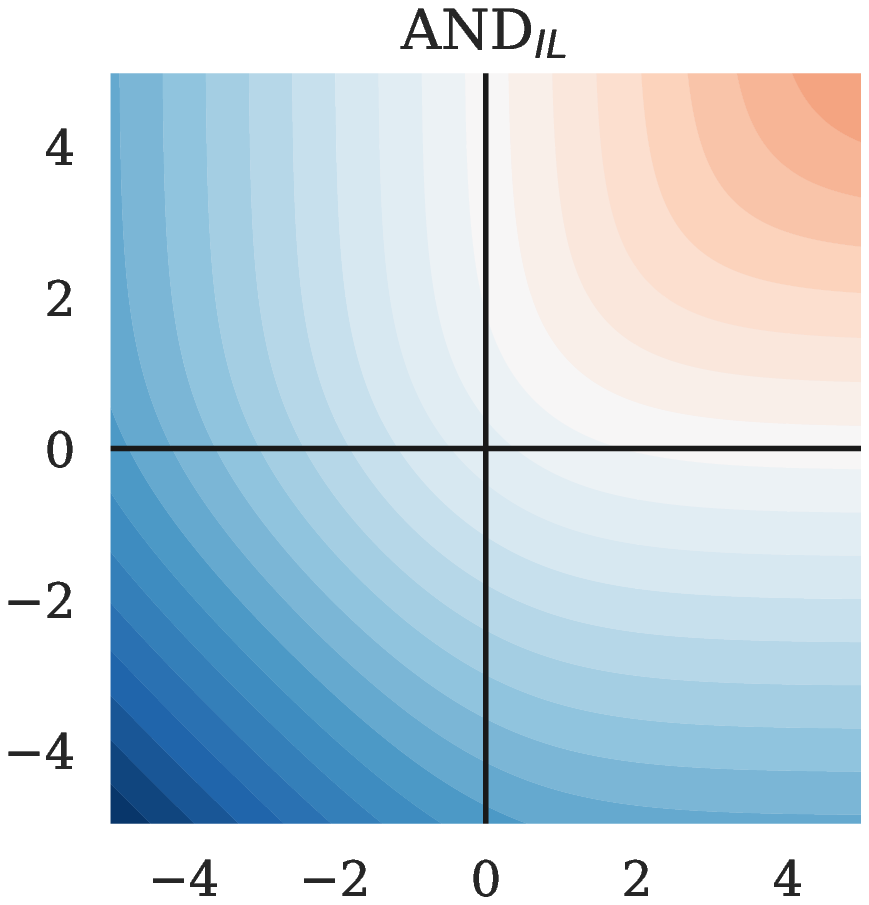}
    \includegraphics[scale=0.47]{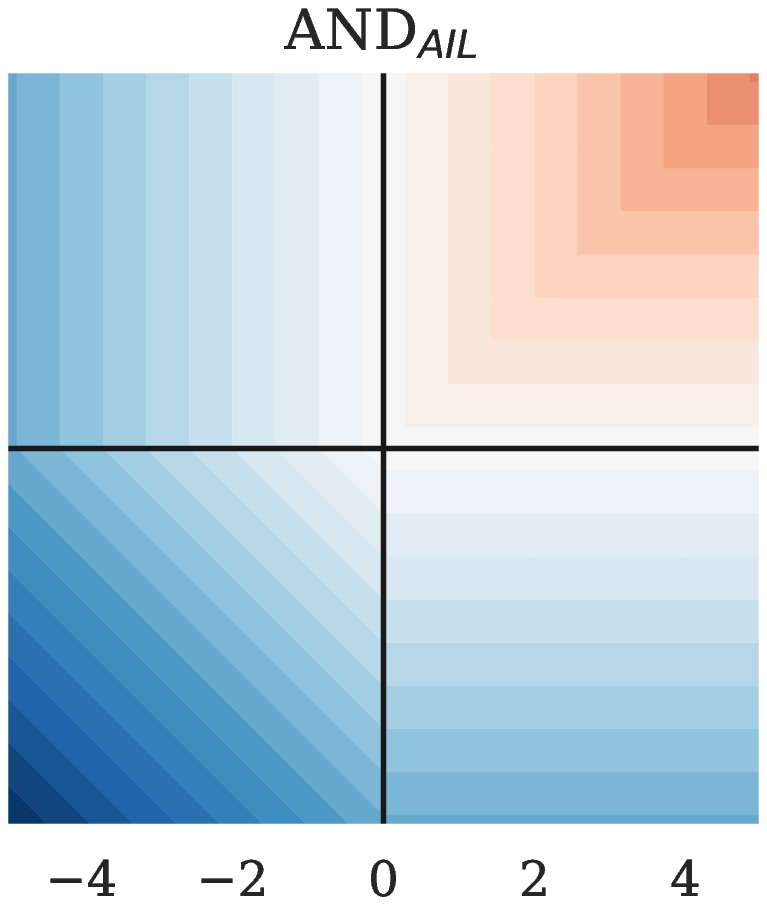}
    \includegraphics[scale=0.47]{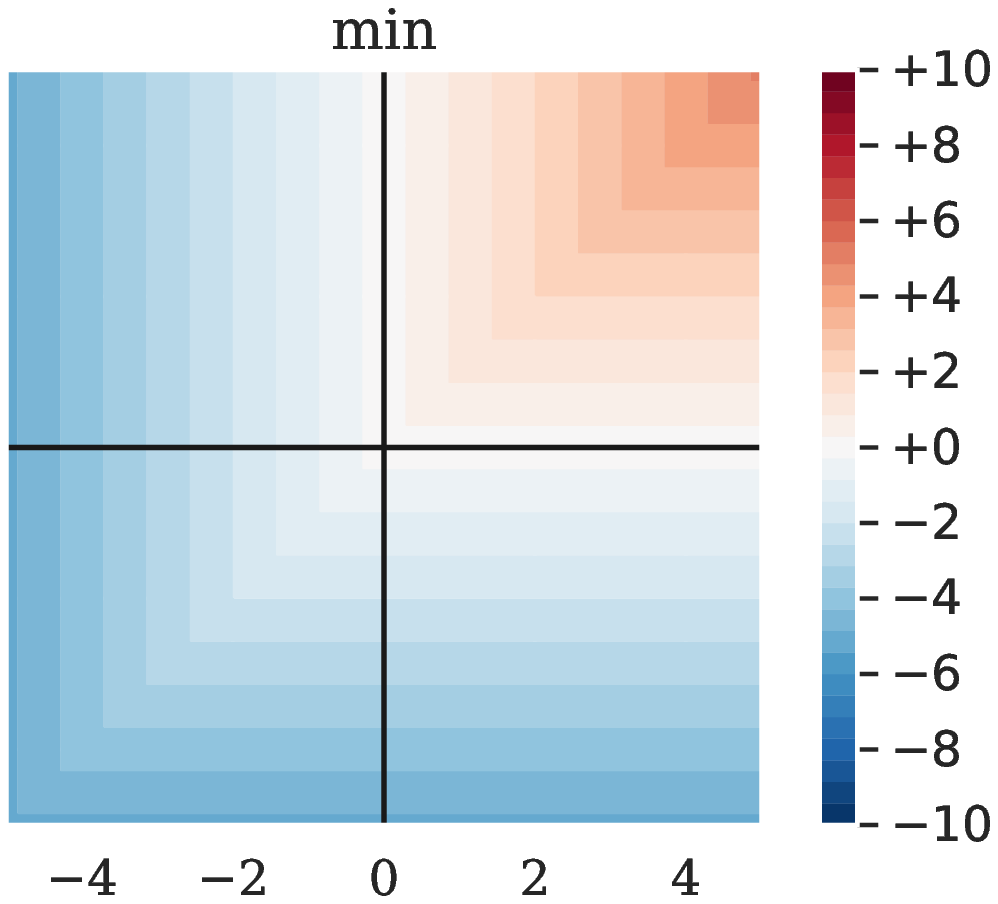}
\caption{
Heatmap comparing the outputs for the exact logit-space probabilistic-and function for independent logits, $\opn{AND_{IL}}(x, y)$; our constructed approximation, $\opn{AND_{AIL}}(x, y)$; and $\opn{max}(x, y)$.
}
\label{fig:AND}
\end{figure}

\subsection{OR}

Similarly, we can construct the logit-space OR function, for independent logits.
For a pair of logits $x$ and $y$, the probability that either of the corresponding events is true is given by $p = 1 - \sigma(-x) \, \sigma(-y)$.
This can be converted into a logit as
\begin{equation}
    \opn{OR_{IL}}(x, y)
    := \opn{logit}(\PP(X \lor Y)_{x \independent y})
    = \log \left( \frac{p}{1 - p} \right), \; \text{with} \; p = 1 - \sigma(-x) \, \sigma(-y)
    \label{eq:or-il}
\end{equation}
which can be roughly approximated by the max function.
This is equivalent to setting the probability of \textbf{either} of event $X$ or $Y$ occurring to be equal to the probability of the \textbf{most likely} event.
This underestimates the upper-right quadrant (below), which we can approximate better as the sum of the input logits, yielding
\begin{equation}
    \opn{OR_{AIL}}(x, y) :=
        \begin{cases}
            x + y, & x > 0, \, y > 0 \\
            \opn{max}(x, y), & \text{otherwise}
        \end{cases}
    \label{eq:or_ail}
\end{equation}

\begin{figure}[h]
    \centering
    \includegraphics[scale=0.47]{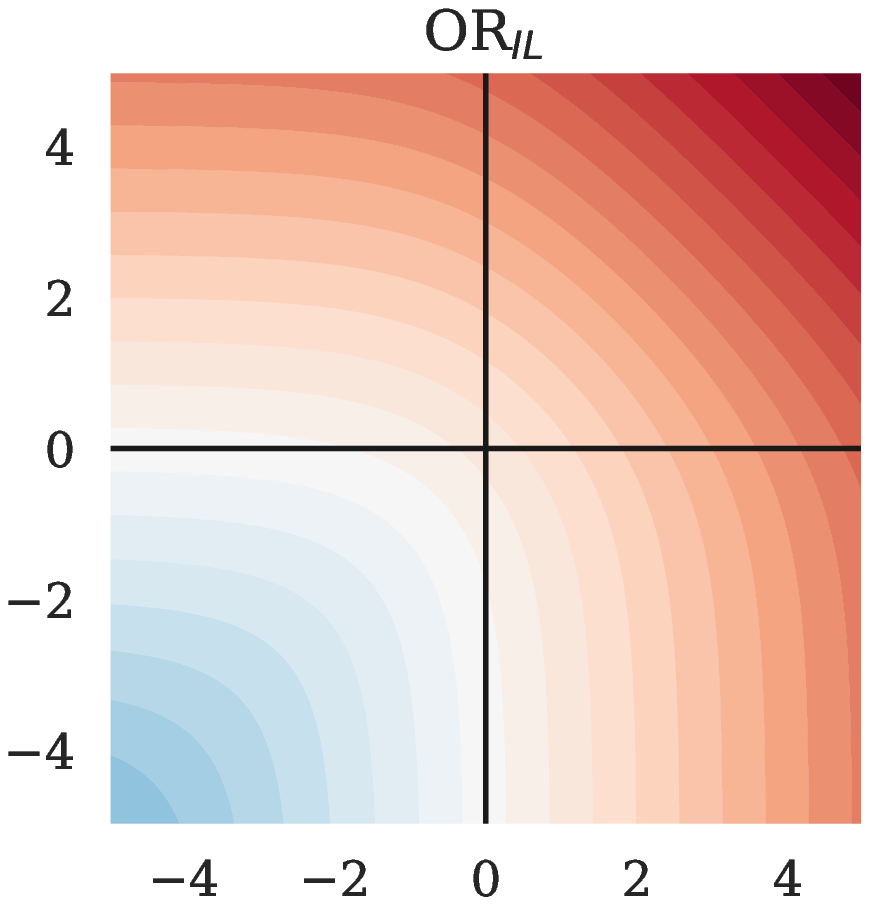}
    \includegraphics[scale=0.47]{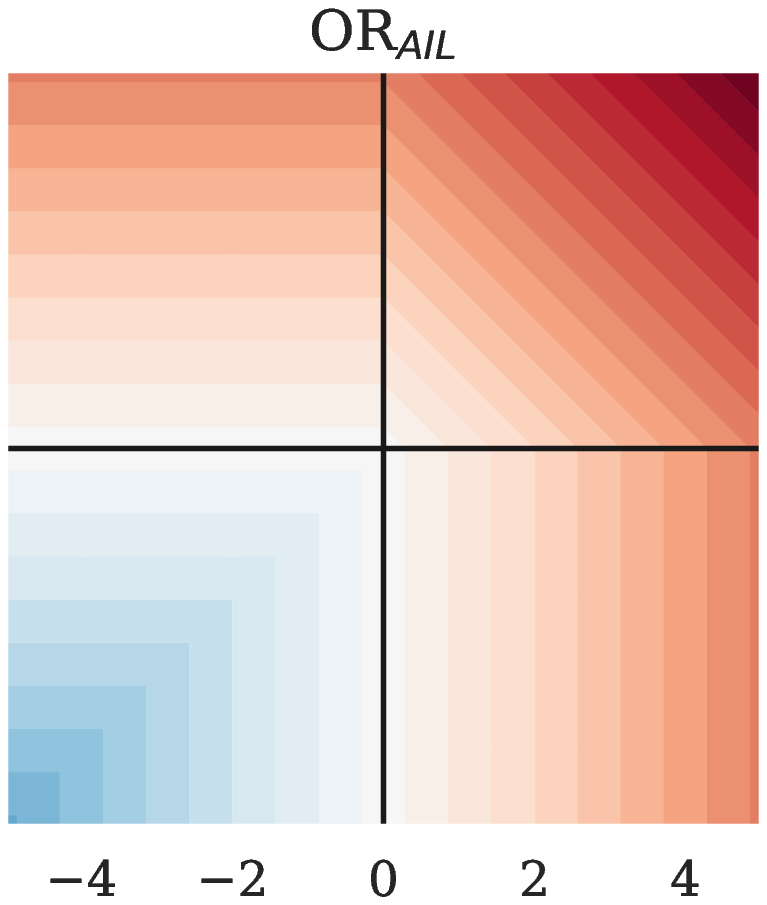}
    \includegraphics[scale=0.47]{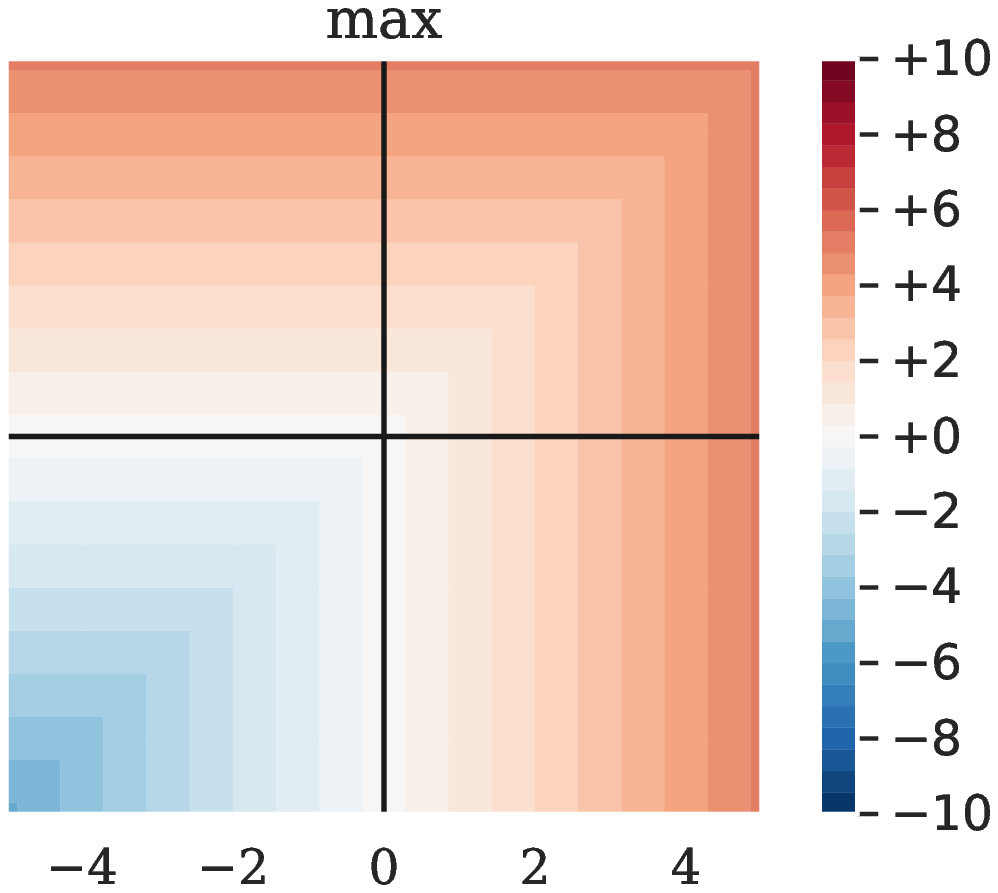}
\caption{
Comparison of the exact logit-space probabilistic-or function for independent logits, $\opn{OR_{IL}}(x, y)$; our constructed approximation, $\opn{OR_{AIL}}(x, y)$; and $\opn{max}(x, y)$.
}
\label{fig:OR}
\end{figure}

\subsection{XNOR}

We also consider the construction of a logit-space XNOR operator.
This is the probability that $X$ and $Y$ occur either together, or not at all, given by
\begin{equation}
    \opn{XNOR_{IL}}(x, y)
    := \opn{logit}(\PP(X \bar\oplus Y)_{x \independent y})
    = \log \left( \frac{p}{1 - p} \right),
    \label{eq:xnor_il}
\end{equation}
where $p = \sigma(x) \, \sigma(y) + \sigma(-x) \, \sigma(-y)$.
We can approximate this with
\begin{equation}
    \opn{XNOR_{AIL}}(x, y) := \sgn(x y) \opn{min}(|x|, |y|),
    \label{eq:xnor_ail}
\end{equation}
which focuses on the logit of the feature \textbf{most likely} to \textbf{flip} the expected \textbf{parity} (\revision{see }\autoref{fig:XNOR}).

We could use other approximations, such as the sign-preserving geometric mean,
\begin{equation}
    \opn{SignedGeomean}(x, y) := \sgn(x y) \sqrt{| x y |},
\end{equation}
but this matches $\opn{XNOR_{IL}}$ less closely, and has a divergent gradient along both $x=0$ and $y=0$.

\begin{figure}[h]
    \centering
    \includegraphics[scale=0.47]{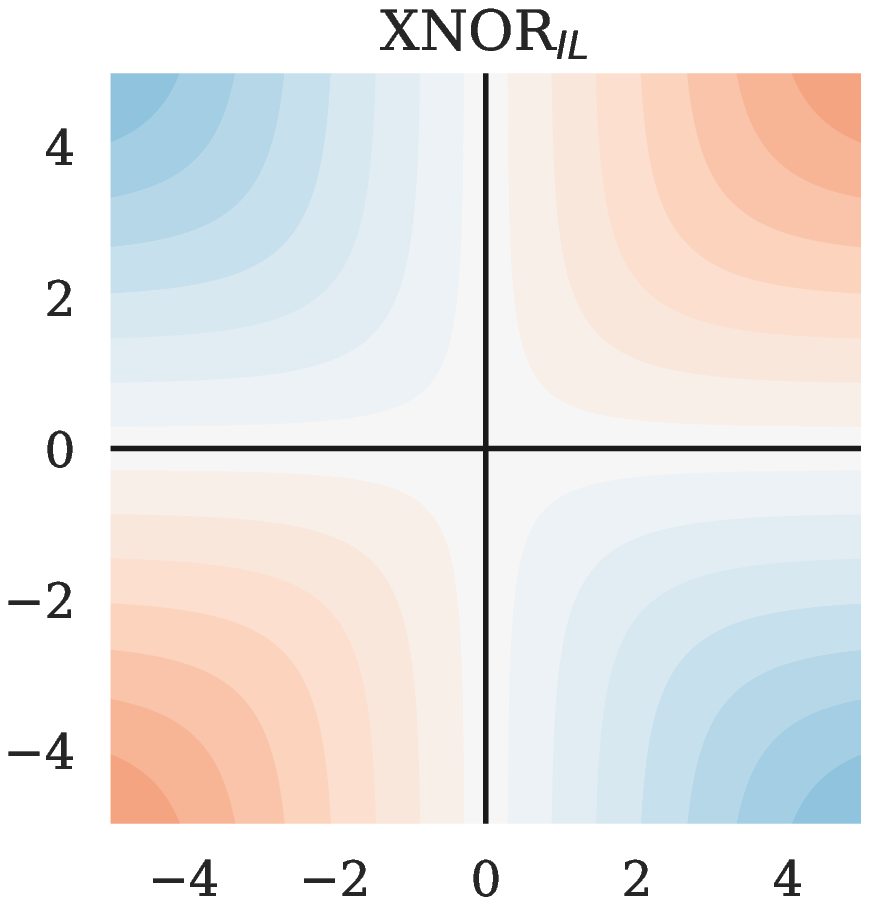}
    \includegraphics[scale=0.47]{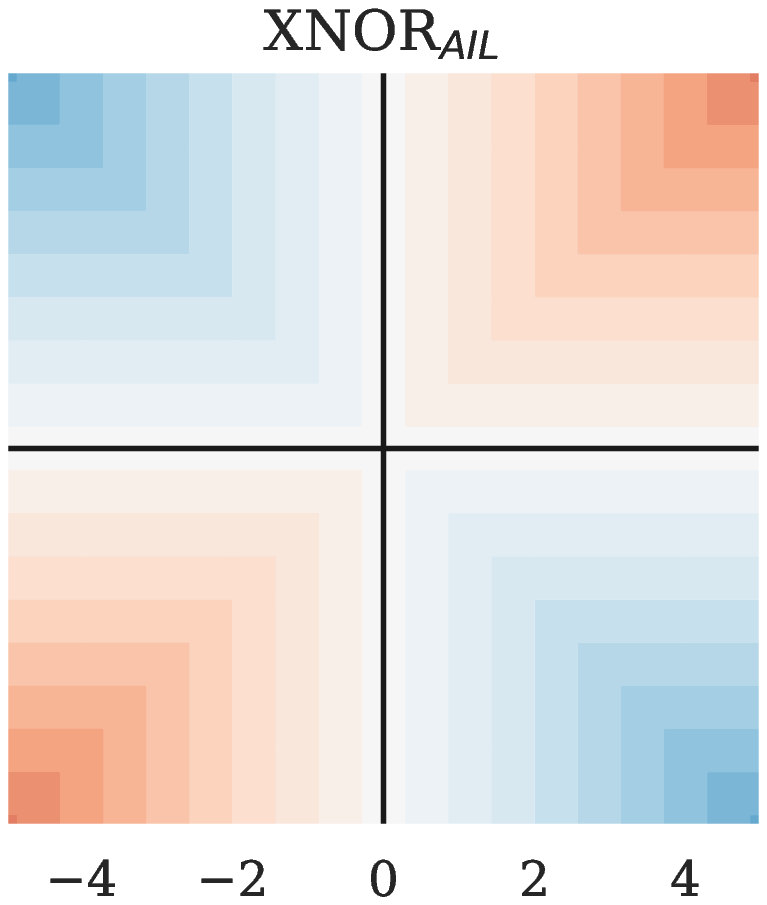}
    \includegraphics[scale=0.47]{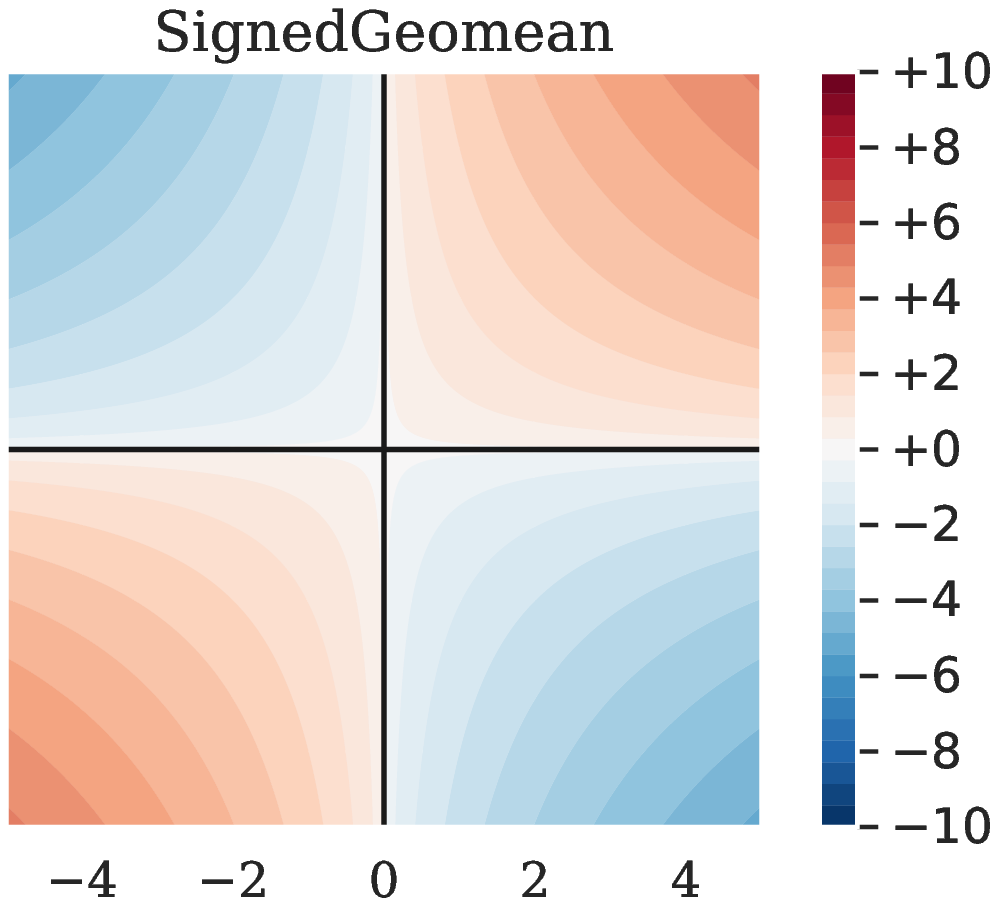}
\caption{
Comparison of the exact logit-space probabilistic-xnor function for independent logits, $\opn{XNOR_{IL}}(x, y)$; our constructed approximation, $\opn{XNOR_{AIL}}(x, y)$; and $\opn{SignedGeomean}(x, y)$.
}
\label{fig:XNOR}
\end{figure}

\subsection{Discussion}

By working via probabilities, and assuming inputs encode independent events, we have derived logit-space equivalents of the Boolean logic functions, AND, OR, and XNOR.
Since these are computationally demanding to compute repeatedly within a neural network, we have constructed approximations of them: $\opn{AND_{AIL}}$, $\opn{OR_{AIL}}$, and $\opn{XNOR_{AIL}}$.
Like ReLU, these involve only comparison, addition, and multiplication operations which are cheap to perform.
In fact, $\opn{AND_{AIL}}$ and $\opn{OR_{AIL}}$ are a generalization of ReLU to an extra dimension, since $\opn{OR_{AIL}}(x, y=0) = max(x, 0)$.

The majority of activation functions are one-dimensional, $f:\R\to\R$.
In contrast to this, our proposed activation functions are all two-dimensional, $f:\R^2\to\R$.
They must be applied to pairs of features from the embedding space, and will reduce the dimensionality of the input space by a factor of 2.
This behaviour is the same as seen in MaxOut networks \citep{maxout} which use $\opn{max}$ as their activation function; $\opn{MaxOut}(x, y; k) := \opn{max}(x, y)$.
Similar to $\opn{MaxOut}$, our activation functions could be generalized to higher dimensional inputs, $f:\R^k\to\R$, by considering the behaviour of the logit-space AND, OR, XNOR operations with regard to more inputs.
For simplicity, we restrict this work to consider only $k\!=\!2$, but note these activation functions also generalize to higher dimensions.

\subsection{Ensembling}
\label{s:ensembling}

By using {multiple} logit-Boolean activation functions \textit{simultaneously} alongside each other, we permit the network multiple options of how to relate features together.
When combining the activation functions, we considered two strategies (illustrated in \aref{a:ensembling}).

In the partition (p) strategy, we split the $n_{c}$ dimensional pre-activation embedding equally into $m$ partitions, apply different activation functions on each partition, and concatenate the results together.
Using AIL activation functions under this strategy, the output dimension will always be half that of the input, as it is for each AIL activation function individually.
In the duplication (d) strategy, we apply $m$ different activation functions in parallel to the same $n_c$ elements.
The output is consequently larger, with dimension $m \, n_c$.
If desired, we can counteract the $2\!\to \!1$ reduction of AIL activation functions by using two of them together under this strategy.
A negative weight in a network is equivalent to the logit-NOT operator.
Hence with sufficient width and depth, a multi-layer network using only the $\opn{OR_{IL}}$ activation function should be able to replicate any probabilistic logic circuit.

Utilizing $\opn{AND_{AIL}}$, $\opn{OR_{AIL}}$ and $\opn{XNOR_{AIL}}$ simultaneously allows our networks to access logit-space equivalent\revision{s} of 12 of the 16 Boolean logical operations with only a single sign inversion (in either one of the inputs or the output).
Including the bias term 
and skip connections, 
the network has easy access to logit-space equivalents of all 16 Boolean logical operations.

We hypothesised that training a network with all three of our activation functions in an ensemble in this manner could yield better results since the network would not need to expend layers having to combine $\opn{OR_{IL}}$ operations together to yield other Boolean operations.

\subsection{Normalization}

Our AIL activation functions are close approximations to exact AND, OR, and XNOR operations in logit-space.
However, when deploying non-linearities within a neural network, it is important that the activation functions have a gain of 1 in order to improve stability during training \citep{selu}, a property the AIL activations do not possess.
We constructed normalized variants of the exact \revision{and approximate} logit operators, dubbed NIL and NAIL respectively, by subtracting the expected mean and dividing by the expected standard deviation, assuming  the operands are sampled independently from the standard normal distribution, $\mathcal{N}(0,1)$.
For more details, see \aref{a:normalization}.

\section{Experiments}
\label{s:experiments}

We evaluated the performance of our AIL activation\revision{s}, both individually and together in an ensemble, on a range of benchmark tasks.
Since $\opn{AND_{AIL}}$ and $\opn{OR_{AIL}}$ are equivalent when the sign of operands and outputs can be freely chosen, we only show results for $\opn{OR_{AIL}}$. 
We compared \revision{against} three primary baselines: (1) ReLU, (2) $\opn{max}(x, y) = \opn{MaxOut}([x, y]; k\!=\!2)$, and (3) the concatenation of $\opn{max}(x, y)$ and $\opn{min}(x, y)$, denoted \{$\opn{Max},\opn{Min}$ (d)\}.
The \{$\opn{Max},\opn{Min}$ (d)\} ensemble is equivalent to $\opn{GroupSort}$ with a group size of 2 \citep{Chernodub2017,groupsort}, sometimes referred to as the $\opn{MaxMin}$ operator\revision{; it} is comparable to the concatenation of $\opn{OR_{AIL}}$ and $\opn{AND_{AIL}}$ under our duplication strategy.


\subsection{Parity}
\label{s:parity}

In a simple initial experiment, we constructed a synthetic dataset whose labels could be derived directly by stacking the logical operation XNOR.
Each sample had four input logits, and target value equal to the parity of the number of positive inputs.
A very small model using $\opn{XNOR}$ with two hidden layers (of 4, then 2 neurons) should be capable of perfect classification accuracy on this dataset with a sparse weight matrix by learning to nest pairwise binary relationships.
We trained such an MLP with either $\opn{ReLU}$ or $\opn{XNOR_{AIL}}$ activations.

\begin{wrapfigure}[26]{r}{0.36\linewidth}
\vspace*{-0.5cm}%
    \centering
    \includegraphics[scale=0.44]{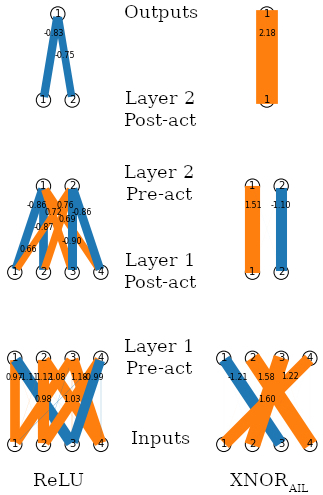}
    \caption{
Visualization of weight matrices learnt by two-layer MLPs on a binary classification task, where the target output is the parity of the inputs.
\revision{Line widths} indicate weight magnitudes (orange: \mbox{+ve}, blue: \mbox{-ve}).
\revision{MLP} with ReLU: 60\% accuracy;  $\opn{XNOR_{AIL}}$: 100\% accuracy.
}
\label{fig:parity}
\end{wrapfigure}

The MLP with $\opn{XNOR_{AIL}}$ activation learned a sparse weight matrix able to perfectly classify any input combination, shown in \autoref{fig:parity}.
In comparison, with $\opn{ReLU}$ the network was only able to produce 60\% classification accuracy.
The accuracy with $\opn{ReLU}$ was improved by increasing the MLP width/depth, but this did not result in a sparse weight matrix.
This experiment demonstrates that $\opn{XNOR_{AIL}}$ can be utilized by a network to find the simplest relationship between inputs that satisfies the objective.
For additional results, see \aref{a:parity}.

\revision{
}

\subsection{MLP on Bach Chorales and Logit Independence}
\label{s:chorales}

The Bach Chorale dataset~\citep{boulanger2012modeling} consists of 382 chorales composed by JS Bach, each $\sim$12 measures long, totalling $\sim$\num{83000} notes.
Represented as discrete sequences of tokens, it has served as a benchmark for music processing for decades, from heuristic methods to HMMs, RNNs, and CNNs \citep{mozer1990connectionist, hild1992harmonet, allan2005harmonising, liang2016bachbot, hadjeres2017deepbach, huang2019counterpoint}.
The chorales are comprised of 4 voices (melodic lines) whose behaviour is guided by soft musical rules depend\revision{ing} on the prior movement of that and other voices, e.g. ``two voices a fifth apart ought not to move in parallel with one another''.
\revision{We tasked 2-layer MLPs with determining} whether a short four-part musical excerpt is taken from a Bach chorale.
Negative examples were created by stochastically corrupting chorale excerpts (see \aref{a:jsb}).
We found \{$\opn{OR},\opn{AND},\opn{XNOR_{AIL}}$ (d)\} \revision{had highest accuracy, but the results were not statistically significant} ($p\!<\!0.1$ between best and worst, \revision{two-tailed} Student's $t$-test, 10 random inits\revision{; see \autoref{s:stats}}).

Additionally, we investigated the independence between logits in the trained pre-activation embeddings.
We expect that an MLP which is optimally engaging its neurons would maintain independence between features in order to maximize information.
We measured correlations between adjacent pre-activations (paired operands for the logical activation functions), and also between non-adjacent pairs of pre-activations.
Our results indicate the network learns features which are independent when they are not passed to the same 2D activation, and \textit{anti-correlated} features when they are.
For more details, see \aref{a:correlations}.

\subsection{CNN and MLP on MNIST}
\label{s:mnist}

We trained 2-layer MLP and 6-layer CNN models on MNIST with ADAM \citep{adam}, 1-cycle schedule \citep{superconvergence,Smith2018}, and using hyperparameters tuned through a random search against a validation set comprised of the last 10k images of the training partition.

\begin{figure}[htb]
    \centering
    \begin{subfigure}[b]{0.49\linewidth}
    \centering
        \includegraphics[scale=0.4]{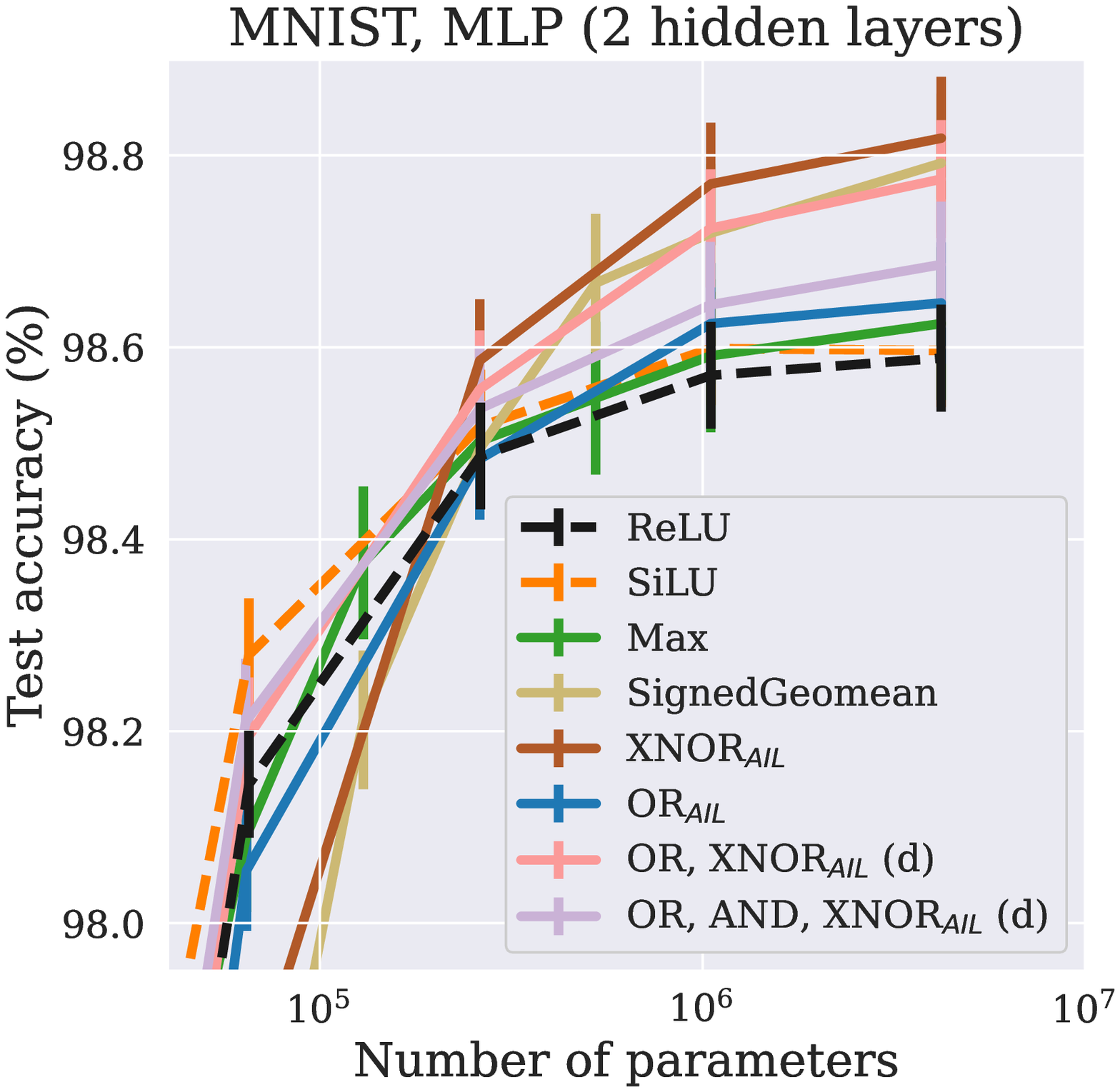}
    \end{subfigure}%
    \hfill{}
    \begin{subfigure}[b]{0.49\linewidth}
    \centering
        \includegraphics[scale=0.4]{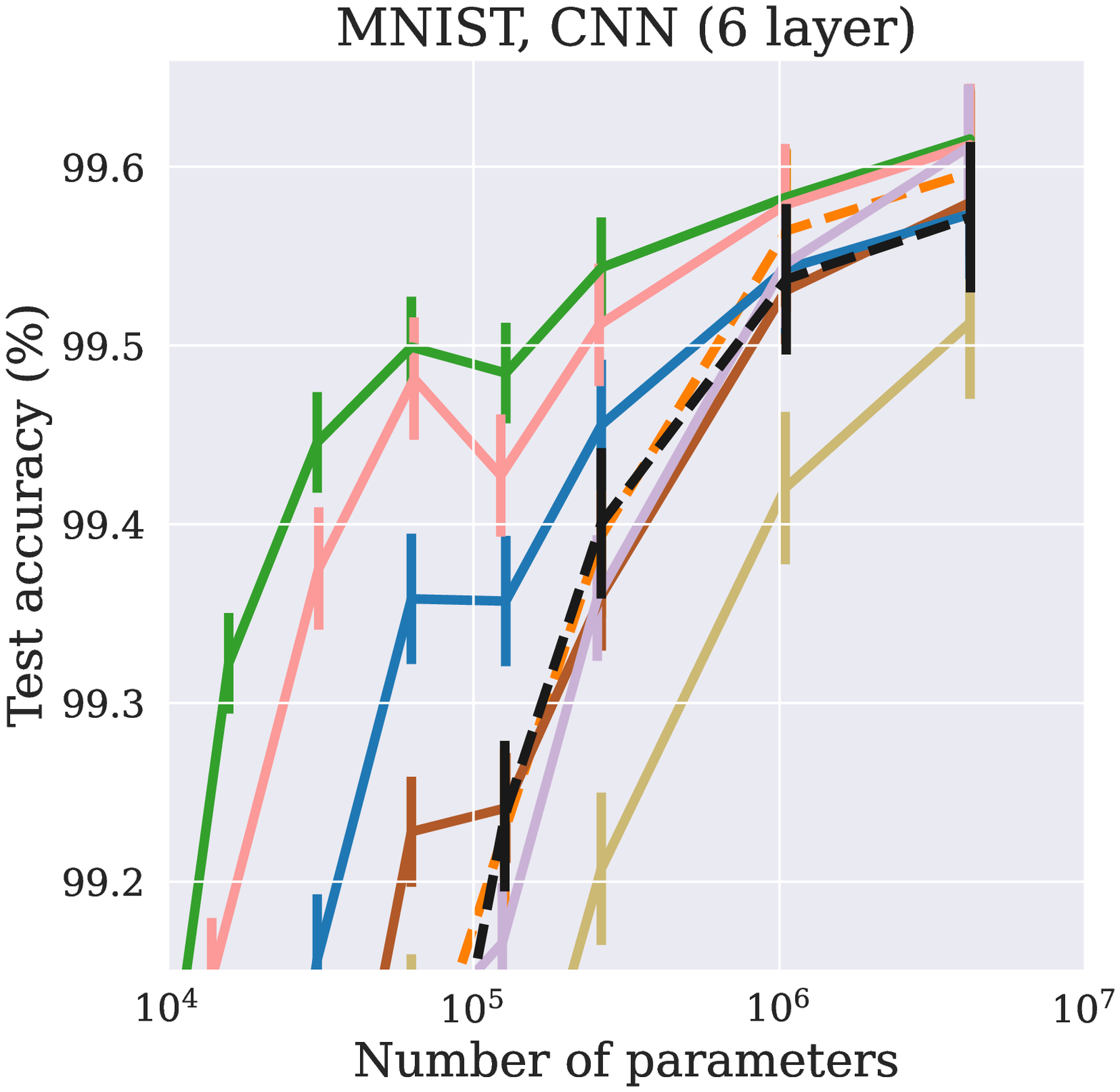}
    \end{subfigure}%
\caption{
We trained CNN on \href{http://yann.lecun.com/exdb/mnist/}{MNIST}, MLP on flattened-MNIST, using ADAM (1-cycle, 10~ep), hyperparams determined by random search. Mean (bars: std~dev) of $n\!=\!40$ weight inits.
}
\label{fig:mnist}
\end{figure}

The MLP used two hidden layers, the widths of which were varied together to evaluate the performance for a range of model sizes.
The CNN used six layers of 3x3 convolution layers, with 2x2 max pooling (stride 2) after every other conv layer, followed by flattening and three MLP layers.
The layer widths were scaled up to explore a range of model sizes (see \aref{a:mnist} for more details).

For the MLP, $\opn{XNOR_{AIL}}$ performed best along with $\opn{SignedGeomean}$ ($p\!<\!0.1$\revision{, two-tailed Student's $t$-test}), ahead of all other activations ($p\!<\!0.01$ \revision{for each}; \autoref{fig:mnist} left panel) \revision{ when considering the best performance across all widths (see \autoref{s:stats} for methodology).
However when the width is reduced below \num{2e5} there is a transition and XNOR-shaped activations perform worst.
We hypothesize this may be because smaller widths embeddings are over-saturated and have individual units corresponding to multiple features, whilst XNOR activations may require single-feature units to perform best.
}

With the CNN, five activation configurations (\{$\opn{OR},\opn{AND},\opn{XNOR_{AIL}}$ (p)\},
\{$\opn{OR},\opn{XNOR_{AIL}}$ (d/p)\}, Max, and SiLU) performed best ($p\!<\!0.05$ \revision{for other activations, two-tailed Student's $t$-test}; \autoref{fig:mnist}, right panel\revision{; \autoref{s:stats}}).
CNNs which used $\opn{OR_{AIL}}$ or $\opn{Max}$ (alone or in an ensemble) maintained high performance with an order of magnitude fewer parameters (\num{3e4}) than \revision{others} (\num{3e5} params).

\subsection{ResNet50 on CIFAR-10/100}
\label{s:cifar}

We explored the impact of our activation functions on deep networks by deploying them in a pre-activation ResNet50 model~\citep{resnet,resnetv2}.
We exchanged all ReLU activation\revision{s} in the network to a candidate activation while maintaining the size of the pass-through embedding.
We experimented with changing the width of the network, scaling up the embedding space and all hidden layers by a common factor, $w$.
The network was trained on CIFAR-10/-100 for 100 epochs using ADAM \citep{adam}, 1-cycle \citep{Smith2018,superconvergence}.
\revision{See \aref{a:cifar} for further details.}

\begin{figure}[htb]
    \centering
    \begin{subfigure}[b]{0.49\linewidth} 
    \centering
        \includegraphics[scale=0.4]{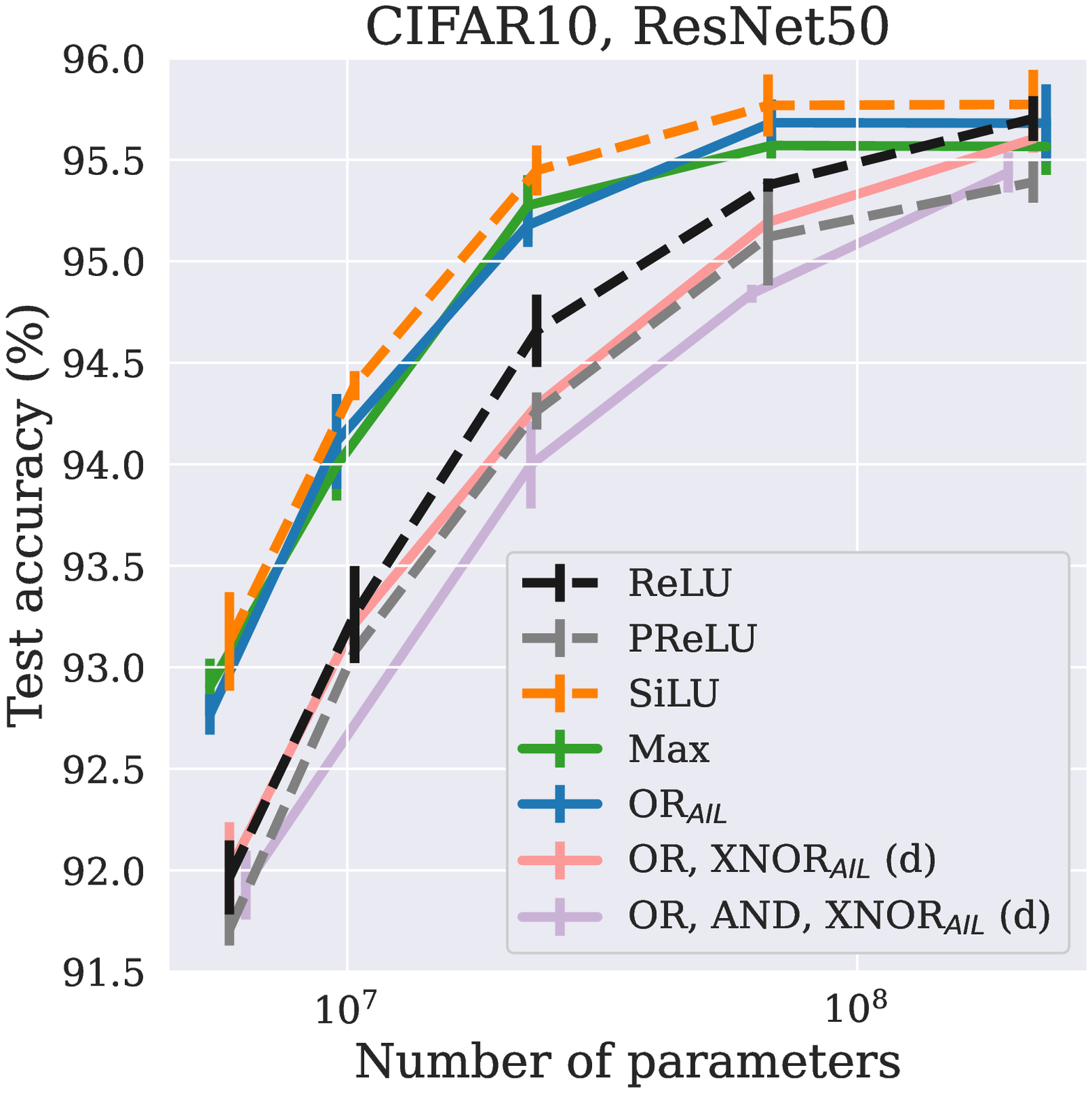} 
    \end{subfigure}%
    \hfill{}
    \begin{subfigure}[b]{0.49\linewidth}
    \centering
        \includegraphics[scale=0.4]{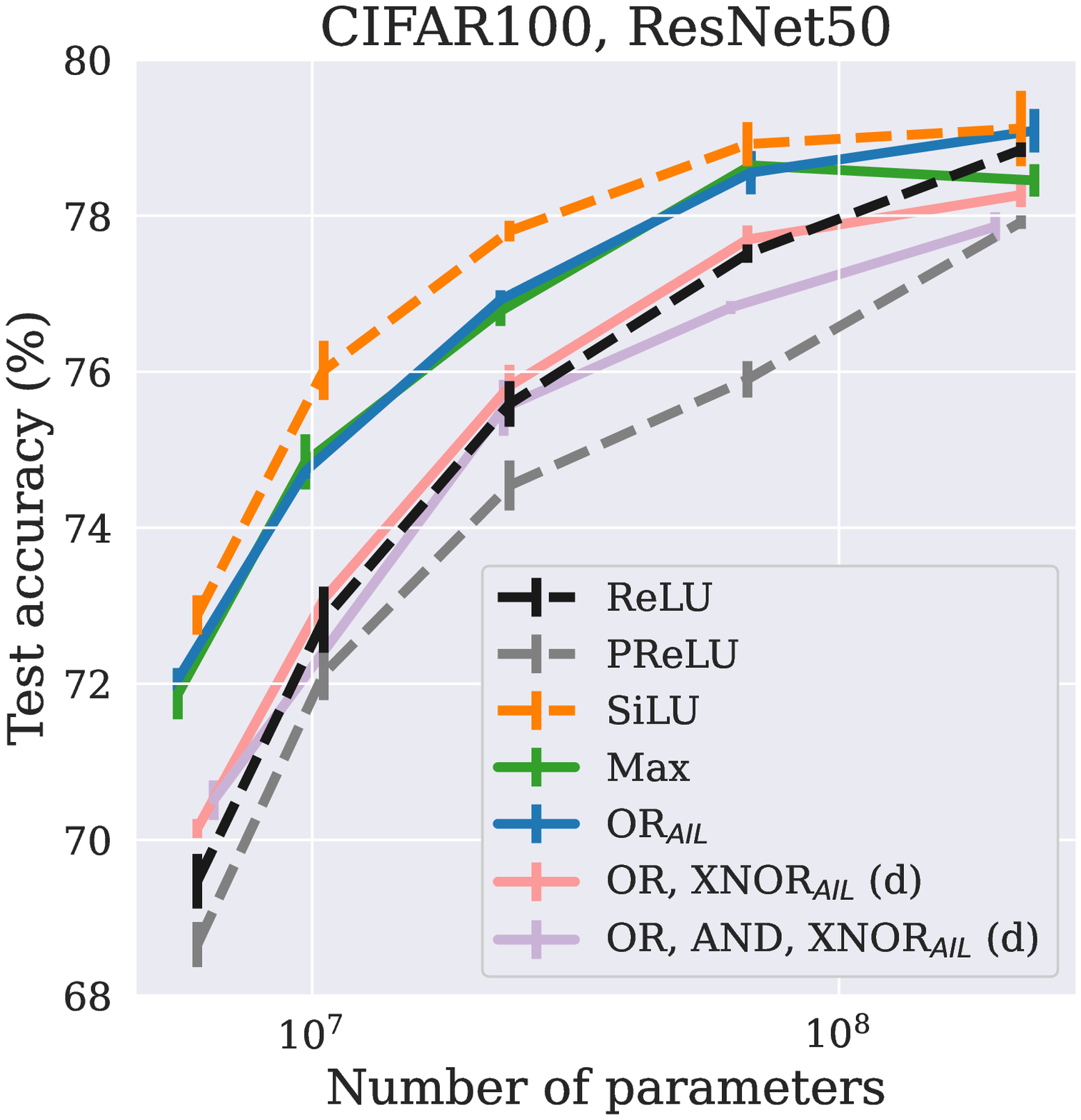}
    \end{subfigure}%
\caption{
ResNet50 on CIFAR-10/100, varying the activation function used through the network.
The width was varied to explore a range of network sizes (see text).
Trained for 100~ep. with ADAM, using hyperparams as determined by random search on CIFAR-100 with width factor $w=2$.
Mean (bars: std dev) of $n\!=\!4$ weight inits.
}
\label{fig:cifar}
\end{figure}

For both CIFAR-10 and -100, SiLU, $\opn{OR_{AIL}}$, and $\opn{Max}$ outperform ReLU across a wide range of width values \revision{(see} \autoref{fig:cifar}\revision{)}.
These three activation functions hold up their performance best as the number of parameters is reduced.
Since SiLU was discovered \revision{by a search of $1\!\to\!1$ activation functions for this type of architecture and task}~\citep{swish}, we expected it to perform well in this setting. 
\revision{Yet, we find $\opn{OR_{AIL}}$ and Max perform better than ReLU and comparable with SiLU generally.}
Meanwhile, other AIL activation\revision{s} perform similarly to ReLU when the width is thin, and slightly worse than ReLU when the width is wide.
When used on its own and not part of an ensemble, the $\opn{XNOR_{AIL}}$ activation function performed poorly (off the bottom of the chart), indicating it is not suited for this task.

\subsection{Transfer learning}
\label{s:transfer}

We considered transfer learning on several image classification datasets.
We used a ResNet18 model~\citep{resnet} pretrained on ImageNet-1k.
The weights were frozen (not fine-tuned) and used to generate embeddings of samples from other image datasets.
We trained a two-layer MLP to classify images from these embeddings using various activation functions.
For a comprehensive set of baselines, we compared against every activation function built into PyTorch~1.10 (see \aref{a:transfer}).
To make the number of parameters similar, we used a width of 512 for activation functions with $1\!\to \!1$ mapping (e.g. ReLU), a width of 650 for activation functions with a $2\!\to \!1$ mapping (e.g. Max, $\opn{OR_{AIL}}$), and a width of 438 for \{$\opn{OR},\opn{AND},\opn{XNOR_{AIL}}$ (d)\}.
See \aref{a:transfer} for more details.

\begin{table}[htb]
\small
  \centering
  \caption{%
Transfer learning from a frozen ResNet-18 architecture pretrained on ImageNet-1k to other computer vision datasets.
Mean (std. error) of $n\!=\!5$ random initializations of the MLP (same pretrained encoder).
Bold: best.
Underlined: top two.
Italic: no sig. diff. from best (two-sided Student's $t$-test, $p\!>\!0.05$).
Background: linear color scale from ReLU baseline (white) to best (black).
}
\label{tab:transfer}
\label{tab:transfer-head-big}
\label{tab:transfer-mink-head-big}
\def\ccAmin{86.5801886792453}
\def\ccAmax{88.34905660377358}
\def\ccA[#1]#2{\heatmapcell{#1}{#2}{\ccAmin}{\ccAmax}}
\def\ccBmin{81.63399999999999}
\def\ccBmax{82.88000000000001}
\def\ccB[#1]#2{\heatmapcell{#1}{#2}{\ccBmin}{\ccBmax}}
\def\ccCmin{58.044000000000004}
\def\ccCmax{60.778}
\def\ccC[#1]#2{\heatmapcell{#1}{#2}{\ccCmin}{\ccCmax}}
\def\ccDmin{90.70731696896436}
\def\ccDmax{93.25785683392981}
\def\ccD[#1]#2{\heatmapcell{#1}{#2}{\ccDmin}{\ccDmax}}
\def\ccEmin{30.969666852908368}
\def\ccEmax{39.835902544398145}
\def\ccE[#1]#2{\heatmapcell{#1}{#2}{\ccEmin}{\ccEmax}}
\def\ccFmin{94.62249999999999}
\def\ccFmax{94.83}
\def\ccF[#1]#2{\heatmapcell{#1}{#2}{\ccFmin}{\ccFmax}}
\def\ccGmin{53.25829748424489}
\def\ccGmax{55.342655196890334}
\def\ccG[#1]#2{\heatmapcell{#1}{#2}{\ccGmin}{\ccGmax}}
\centerline{
  \scalebox{0.835}{
\begin{tabular}{lrrrrrrr}
\toprule
 & \multicolumn{7}{c}{Test Accuracy (\%)} \\
\cmidrule(l){2-8}
Activation function                                     & \multicolumn{1}{c}{Cal101} & \multicolumn{1}{c}{CIFAR10} & \multicolumn{1}{c}{CIFAR100} & \multicolumn{1}{c}{Flowers} & \multicolumn{1}{c}{StfCars} & \multicolumn{1}{c}{STL-10} & \multicolumn{1}{c}{SVHN} \\
\toprule
Linear layer only                                        & \ccA[88.3491]{$\mbf{\mbns{88.35}}\wpm{ 0.15}$} & \ccB[78.5640]{$78.56\wpm{ 0.09}$} & \ccC[57.3920]{$57.39\wpm{ 0.09}$} & \ccD[92.3171]{$92.32\wpm{ 0.20}$} & \ccE[33.5132]{$33.51\wpm{ 0.06}$} & \ccF[94.6825]{$94.68\wpm{ 0.02}$} & \ccG[46.6011]{$46.60\wpm{ 0.14}$} \\
\midrule
$\opn{ReLU}$                                             & \ccA[86.5802]{$86.58\wpm{ 0.17}$} & \ccB[81.6340]{$81.63\wpm{ 0.05}$} & \ccC[58.0440]{$58.04\wpm{ 0.11}$} & \ccD[90.7073]{$90.71\wpm{ 0.26}$} & \ccE[30.9697]{$30.97\wpm{ 0.26}$} & \ccF[94.6225]{$94.62\wpm{ 0.06}$} & \ccG[53.2583]{$53.26\wpm{ 0.08}$} \\
PReLU                                                    & \ccA[87.8302]{$\mbns{87.83}\wpm{ 0.21}$} & \ccB[81.0300]{$81.03\wpm{ 0.13}$} & \ccC[58.8980]{$58.90\wpm{ 0.18}$} & \ccD[93.1707]{$\mbs{\mbns{93.17}}\wpm{ 0.19}$} & \ccE[39.8359]{$\mbf{\mbns{39.84}}\wpm{ 0.18}$} & \ccF[94.5375]{$94.54\wpm{ 0.05}$} & \ccG[53.4750]{$53.47\wpm{ 0.08}$} \\
\midrule
SELU                                                     & \ccA[87.7358]{$87.74\wpm{ 0.09}$} & \ccB[79.9340]{$79.93\wpm{ 0.13}$} & \ccC[58.2380]{$58.24\wpm{ 0.06}$} & \ccD[92.2683]{$92.27\wpm{ 0.13}$} & \ccE[37.5112]{$37.51\wpm{ 0.17}$} & \ccF[94.5300]{$94.53\wpm{ 0.07}$} & \ccG[50.9365]{$50.94\wpm{ 0.12}$} \\
GELU                                                     & \ccA[87.0991]{$87.10\wpm{ 0.15}$} & \ccB[81.3860]{$81.39\wpm{ 0.09}$} & \ccC[58.5060]{$58.51\wpm{ 0.13}$} & \ccD[91.5122]{$91.51\wpm{ 0.15}$} & \ccE[33.4286]{$33.43\wpm{ 0.15}$} & \ccF[94.6250]{$94.62\wpm{ 0.06}$} & \ccG[53.4250]{$53.43\wpm{ 0.23}$} \\
Mish                                                     & \ccA[87.1108]{$87.11\wpm{ 0.12}$} & \ccB[81.0880]{$81.09\wpm{ 0.11}$} & \ccC[58.3720]{$58.37\wpm{ 0.10}$} & \ccD[91.6098]{$91.61\wpm{ 0.15}$} & \ccE[33.7469]{$33.75\wpm{ 0.14}$} & \ccF[94.6075]{$94.61\wpm{ 0.05}$} & \ccG[53.0539]{$53.05\wpm{ 0.12}$} \\
\midrule
Tanh                                                     & \ccA[87.4764]{$87.48\wpm{ 0.06}$} & \ccB[80.5560]{$80.56\wpm{ 0.07}$} & \ccC[57.3480]{$57.35\wpm{ 0.08}$} & \ccD[90.3171]{$90.32\wpm{ 0.20}$} & \ccE[29.5077]{$29.51\wpm{ 0.12}$} & \ccF[94.6275]{$\mbns{94.63}\wpm{ 0.07}$} & \ccG[51.8600]{$51.86\wpm{ 0.05}$} \\
\midrule
$\opn{Max}$                                              & \ccA[86.9575]{$86.96\wpm{ 0.20}$} & \ccB[81.7560]{$81.76\wpm{ 0.14}$} & \ccC[58.5980]{$58.60\wpm{ 0.12}$} & \ccD[90.9756]{$90.98\wpm{ 0.18}$} & \ccE[33.3690]{$33.37\wpm{ 0.15}$} & \ccF[94.7025]{$\mbns{94.70}\wpm{ 0.06}$} & \ccG[53.5257]{$53.53\wpm{ 0.16}$} \\
$\opn{Max},\opn{Min}$ (d)                                & \ccA[87.2288]{$87.23\wpm{ 0.13}$} & \ccB[82.3060]{$82.31\wpm{ 0.10}$} & \ccC[59.0520]{$59.05\wpm{ 0.10}$} & \ccD[91.6829]{$91.68\wpm{ 0.18}$} & \ccE[34.9105]{$34.91\wpm{ 0.12}$} & \ccF[94.6400]{$94.64\wpm{ 0.04}$} & \ccG[53.9136]{$53.91\wpm{ 0.13}$} \\
\midrule
SignedGeomean                                            & \ccA[87.0324]{$87.03\wpm{ 0.23}$} & \ccB[51.4520]{$\mbns{51.45}\wpm{16.92}$} & \ccC[11.7980]{$11.80\wpm{10.80}$} & \ccD[91.3447]{$91.34\wpm{ 0.34}$} & \ccE[26.3748]{$\mbns{26.37}\wpm{ 6.46}$} & \ccF[94.6825]{$94.68\wpm{ 0.06}$} & \ccG[37.1635]{$37.16\wpm{ 7.18}$} \\
\midrule
$\opn{XNOR_{IL}}$                                        & \ccA[85.0118]{$85.01\wpm{ 0.17}$} & \ccB[79.6180]{$79.62\wpm{ 0.09}$} & \ccC[57.1360]{$57.14\wpm{ 0.07}$} & \ccD[84.7561]{$84.76\wpm{ 0.43}$} & \ccE[ 1.3401]{$ 1.34\wpm{ 0.11}$} & \ccF[94.5125]{$94.51\wpm{ 0.03}$} & \ccG[51.9914]{$51.99\wpm{ 0.16}$} \\
$\opn{OR_{IL}}$                                          & \ccA[87.1108]{$87.11\wpm{ 0.08}$} & \ccB[79.7480]{$79.75\wpm{ 0.05}$} & \ccC[58.0700]{$58.07\wpm{ 0.11}$} & \ccD[91.1220]{$91.12\wpm{ 0.36}$} & \ccE[33.1179]{$33.12\wpm{ 0.12}$} & \ccF[94.6000]{$94.60\wpm{ 0.03}$} & \ccG[51.2077]{$51.21\wpm{ 0.17}$} \\
\midrule
$\opn{XNOR_{NIL}}$                                       & \ccA[87.2477]{$87.25\wpm{ 0.22}$} & \ccB[82.8800]{$\mbf{\mbns{82.88}}\wpm{ 0.08}$} & \ccC[60.7780]{$\mbf{\mbns{60.78}}\wpm{ 0.08}$} & \ccD[93.2579]{$\mbf{\mbns{93.26}}\wpm{ 0.26}$} & \ccE[39.4747]{$\mbns{39.47}\wpm{ 0.20}$} & \ccF[94.8300]{$\mbf{\mbns{94.83}}\wpm{ 0.06}$} & \ccG[55.3427]{$\mbf{\mbns{55.34}}\wpm{ 0.19}$} \\
$\opn{OR_{NIL}}$                                         & \ccA[87.1873]{$87.19\wpm{ 0.16}$} & \ccB[79.6100]{$79.61\wpm{ 0.05}$} & \ccC[58.4380]{$58.44\wpm{ 0.10}$} & \ccD[91.6458]{$91.65\wpm{ 0.29}$} & \ccE[35.8205]{$35.82\wpm{ 0.04}$} & \ccF[94.5775]{$94.58\wpm{ 0.03}$} & \ccG[50.9488]{$50.95\wpm{ 0.15}$} \\
$\opn{OR}, \opn{AND_{NIL}}$ (d)                          & \ccA[86.8201]{$86.82\wpm{ 0.18}$} & \ccB[80.0900]{$80.09\wpm{ 0.09}$} & \ccC[58.5980]{$58.60\wpm{ 0.07}$} & \ccD[91.4425]{$91.44\wpm{ 0.20}$} & \ccE[37.0252]{$37.03\wpm{ 0.11}$} & \ccF[94.6500]{$94.65\wpm{ 0.05}$} & \ccG[52.4915]{$52.49\wpm{ 0.09}$} \\
$\opn{OR}, \opn{XNOR_{NIL}}$ (d)                         & \ccA[87.8184]{$\mbns{87.82}\wpm{ 0.19}$} & \ccB[82.6720]{$\mbns{82.67}\wpm{ 0.05}$} & \ccC[60.6000]{$\mbs{\mbns{60.60}}\wpm{ 0.11}$} & \ccD[92.9268]{$\mbns{92.93}\wpm{ 0.12}$} & \ccE[39.2190]{$39.22\wpm{ 0.17}$} & \ccF[94.6275]{$94.63\wpm{ 0.06}$} & \ccG[54.8748]{$\mbns{54.87}\wpm{ 0.09}$} \\
$\opn{OR}, \opn{AND}, \opn{XNOR_{NIL}}$ (d)              & \ccA[87.4114]{$87.41\wpm{ 0.27}$} & \ccB[82.8360]{$\mbs{\mbns{82.84}}\wpm{ 0.06}$} & \ccC[60.3760]{$60.38\wpm{ 0.10}$} & \ccD[92.9756]{$\mbns{92.98}\wpm{ 0.17}$} & \ccE[39.4166]{$\mbns{39.42}\wpm{ 0.18}$} & \ccF[94.7125]{$\mbns{94.71}\wpm{ 0.03}$} & \ccG[55.1137]{$\mbs{\mbns{55.11}}\wpm{ 0.12}$} \\
\midrule
$\opn{XNOR_{AIL}}$                                       & \ccA[86.9693]{$86.97\wpm{ 0.18}$} & \ccB[81.8330]{$81.83\wpm{ 0.06}$} & \ccC[58.4610]{$58.46\wpm{ 0.10}$} & \ccD[90.9268]{$90.93\wpm{ 0.15}$} & \ccE[32.5559]{$32.56\wpm{ 0.10}$} & \ccF[94.7050]{$\mbns{94.71}\wpm{ 0.06}$} & \ccG[53.7485]{$53.75\wpm{ 0.14}$} \\
$\opn{OR_{AIL}}$                                         & \ccA[87.4528]{$87.45\wpm{ 0.14}$} & \ccB[81.8800]{$81.88\wpm{ 0.07}$} & \ccC[59.0960]{$59.10\wpm{ 0.09}$} & \ccD[92.0000]{$92.00\wpm{ 0.15}$} & \ccE[36.0094]{$36.01\wpm{ 0.12}$} & \ccF[94.6875]{$\mbns{94.69}\wpm{ 0.04}$} & \ccG[53.6809]{$53.68\wpm{ 0.14}$} \\
\midrule
$\opn{XNOR_{NAIL}}$                                      & \ccA[87.6135]{$87.61\wpm{ 0.23}$} & \ccB[82.3840]{$82.38\wpm{ 0.07}$} & \ccC[59.7680]{$59.77\wpm{ 0.13}$} & \ccD[93.0732]{$\mbns{93.07}\wpm{ 0.20}$} & \ccE[39.7737]{$\mbs{\mbns{39.77}}\wpm{ 0.04}$} & \ccF[94.8100]{$\mbs{\mbns{94.81}}\wpm{ 0.03}$} & \ccG[53.9136]{$53.91\wpm{ 0.05}$} \\
$\opn{OR_{NAIL}}$                                        & \ccA[87.1934]{$87.19\wpm{ 0.16}$} & \ccB[81.7920]{$81.79\wpm{ 0.09}$} & \ccC[59.4040]{$59.40\wpm{ 0.09}$} & \ccD[92.1220]{$92.12\wpm{ 0.12}$} & \ccE[37.3247]{$37.32\wpm{ 0.17}$} & \ccF[94.6525]{$94.65\wpm{ 0.04}$} & \ccG[53.8199]{$53.82\wpm{ 0.21}$} \\
$\opn{OR}, \opn{AND_{NAIL}}$ (d)                         & \ccA[87.6179]{$87.62\wpm{ 0.11}$} & \ccB[82.2800]{$82.28\wpm{ 0.10}$} & \ccC[59.7140]{$59.71\wpm{ 0.05}$} & \ccD[92.0976]{$92.10\wpm{ 0.20}$} & \ccE[37.6977]{$37.70\wpm{ 0.12}$} & \ccF[94.6125]{$\mbns{94.61}\wpm{ 0.08}$} & \ccG[53.8614]{$53.86\wpm{ 0.10}$} \\
$\opn{OR}, \opn{XNOR_{NAIL}}$ (d)                        & \ccA[87.8538]{$\mbs{\mbns{87.85}}\wpm{ 0.22}$} & \ccB[82.5240]{$82.52\wpm{ 0.11}$} & \ccC[60.0180]{$60.02\wpm{ 0.10}$} & \ccD[93.1220]{$\mbns{93.12}\wpm{ 0.13}$} & \ccE[39.6370]{$\mbns{39.64}\wpm{ 0.09}$} & \ccF[94.7450]{$\mbns{94.75}\wpm{ 0.03}$} & \ccG[54.1334]{$54.13\wpm{ 0.05}$} \\
$\opn{OR}, \opn{AND}, \opn{XNOR_{NAIL}}$ (d)             & \ccA[87.7830]{$87.78\wpm{ 0.14}$} & \ccB[82.6680]{$\mbns{82.67}\wpm{ 0.06}$} & \ccC[60.0100]{$60.01\wpm{ 0.21}$} & \ccD[93.1220]{$\mbns{93.12}\wpm{ 0.21}$} & \ccE[39.6544]{$\mbns{39.65}\wpm{ 0.14}$} & \ccF[94.7750]{$\mbns{94.78}\wpm{ 0.03}$} & \ccG[54.5790]{$54.58\wpm{ 0.12}$} \\
\bottomrule
\end{tabular}
  }
}
\end{table}

Our results are shown in \autoref{tab:transfer-head-big}.
We found that all our (N)AIL activation functions outperformed ReLU on every transfer task.
Normalization had little impact on the performance of $\opn{OR_{\ast{}IL}}$, but a large impact on $\opn{XNOR_{\ast{}IL}}$.
The best overall performance was attained by the exact $\opn{XNOR_{NIL}}$, closely followed by our approximate $\opn{XNOR_{NAIL}}$ and ensembles containing either of these.
Surprisingly, our approximation $\opn{OR_{\ast{}AIL}}$ outperformed the exact form $\opn{OR_{\ast{}IL}}$.
The proposed activation functions were beaten only by PReLU, on Stanford Cars.
On Caltech101, all MLPs were beaten by a linear layer, suggesting our MLPs were overfitting for that dataset.
For further discussion, see \aref{a:transfer}.

\subsection{Additional results}

For results on Covertype, abstract reasoning and compositional zero-shot learning tasks, please see \aref{a:covtype}, \aref{a:abstractreasoning} and \aref{a:czsl}, respectively.

\section{Conclusion}

In this work, we motivated and introduced novel activation functions analogous to Boolean operators in logit-space.
We designed the AIL functions, fast approximates to the true logit-space probabilistic Boolean operations, and demonstrated their effectiveness on a wide range of tasks.

Although our activation functions assume independence (which is generally approximately true for pre-activation features learnt with 1D activation\revision{s}), we found the network learnt to induce anti-correlations between features which were paired together by our activation\revision{s} (\aref{a:correlations}).
This suggests exact independence of the features is not essential to the performance of our proposed activation\revision{s}.


We found $\opn{XNOR_{\ast{}AIL}}$ was highly effective \revision{for} shallow \revision{MLPs}.
Meanwhile, $\opn{OR_{AIL}}$  was highly effective for representation learning in the setting of a deep ResNet architecture trained on images.
In scenarios which involve manipulating high-level features extracted by an embedding network, we find that using $\opn{XNOR_{\ast{}IL}}$ or an ensemble of AIL activation functions together works best, and that the duplication ensembling strategy outperforms partitioning.
In this work we restricted ourselves to only considering using a single activation function (or ensemble) throughout the network, however our results together indicate stronger results may be found by using $\opn{OR_{\ast{}AIL}}$ for feature extraction and either $\opn{XNOR_{\ast{}IL}}$ or ensemble \{$\opn{OR}, \opn{XNOR_{\ast{}IL}}$ (d)\} for later higher-order reasoning layers.

Our work shows there is more to learn about the importance of more complex activation functions, both for \revision{ANN applications} and \revision{for} non-linear dendritic integration in \revision{biological} neuronal networks.

\hider{
\subsubsection*{Reproducibility Statement}

We have shared code to run all experiments considered in this paper with the reviewers via an anonymous download URL.
The code base contains detailed instructions on how to setup each experiment, including downloading the datasets and installing environments with pinned dependencies.
It should be possible to reproduce all our experimental results with this code.
For the final version of the paper, we will make the code publicly available on an online repository and share a link to it within the paper.
}

%
\begin{ack}
We are grateful to Chandramouli Shama Sastry, Eleni Triantafillou, and Finlay Maguire for insightful discussions.

Resources used in preparing this research were provided, in part, by the Province of Ontario, the Government of Canada through CIFAR, and companies sponsoring the Vector Institute \url{https://vectorinstitute.ai/partners/}, and in part by ACENET \url{https://ace-net.ca/} and Compute Canada \url{https://www.computecanada.ca/}.
Additionally, we gratefully acknowledge the support of NVIDIA Corporation with the donation of the Titan Xp GPU used for this research.
\end{ack}

\FloatBarrier


\bibliography{main}
\bibliographystyle{iclr2022_conference}

\section*{Checklist}

\hider{
The checklist follows the references.  Please
read the checklist guidelines carefully for information on how to answer these
questions.  For each question, change the default \answerTODO{} to \answerYes{},
\answerNo{}, or \answerNA{}.  You are strongly encouraged to include a {\bf
justification to your answer}, either by referencing the appropriate section of
your paper or providing a brief inline description.  For example:
\begin{itemize}
  \item Did you include the license to the code and datasets? \answerYes{See Section~\ref{gen_inst}.}
  \item Did you include the license to the code and datasets? \answerNo{The code and the data are proprietary.}
  \item Did you include the license to the code and datasets? \answerNA{}
\end{itemize}
Please do not modify the questions and only use the provided macros for your
answers.  Note that the Checklist section does not count towards the page
limit.  In your paper, please delete this instructions block and only keep the
Checklist section heading above along with the questions/answers below.
}

\begin{enumerate}

\item For all authors...
\begin{enumerate}
  \item Do the main claims made in the abstract and introduction accurately reflect the paper's contributions and scope?
    \answerYes{}
  \item Did you describe the limitations of your work?
    \answerYes{}
  \item Did you discuss any potential negative societal impacts of your work?
    \answerNo{The methodology we discuss is general-purpose and any societal impacts we could discuss would be highly general and completely speculative at this stage.}
  \item Have you read the ethics review guidelines and ensured that your paper conforms to them?
    \answerYes{}
\end{enumerate}

\item If you are including theoretical results...
\begin{enumerate}
  \item Did you state the full set of assumptions of all theoretical results?
    \answerYes{}
        \item Did you include complete proofs of all theoretical results?
    \answerYes{}
\end{enumerate}

\item If you ran experiments...
\begin{enumerate}
  \item Did you include the code, data, and instructions needed to reproduce the main experimental results (either in the supplemental material or as a URL)?
    \answerYes{See supplemental material.}
  \item Did you specify all the training details (e.g., data splits, hyperparameters, how they were chosen)?
    \answerYes{See \autoref{s:parity}--\ref{s:transfer} and \aref{a:parity}--\ref{a:czsl}.}
        \item Did you report error bars (e.g., with respect to the random seed after running experiments multiple times)?
    \answerYes{All results figures and tables have error over several seeds, and the number of seeds used, clearly indicated.}
        \item Did you include the total amount of compute and the type of resources used (e.g., type of GPUs, internal cluster, or cloud provider)?
    \answerYes{See \aref{a:hardware}. The time to run each experiment is also indicated in the README files which accompany the code.}
\end{enumerate}

\item If you are using existing assets (e.g., code, data, models) or curating/releasing new assets...
\begin{enumerate}
  \item If your work uses existing assets, did you cite the creators?
    \answerYes{See \aref{a:datasets}, \autoref{tab:datasets}.}
  \item Did you mention the license of the assets?
    \answerNo{It is surprisingly difficult to discover the license under which popular ML datasets, such as CIFAR-10, are released! Most of the popular datasets within the community do not appear to be released under open licenses.}
  \item Did you include any new assets either in the supplemental material or as a URL?
    \answerYes{Code to replicate our results is included in the supplemental material zip.}
  \item Did you discuss whether and how consent was obtained from people whose data you're using/curating?
    \answerNA{}
  \item Did you discuss whether the data you are using/curating contains personally identifiable information or offensive content?
    \answerNo{We describe the nature of the datasets we are using. There is no reason to expect that the data we are using contains personally identifiable information.}
\end{enumerate}

\item If you used crowdsourcing or conducted research with human subjects...
\begin{enumerate}
  \item Did you include the full text of instructions given to participants and screenshots, if applicable?
    \answerNA{}
  \item Did you describe any potential participant risks, with links to Institutional Review Board (IRB) approvals, if applicable?
    \answerNA{}
  \item Did you include the estimated hourly wage paid to participants and the total amount spent on participant compensation?
    \answerNA{}
\end{enumerate}

\end{enumerate}


\clearpage

\appendix
\section{Appendix}

\subsection{Approximating logistic AND and OR using ReLU}
\label{a:relu}

In \autoref{fig:ap:relu} (upper panels), we show a representation of a 2D linear layer followed by the ReLU activation function.
Changing projection used in the linear layer allows us to rotate and stretch the output only; each unit is, by construction, empty in half the plane.

\begin{figure}[h]
    \centering
    \includegraphics[scale=0.45]{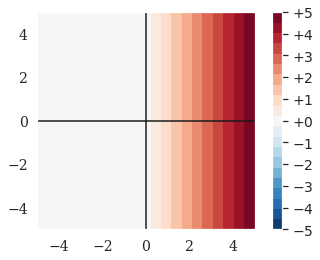}
    \includegraphics[scale=0.45]{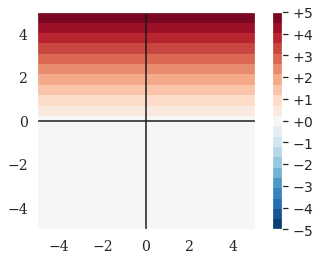}
    \includegraphics[scale=0.45]{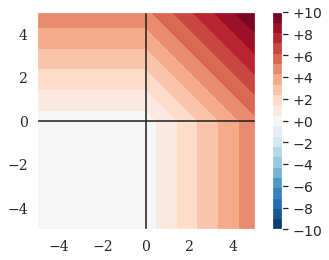}
\caption{
ReLU unit, and ReLU units followed by a linear layer to try to approximate $\opn{OR_{IL}}$, leaving a dead space where the negative logits should be.
}
\label{fig:ap:relu}
\end{figure}

If we apply a second linear layer on top of the outputs of the first two units, we can try to approximate the logit AND or OR function with $z = \opn{ReLU}(x) + \opn{ReLU}(y)$.
However, as shown in  \autoref{fig:ap:relu}, lower panel, the solution using ReLU leaves a quadrant of the output space set to zero due to its behaviour at truncating away information.
This illustrates why a ReLU-based neural network can not accurately approximate the logistic AND and OR functions between a pair of logits with only two hidden units.
Comparing this with Figures \ref{fig:AND} and \ref{fig:OR} demonstrates the advantage of $\opn{AND_{AIL}}$ and $\opn{OR_{AIL}}$ in situations where the network needs to use probabilistic Boolean logic as its basis.


\FloatBarrier

\subsection{Solving XOR}
\label{a:xor}

A long-standing criticism of artificial neural networks is their inability to solve XOR with a single layer \citep{perceptrons}.
Of course, adding a single hidden layer allows a network using ReLU to solve XOR.
However, the way that it solves the problem is to join two of the disconnected regions together in a stripe (see \autoref{fig:toy-xor}), which does not extrapolate to correctly solve the task beyond the bounds of the training data.
Meanwhile, our $\opn{XNOR_{NIL}}$ and $\opn{XNOR_{NAIL}}$ activations are trivially able to solve the XOR problem without any hidden layers.
For comparison here, we include one hidden layer with 2 units for each network.
Including a layer before the activation function makes the task harder for networks using $\opn{XNOR_{*IL}}$, activations which must learn how to project the input space in order to compute the desired separation.
Also, including the linear layer allows the network to generalize to rotations and offset versions of the task.

\begin{figure}[h]
    \centering
    \includegraphics[scale=0.47]{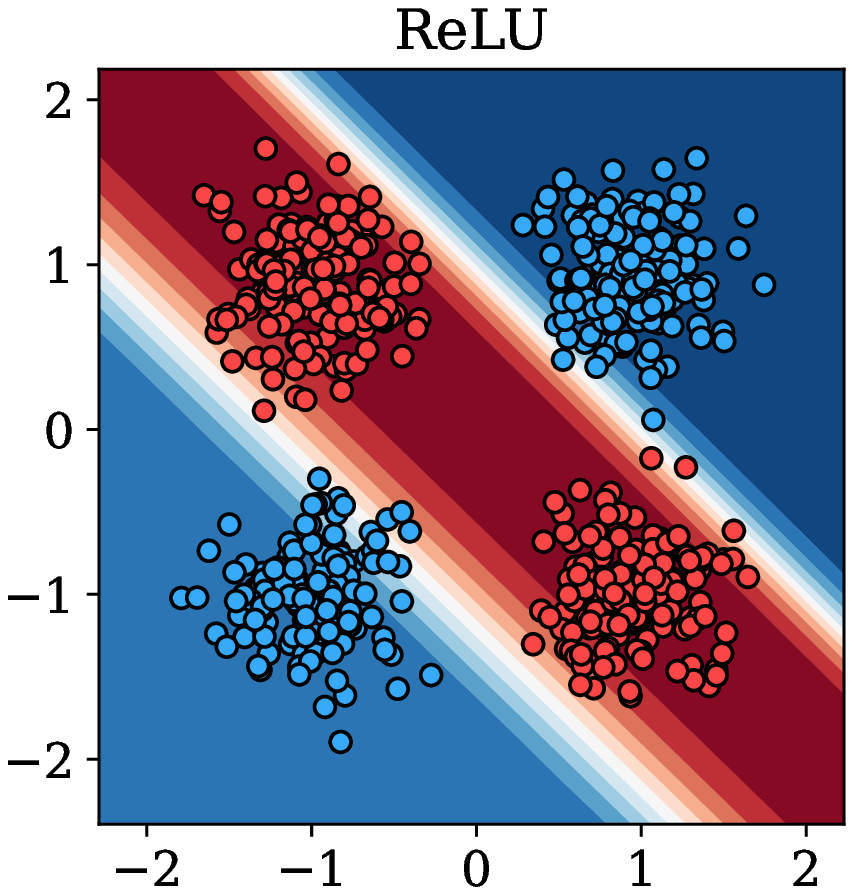}
    \includegraphics[scale=0.47]{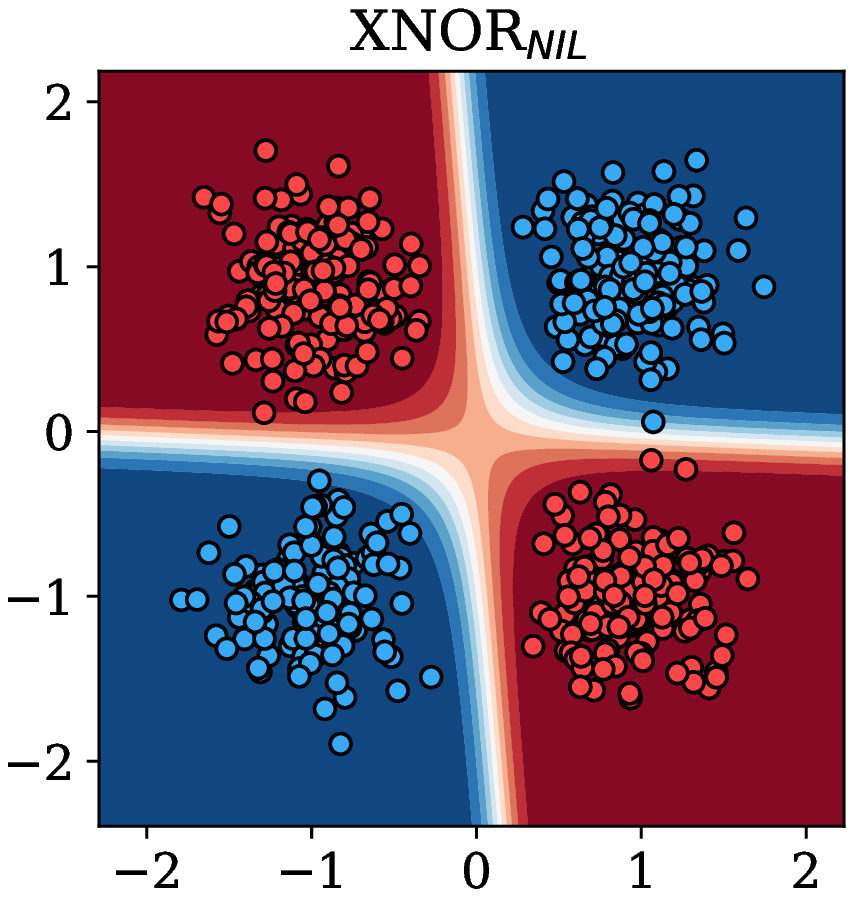}
    \includegraphics[scale=0.47]{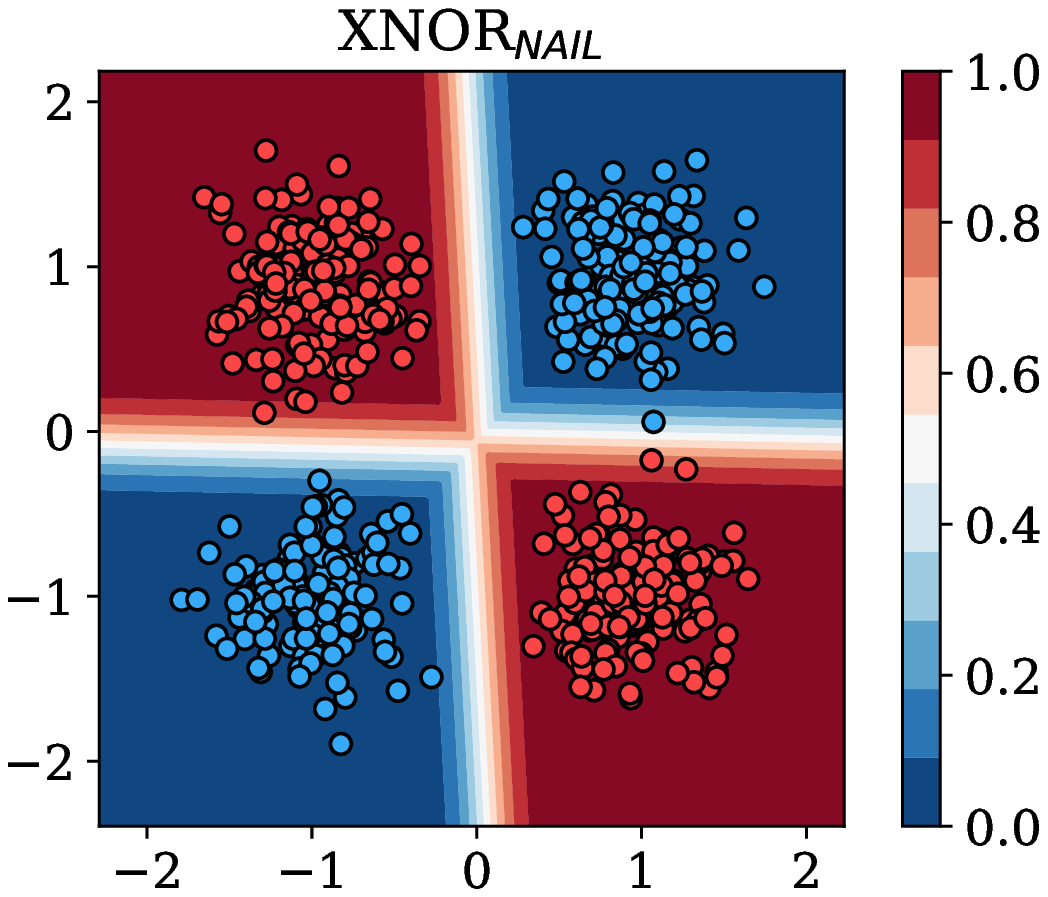}
\caption{
    Solving XOR with a single hidden layer of 2 units, using either ReLU, $\opn{XNOR_{NIL}}$, or $\opn{XNOR_{NAIL}}$ activation.
    Circles indicate negative (blue) and positive (red) training samples.
    The heatmaps indicate the output probabilities of the networks.
}
\label{fig:toy-xor}
\end{figure}

\FloatBarrier

\subsection{Linear layers and Bayes' Rule in log-odds form}
\label{a:bayesrule}

Bayes' Theorem or Bayes' Rule is given by
\begin{equation}
\PP(H|X) = \frac{\PP(X|H) \, \PP(H)}{\PP(X)}
\label{eq:bayesruleprob}
.\end{equation}
In this case, we update the probability of our hypothesis, $H$, based on the event of the observation of a new piece of evidence, $X$.
Our prior belief for the hypothesis is $\PP(H)$, and posterior is $\PP(H|X)$.
To update our belief from the prior and yield the posterior, we multiply by the Bayes factor for the evidence which is given by $\nicefrac{\PP(X|H)}{\PP(X)}$.

Converting the probabilities into log-odds ratios (logits) yields the following representation of Bayes' Rule.
\begin{equation}
\log \left( \frac{\PP(H|X)}{\PP(H^C|X)} \right) = \log \left( \frac{\PP(H)}{\PP(H^C)} \right) + \log \left( \frac{\PP(X|H)}{\PP(X|H^C)} \right)
\label{eq:bayesrulelogit}
\end{equation}
Here, $H^C$ is the complement to $H$ (the event that the hypothesis is false), and $\PP(H^C) = 1 - \PP(H)$.
Our prior log-odds ratio is $\log\left((\nicefrac{\PP(H)}{\PP(H^C)}\right)$, and our posterior after updating based on the observation of new evidence $X$ is $\log\left(\nicefrac{\PP(H|X)}{\PP(H^C|X)}\right)$.
To update our belief from the prior and yield the posterior, we add the log-odds Bayes factor for the evidence which is given by $\log\left(\nicefrac{\PP(X|H)}{\PP(X|H^C)}\right)$.

In log-odds space, a series of updates with multiple pieces of independent evidence can be performed at once with a summation operation.
\begin{equation}
\log \left( \frac{\PP(H|\vx)}{\PP(H^C|\vx)} \right) = \log \left( \frac{\PP(H)}{\PP(H^C)} \right) + \sum_i \log \left( \frac{\PP(X_i|H)}{\PP(X_i|H^C)} \right)
\label{eq:bayesrulelogitmulti}
.\end{equation}

This operation can be represented by the linear layer in an artificial neural network, $z_k=b_k+\vw_k^T\vx$.
Here, the bias term $b_k = \log\left((\nicefrac{\PP(H)}{\PP(H^C)}\right)$ is the prior for hypothesis (the presence of the feature represented by the k-th neuron), and the series of weighted inputs from the previous layer, $w_{ki}\,x_i$ provide evidential updates.
This is also equivalent to the operation of a multinomial naïve Bayes classifier, expressed in log-space, if we choose $w_{{ki}}=\log p_{{ki}}$ \citep{Rennie2003}.

\subsection{Difference between AIL and IL functions}
\label{a:diff}

Here, we measure and show the difference between the true logit-space operations and our AIL approximations, shown in \autoref{fig:ap:AND}, \autoref{fig:ap:OR}, and \autoref{fig:ap:XNOR}.

In each case, we observe that the magnitude of the difference is never more than 1, which occurs along the boundary lines in AIL.
Since the magnitude of the three functions increase as we move away from the origin, the relative difference decreases in magnitude as the size of $x$ and $y$ increase.

\begin{figure}[h]
    \centering
    \includegraphics[scale=0.45]{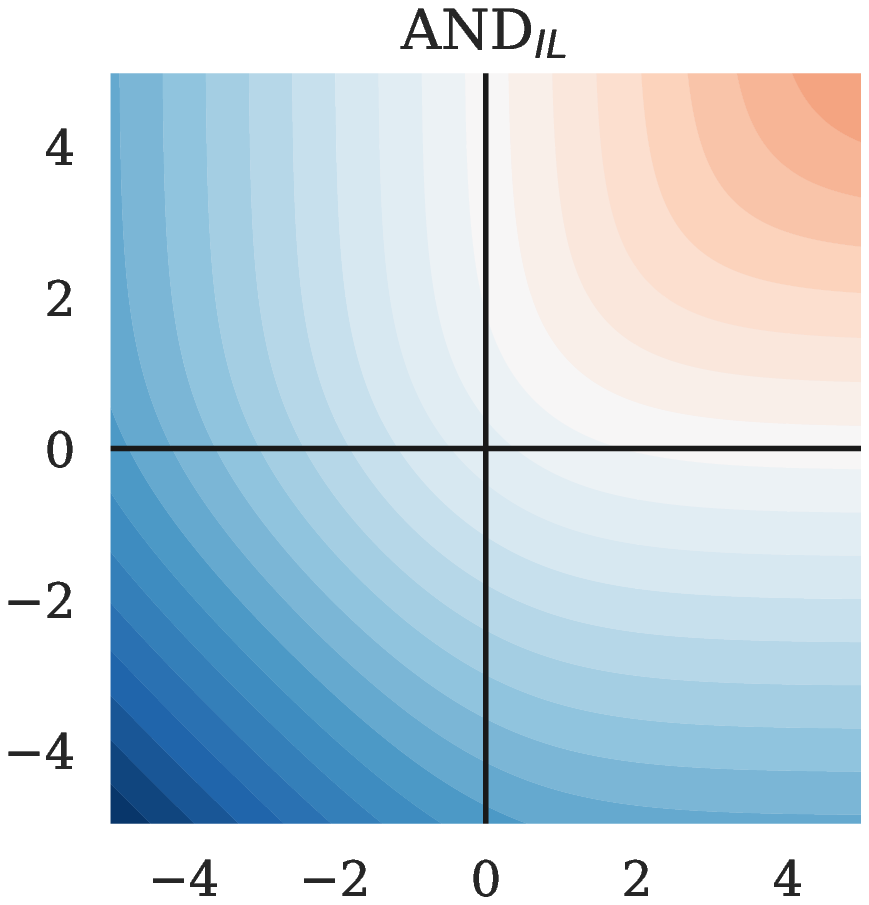}
    \includegraphics[scale=0.45]{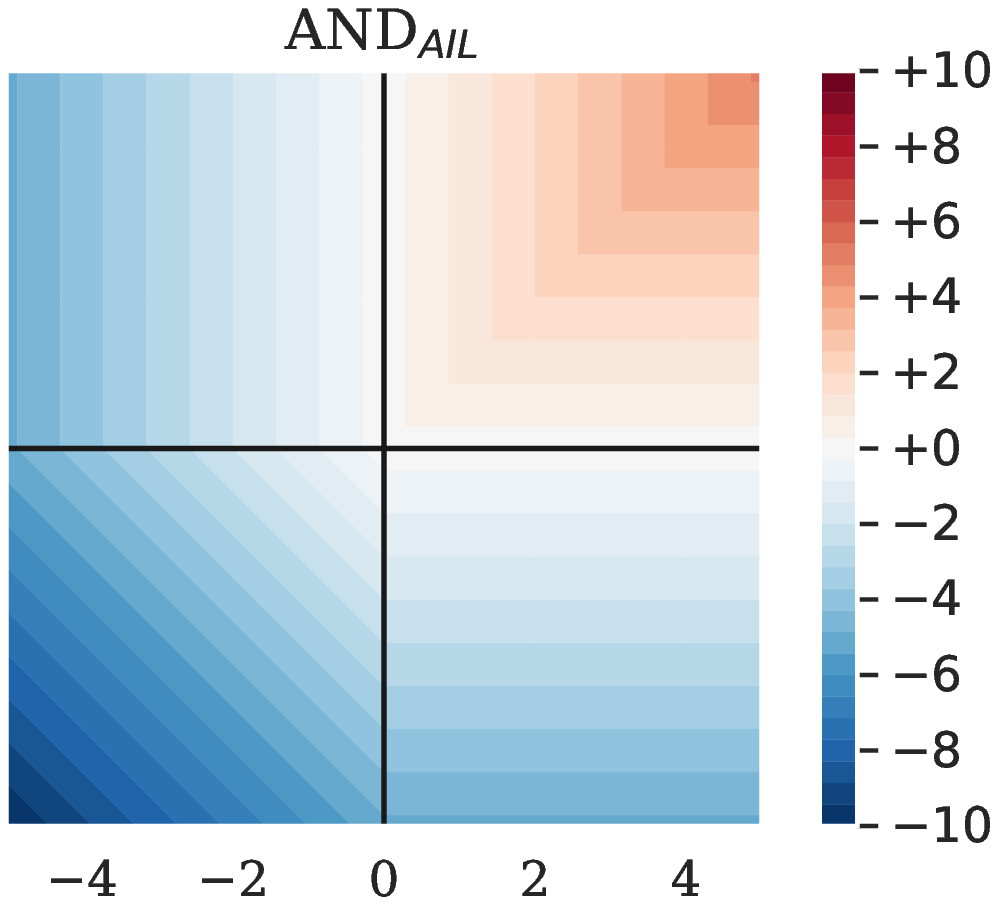}
    \\
    \includegraphics[scale=0.45]{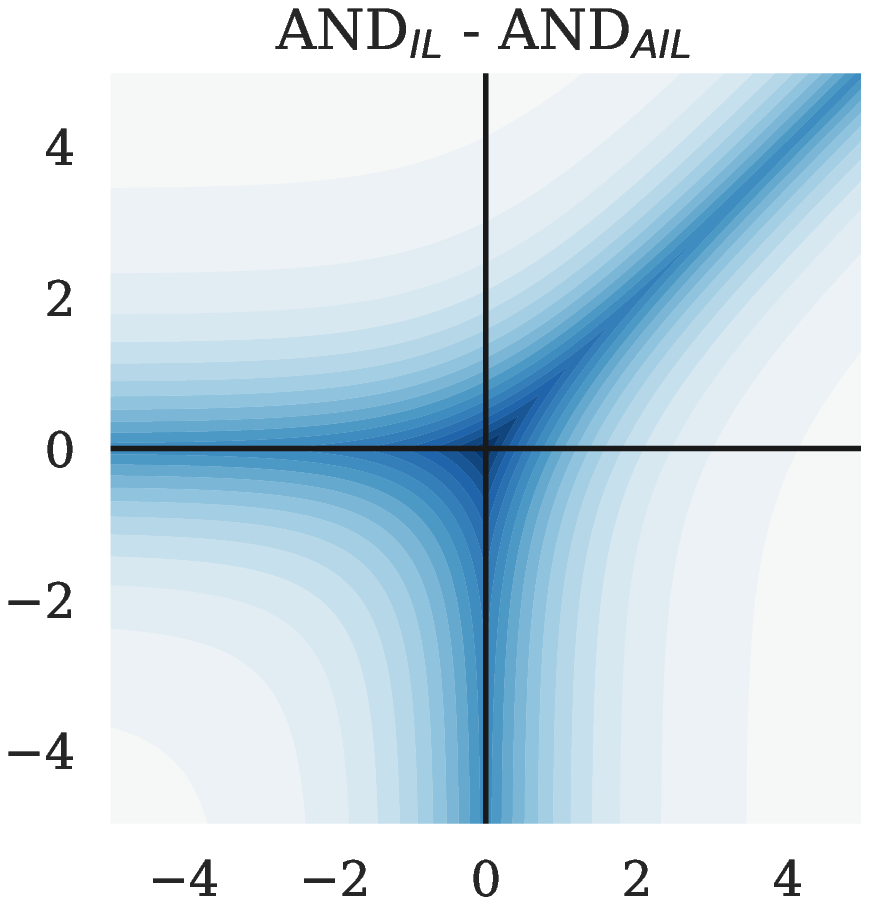}
    \includegraphics[scale=0.45]{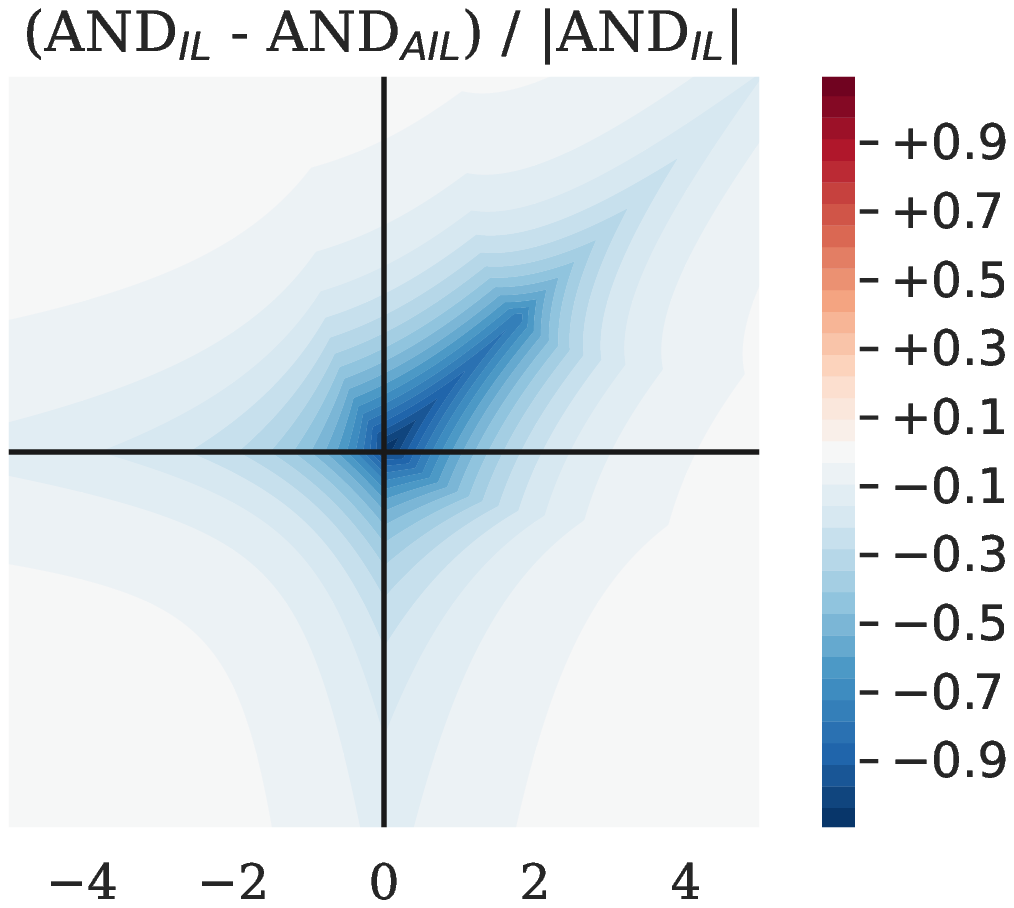}
\caption{
Heatmaps showing $\opn{AND_{IL}}$, $\opn{AND_{AIL}}$, their difference, and their relative difference.
}
\label{fig:ap:AND}
\end{figure}

\begin{figure}[h]
    \centering
    \includegraphics[scale=0.45]{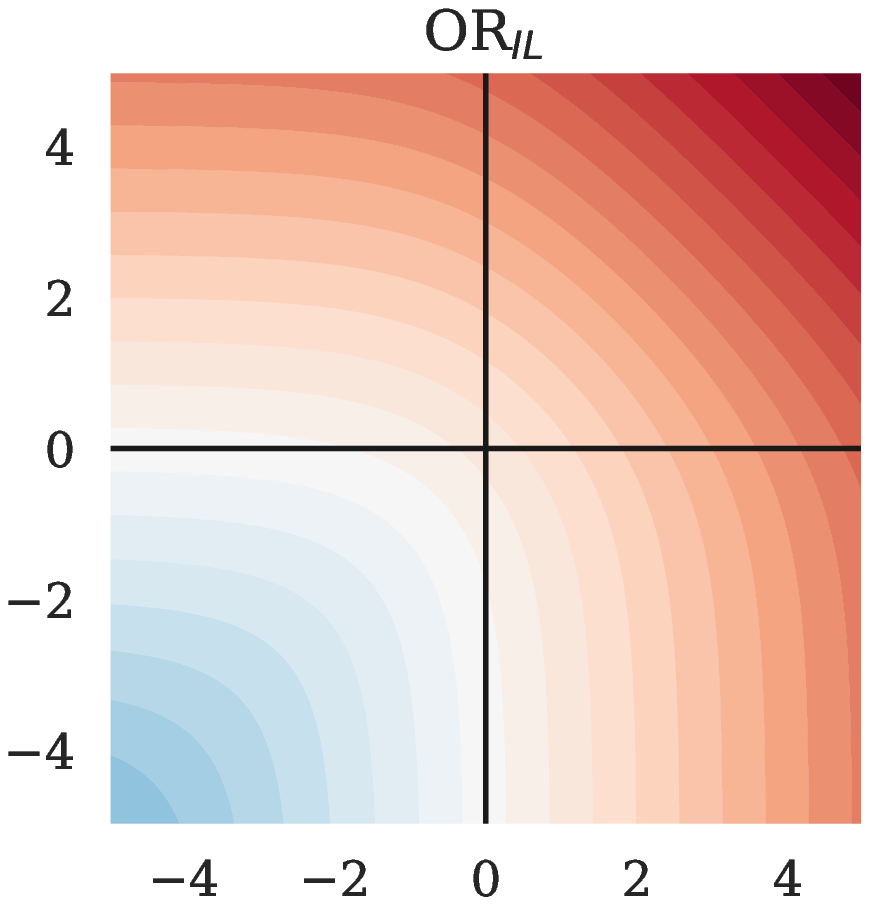}
    \includegraphics[scale=0.45]{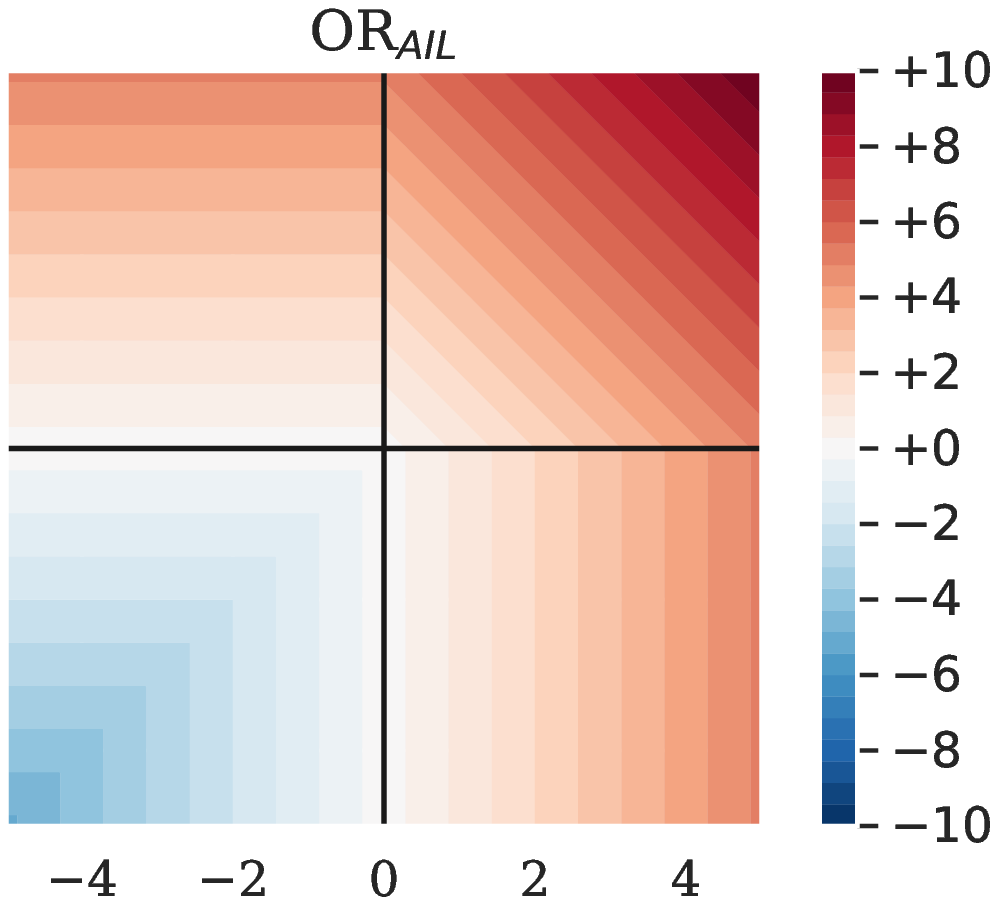}
    \\
    \includegraphics[scale=0.45]{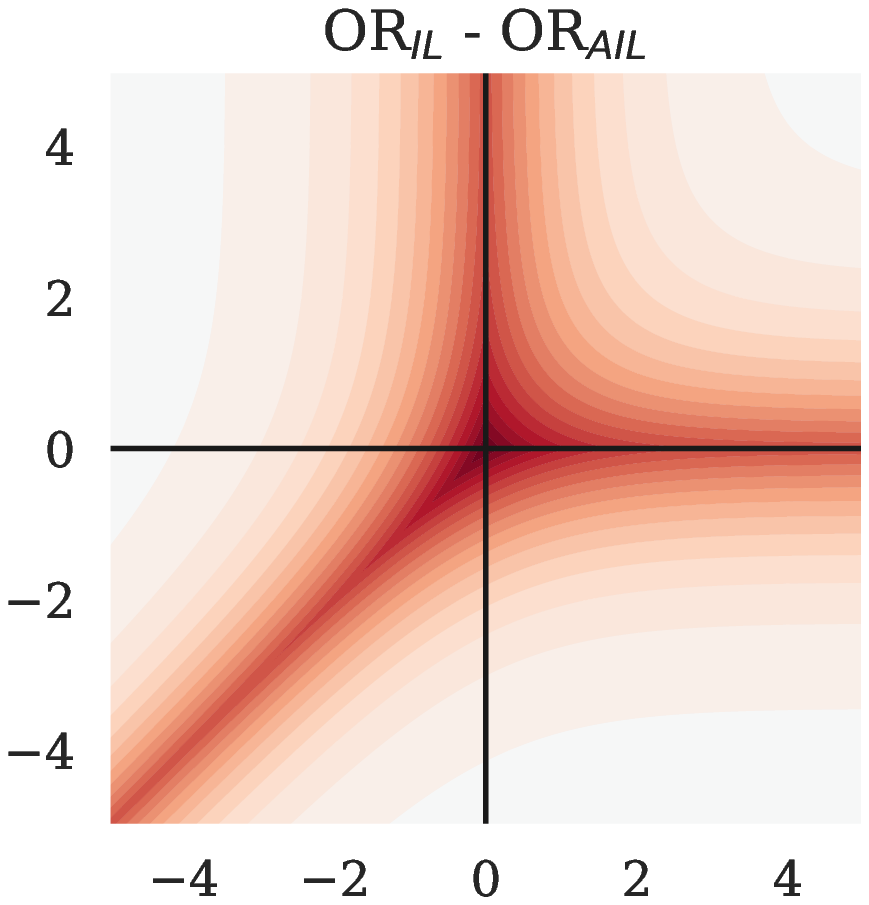}
    \includegraphics[scale=0.45]{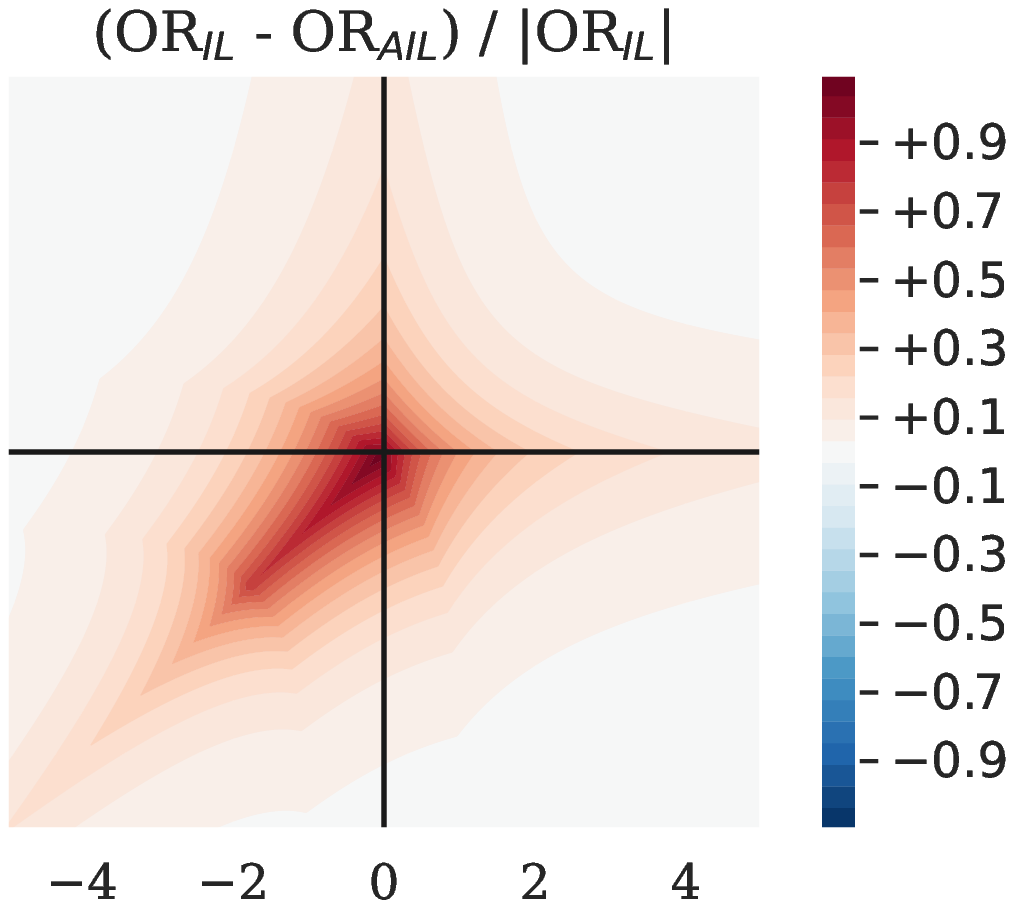}
\caption{
Heatmaps showing $\opn{OR_{IL}}$, $\opn{OR_{AIL}}$, their difference, and their relative difference.
}
\label{fig:ap:OR}
\end{figure}

\begin{figure}[h]
    \centering
    \includegraphics[scale=0.45]{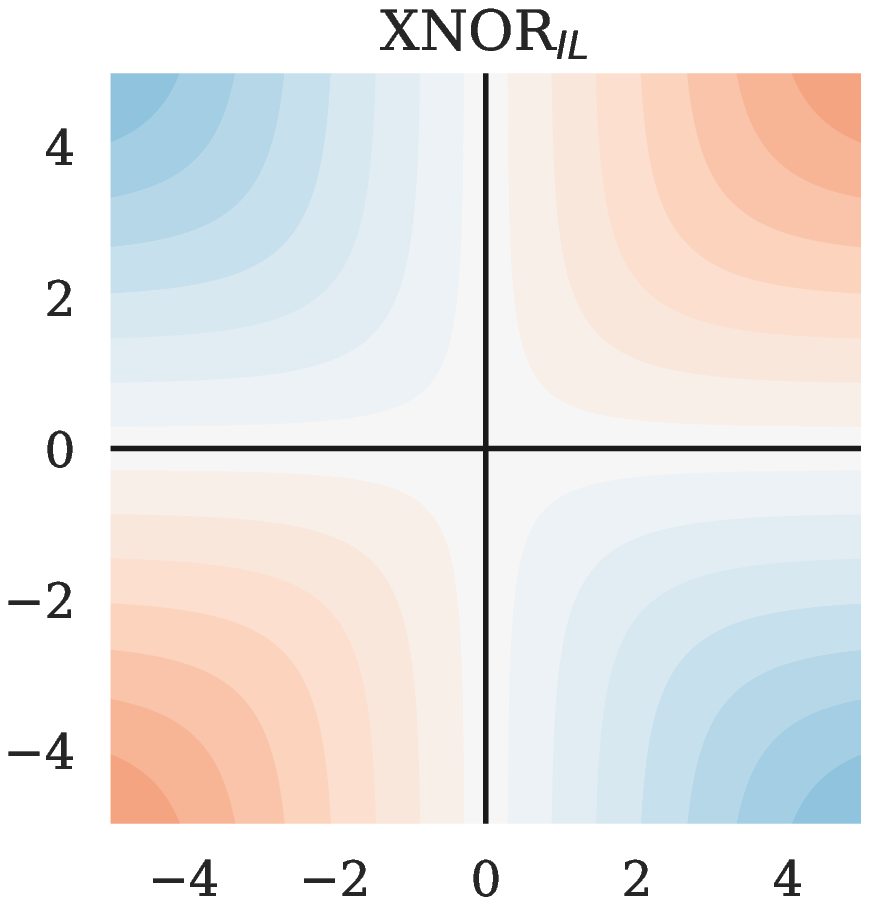}
    \includegraphics[scale=0.45]{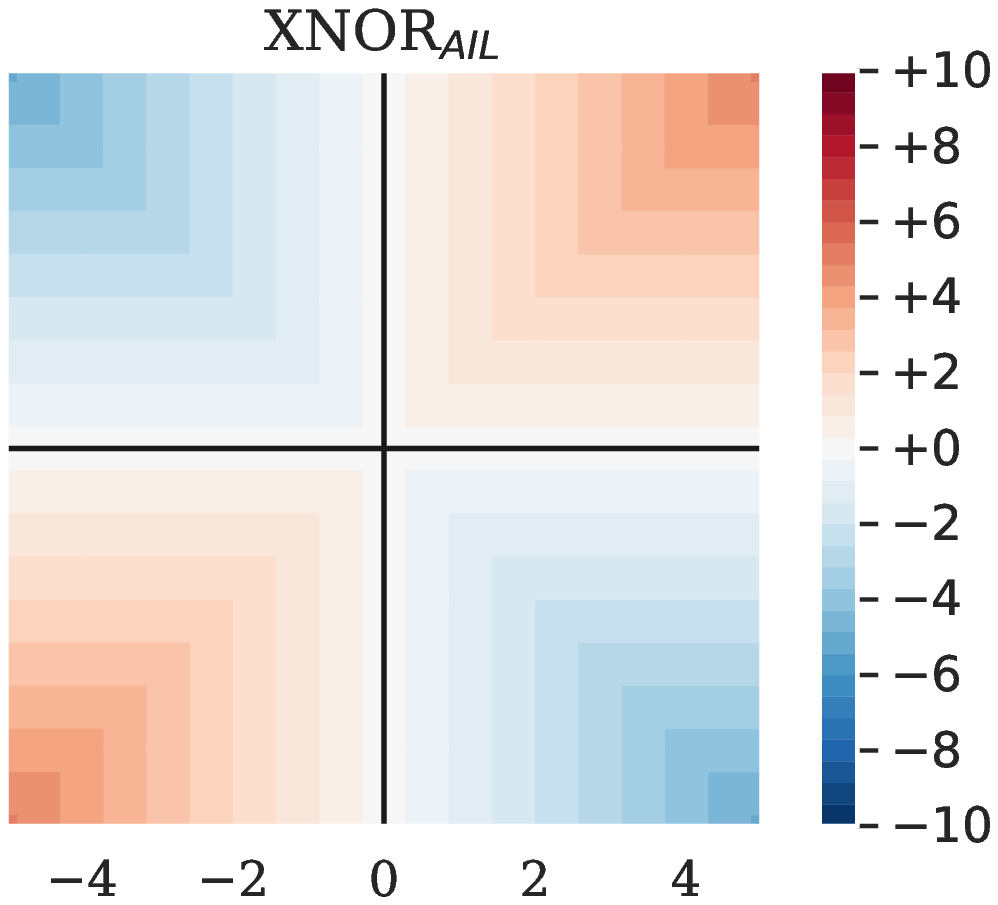}
    \\
    \includegraphics[scale=0.45]{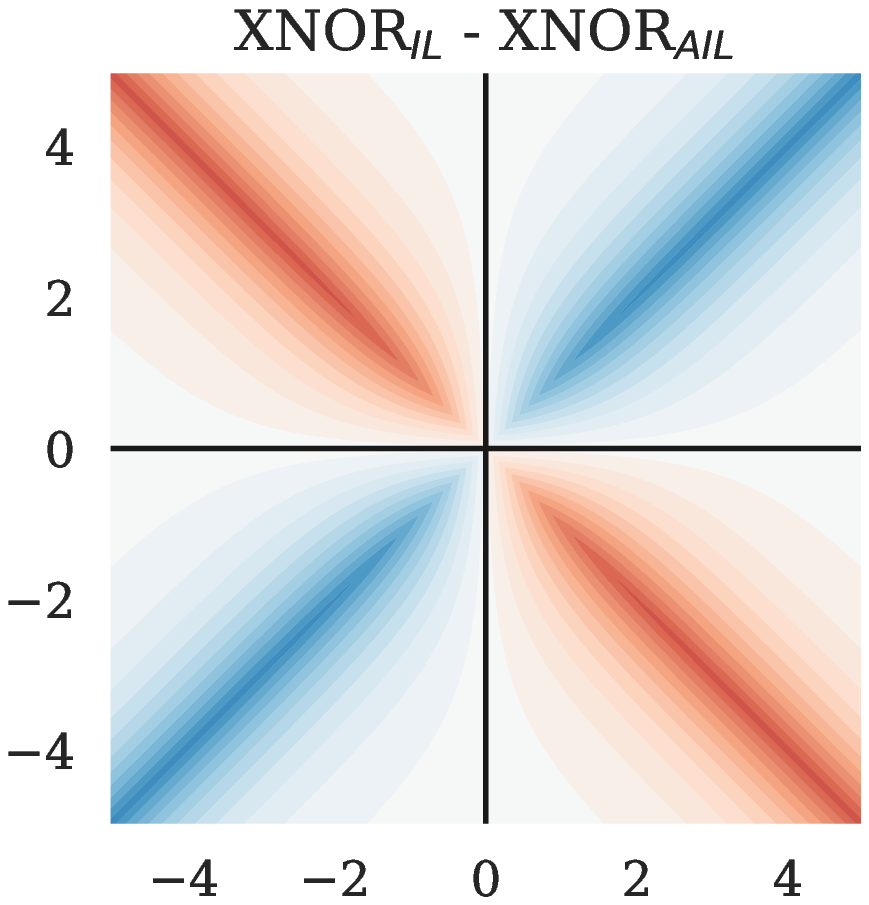}
    \includegraphics[scale=0.45]{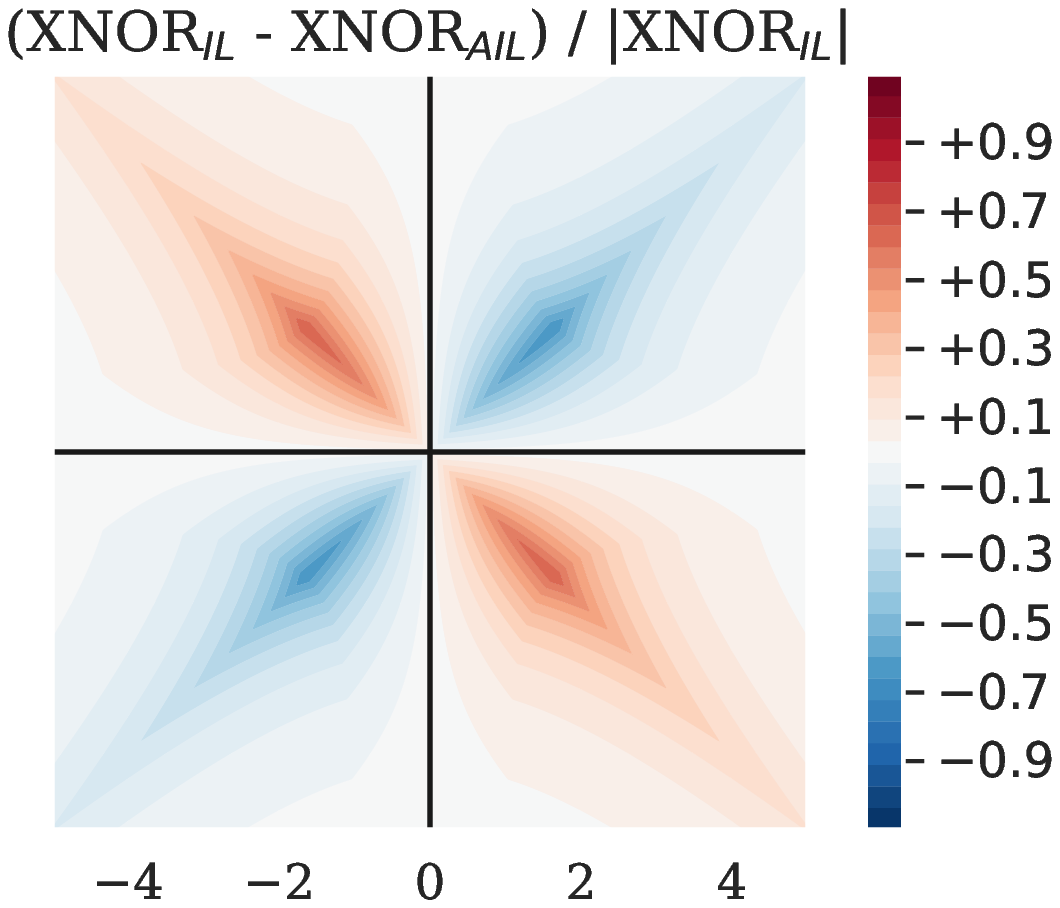}
\caption{
Heatmaps showing $\opn{XNOR_{IL}}$, $\opn{XNOR_{AIL}}$, their difference, and their relative difference.
}
\label{fig:ap:XNOR}
\end{figure}

\FloatBarrier

\subsection{Gradient of AIL and IL functions}
\label{a:gradients}

We show the gradient of each of the logit-space Boolean operators and their AIL approximates in \autoref{fig:ap:grad:AND}, \autoref{fig:ap:grad:OR}, and \autoref{fig:ap:grad:XNOR}.
By the symmetry of each of the functions, the derivative with respect to $y$ is a reflected copy of the gradient with respect to $x$.

We find that the gradient of each AIL function closely matches that of the exact form.
Whilst there are ``dead'' regions where the gradient is zero, this only occurs for one of the derivatives at a time (there is always a gradient with respect to at least one of $x$ and $y$).

\begin{figure}[h]
    \centering
    \includegraphics[scale=0.45]{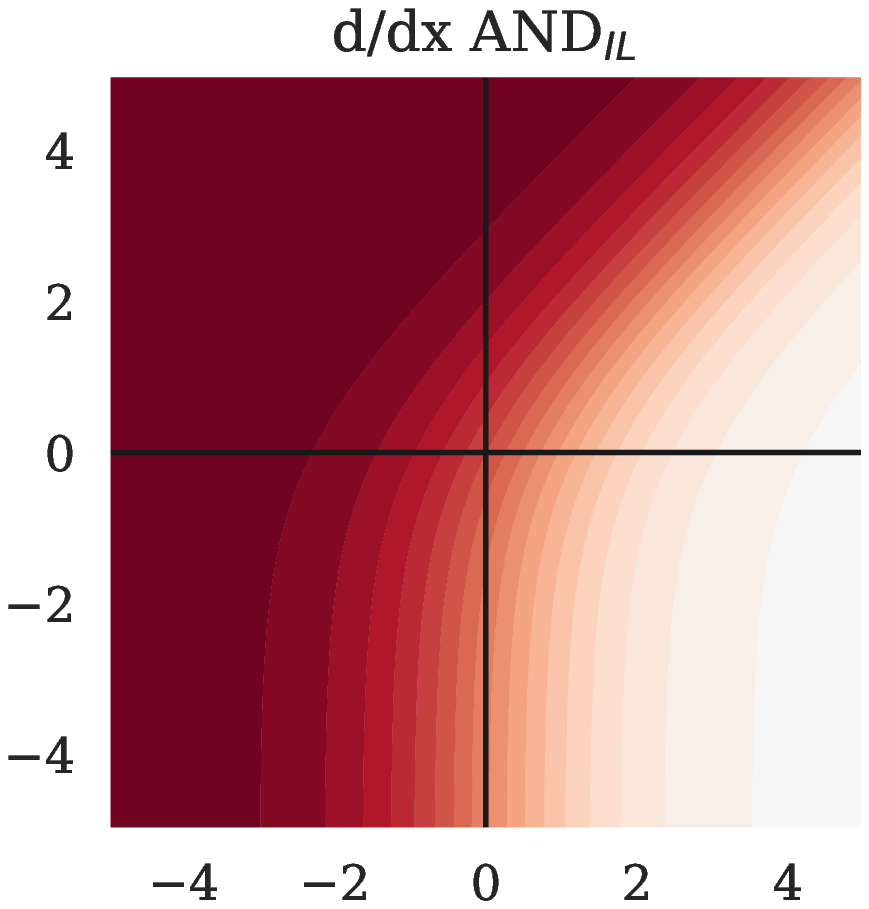}
    \includegraphics[scale=0.45]{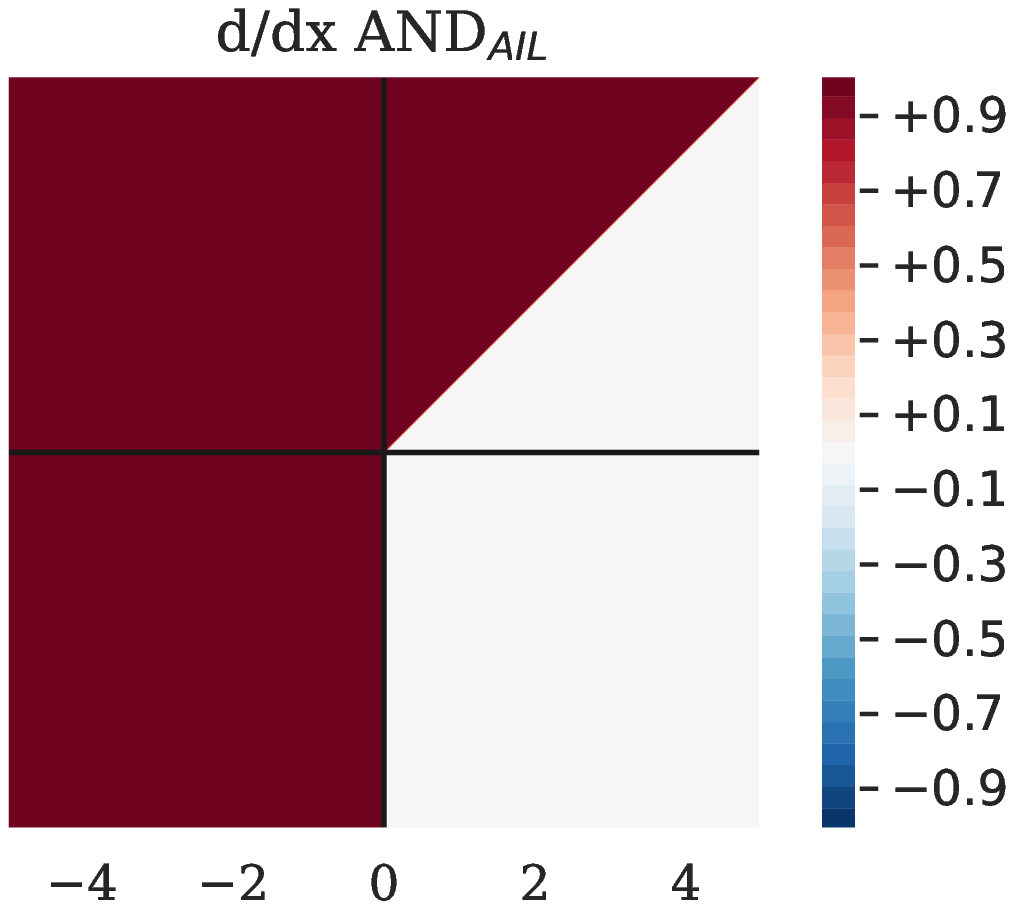}
    \\
    \includegraphics[scale=0.45]{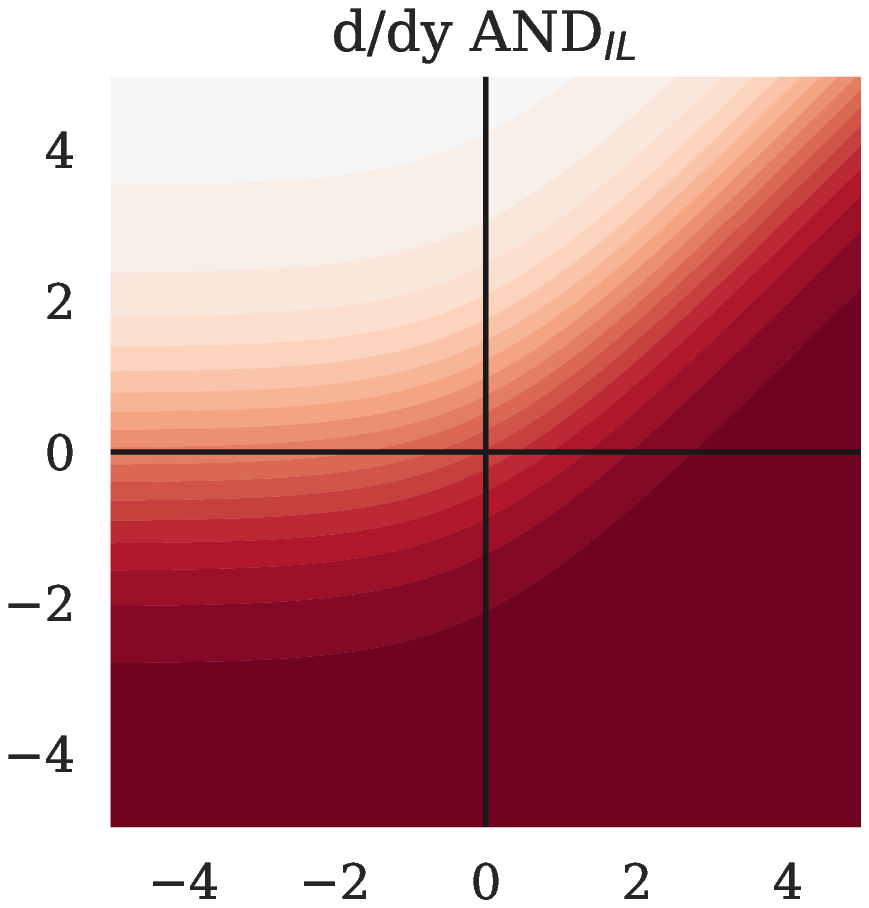}
    \includegraphics[scale=0.45]{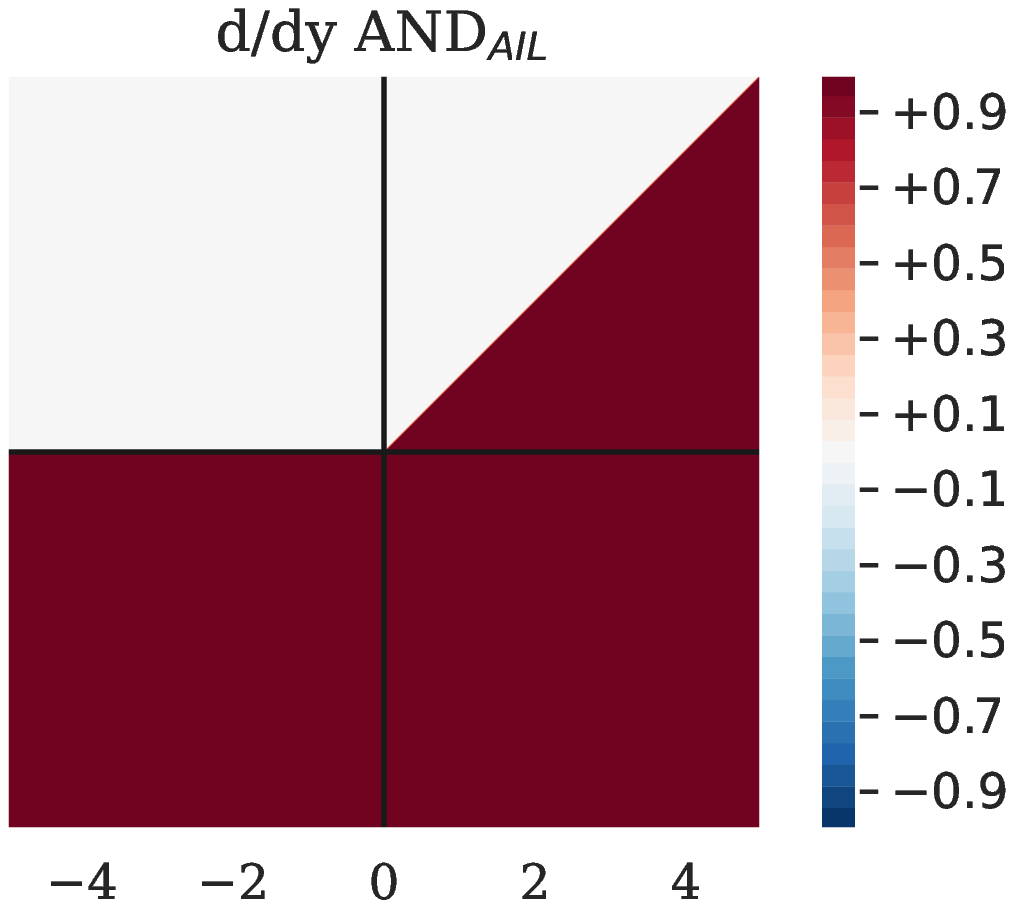}
\caption{
Heatmaps showing the gradient with respect to $x$ and $y$ of $\opn{AND_{IL}}$ and $\opn{AND_{AIL}}$.
}
\label{fig:ap:grad:AND}
\end{figure}

\begin{figure}[h]
    \centering
    \includegraphics[scale=0.45]{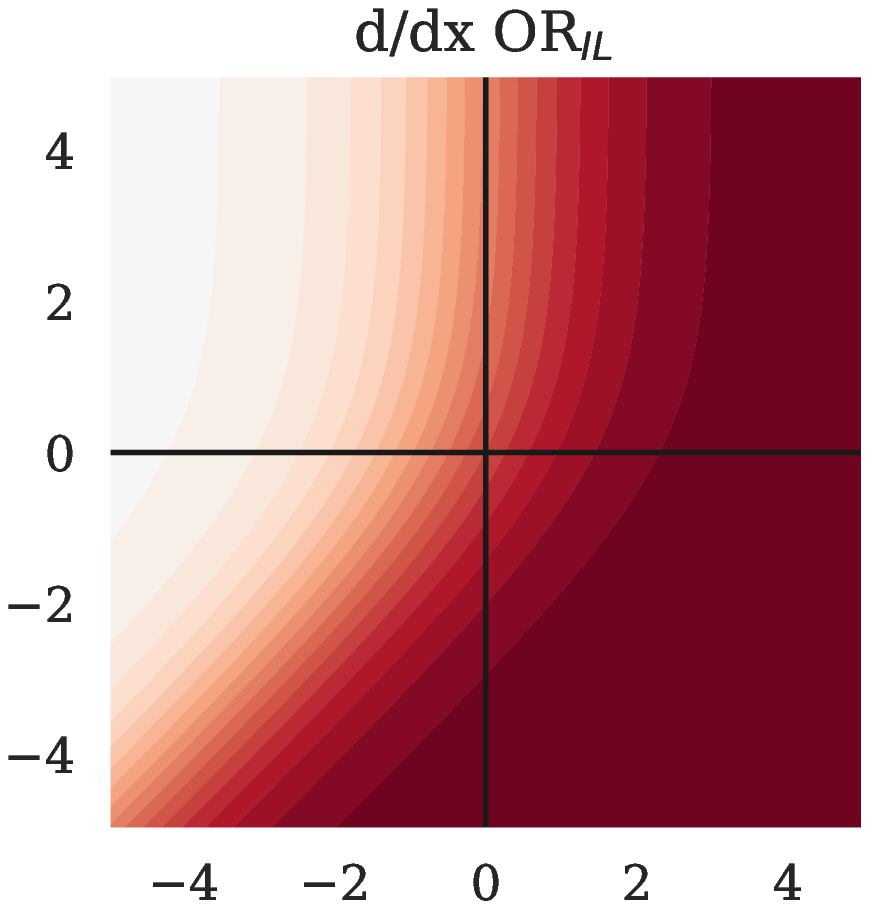}
    \includegraphics[scale=0.45]{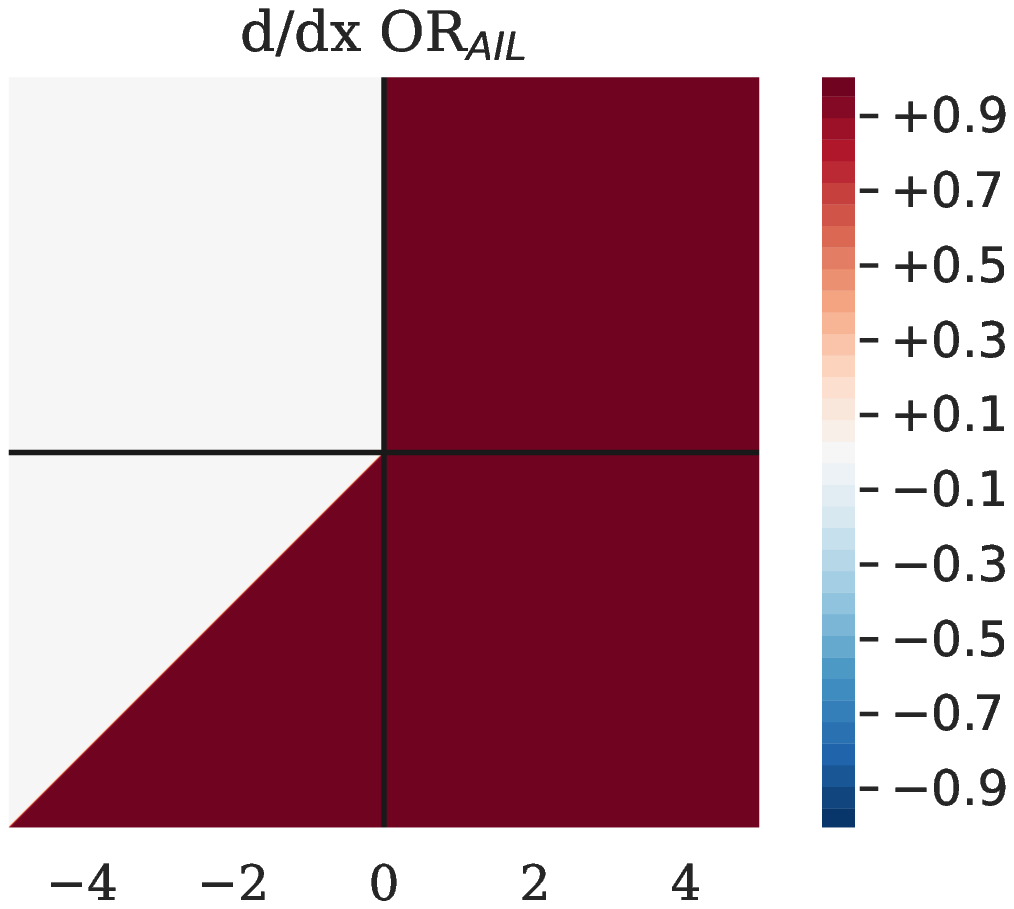}
    \\
    \includegraphics[scale=0.45]{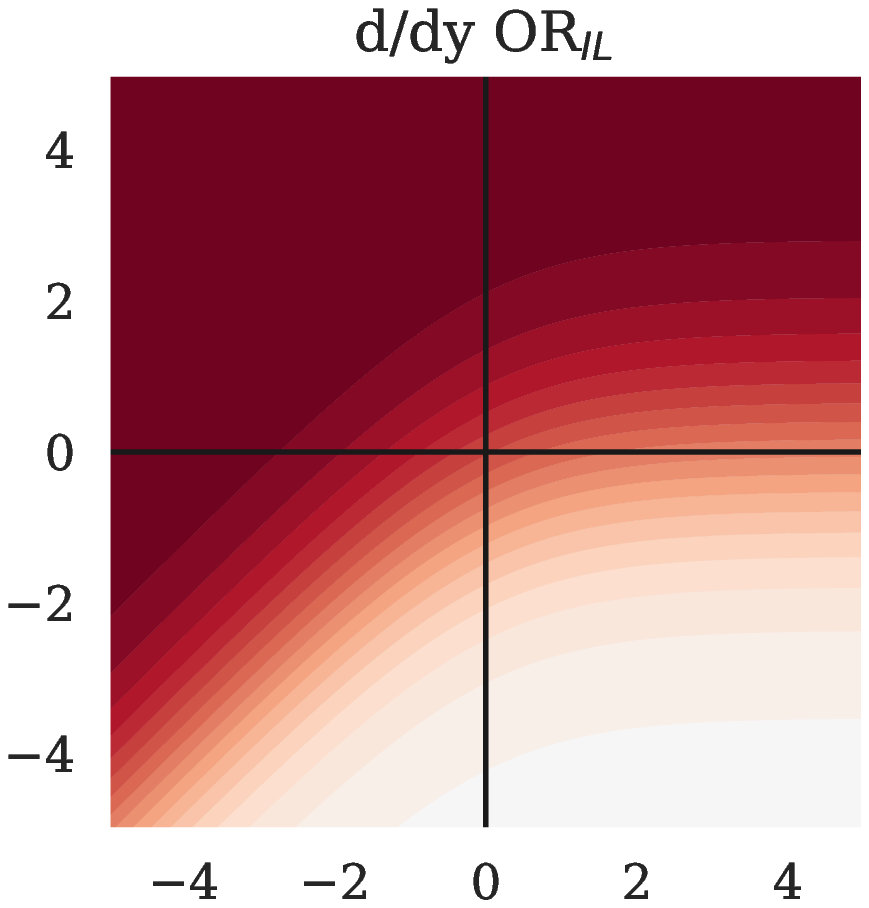}
    \includegraphics[scale=0.45]{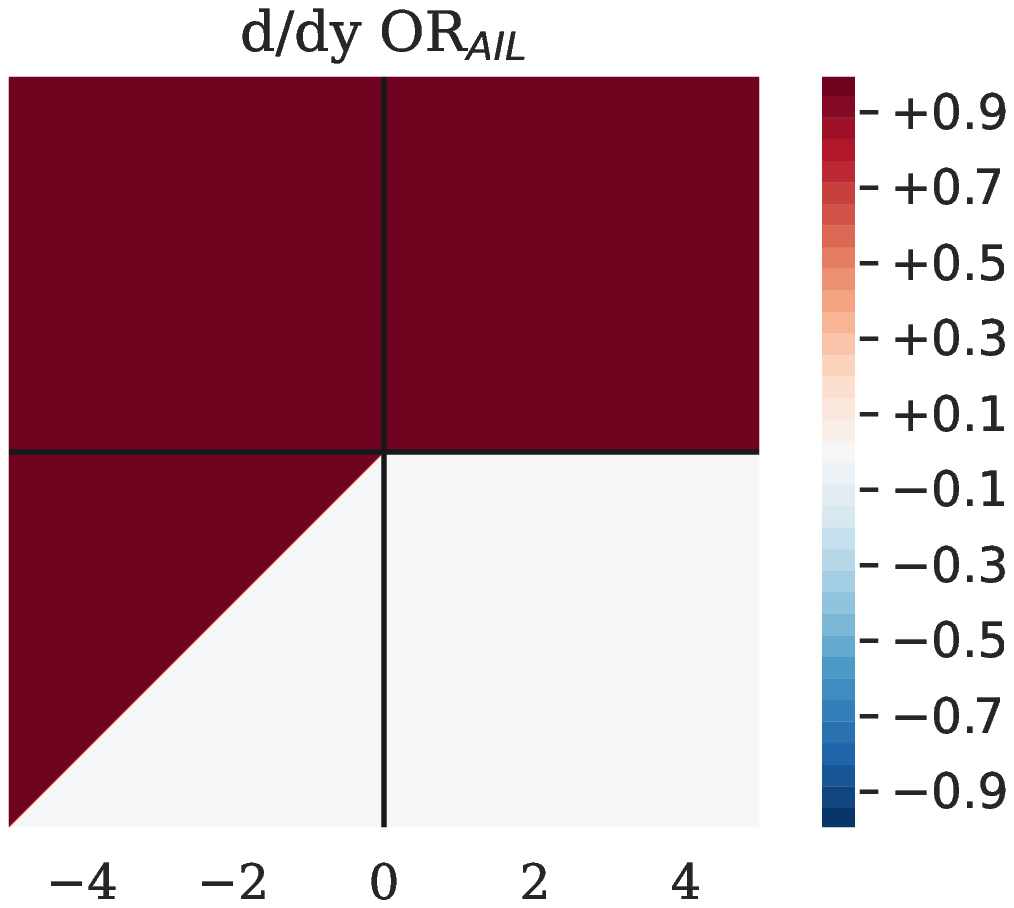}
\caption{
Heatmaps showing the gradient with respect to $x$ and $y$ of $\opn{OR_{IL}}$ or $\opn{OR_{AIL}}$.
}
\label{fig:ap:grad:OR}
\end{figure}

\begin{figure}[h]
    \centering
    \includegraphics[scale=0.45]{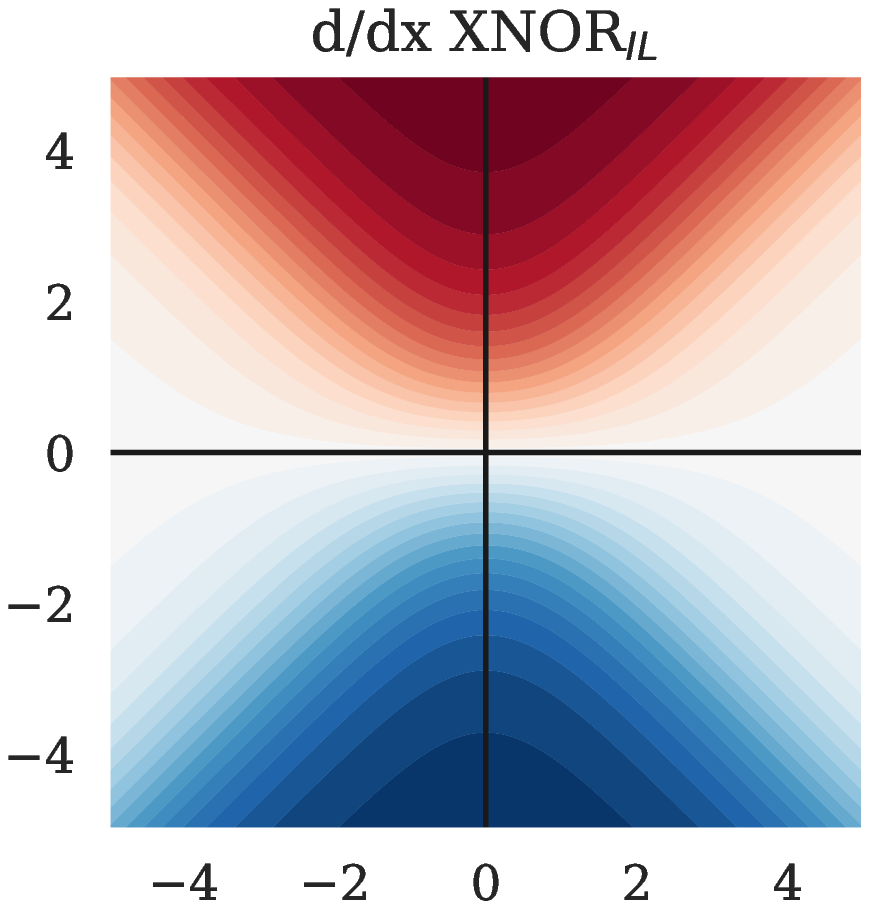}
    \includegraphics[scale=0.45]{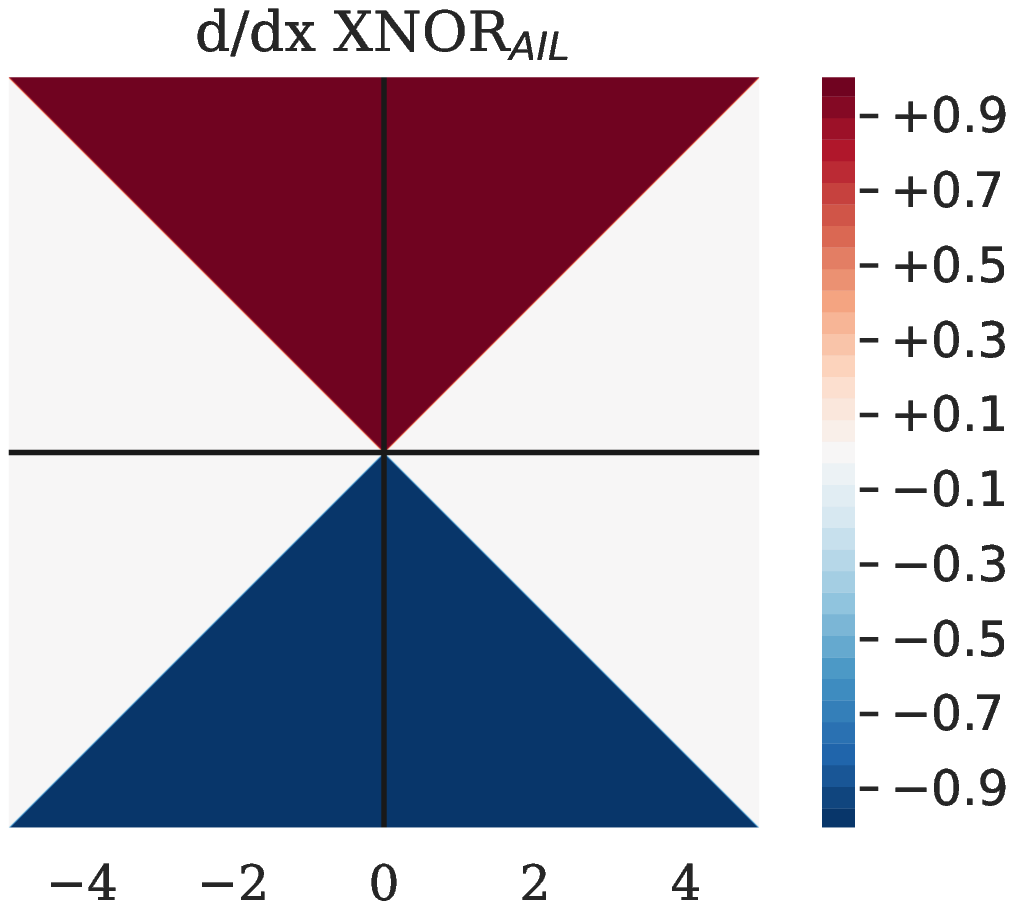}
    \\
    \includegraphics[scale=0.45]{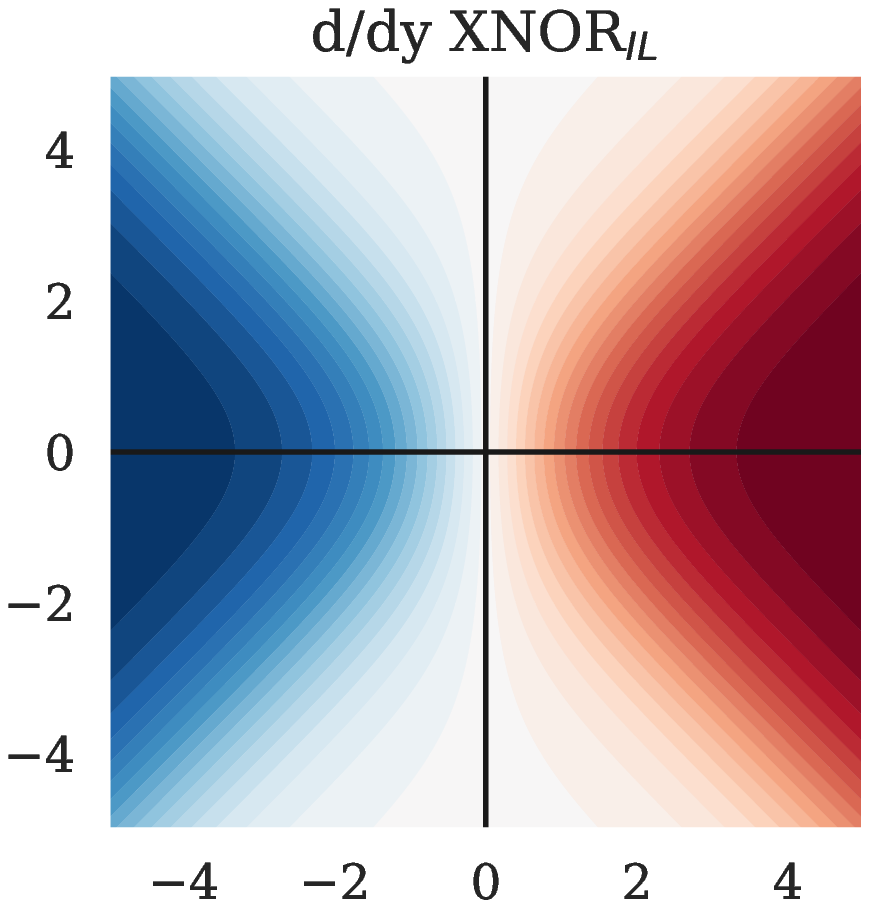}
    \includegraphics[scale=0.45]{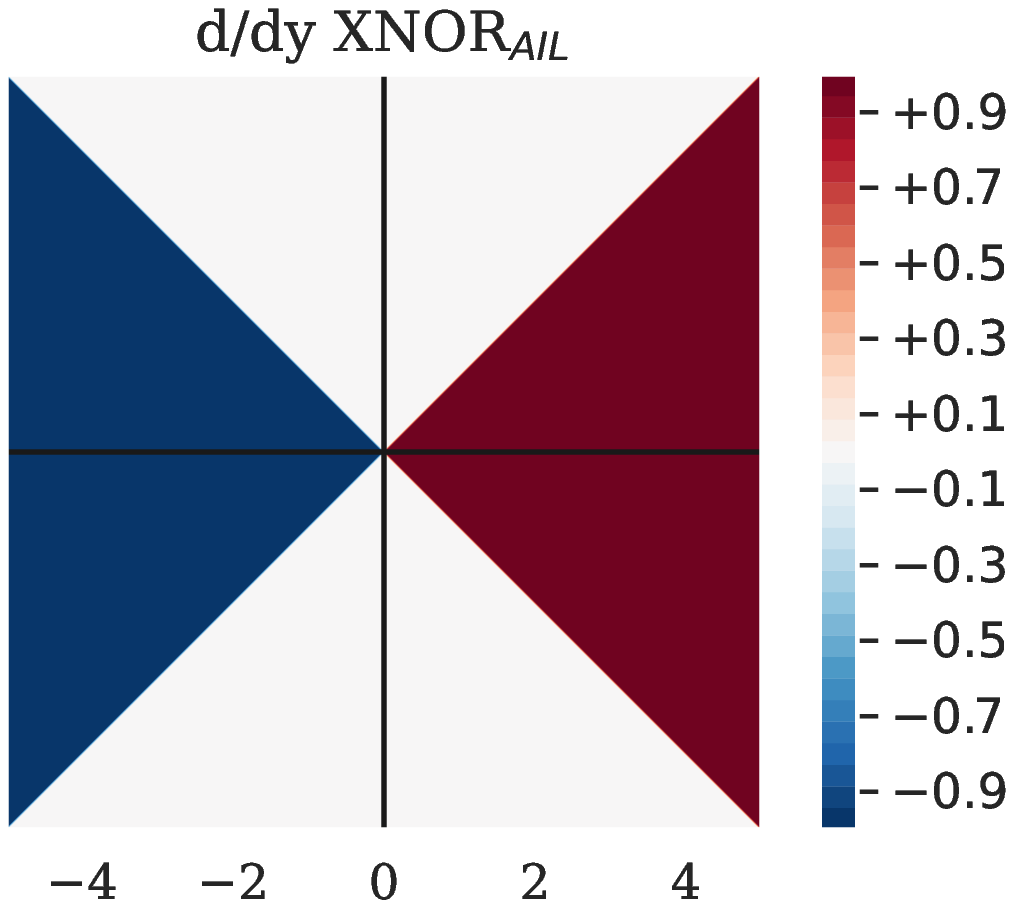}
\caption{
Heatmaps showing the gradient with respect to $x$ and $y$ of $\opn{XNOR_{IL}}$ xnor $\opn{XNOR_{AIL}}$.
}
\label{fig:ap:grad:XNOR}
\end{figure}

\FloatBarrier

\subsection{Activation pathing diagrams}
\label{a:ensembling}

Here we illustrate the implications of using a $2\!\to\!1$ activation function on the network achitecture, and the layout of the partition and duplication ensembling methods described in \autoref{s:ensembling}.

A standard 1D activation function with a $1\!\to\!1$ mapping such as $\opn{ReLU}$ processes each pre-activation logit independently, and the number of post-activation channels equals the pre-activation channels (\autoref{fig:pathway-relu}).
The total number of parameters per layer is $C^2 + C \sim C^2$.

\begin{figure}[h]
    \centering
    \begin{subfigure}[b]{0.47\linewidth}
        \includegraphics[width=0.95\linewidth]{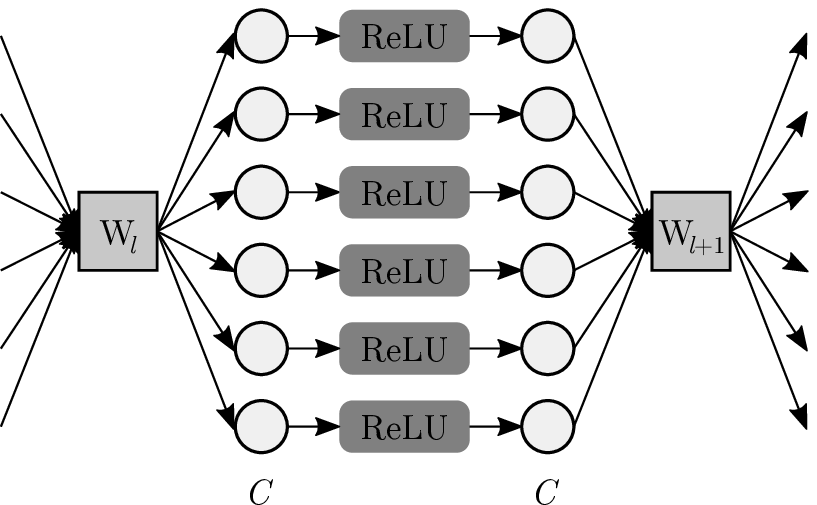}
        \caption{ReLU layer ($1\!\to\!1$ activation function).}
        \label{fig:pathway-relu}
    \end{subfigure}%
    ~~~
    \begin{subfigure}[b]{0.47\linewidth}
        \includegraphics[width=0.95\linewidth]{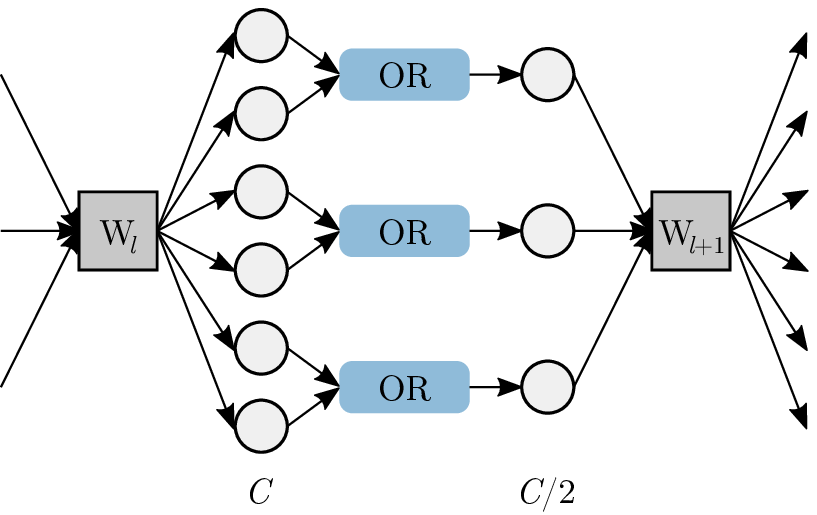}
        \caption{$\opn{OR_{AIL}}$ layer ($2\!\to\!1$ activation function).}
        \label{fig:pathway-or}
    \end{subfigure}%
    \\
    \begin{subfigure}[b]{0.47\linewidth}
        \includegraphics[width=0.95\linewidth]{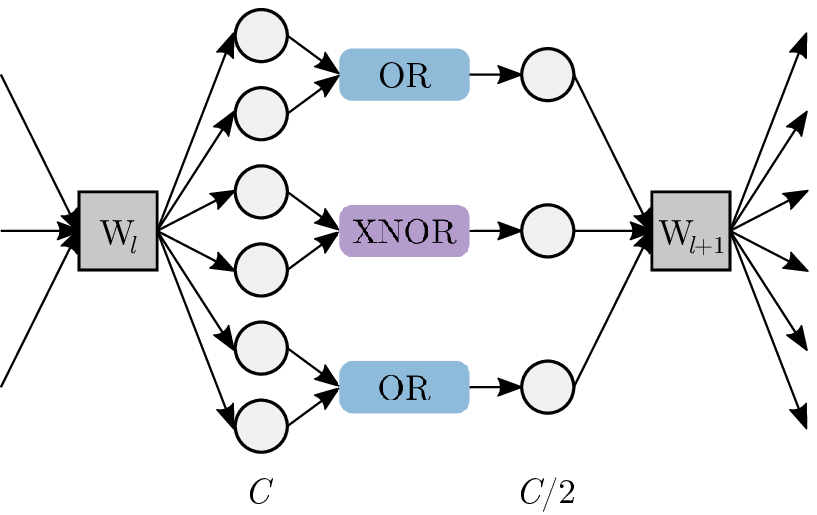}
        \caption{Partition ensembling strategy, using $\opn{OR}$ and $\opn{XNOR}$ (overall $2\!\to\!1$).}
        \label{fig:pathway-partition}
    \end{subfigure}%
    ~~~
    \begin{subfigure}[b]{0.47\linewidth}
        \includegraphics[width=0.95\linewidth]{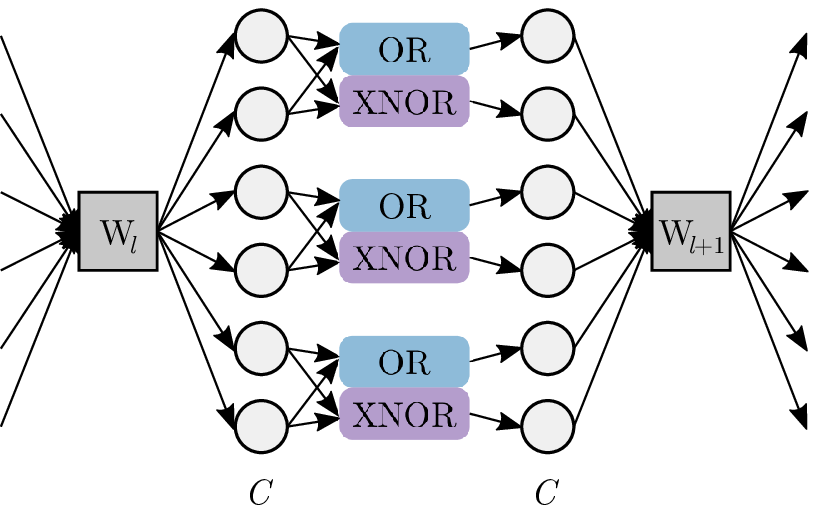}
        \caption{Duplication ensembling strategy, $\opn{OR}$ and $\opn{XNOR}$ (overall $2\!\to\!2$).}
        \label{fig:pathway-duplication}
    \end{subfigure}%
    \\
    \begin{subfigure}[b]{0.47\linewidth}
        \includegraphics[width=0.95\linewidth]{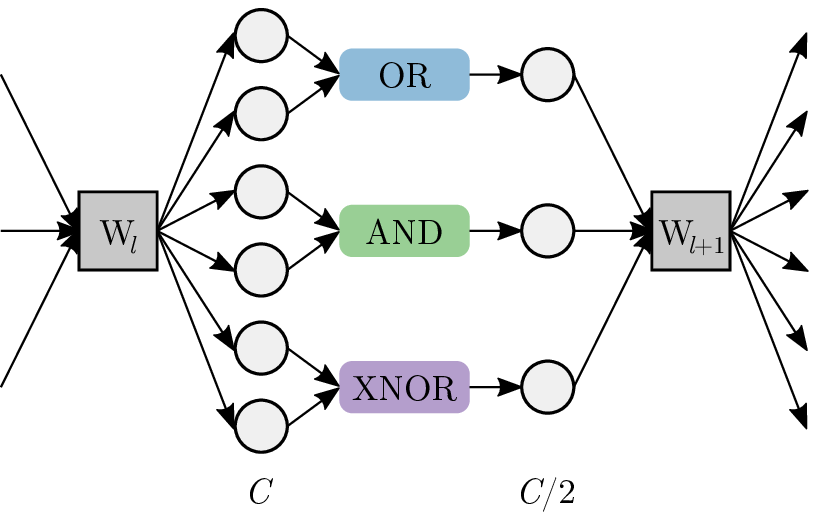}
        \caption{Partition ensembling strategy, using $\opn{OR}$, $\opn{AND}$, and $\opn{XNOR}$ (overall $2\!\to\!1$).}
        \label{fig:pathway-partition3}
    \end{subfigure}%
    ~~~
    \begin{subfigure}[b]{0.47\linewidth}
        \includegraphics[width=0.95\linewidth]{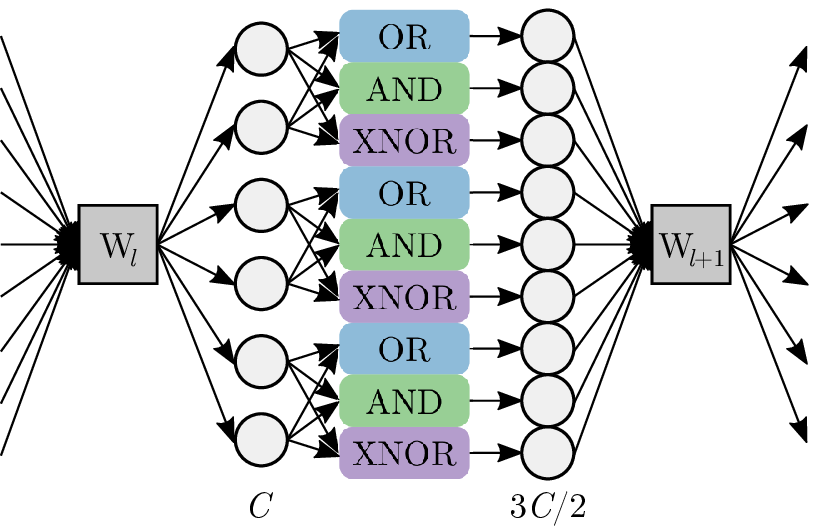}
        \caption{Duplication ensembling strategy, $\opn{OR}$, $\opn{AND}$ and $\opn{XNOR}$ (overall $2\!\to\!3$).}
        \label{fig:pathway-duplication3}
    \end{subfigure}%
    \caption{Network architectures for one layer within an MLP network using various activation functions/ensembles (centered around the activation function). The preceding weight matrix for this layer, $W_l$, and the weight matrix for the next layer, $W_{l+1}$, each produce $C$ features (channels). The activation functions may change the number of features, depending on their mapping.}
    \label{fig:pathways}
\end{figure}

For a $2\!\to\!1$ activation function such as $\opn{OR_{AIL}}$, each activation function requires two arguments and produces a single output (\autoref{fig:pathway-or}).
The number of output channels is half that of the input, and there are $\frac{C^2}{2} + C \sim \frac{C^2}{2}$ parameters per layer.

If we use a combination of $2\!\to\!1$ activation functions with the partition ensembling strategy, each pre-activation neuron is used once, and only seen by one activation function.
The number of output channels is again half that of the input, as seen in \autoref{fig:pathway-partition} and \subref{fig:pathway-partition3}.

Using the duplication ensembling strategy, each $2\!\to\!1$ activation function is applied in parallel to the same set of operands (\autoref{fig:pathway-duplication} and \subref{fig:pathway-duplication3}).
The number of output channels scales up with the number of activation functions used in the ensemble.
With two activation functions, the ensemble performs a $2\!\to\!2$ mapping; the number of channels is unchanged by the layer and the number of parameters per layer is equal to that of a $1\!\to\!1$ activation function, $C^2 + C \sim C^2$.

\FloatBarrier

\subsection{Normalization methodology}
\label{a:normalization}

To normalize the IL and AIL functions, we assumed the pair of arguments to the activation function were drawn independently from the standard normal distribution, $\mathcal{N}(0, 1)$, with probability density function
\begin{equation}
\phi(x) = \frac{1}{\sqrt{2\pi}} \,\exp\left(-\frac{x^2}{2}\right)
.\end{equation}

We found the expected value of the operator, and the standard deviation, and then subtracted by the mean and divided by the standard deviation,
\begin{equation}
\opn{OP}_{N\ast{}IL} = \frac{\opn{OP}_{\ast{}IL} - \mu}{\sigma}
.\end{equation}

For the normalized version of the exact operators ($\text{OR}_{\text{NIL}}$, etc.), the mean and standard deviation were estimated empirically using 600 million samples.
For our AIL functions, we determined the mean and standard deviation analytically by integration.
The values are shown in \autoref{tab:normalization-constants}.

\begin{table}[htp]
  \centering
  \caption{%
  Mean and standard deviation of IL and AIL activation functions, for inputs drawn independently from the standard normal distribution.
}
\label{tab:normalization-constants}
\begin{tabular}{lrr}
\toprule
Activation function                     & Mean & Standard deviation  \\
\midrule
$\opn{OR_{IL}}$                         & \num{ 1.29895} & \num{0.94834} \\
$\opn{AND_{IL}}$                        & \num{-1.29895} & \num{0.94834} \\
$\opn{XNOR_{IL}}$                       & \num{       0} & \num{0.36641} \\
\midrule
$\opn{OR_{AIL}}$                        & \num{ 0.68104} & \num{0.97229} \\
$\opn{AND_{AIL}}$                       & \num{-0.68104} & \num{0.97229} \\
$\opn{XNOR_{AIL}}$                      & \num{       0} & \num{0.60281} \\
\bottomrule
\end{tabular}
\end{table}

\subsubsection{Derivations for $\text{OR}_{\text{AIL}}$}

Here, we analytically derive the mean and variance of $\opn{OR_{AIL}}$.
Recall 
\begin{equation*}
    \opn{OR_{AIL}}(x, y) :=
        \begin{cases}
            x + y           & x > 0, \, y > 0 \\
            \opn{max}(x, y) & \text{otherwise}
        \end{cases}
\end{equation*}
which can be rewritten as
\begin{equation}
    \opn{OR_{AIL}}(x, y)
        = \begin{cases}
            x + y   & \text{if } x > 0, \, y > 0 \\
            x       & \text{if } y \le 0, \, x \ge y \\
            y       & \text{if } x \le 0, \, x < y
        \end{cases}
    \label{eq:or-ail-cases}
\end{equation}

\begin{proposition}
\label{thm:or-mean}
The expected value of $\opn{OR_{AIL}}(x, y)$ for independently sampled $x,y \sim \mathcal{N}(0, 1)$ is
\begin{equation}
\E[\opn{OR_{AIL}}(x, y)] = \frac{1}{\sqrt{2\pi}} + \frac{1}{2\sqrt{\pi}}
.\end{equation}
\end{proposition}

\begin{proof}
Using \autoref{eq:or-ail-cases},
\begin{align*}
\E[\opn{OR_{AIL}}(x, y)]
=& \int_{-\infty}^{\infty} \int_{-\infty}^{\infty} \opn{OR_{AIL}}(x, y) \phi(x) \,\phi(y) \dif x \dif y \\
=& \iint_{x>0,\, y>0} (x + y) \, \phi(x) \,\phi(y) \dif x \dif y \\
 & \qquad + \iint_{x \le 0, \, x \ge y} x \, \phi(x) \,\phi(y) \dif x \dif y
 + \iint_{y \le 0, \, x < y} y \, \phi(x) \,\phi(y) \dif x \dif y
\end{align*}

Let us consider the first component, where $x>0\,y>0$.
\begin{align*}
A :=& \iint_{x>0,\, y>0} (x + y) \, \phi(x) \,\phi(y) \dif x \dif y \\
=& \int_{x=0}^\infty \dif x \, \phi(x) \int_{y=0}^\infty (x + y) \,\phi(y) \dif y \\
=& \int_{x=0}^\infty \dif x \, \phi(x) \int_{y=0}^\infty (x \, \phi(y) + y\,\phi(y)) \dif y
\end{align*}

Note that the integral of the probability density function $\phi(x)$ is its cumulative distribution function,
\begin{equation}
\int \phi(u) \, \dif u = \Phi(u) + C
\end{equation}
which has the limits $\lim_{u\to-\infty} \Phi(u) = 0$ and $\lim_{u\to+\infty} \Phi(u) = 1$, and central value $\Phi(0) = 0.5$.

The probability density function $\phi(x)$ has limits $\lim_{u\to-\infty} \phi(u) = \lim_{u\to+\infty} \phi(u) = 0$, and central value $\Phi(0) = \nicefrac{1}{\sqrt{2\pi}}$.

From the integrals of normal densities, equation 11 given in \citet{Owen1980}, we have
\begin{equation}
\int u \, \phi(u) \dif u = -\phi(u) + C
\label{eq:owen-u-phi}
.\end{equation}

Thus,
\begin{align*}
A
&= \int_{x=0}^\infty \dif x \, \phi(x)  [x \, \Phi(y) - \phi(y)]_{y=0}^\infty \\
&= \int_{x=0}^\infty \dif x \, \phi(x)  \left[ \frac{x}{2} + \frac{1}{\sqrt{2\pi}} \right] \\
&= \left[ -\frac{1}{2}\phi(x) + \frac{1}{\sqrt{2\pi}}\Phi(x) \right]_{x=0}^\infty \\
&= \frac{1}{2\sqrt{2\pi}} + \frac{1}{\sqrt{2\pi}} - \frac{1}{2\sqrt{2\pi}} \\
&= \frac{1}{\sqrt{2\pi}}
\end{align*}

Next let us consider the second component of the integral.
\begin{align*}
B
:=& \iint_{x \le 0, \, x \ge y} x \, \phi(x) \,\phi(y) \dif x \dif y \\
=& \int_{y=-\infty}^0 \phi(x) \int_{x=y}^\infty \,\phi(y) \, x \dif x \dif y \\
=& \int_{y=-\infty}^0 \dif y \, \phi(y) \int_{x=y}^\infty x\,\phi(x)\dif x \\
=& \int_{y=-\infty}^0 \dif y \, \phi(y) \left[ -\phi(x) \right]_{x=y}^\infty \\
=& \int_{y=-\infty}^0 \phi(y)^2 \dif y \\
=& \left[ \frac{1}{2\sqrt{\pi}} \Phi(x\sqrt{2}) \right]_{y=-\infty}^0 \\
=& \frac{1}{4\pi}
\end{align*}

By change of variables,
\begin{align*}
\iint_{y \le 0, \, x < y} y \, \phi(x) \,\phi(y) \dif x \dif y
=& \iint_{x \le 0, \, y < x} x \, \phi(y) \,\phi(x) \dif y \dif x \\
=& B
\end{align*}

Thus we conclude
\begin{align*}
\E[\opn{OR_{AIL}}(x, y)]
&= A + 2B \\
&= \frac{1}{\sqrt{2\pi}} + \frac{2}{4\pi} \\
&= \frac{1}{\sqrt{2\pi}} + \frac{1}{2\pi}
\end{align*}
\end{proof}

\begin{proposition}
\label{thm:or-var}
The variance of $\opn{OR_{AIL}}(x, y)$ for independently sampled $x,y \sim \mathcal{N}(0, 1)$ is
\begin{equation}
\Var(\opn{OR_{AIL}}(x, y)) = \frac{5}{4} - \frac{1}{\sqrt{2}\pi} - \frac{1}{4\pi}
\end{equation}
\end{proposition}

\begin{proof}
Using \autoref{eq:or-ail-cases},
\begin{align*}
\E[\opn{OR_{AIL}}(x, y)^2]
=& \int_{-\infty}^{\infty} \int_{-\infty}^{\infty} \opn{OR_{AIL}}(x, y)^2 \phi(x) \,\phi(y) \dif x \dif y \\
=& \iint_{x>0,\, y>0} (x + y)^2 \, \phi(x) \,\phi(y) \dif x \dif y \\
 & \qquad + \iint_{x \le 0, \, x^2 \ge y} x^2 \, \phi(x) \,\phi(y) \dif x \dif y
 + \iint_{y \le 0, \, x < y} y^2 \, \phi(x) \,\phi(y) \dif x \dif y
\end{align*}

Let us consider the first component, where $x>0,\,y>0$.
\begin{align*}
A :=& \iint_{x>0,\, y>0} (x + y)^2 \, \phi(x) \,\phi(y) \dif x \dif y \\
=& \iint_{x>0,\, y>0} (x^2 + 2xy + y^2) \, \phi(x) \,\phi(y) \dif x \dif y \\
=& \int_{x=0}^\infty \int_{y=0}^\infty (x^2 + 2xy + y^2) \, \phi(x) \,\phi(y) \dif x \dif y \\
=& \int_{x=0}^\infty \int_{y=0}^\infty (x^2 + 2xy + y^2) \, \phi(x) \,\phi(y) \dif x \dif y
\end{align*}

From equation 12 of \citet{Owen1980}, we note the identity
\begin{equation}
\int u^2 \,\phi(u) \dif u = \Phi(u) - u\,\phi(u) + C
.\end{equation}

Additionally, we note that
\begin{align*}
\lim_{u\to\infty} u\,\phi(u)
&= \lim_{u\to\infty} \frac{u}{\sqrt{2\pi}}\,\exp\left(-\frac{u^2}{2}\right) \\
&= 0
\end{align*}

Using this in addition to \autoref{eq:owen-u-phi}, 
\begin{align*}
A
=& \int_{x=0}^\infty \dif x \int_{y=0}^\infty x^2\,\phi(x)\,\phi(y) + 2xy\,\phi(x)\,\phi(y) + y^2\,\phi(x)\,\phi(y) \dif y \\
=& \int_{x=0}^\infty \dif x \left[ x^2\,\phi(x)\,\Phi(y) - 2x\,\phi(x)\,\phi(y) + \phi(x) (\Phi(y) - y\,\phi(y) ) \right]_{y=0}^\infty \\
=& \int_{x=0}^\infty \dif x \left[
(
x^2\,\phi(x) + \phi(x)
)
- \left(
\frac{x^2\,\phi(x)}{2} - \frac{2\,x\,\phi(x)}{\sqrt{2\pi}} + \frac{\phi(x)}{2}
\right)
\right] \\
=& \int_{x=0}^\infty \dif x \left(\frac{\phi(x)}{2} + \frac{\sqrt{2}\,x\,\phi(x)}{\sqrt{\pi}} + \frac{x^2\,\phi(x)}{2} \right) \\
=& \left[ \frac{\Phi(x)}{2} - \frac{\sqrt{2}}{\sqrt{\pi}}\,\phi(x) + \frac{1}{2}(\Phi(x) -x\,\phi(x)) \right]_{x=0}^\infty \\
=& 
\left(
\frac{1}{2} + \frac{1}{2}
\right)
- \left(
\frac{1}{4} - \frac{\sqrt{2}}{\sqrt{\pi}}\,\frac{1}{\sqrt{2\pi}} + \frac{1}{4}
\right) \\
=& \frac{1}{2} + \frac{1}{\pi}
\end{align*}

Next we consider the second component,
\begin{align*}
B
:=& \iint_{x \le 0, \, x \ge y} x^2 \, \phi(x) \,\phi(y) \dif x \dif y \\
=& \int_{y=-\infty}^0 \int_{x=y}^\infty \phi(x) \,\phi(y) \, x^2 \dif x \dif y \\
=& \int_{y=-\infty}^0 \dif y \, \phi(y) \int_{x=y}^\infty x^2\,\phi(x)\dif x \\
=& \int_{y=-\infty}^0 \dif y \, \phi(y) \left[ \Phi(x) - x\,\phi(x) \right]_{x=y}^{\infty} \\
=& \int_{y=-\infty}^0 \dif y \, \phi(y) \left[ 1 - \Phi(y) + y\,\phi(y) \right] \\
=& \int_{y=-\infty}^0 \dif y \, \left[ \phi(y) - \phi(y)\,\Phi(y) + y\,\phi(y)^2 \right] \\
\end{align*}

From equation n1 of \citet{Owen1980},
\begin{align*}
\int u \,\phi(u)^n \dif u &= \frac{-\phi\left(\sqrt{n}u\right)}{n(2\pi)^{(n-1)/2}} \\
\implies
\int u \,\phi(u)^2 \dif u &= \frac{-\phi\left(\sqrt{2}u\right)}{2\sqrt{2\pi}}
.\end{align*}
Additionally from equation 1,010.4 of \citet{Owen1980},
\begin{align*}
\int_{-\infty}^0 \phi(au)\,\Phi(bu)\dif u
&= \frac{1}{2\pi|a|}\left(\frac{\pi}{2} - \arctan\left(\frac{b}{|a|}\right)\right) \\
\implies
\int_{-\infty}^0 \phi(u)\,\Phi(u)\dif u
&= \frac{1}{2\pi}\left(\frac{\pi}{2} - \frac{\pi}{4} \right)
= \frac{1}{8}
.\end{align*}

Thus
\begin{align*}
B
=& \int_{y=-\infty}^0 \dif y \, \left[ \phi(y) - \phi(y)\,\Phi(y) + y\,\phi(y)^2 \right] \\
=& -\frac{1}{8} + \left[ \Phi(y) - \frac{\phi\left(\sqrt{2}y\right)}{2\sqrt{2\pi}} \right]_{y=-\infty}^0 \\
=& -\frac{1}{8} + \frac{1}{2} - \frac{1}{\sqrt{2\pi}} \frac{1}{2\sqrt{2\pi}} \\
=& \frac{3}{8} - \frac{1}{4\pi}
\end{align*}

By change of variables,
\begin{align*}
\iint_{y \le 0, \, x < y} y^2 \, \phi(x) \,\phi(y) \dif x \dif y
=& \iint_{x \le 0, \, y < x} x^2 \, \phi(y) \,\phi(x) \dif y \dif x \\
=& B
\end{align*}

Thus we conclude
\begin{align*}
\E[\opn{OR_{AIL}}(x, y)^2]
&= A + 2B \\
&= \frac{1}{2} + \frac{1}{\pi} + \frac{3}{4} - \frac{1}{2\pi} \\
&= \frac{5}{4} + \frac{1}{2\pi}
\end{align*}

Recall that
\begin{align*}
\Var(X)
&= \E[(X-\E[X])^2] \\
&= \E[X^2] - \E[X]^2
.\end{align*}
Hence
\begin{align*}
\Var(\opn{OR_{AIL}}(x, y))
&= \frac{5}{4} + \frac{1}{2\pi} - \left(\frac{1}{\sqrt{2\pi}} + \frac{1}{2\sqrt{\pi}} \right)^2 \\
&= \frac{5}{4} + \frac{1}{2\pi} - \left(
\frac{1}{2\pi}
+ 2\frac{1}{\sqrt{2\pi}}\frac{1}{2\sqrt{\pi}}
+ \frac{1}{4\pi}
\right) \\
&= \frac{5}{4} + \frac{1}{2\pi} - \left(
\frac{3}{4\pi}
+ \frac{1}{\sqrt{2}\pi}
\right) \\
&= \frac{5}{4} - \frac{1}{\sqrt{2}\pi} - \frac{1}{4\pi}
\end{align*}

\end{proof}

\subsubsection{Derivations for $\text{AND}_{\text{AIL}}$}

Follows from $\text{AND}_{\text{AIL}}(x,y) = -\opn{OR}_{\text{AIL}}(-x,-y)$.

\subsubsection{Derivations for $\text{XNOR}_{\text{AIL}}$}

Here, we analytically derive the mean and variance of $\opn{XNOR_{AIL}}$.
Recall 
\begin{equation}
    \opn{XNOR_{AIL}}(x, y) := \sgn(x y) \opn{min}(|x|, |y|),
    \label{eq:xnor_ail}
\end{equation}


\begin{proposition}
The expected value of $\opn{XNOR_{AIL}}(x, y)$ for independently sampled $x,y \sim \mathcal{N}(0, 1)$ is $0$.
\label{thm:xnor-mean}
\end{proposition}

\begin{proof}
The expected value is given by
\begin{equation}
\E[\opn{XNOR_{AIL}}(x, y)] = \int_{x=-\infty}^{\infty} \int_{y=-\infty}^{\infty} \opn{XNOR_{AIL}}(x, y) \,\phi(x) \,\phi(y) \dif x \dif y
.\end{equation}

Note that $\phi(x)$ is an even function, i.e. $\phi(x) = \phi(-x)$.
Moreover, $\opn{XNOR_{AIL}}(x, y)$ is an odd function with $\opn{XNOR_{AIL}}(-x, y) = -\opn{XNOR_{AIL}}(x, y)$.
Thus $\opn{XNOR_{AIL}}(x, y) \,\phi(x) \,\phi(y)$ is an odd function.
Therefore,
\begin{equation}
\E[\opn{XNOR_{AIL}}(x, y)] = 0
\end{equation}
\end{proof}


\begin{proposition}
The variance of $\opn{XNOR_{AIL}}(x, y)$ for independently sampled $x,y \sim \mathcal{N}(0, 1)$ is
\begin{equation}
\Var(\opn{XNOR_{AIL}}(x, y)) = 1 - \frac{2}{\pi}
\end{equation}
\end{proposition}

\begin{proof}
Recall that
\begin{align*}
\Var(X)
&= \E[(X-\E[X])^2] \\
&= \E[X^2] - \E[X]^2
.\end{align*}

Using \autoref{thm:xnor-mean},
\begin{align*}
\Var(\opn{XNOR_{AIL}}(x, y))
&= \E[\opn{XNOR_{AIL}}(x, y)^2] - \E[\opn{XNOR_{AIL}}(x, y)]^2 \\
&= \E[\opn{XNOR_{AIL}}(x, y)^2] - 0 \\
&= \E[\sgn(x y)^2 \opn{min}(|x|, |y|)^2] \\
&= \E[\min(|x|, |y|)^2] \\
&= \E[\min(x^2, y^2)]
\end{align*}
since $x\in\Real$, $y\in\Real$.

\begin{align*}
\E[\opn{XNOR_{AIL}}(x, y)^2]
&= \int_{-\infty}^{\infty} \int_{-\infty}^{\infty} (\opn{XNOR_{AIL}}(x, y))^2 \,\phi(x) \,\phi(y) \dif x \dif y \\
&= \int_{-\infty}^{\infty} \int_{-\infty}^{\infty} \min(x^2, y^2) \,\phi(x) \,\phi(y) \dif x \dif y \\
&=: 8 \, A
\end{align*}

The integrand has three symmetries: $f(-x)=f(x)$, $f(-y)=f(y)$, and $f(x,y)=f(y,x)$.
Using this, we can break the integral up into 8 slices, each of which have an equal area, $A$.

We next integrate over one of these slices to find $A$.
\begin{align*}
A
&= \int_{y=0}^{\infty} \int_{x=y}^{\infty} y^2 \,\phi(x) \,\phi(y) \dif x \dif y \\
&= \int_{y=0}^{\infty} y^2 \,\phi(y) \dif y \int_{x=y}^{\infty} \phi(x) \dif x \\
&= \int_{y=0}^{\infty} y^2 \,\phi(y) \dif y \left[ \Phi(\infty) - \Phi(y) \right] \\
&= \int_{y=0}^{\infty} y^2 \,\phi(y) \left( 1 - \Phi(y) \right) \dif y \\
&= \int_{y=0}^{\infty} y^2 \,\phi(y) \dif y - \int_{y=0}^{\infty} y^2 \,\phi(y) \,\Phi(y) \dif y
\end{align*}
Equation (1;012-1) from \citet{Owen1980} provides us with the identity
\begin{equation*}
\int_{u=0}^{\infty} u^2 \,\phi(u) \,\Phi(u) \dif u = \frac{\Phi(u)^2}{2} - \frac{\phi(u)^2}{2} - u\,\phi(u)\,\Phi(u)
.\end{equation*}
Thus,
\begin{align*}
A
&= \left[ \left(\Phi(y)-y\,\phi(y)\right) - \left( \frac{\Phi(y)^2}{2} - \frac{\phi(y)^2}{2} - y\,\phi(y)\,\Phi(y) \right) \right]_{y=0}^{\infty} \\
&= \left[ \Phi(y) - y\,\phi(y) - \frac{\Phi(y)^2}{2} + \frac{\phi(y)^2}{2} + y\,\phi(y)\,\Phi(y) \right]_{y=0}^{\infty} \\
&= \left( 1 - 0 - \frac{1}{2} + 0 + 0 \right) - \left( \frac{1}{2} - 0 - \frac{1}{8} + \frac{1}{4\pi} + 0 \right) \\
&= \frac{1}{8} - \frac{1}{4\pi} \\
\implies 8A &= 1 - \frac{2}{\pi}
\end{align*}
$$
\implies \Var(\opn{XNOR_{AIL}}(x, y)) = 1 - \frac{2}{\pi}
$$
\end{proof}

\begin{corollary}
The standard deviation of $\opn{XNOR_{AIL}}(x, y)$ for independently sampled $x,y \sim \mathcal{N}(0, 1)$ is
\begin{equation}
\sigma = \sqrt{1 - \frac{2}{\pi}}
.\end{equation}
\end{corollary}

\subsection{Dataset summary}
\label{a:datasets}

The datasets used in this work are summarized in \autoref{tab:datasets}.

\begin{table}[h]
\small
  \centering
  \caption{%
Dataset summaries.
}
\label{tab:datasets}
\centerline{
\begin{tabular}{lrrrl}
\toprule
                    & \multicolumn{2}{c}{\textnumero{} Samples} & & \\
\cmidrule(r){2-3}
Dataset             & Train       & Test        & Classes   & Reference \\
\midrule
Bach Chorales       &    \num{229} &    \num{77} &   \num{2} & \citet{boulanger2012modeling} \\
Caltech101          &   \num{6162} &  \num{1695} & \num{101} & \citet{fei2006one} \\
CIFAR-10            &  \num{50000} & \num{10000} &  \num{10} & \citet{cifar-report} \\
CIFAR-100           &  \num{50000} & \num{10000} & \num{100} & \citet{cifar-report} \\
Covertype           & \num{464810} &\num{116202} &   \num{7} & \citet{covertype-phd, covertype-paper} \\
I-RAVEN             &   \num{6000} &  \num{2000} &       --- & \citet{sran} \\
MIT-States          &  \num{30338} & \num{12995} & \num{245} obj, \num{115} attr & \citet{mit-states} \\
MNIST               &  \num{60000} & \num{10000} &  \num{10} & \citet{lecun1998gradient} \\
Oxford Flowers      &   \num{6552} &   \num{818} & \num{102} & \citet{nilsback2008automated} \\
Stanford Cars       &   \num{8144} &  \num{8041} & \num{196} & \citet{KrauseStarkDengFei2013} \\
STL-10              &   \num{5000} &  \num{8000} &  \num{10} & \citet{coates2011analysis} \\
SVHN                &  \num{73257} & \num{26032} &  \num{10} & \citet{svhn} \\
\bottomrule
\end{tabular}
}
\end{table}

We used the same splits for Caltech101 as used in \citet{spinalnet}.

The Covertype dataset was chosen as a dataset that contains only simple features (and not pixels of an image) on which we could study a simple MLP architecture, and was selected based on its popularity on the UCI ML repository.

The Bach Chorales dataset was chosen because --- in addition to being in continued use by ML researchers for decades --- it presents an interesting opportunity to consider a task where logical activation functions are intuitively applicable, as it is a relatively small dataset that has also been approached with rule-based frameworks, e.g. the expert system by Ebcioglu (1988).

MNIST, CIFAR-10, CIFAR-100 are standard image datasets, commonly used. We used small MLP and CNN architectures for the experiments on MNIST so we could investigate the performance of the network for many configurations (varying the size of the network). We used ResNet-50, a very common deep convolutional architecture within the computer vision field, on CIFAR-10/100 to evaluate the performance in the context of a deep network.

The datasets used for the transfer learning task are all common and popular natural image datasets, with some containing coarse-grained classification (CIFAR-10), others fine-grained (Stanford Cars), and with a varying dataset size (\num{5000}–-\num{75000} training samples). We chose to do an experiment involving transfer learning because it is a common practical situation where one must train only a small network that handles high-level features, and is the sort of situation which involves manipulating high-level features, relying on the pretrained network to do the feature extraction.

We considered other domains where logical reasoning is involved as a component of the task, and isolated abstract reasoning and compositional zero-shot learning as suitable tasks.

For abstract reasoning, we wanted to use an IQ style test, and determined that I-RAVEN was a state-of-the-art dataset within this domain (having fixed some problems with the pre-existing RAVEN dataset). We determined that the SRAN architecture from the paper which introduced I-RAVEN was still state-of-the-art on this task, and so used this.

Another problem domain in which we thought it would be interesting to study our activation functions was compositional zero-shot learning (CZSL). This is because the task inherently involves combining an attribute with an object (i.e. the AND operation). For CZSL, we looked at SOTA methods on \url{https://paperswithcode.com}. The best performance was from SymNet, but this was only implemented in TensorFlow and our code was set up in PyTorch so we built our experiments on the back of the second-best instead, which is the TMN architecture. In the TMN paper, they used two datasets: MIT-States and UT-Zappos-1. In our preliminary experiments, we found that the network started overfitting on MIT-States after around 6 epochs, but on UT-Zappos-1 it was overfitting after the first or second epoch (one can not tell beyond the fact the val performance is best for the first epoch). In the context of zero-shot learning, an epoch uses every image once, but there are also only a finite number of tasks in the dataset. Because there are multiple samples for each noun/adjective pair, and each noun only appears with a handful of adjectives and vice versa, there is in a way fewer tasks in one epoch than there are images. Hence it is possible for a zero-shot learning model to overfit to the training tasks in less than one epoch (recall that the network includes a pretrained ResNet model for extracting features from the images). For simplicity, we dropped UT-Zappos-1 and focused on MIT-States.

\subsection{Experiment configuration and hardware}
\label{a:hardware}

Experiments were run using PyTorch (typically v1.10) on an internal cluster of NVIDIA RTX-6000 and Tesla~T4 GPUs.
Our experiments each used either 1 or 4 GPUs.

{
\subsection{Statistical Methodology}
\label{s:stats}

Here we provide information on how we carried out some of the detailed comparisons between activation functions. 

In our experiments on Bach chorales (\autoref{s:chorales}), MNIST (\autoref{s:mnist}), CIFAR-10/100 (\autoref{s:cifar}), and Covertype (\aref{a:covtype}), we explored the performance of networks using various activation functions whilst changing the number of parameters in the network.
This was achieved by scaling up the width of the networks, whilst keeping the depth (number of layers) fixed.

Since the different structure of 1-to-1 and 2-to-1 activations (e.g. ReLU vs Max and the activations introduced in this paper) imply a different organization of the network, to explore the same range of values for the parameter count, the width values spanned different ranges depending on the activation function.
Note that in other cases, it made more sense to preserve a characteristic such as the size of the embedding space, in which case we discuss this explicitly (e.g. \aref{s:czsl}).

In the cases where we explore a range of values for the parameter count, we adopt the following methodology for our statistical tests for the ``best'' activation function.
First we collapsed across the number of params dimension by selecting the best width value for each activation function. Typically this is the widest value, but sometimes some activation functions had their performance peak with fewer parameters (especially on Covertype, where we had disabled weight decay). Then we identified the best performing activation function as the one with the highest accuracy.

For example, in the case of MLP on MNIST (\autoref{s:mnist}), the best activation function (over all num. params considered) was $\opn{XNOR_{AIL}}$. We compared its performance to each of the other activation functions using a (non-paired) two-tailed Student's $t$-test, implemented with \texttt{scipy.stats.ttest\_ind}, one test for each of the other activations. The total number of tests performed is $(\text{num\_actfuns} - 1)$. We did not use a Bonferroni correction for multiple comparisons. In the case of the MLP on MNIST, the p-value vs $\opn{SignedGeomean}$ was between 0.05 and 0.1 and hence did not meet our threshold for statistical significance. For all the other activations, the p-value was less than 0.01, more than exceeding our threshold for significance. Hence, we found that $\opn{XNOR_{AIL}}$ was significantly better than all other activations besides $\opn{SignedGeomean}$.

Note that this methodology (collapsing across the number of parameters by selecting the best for each activation function) compares each of the activation functions at their best. This is in keeping with how results are frequently reported in the literature (i.e. in a table showing one value per comparator, after performing a hyperparam search), and makes it possible to do a simple statistical test, such as a $t$-test. However, such a simplification does not tell the whole story about the results. For example, \autoref{fig:mnist} also shows that for the MLP, XNOR-shaped activation functions also lose performance sooner than other activation functions when the number of parameters is reduced. Meanwhile for the CNN, some activation functions which perform highly with a large number of parameters maintain high performance even as the number of parameters is reduced, while others do not. This seemed to be associated with those activation functions that included Max/OR operations as opposed to those that did not. The methodology we used to evaluate the performance was the same as described above for the MLP. Selecting the best width for each activation function, we found there are 5 functions/ensembles which have indistinguishably best performance: ({OR, AND, $XNOR_{AIL}$ (p)}, {OR, XNOR AIL (d/p)}, Max, and SiLU) performed best, whilst other activations had significantly lower performance than top-performing activation function, $p < 0.05$.

}

\subsection{Parity Experiments}
\label{a:parity}

\revision{
Our parity experiment shown in \autoref{s:parity} was trained for 100 epochs using ADAM, one-cycle learning rate schedule, max LR~\num{0.01}, weight decay~\num{1e-4}.
}

Following on from the parity experiment described in the main text, we also introduced a second synthetic dataset with a labelling function that, while slightly more complex than the first, was still solvable by applying the logical XNOR operation to the network inputs. In this dataset increased our number of inputs to 8, and derived our labels by applying a set of nested $\opn{XNOR_{IL}}$ operations: 

\begin{align*}
\opn{XNOR_{IL}}( &
     \opn{XNOR_{IL}}( \opn{XNOR_{IL}}( x_2, x_5 ), \opn{XNOR_{IL}}( x_3, x_4 ) ), \\
    &\opn{XNOR_{IL}}( \opn{XNOR_{IL}}( x_6, x_7 ), \opn{XNOR_{IL}}( x_0, x_1 )).
\end{align*}

For this more difficult task we also reformulated our initial experiment into a regression problem, as the continuous targets produced by this labelling function are more informative than the rounded binary targets used in the first experiment. We also adjusted our network setup to have an equal number of neurons at each hidden layer as we found that this significantly improved model performance\footnote{We hypothesize that, because our $\opn{XNOR_{AIL}}$ activation function reduces the number of hidden layer neurons by a factor of $k$, having a reduced number of neurons at each layer creates a bottleneck in the later layers of the network which restricts the amount of information that made its way through to the final layer.}. We again trained using the same model hyperparameters for 100 epochs.

\begin{figure}[tbh]
    \centering
    \begin{subfigure}[b]{0.49\linewidth}
        \includegraphics[width=\linewidth]{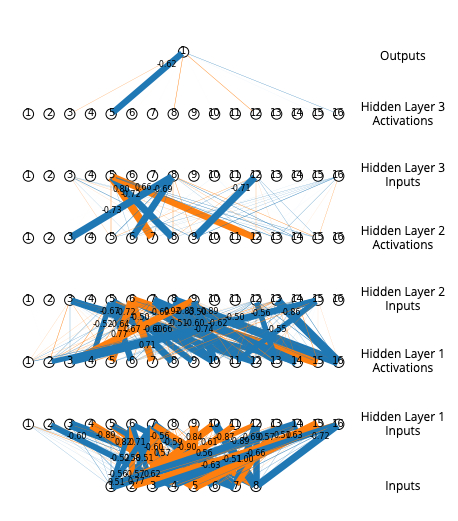}
        \caption{ReLU learned weight matrix}
    \end{subfigure}%
    ~~~
    \begin{subfigure}[b]{0.49\linewidth}
        \includegraphics[width=\linewidth]{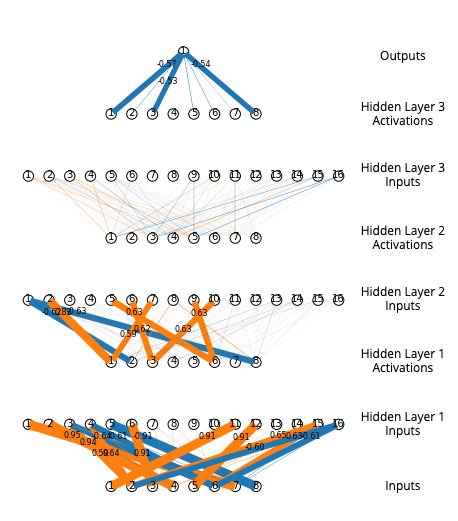}
        \caption{$\opn{XNOR_{AIL}}$ learned weight matrix}
    \end{subfigure}%
    \caption{Training results, regression experiment on second synthetic dataset.}
    \label{fig:2nd_parity}
\end{figure}

While this time the model was not able to learn a sparse weight matrix that exactly reflected our labelling function (see \autoref{fig:2nd_parity}), the model was again able to leverage the $\opn{XNOR_{AIL}}$ activation function to significantly outperform an identical model utilizing the ReLU activation function.

We found that a simple model with three hidden layers, each with eight neurons, utilizing $\opn{XNOR_{AIL}}$ was able to go from a validation RMSE of 0.287 at the beginning of training to a validation RMSE of 0.016 after 100 epochs. Comparatively, an identical model utilizing the ReLU activation function was only able to achieve a validation RMSE of 0.271 after 100 epochs. In order for our ReLU network to match the validation RMSE of our 8-neuron-per-layer $\opn{XNOR_{AIL}}$ model, we had to increase the model size by 32 times to 256 neurons at each hidden layer.

\subsection{MLP on Covertype}
\label{a:covtype}

We trained small one, two, and three-layer MLP models on the Covertype dataset~\citep{covertype-dataset,covertype-paper} from the UCI Machine Learning Repository.
The Covertype dataset is a classification task consisting of \num{581012} samples of forest coverage across 7 classes.
Each sample has 54 attributes.

Networks were trained using one-cycle \citep{superconvergence,Smith2018} for 50 epochs, with a batch size of \num{1024}, without weight decay.
We used a fixed random 80:20 train:test split throughout our experiments.
The learning rate was selected using an automated learning rate finder approach.
No data augmentation was performed.

\begin{figure}[tbh]
    \centering
    \begin{subfigure}[b]{0.8\linewidth}
        \centering
        \includegraphics[scale=0.4]{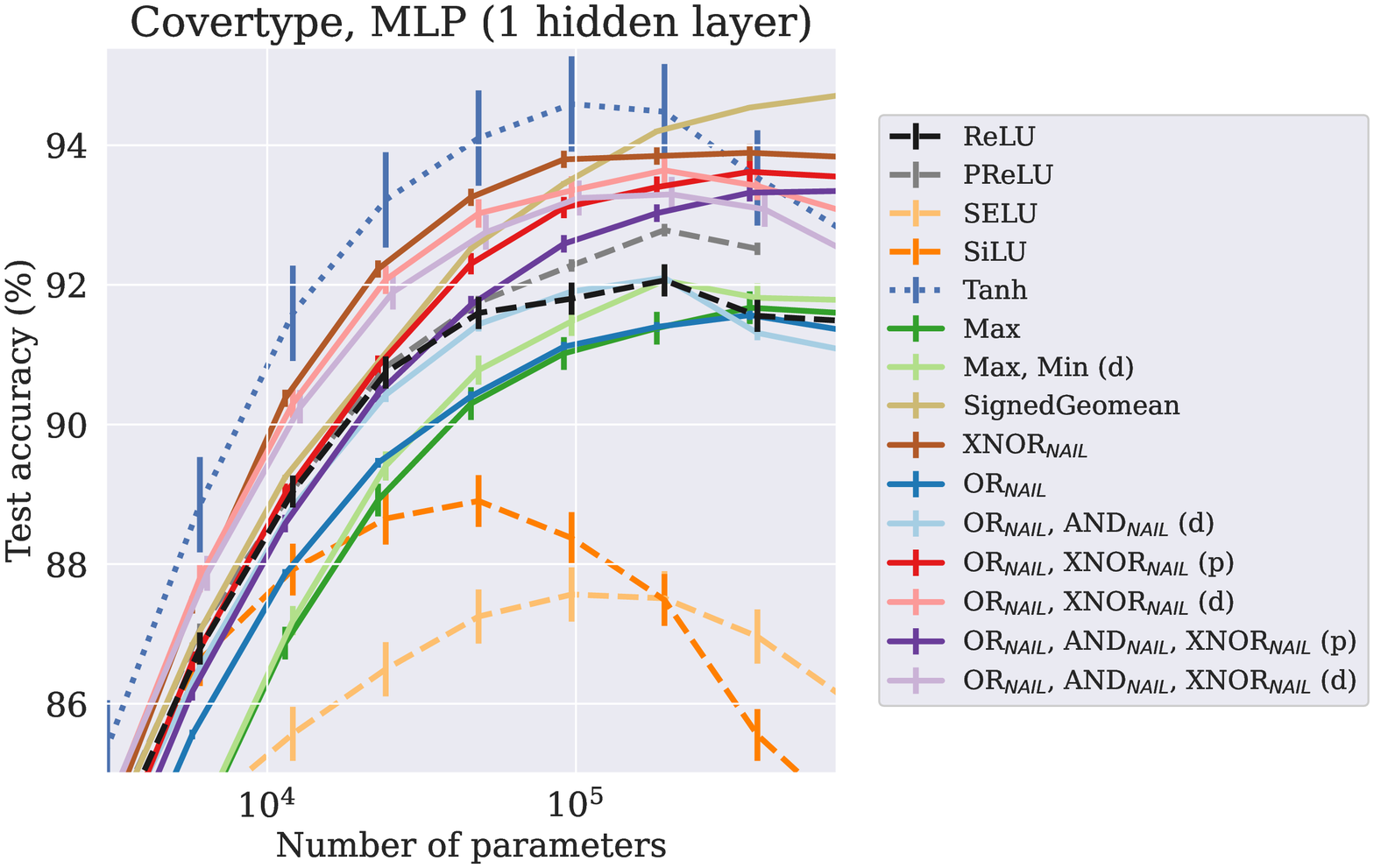}
    \end{subfigure}%
    \\
    \begin{subfigure}[b]{0.48\linewidth}
        \centering
        \includegraphics[scale=0.4]{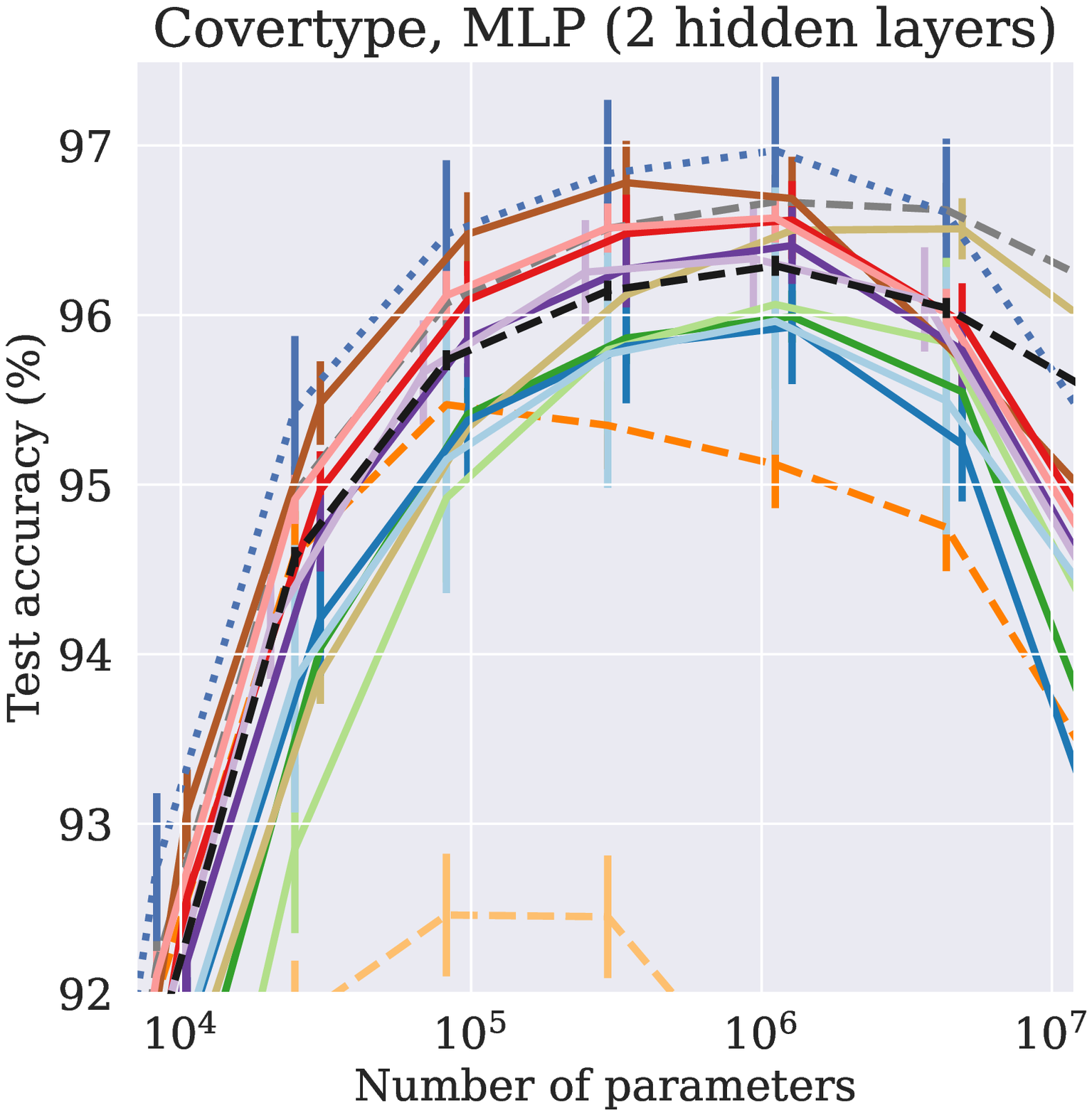}
    \end{subfigure}%
    \hfill
    \begin{subfigure}[b]{0.48\linewidth}
        \centering
        \includegraphics[scale=0.4]{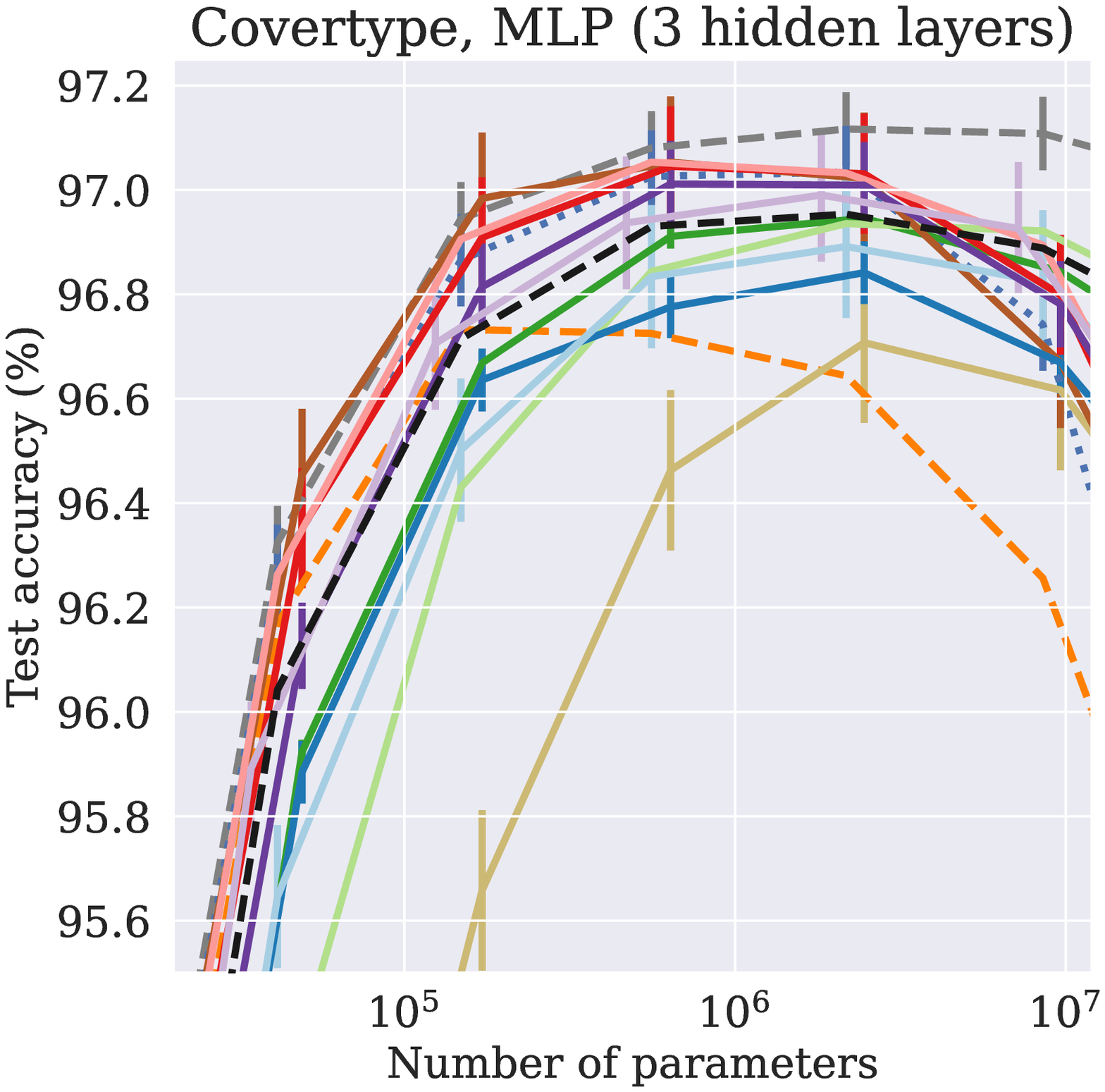}
    \end{subfigure}%
\caption{
We trained MLPs on the Covertype dataset, with a fixed 80:20 random split.
Trained with ADAM, 50~ep., 1-cycle, using LRs determined automatically with LR-finder.
Mean (bars: std~dev) of $n\!=\!5$ weight inits.
}
\label{fig:covtype}
\end{figure}

\begin{table}[tbhp]
\small
  \centering
  \caption{%
MLP on Covertype, with 1--3 hidden layers, while varying the activation function and network width.
For each activation, we show the best performance across all widths.
Bold: best.
Underlined: top two.
Italic: no sig. diff. from best (two-sided Student's $t$-test, $p\!>\!0.05$, $n\!=\!5$ weight inits).
Background: linear color scale from worst (white) to best (black) with a given number of layers.
}
\label{tab:mlp-covtype}
\def\ccAmin{84.78507439566962}
\def\ccAmax{94.7213066788293}
\def\ccA[#1]#2{\heatmapcell{#1}{#2}{\ccAmin}{\ccAmax}}
\def\ccBmin{92.46215674294123}
\def\ccBmax{96.97012985895373}
\def\ccB[#1]#2{\heatmapcell{#1}{#2}{\ccBmin}{\ccBmax}}
\def\ccCmin{94.45281102897516}
\def\ccCmax{97.1178024663735}
\def\ccC[#1]#2{\heatmapcell{#1}{#2}{\ccCmin}{\ccCmax}}
\centerline{
\begin{tabular}{lcrrr}
\toprule
 & & \multicolumn{3}{c}{Test Accuracy (\%) by \textnumero{} Layers} \\
\cmidrule(l){3-5}
Activation function                                     & Map & \multicolumn{1}{c}{1} & \multicolumn{1}{c}{2} & \multicolumn{1}{c}{3} \\
\toprule
$\opn{ReLU}$                                            & $1\!\to\! 1$ & \ccA[92.0644]{$92.06\wpm{ 0.07}$} & \ccB[96.2906]{$96.29\wpm{ 0.02}$} & \ccC[96.9536]{$96.95\wpm{ 0.01}$} \\
LeakyReLU                                               & $1\!\to\! 1$ & \ccA[92.1176]{$92.12\wpm{ 0.18}$} & \ccB[96.2448]{$96.24\wpm{ 0.02}$} & \ccC[96.9330]{$96.93\wpm{ 0.01}$} \\
PReLU                                                   & $1\!\to\! 1$ & \ccA[92.7844]{$92.78\wpm{ 0.04}$} & \ccB[96.6688]{$96.67\wpm{ 0.01}$} & \ccC[97.1168]{$\mbf{\mbns{97.12}}\wpm{ 0.01}$} \\
Softplus                                                & $1\!\to\! 1$ & \ccA[86.4561]{$86.46\wpm{ 0.05}$} & \ccB[93.5117]{$93.51\wpm{ 0.10}$} & \ccC[95.5044]{$95.50\wpm{ 0.12}$} \\
\midrule
ELU                                                     & $1\!\to\! 1$ & \ccA[86.7706]{$86.77\wpm{ 0.14}$} & \ccB[93.4177]{$93.42\wpm{ 0.05}$} & \ccC[95.4406]{$95.44\wpm{ 0.06}$} \\
CELU                                                    & $1\!\to\! 1$ & \ccA[86.7150]{$86.71\wpm{ 0.14}$} & \ccB[93.4177]{$93.42\wpm{ 0.05}$} & \ccC[95.4406]{$95.44\wpm{ 0.06}$} \\
SELU                                                    & $1\!\to\! 1$ & \ccA[87.5643]{$87.56\wpm{ 0.04}$} & \ccB[92.4622]{$92.46\wpm{ 0.10}$} & \ccC[94.4528]{$94.45\wpm{ 0.12}$} \\
GELU                                                    & $1\!\to\! 1$ & \ccA[90.7412]{$90.74\wpm{ 0.16}$} & \ccB[96.0855]{$96.09\wpm{ 0.01}$} & \ccC[96.9108]{$96.91\wpm{ 0.02}$} \\
SiLU                                                    & $1\!\to\! 1$ & \ccA[88.9041]{$88.90\wpm{ 0.07}$} & \ccB[95.4721]{$95.47\wpm{ 0.04}$} & \ccC[96.7326]{$96.73\wpm{ 0.02}$} \\
Hardswish                                               & $1\!\to\! 1$ & \ccA[87.4547]{$87.45\wpm{ 0.05}$} & \ccB[95.2154]{$95.22\wpm{ 0.05}$} & \ccC[96.7927]{$96.79\wpm{ 0.03}$} \\
Mish                                                    & $1\!\to\! 1$ & \ccA[89.2366]{$89.24\wpm{ 0.18}$} & \ccB[95.4705]{$95.47\wpm{ 0.03}$} & \ccC[96.7682]{$96.77\wpm{ 0.04}$} \\
\midrule
Softsign                                                & $1\!\to\! 1$ & \ccA[94.2401]{$94.24\wpm{ 0.05}$} & \ccB[96.6099]{$96.61\wpm{ 0.03}$} & \ccC[96.8459]{$96.85\wpm{ 0.02}$} \\
Tanh                                                    & $1\!\to\! 1$ & \ccA[94.5898]{$\mbs{94.59}\wpm{ 0.03}$} & \ccB[96.9701]{$\mbf{\mbns{96.97}}\wpm{ 0.02}$} & \ccC[97.0343]{$\mbns{97.03}\wpm{ 0.04}$} \\
\midrule
GLU                                                     & $2\!\to\! 1$ & \ccA[93.9683]{$93.97\wpm{ 0.08}$} & \ccB[96.8550]{$96.85\wpm{ 0.03}$} & \ccC[97.0942]{$\mbns{97.09}\wpm{ 0.01}$} \\
\midrule
$\opn{Max}$                                             & $2\!\to\! 1$ & \ccA[91.6728]{$91.67\wpm{ 0.05}$} & \ccB[95.9913]{$95.99\wpm{ 0.03}$} & \ccC[96.9457]{$96.95\wpm{ 0.01}$} \\
$\opn{Max},\opn{Min}$ (d)                               & $2\!\to\! 2$ & \ccA[92.0613]{$92.06\wpm{ 0.01}$} & \ccB[96.0629]{$96.06\wpm{ 0.01}$} & \ccC[96.9355]{$96.94\wpm{ 0.02}$} \\
\midrule
SignedGeomean                                           & $2\!\to\! 1$ & \ccA[94.7213]{$\mbf{\mbns{94.72}}\wpm{ 0.02}$} & \ccB[96.5087]{$96.51\wpm{ 0.06}$} & \ccC[96.7075]{$96.71\wpm{ 0.02}$} \\
\midrule
$\opn{XNOR_{IL}}$                                       & $2\!\to\! 1$ & \ccA[92.2274]{$92.23\wpm{ 0.03}$} & \ccB[96.6416]{$96.64\wpm{ 0.05}$} & \ccC[97.0391]{$97.04\wpm{ 0.02}$} \\
$\opn{OR_{IL}}$                                         & $2\!\to\! 1$ & \ccA[85.4200]{$85.42\wpm{ 0.06}$} & \ccB[93.3849]{$93.38\wpm{ 0.05}$} & \ccC[95.4994]{$95.50\wpm{ 0.06}$} \\
$\opn{OR}, \opn{AND_{IL}}$ (d)                          & $2\!\to\! 2$ & \ccA[84.7851]{$84.79\wpm{ 0.23}$} & \ccB[92.8284]{$92.83\wpm{ 0.18}$} & \ccC[95.1860]{$95.19\wpm{ 0.10}$} \\
$\opn{OR}, \opn{XNOR_{IL}}$ (p)                         & $2\!\to\! 1$ & \ccA[91.3643]{$91.36\wpm{ 0.08}$} & \ccB[96.1843]{$96.18\wpm{ 0.05}$} & \ccC[96.8868]{$96.89\wpm{ 0.01}$} \\
$\opn{OR}, \opn{XNOR_{IL}}$ (d)                         & $2\!\to\! 2$ & \ccA[89.6419]{$89.64\wpm{ 0.06}$} & \ccB[95.3395]{$95.34\wpm{ 0.01}$} & \ccC[96.5106]{$96.51\wpm{ 0.05}$} \\
$\opn{OR}, \opn{AND}, \opn{XNOR_{IL}}$ (p)              & $2\!\to\! 1$ & \ccA[90.4503]{$90.45\wpm{ 0.17}$} & \ccB[95.5996]{$95.60\wpm{ 0.13}$} & \ccC[96.6564]{$96.66\wpm{ 0.03}$} \\
$\opn{OR}, \opn{AND}, \opn{XNOR_{IL}}$ (d)              & $2\!\to\! 3$ & \ccA[88.5864]{$88.59\wpm{ 0.09}$} & \ccB[94.5075]{$94.51\wpm{ 0.04}$} & \ccC[96.0540]{$96.05\wpm{ 0.02}$} \\
\midrule
$\opn{XNOR_{NIL}}$                                      & $2\!\to\! 1$ & \ccA[90.5331]{$90.53\wpm{ 0.07}$} & \ccB[96.2288]{$96.23\wpm{ 0.03}$} & \ccC[96.9269]{$96.93\wpm{ 0.01}$} \\
$\opn{OR_{NIL}}$                                        & $2\!\to\! 1$ & \ccA[85.0820]{$85.08\wpm{ 0.09}$} & \ccB[93.0943]{$93.09\wpm{ 0.01}$} & \ccC[95.3215]{$95.32\wpm{ 0.05}$} \\
$\opn{OR}, \opn{AND_{NIL}}$ (d)                         & $2\!\to\! 2$ & \ccA[85.1648]{$85.16\wpm{ 0.11}$} & \ccB[92.8635]{$92.86\wpm{ 0.11}$} & \ccC[95.3540]{$95.35\wpm{ 0.06}$} \\
$\opn{OR}, \opn{XNOR_{NIL}}$ (p)                        & $2\!\to\! 1$ & \ccA[89.9915]{$89.99\wpm{ 0.10}$} & \ccB[95.8199]{$95.82\wpm{ 0.06}$} & \ccC[96.6689]{$96.67\wpm{ 0.03}$} \\
$\opn{OR}, \opn{XNOR_{NIL}}$ (d)                        & $2\!\to\! 2$ & \ccA[88.6311]{$88.63\wpm{ 0.16}$} & \ccB[95.3731]{$95.37\wpm{ 0.05}$} & \ccC[96.6245]{$96.62\wpm{ 0.03}$} \\
$\opn{OR}, \opn{AND}, \opn{XNOR_{NIL}}$ (p)             & $2\!\to\! 1$ & \ccA[89.9056]{$89.91\wpm{ 0.18}$} & \ccB[95.5289]{$95.53\wpm{ 0.06}$} & \ccC[96.4767]{$96.48\wpm{ 0.06}$} \\
$\opn{OR}, \opn{AND}, \opn{XNOR_{NIL}}$ (d)             & $2\!\to\! 3$ & \ccA[87.9864]{$87.99\wpm{ 0.10}$} & \ccB[94.7618]{$94.76\wpm{ 0.07}$} & \ccC[96.1562]{$96.16\wpm{ 0.06}$} \\
\midrule
$\opn{XNOR_{AIL}}$                                      & $2\!\to\! 1$ & \ccA[94.2189]{$94.22\wpm{ 0.07}$} & \ccB[96.9021]{$\mbs{\mbns{96.90}}\wpm{ 0.02}$} & \ccC[97.1178]{$\mbf{\mbns{97.12}}\wpm{ 0.01}$} \\
$\opn{OR_{AIL}}$                                        & $2\!\to\! 1$ & \ccA[91.1383]{$91.14\wpm{ 0.05}$} & \ccB[95.7404]{$95.74\wpm{ 0.05}$} & \ccC[96.8271]{$96.83\wpm{ 0.01}$} \\
$\opn{OR}, \opn{AND_{AIL}}$ (d)                         & $2\!\to\! 2$ & \ccA[91.3247]{$91.32\wpm{ 0.16}$} & \ccB[95.8492]{$95.85\wpm{ 0.03}$} & \ccC[96.8701]{$96.87\wpm{ 0.01}$} \\
$\opn{OR}, \opn{XNOR_{AIL}}$ (p)                        & $2\!\to\! 1$ & \ccA[93.6577]{$93.66\wpm{ 0.06}$} & \ccB[96.6046]{$96.60\wpm{ 0.04}$} & \ccC[97.1111]{$\mbns{97.11}\wpm{ 0.01}$} \\
$\opn{OR}, \opn{XNOR_{AIL}}$ (d)                        & $2\!\to\! 2$ & \ccA[93.5842]{$93.58\wpm{ 0.04}$} & \ccB[96.5958]{$96.60\wpm{ 0.01}$} & \ccC[97.0624]{$97.06\wpm{ 0.01}$} \\
$\opn{OR}, \opn{AND}, \opn{XNOR_{AIL}}$ (p)             & $2\!\to\! 1$ & \ccA[93.2952]{$93.30\wpm{ 0.05}$} & \ccB[96.4288]{$96.43\wpm{ 0.03}$} & \ccC[97.0601]{$\mbns{97.06}\wpm{ 0.02}$} \\
$\opn{OR}, \opn{AND}, \opn{XNOR_{AIL}}$ (d)             & $2\!\to\! 3$ & \ccA[93.0940]{$93.09\wpm{ 0.11}$} & \ccB[96.3159]{$96.32\wpm{ 0.02}$} & \ccC[96.9894]{$96.99\wpm{ 0.00}$} \\
\midrule
$\opn{XNOR_{NAIL}}$                                     & $2\!\to\! 1$ & \ccA[93.8921]{$93.89\wpm{ 0.03}$} & \ccB[96.7810]{$96.78\wpm{ 0.01}$} & \ccC[97.0533]{$97.05\wpm{ 0.01}$} \\
$\opn{OR_{NAIL}}$                                       & $2\!\to\! 1$ & \ccA[91.5661]{$91.57\wpm{ 0.03}$} & \ccB[95.9301]{$95.93\wpm{ 0.03}$} & \ccC[96.8416]{$96.84\wpm{ 0.02}$} \\
$\opn{OR}, \opn{AND_{NAIL}}$ (d)                        & $2\!\to\! 2$ & \ccA[92.0992]{$92.10\wpm{ 0.09}$} & \ccB[95.9648]{$95.96\wpm{ 0.05}$} & \ccC[96.8916]{$96.89\wpm{ 0.03}$} \\
$\opn{OR}, \opn{XNOR_{NAIL}}$ (p)                       & $2\!\to\! 1$ & \ccA[93.6200]{$93.62\wpm{ 0.05}$} & \ccB[96.5605]{$96.56\wpm{ 0.02}$} & \ccC[97.0453]{$97.05\wpm{ 0.01}$} \\
$\opn{OR}, \opn{XNOR_{NAIL}}$ (d)                       & $2\!\to\! 2$ & \ccA[93.6404]{$93.64\wpm{ 0.07}$} & \ccB[96.5727]{$96.57\wpm{ 0.01}$} & \ccC[97.0534]{$97.05\wpm{ 0.01}$} \\
$\opn{OR}, \opn{AND}, \opn{XNOR_{NAIL}}$ (p)            & $2\!\to\! 1$ & \ccA[93.3451]{$93.35\wpm{ 0.06}$} & \ccB[96.4102]{$96.41\wpm{ 0.01}$} & \ccC[97.0114]{$97.01\wpm{ 0.01}$} \\
$\opn{OR}, \opn{AND}, \opn{XNOR_{NAIL}}$ (d)            & $2\!\to\! 3$ & \ccA[93.2960]{$93.30\wpm{ 0.02}$} & \ccB[96.3330]{$96.33\wpm{ 0.02}$} & \ccC[96.9903]{$96.99\wpm{ 0.01}$} \\
\bottomrule
\end{tabular}
}
\end{table}

For each activation function, we varied the number of hidden units per layer to investigate how the activation functions affected the performance of the networks as its capacity changed.
We did not use weight decay or data augmentation for this experiment, and so the network exhibits clear overfitting with larger architectures when the network is over-parameterized.

As shown in \autoref{fig:covtype}, we found that $\opn{XNOR_{NAIL}}$ performed well on this task (outperforming ReLU), whilst $\opn{OR_{NAIL}}$ and Max were less successful, performing worse than ReLU.
SignedGeomean performed well when using 1 hidden layer, but its relative performance drastically decreased as the MLP depth was increased.
Tanh performed well throughout, and was best when using 2 hidden layers, but the margin vs $\opn{XNOR_{NAIL}}$ was not statistically significant (two-sided Student's $t$-test, $p\!>\!0.05$).
SiLU performed poorly at this task, and the ELU family of activation functions (of which SELU is shown in the figures) performed even worse (reaching only $94.15\pm0.09\%$ with 3 hidden layers).

\subsection{Bach Chorale training details}
\label{a:jsb}

For each of the 4 voices, we restricted the available pitches to 3 octaves, resulting in 37 one-hot tokens (including silence). Since we only planned on feeding small time-windows of the chorale into our model, for pre-processing, we converted the data into a shape of $(L, 4, 37)$, where we set the sequence length to be $L=4$.

To generate training examples for the discriminator, we transposed by $\{-5,-4,\dots,5,6\}$ semitones, chosen uniformly at random, and there was a $0.5$ probability that the sample was corrupted by the following method:
\begin{itemize}
    \setlength\itemsep{0em}
    \item Choose 2--3 notes in the $4\,L$ window to be corrupted
    \item For each note, corrupt by the following mixture distribution:
    \begin{itemize}
        \item $(p=0.6)$ Sample a pitch from a Gaussian centered on the existing note, with $\sigma = 3$ semitones, forcing the new pitch to be distinct
        \item $(p=0.2)$ Copy a pitch from the current voice, forcing the new pitch to be distinct
        \item $(p=0.2)$ Extend the previous note in time
        \item $(p=0.1)$ Sample uniformly from all 37 possible tokens
    \end{itemize}
\end{itemize}

\begin{figure}[htb]
    \centering
    \includegraphics[scale=0.45]{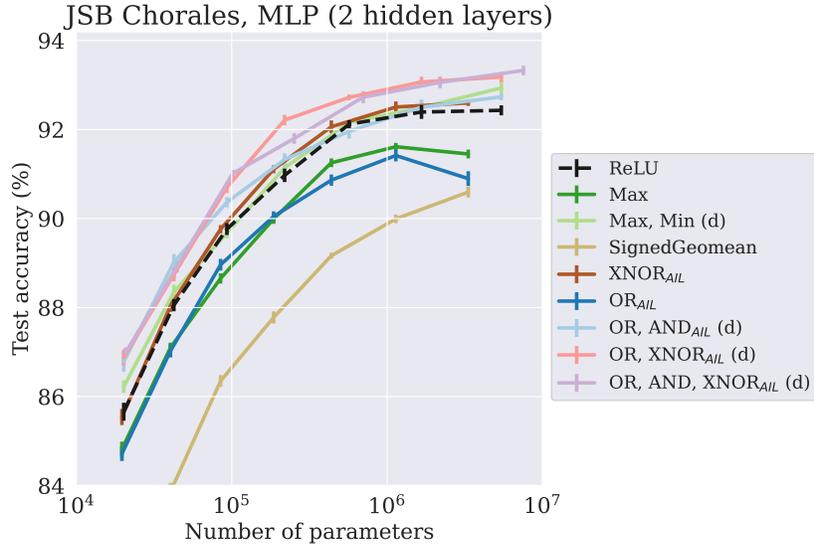}
\caption{
We trained 2-layer MLPs on JSB Chorales using ADAM (constant LR \num{1e-3}, 150~ep.), Mean (bars: std dev) of $n\!=\!10$ weight inits.
}
\label{fig:jsb}
\end{figure}

\subsection{Correlations between pre-activations}
\label{a:correlations}

\revision{
To capture the existence of correlations, we took the cosine similarity between rows of the weight matrix.
Since the inputs to all features in a given layer are the same, this is equivalent to measuring the similarity between corresponding pair of pre-activation features.
}

Results on correlations between weights in the JSB Chorale models are shown in \autoref{fig:correlations}. We found that when taking all pre-activations into account, every activation function generally showed independence between features.
Interestingly, the cosine similarities between inputs that were paired together for the bivariate activation functions showed anticorrelation in almost all cases where $\opn{Max}$ or $\opn{OR_{AIL}}$ were used, and other cases generally showed more correlation than $\opn{ReLU}$.

\begin{figure}[htb]
    \centering
    \begin{subfigure}[b]{\linewidth}
        \includegraphics[width=\linewidth]{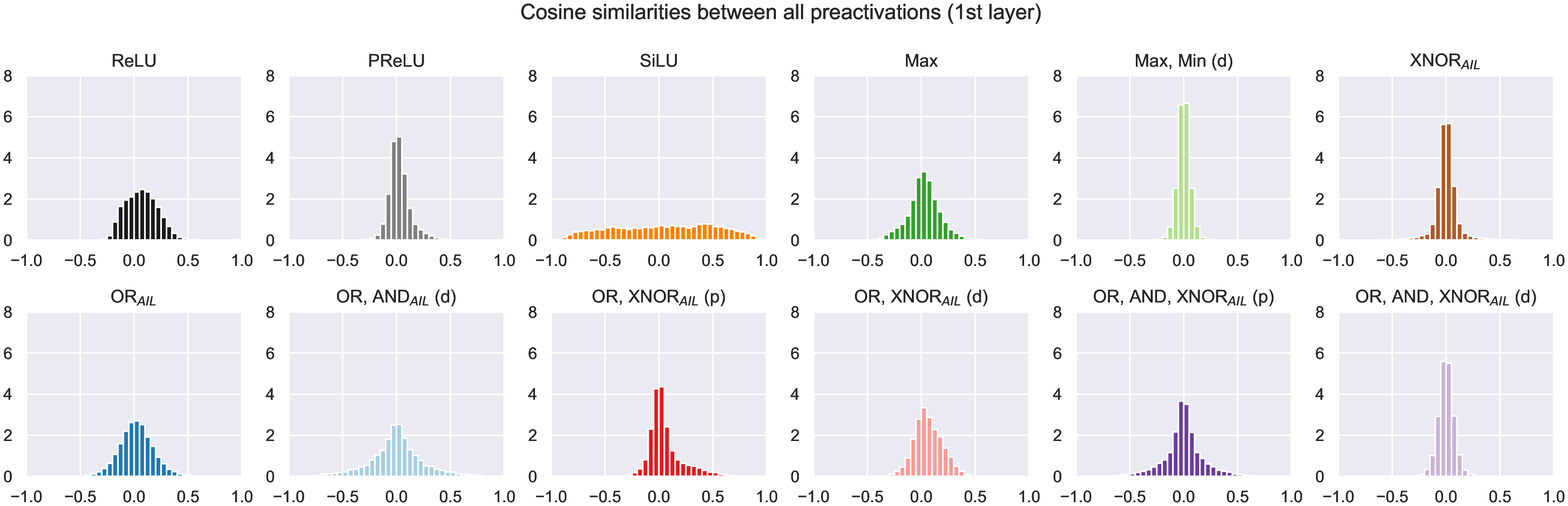}
    \end{subfigure}
    \begin{subfigure}[b]{\linewidth}
        \includegraphics[width=\linewidth]{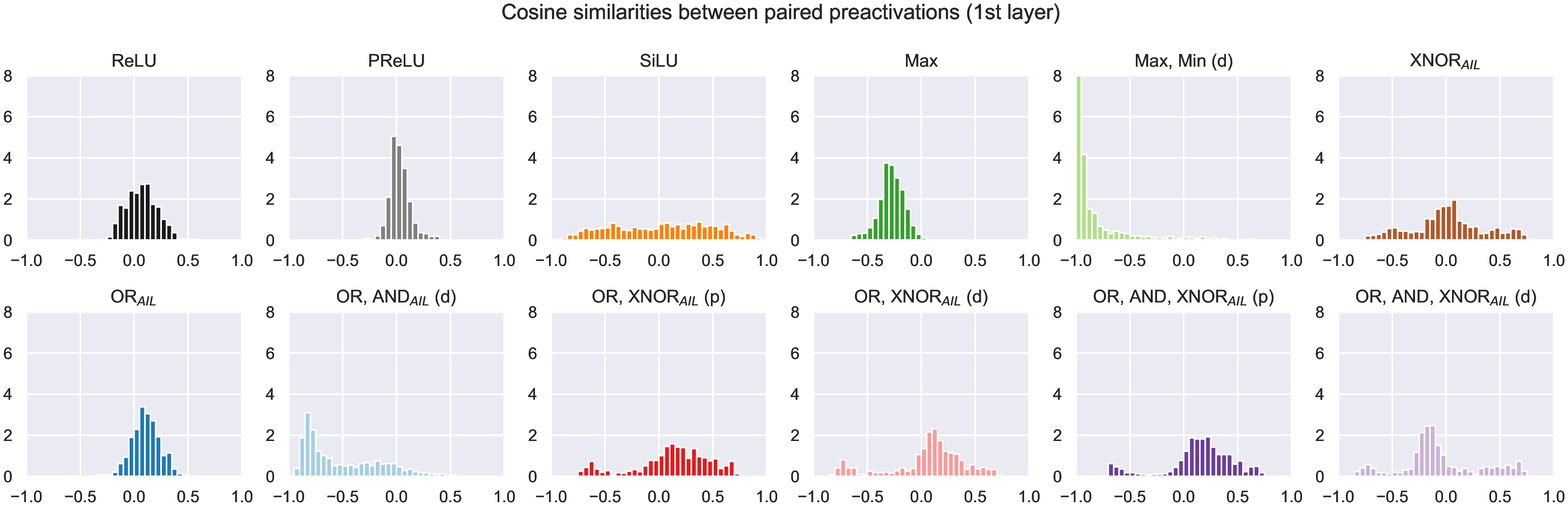}
    \end{subfigure}
    \caption{Cosine similarities between pre-activation weights of two activation functions in the first layer of an MLP trained on JSB Chorales. }
    \label{fig:correlations}
\end{figure}

We found that randomly selected pairs of preactivation features within the same layer have correlations that are given by a Gaussian-like distribution centered around zero.
This was the case for all of the activation functions we tested.
The behaviour of randomly selected pairs of features is thus reasonably consistent with the assumption of independence which we have made.
We also investigated the correlation between the pairs of preactivation features which were passing into our two-dimensional activation functions.
Here, we found the correlation structure is different, and the correlation depends on the activation function being used. With Max and $\opn{OR_{AIL}}$ activations, the network learns to make the columns of the weight matrix (and hence the preactivation scores for the pair of features) be inversely correlated.
With $\opn{XNOR_{AIL}}$, the network learns features which are either positively or negatively correlated (a wider distribution of correlations than seen with random pairs of features).
We observe that (in all cases) the network learns to make the features passed to the AIL activation functions be correlated instead of independent, despite our assumption of independence.
So it appears clear that the assumption of independence is violated, but also that it doesn't \textit{really} matter because the network is choosing to break the assumption and induce these correlations between the features to get better performance.

\subsection{CNN and MLP on MNIST}
\label{a:mnist}

In this experiment we trained MLP and CNN models for 10 epochs on the MNIST dataset using ADAM optimizer, one-cycle learning rate schedule, and cross entropy loss, with batch size of 256. We augmented our training samples using a random affine transformation from: rotation of $\pm10$ degrees, scale of factor $0.8$ to $1.2$, translation with max absolute fraction for horizontal and vertical directions of $0.08$, and shear parallel to the x-axes of $\pm0.3$.

The hyperparameters for the optimizer and scheduler were selected through a random search of the hyperparameter space.
We chose to do random search instead of grid search because it typically yields better results for the same number of test cases.

For the hyperparameter search, we trained on the first \num{50000} samples of the training partition and used the final \num{10000} samples as a validation set.
We ran the search for four iterations, with each iteration sampling 120 different hyperparameter settings. The initial bounds for our hyperparameter samples were set as described in \autoref{tab:hparam-search}.
These bounds were chosen to be suitably wide such that the optimal configuration should be contained within them for all activation functions considered.

\begin{table}[ht]
\small
  \centering
  \caption{%
Hyperparameter random search parameters.
}
\label{tab:hparam-search}
\centerline{
\begin{tabular}{lrl}
\toprule
Hyperparameter  & Variable & Sampling (initial bounds) \\
\midrule
ADAM beta1       & $1-10^x$ & $x\sim \mathrm{Uniform}(-3, -0.5)$ \\
ADAM beta2       & $1-10^x$ & $x\sim \mathrm{Uniform}(-5, -1)$   \\
ADAM epsilon     &   $10^x$ & $x\sim \mathrm{Uniform}(-10, -6)$  \\
Weight decay     &   $10^x$ & $x\sim \mathrm{Uniform}(-7, -3)$   \\
One-cycle max LR &   $10^x$ & $x\sim \mathrm{Uniform}(-4, 0)$    \\
One-cycle peak   &    $x$   & $x\sim \mathrm{Uniform}(0.1, 0.5)$ \\
\bottomrule
\end{tabular}
}
\end{table}

The bounds on our uniform random variable $x$ were tightened at each iteration by selecting the top-5 performing hyperparameter settings, taking the mean $\bar{x}$ and weighted standard deviation $\sigma$ across those settings, and re-setting the bounds for the next iteration to be equal to $\bar{x}\pm 1.5\,\sigma$. After the fourth iteration, we ran a final iteration where we selected the top-10 performing hyperparameter settings across the four previous iterations, and then re-ran these for another 120 seeds (i.e. random weight initializations). We then selected the hyperparameter settings which had the highest performance across all 120 seeds.

We found that the $\opn{XNOR_{AIL}}$ activation function required quite different hyperparameters than the other activation functions, but $\opn{AND_{AIL}}$ and $\opn{Max}$ activation functions used similar hyperparameters to ReLU. However, we have no reason to believe that the proposed AIL activation functions are more susceptible than others to the choice of hyperparameters or the way in which hyperparameters are selected.

Our MLP model consisted of two hidden layers of equal size with batch norm applied to each. The number of neurons in each layer was set, taking into consideration the current activation function being tested, to ensure the number of trainable parameters in the network remained static across experiments

Our CNN model consisted of six layers, where each layer was comprised of a set of 2D convolution filters (kernel size 3, stride 1, padding 1), batch normalization, and a non-linear activation. The network also applied a pooling layer (kernel size 2, stride 2) at the end of every second layer. After the 6 CNN layers the network flattens the output and applies three linear layers. The number of output channels (i.e. pre-activations) for the six 2D convolutions were $[c, 2c, 4c, 4c, 8c, 8c]$, and the number of pre-activation neurons produced by the subsequent linear layers were $[32c, 16c]$. Similar to the MLP model, $c$ was chosen to ensure the number of trainable parameters in the network remained fixed across experiments. When varying the number of network parameters in our experiments, if the number of parameters dropped below \num{1e6} in the CNN model, we changed the structure of the convolution layers and linear layers to $[c, c, c, c, 1.5c, 1.5c]$ and $[2c, 1.5c]$, respectively. This ensured that each layer had at least two channels for our activation functions to aggregate.

Once our optimal hyperparameters were selected for both our MLP and CNN models, we then trained each model several times with varying number of trainable parameters. The reason we vary the number of parameters in the network instead of the network size/structure is because a network using our logical activation functions can have a significantly different number of parameters than an identical network using ReLU activation, because our logical activations aggregate across neurons at each layer.

\begin{figure}[htb]
    \centering
    \begin{subfigure}[b]{0.49\linewidth}
    \centering
        \includegraphics[scale=0.4]{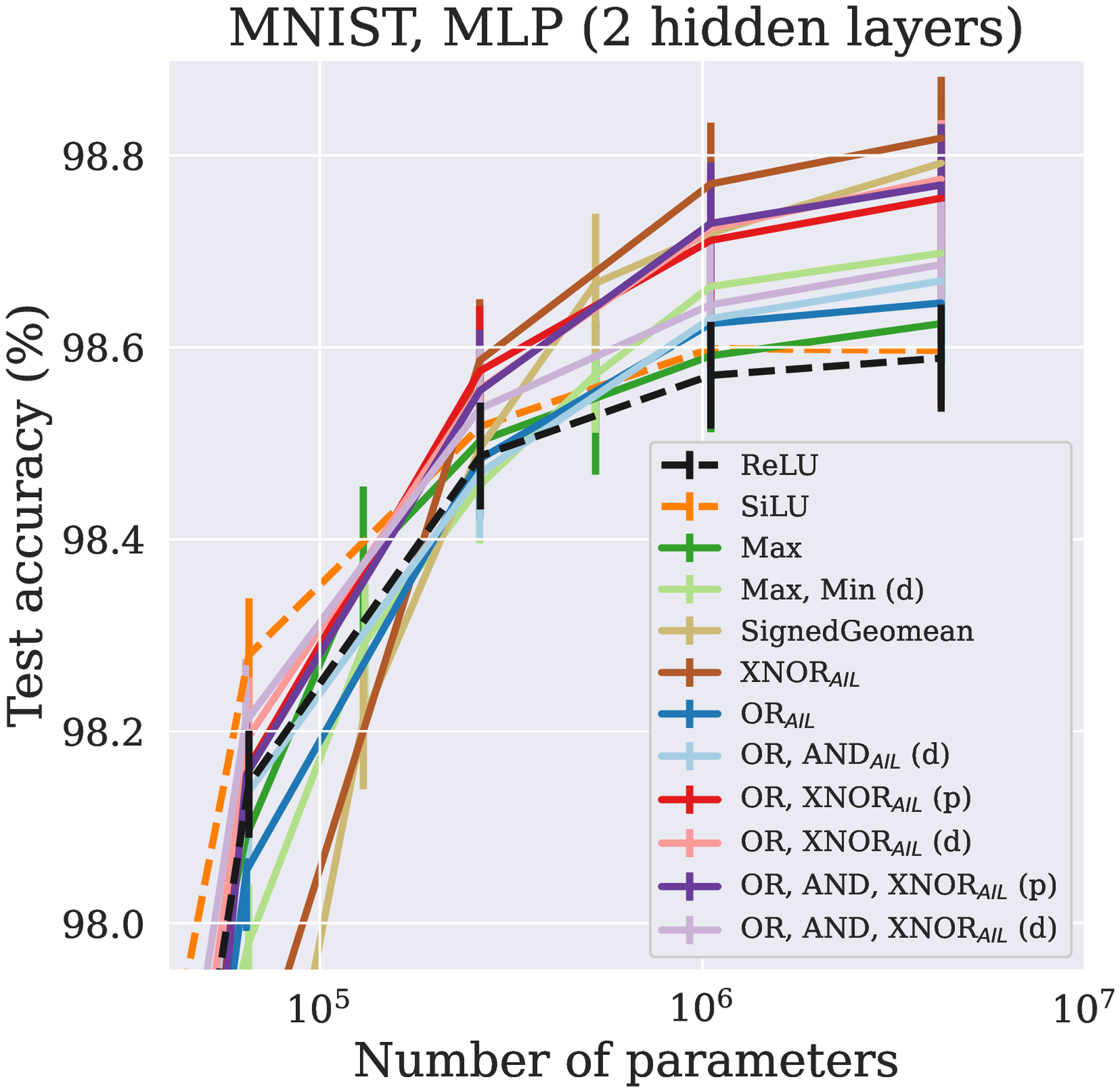}
    \end{subfigure}%
    \hfill{}
    \begin{subfigure}[b]{0.49\linewidth}
    \centering
        \includegraphics[scale=0.4]{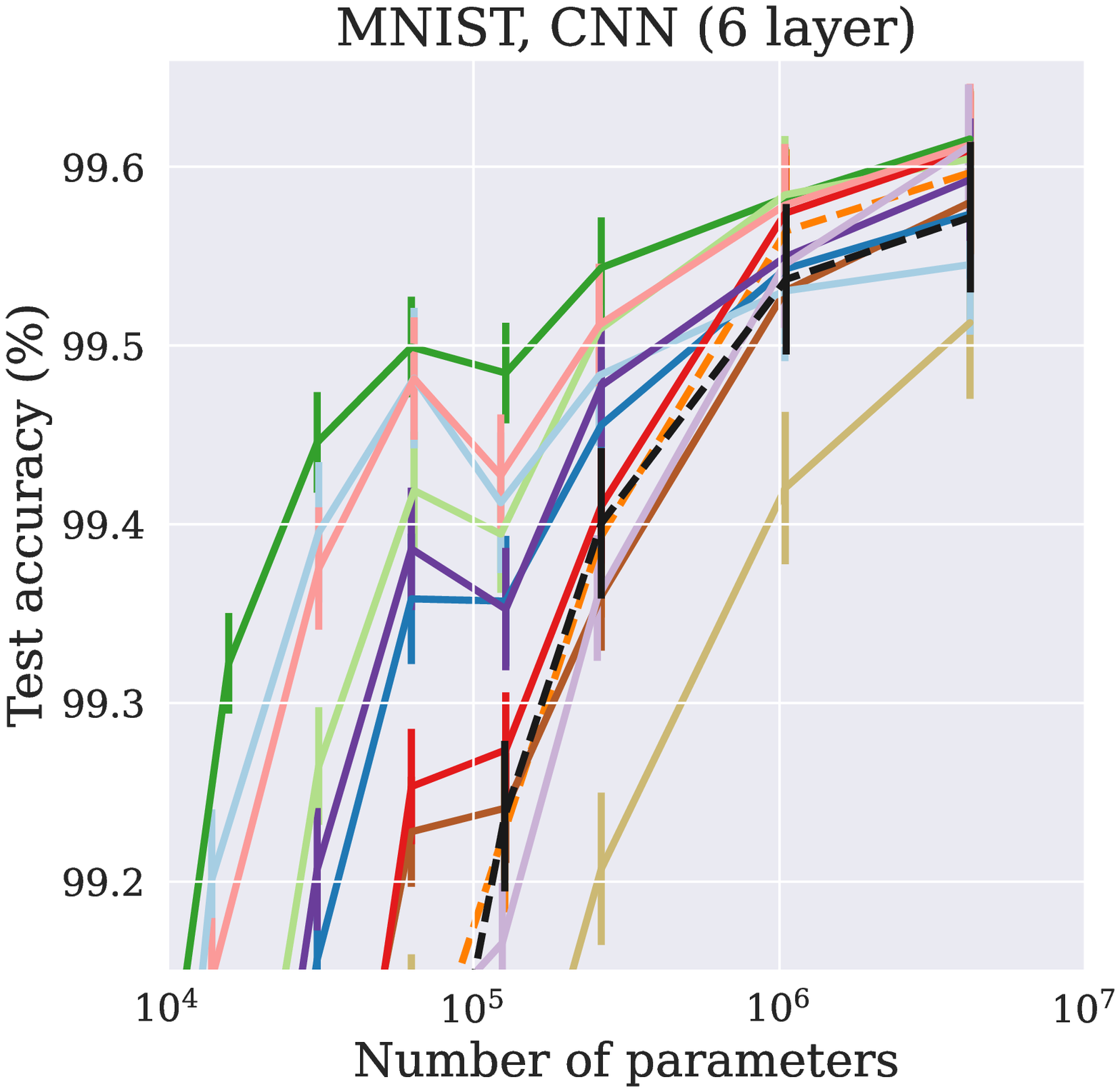}
    \end{subfigure}%
\caption{
We trained CNN on \href{http://yann.lecun.com/exdb/mnist/}{MNIST}, MLP on flattened-MNIST, using ADAM (1-cycle, 10~ep), hyperparams determined by random search. Mean (bars: std~dev) of $n\!=\!40$ weight inits.
}
\label{fig:mnist-extra}
\end{figure}

Our abridged results are plotted in \autoref{fig:mnist}, with results for more activation functions shown in \autoref{fig:mnist-extra}.
Performance measurements from four slices of total parameter count are tabulated in \autoref{tab:mlp-mnist} and \autoref{tab:cnn-mnist}.

\begin{table}[tbhp]
\small
  \centering
  \caption{%
MLP on MNIST, with two hidden layers, while varying the activation function and network width.
Bold: best.
Underlined: top two.
Italic: no sig. diff. from best (two-sided Student's $t$-test, $p\!>\!0.05$).
Background: linear color scale from worst (white) to best (black) with a given number of parameters.
}
\label{tab:mlp-mnist}
\def\ccAmin{91.03449746966362}
\def\ccAmax{97.03299760818481}
\def\ccA[#1]#2{\heatmapcell{#1}{#2}{\ccAmin}{\ccAmax}}
\def\ccBmin{97.59056365714883}
\def\ccBmax{98.2794976234436}
\def\ccB[#1]#2{\heatmapcell{#1}{#2}{\ccBmin}{\ccBmax}}
\def\ccCmin{98.45604943601708}
\def\ccCmax{98.58624771237373}
\def\ccC[#1]#2{\heatmapcell{#1}{#2}{\ccCmin}{\ccCmax}}
\def\ccDmin{98.58874768018723}
\def\ccDmax{98.81799772381783}
\def\ccD[#1]#2{\heatmapcell{#1}{#2}{\ccDmin}{\ccDmax}}
\centerline{
\begin{tabular}{lrrrr}
\toprule
 & \multicolumn{4}{c}{Test Accuracy (\%) by \textnumero{} Params.} \\
\cmidrule(l){2-5}
Activation function                                     & \multicolumn{1}{c}{$\sim\! 16\,\text{k}$} & \multicolumn{1}{c}{$\sim\! 65\,\text{k}$} & \multicolumn{1}{c}{$\sim\! 262\,\text{k}$} & \multicolumn{1}{c}{$\sim\! 4\,\text{M}$} \\
\toprule
$\opn{ReLU}$                                             & \ccA[96.6022]{$96.60\wpm{ 0.03}$} & \ccB[98.1448]{$98.14\wpm{ 0.01}$} & \ccC[98.4868]{$98.49\wpm{ 0.01}$} & \ccD[98.5887]{$98.59\wpm{ 0.01}$} \\
\midrule
SiLU                                                     & \ccA[97.0330]{$\mbf{\mbns{97.03}}\wpm{ 0.02}$} & \ccB[98.2795]{$\mbf{\mbns{98.28}}\wpm{ 0.01}$} & \ccC[98.5177]{$98.52\wpm{ 0.01}$} & \ccD[98.5967]{$98.60\wpm{ 0.01}$} \\
\midrule
$\opn{Max}$                                              & \ccA[96.3570]{$96.36\wpm{ 0.03}$} & \ccB[98.0913]{$98.09\wpm{ 0.01}$} & \ccC[98.5022]{$98.50\wpm{ 0.01}$} & \ccD[98.6247]{$98.62\wpm{ 0.01}$} \\
$\opn{Max},\opn{Min}$ (d)                                & \ccA[96.3828]{$96.38\wpm{ 0.03}$} & \ccB[97.9797]{$97.98\wpm{ 0.02}$} & \ccC[98.4560]{$98.46\wpm{ 0.01}$} & \ccD[98.6981]{$98.70\wpm{ 0.01}$} \\
\midrule
SignedGeomean                                            & \ccA[91.0345]{$91.03\wpm{ 0.15}$} & \ccB[97.5906]{$97.59\wpm{ 0.02}$} & \ccC[98.4943]{$98.49\wpm{ 0.01}$} & \ccD[98.7920]{$\mbs{\mbns{98.79}}\wpm{ 0.01}$} \\
\midrule
$\opn{XNOR_{AIL}}$                                       & \ccA[91.9057]{$91.91\wpm{ 0.13}$} & \ccB[97.8132]{$97.81\wpm{ 0.02}$} & \ccC[98.5862]{$\mbf{\mbns{98.59}}\wpm{ 0.01}$} & \ccD[98.8180]{$\mbf{\mbns{98.82}}\wpm{ 0.01}$} \\
$\opn{OR_{AIL}}$                                         & \ccA[96.2232]{$96.22\wpm{ 0.03}$} & \ccB[98.0550]{$98.05\wpm{ 0.01}$} & \ccC[98.4835]{$98.48\wpm{ 0.01}$} & \ccD[98.6462]{$98.65\wpm{ 0.01}$} \\
$\opn{OR}, \opn{AND_{AIL}}$ (d)                          & \ccA[96.6400]{$\mbs{96.64}\wpm{ 0.03}$} & \ccB[98.1370]{$98.14\wpm{ 0.01}$} & \ccC[98.4687]{$98.47\wpm{ 0.01}$} & \ccD[98.6692]{$98.67\wpm{ 0.01}$} \\
$\opn{OR}, \opn{XNOR_{AIL}}$ (p)                         & \ccA[95.8482]{$95.85\wpm{ 0.03}$} & \ccB[98.1587]{$98.16\wpm{ 0.01}$} & \ccC[98.5752]{$\mbs{\mbns{98.58}}\wpm{ 0.01}$} & \ccD[98.7555]{$98.76\wpm{ 0.01}$} \\
$\opn{OR}, \opn{XNOR_{AIL}}$ (d)                         & \ccA[96.4980]{$96.50\wpm{ 0.03}$} & \ccB[98.1945]{$98.19\wpm{ 0.01}$} & \ccC[98.5562]{$\mbns{98.56}\wpm{ 0.01}$} & \ccD[98.7752]{$98.78\wpm{ 0.01}$} \\
$\opn{OR}, \opn{AND}, \opn{XNOR_{AIL}}$ (p)              & \ccA[95.6955]{$95.70\wpm{ 0.06}$} & \ccB[98.1552]{$98.16\wpm{ 0.01}$} & \ccC[98.5547]{$\mbns{98.55}\wpm{ 0.01}$} & \ccD[98.7692]{$98.77\wpm{ 0.01}$} \\
$\opn{OR}, \opn{AND}, \opn{XNOR_{AIL}}$ (d)              & \ccA[96.5752]{$96.58\wpm{ 0.02}$} & \ccB[98.2097]{$\mbs{98.21}\wpm{ 0.01}$} & \ccC[98.5360]{$98.54\wpm{ 0.01}$} & \ccD[98.6862]{$98.69\wpm{ 0.01}$} \\
\bottomrule
\end{tabular}
}
\end{table}

\begin{table}[tbhp]
\small
  \centering
  \caption{%
CNN on MNIST, while varying the activation function and network width.
Bold: best.
Underlined: top two.
Italic: no sig. diff. from best (two-sided Student's $t$-test, $p\!>\!0.05$).
Background: linear color scale from worst (white) to best (black) with a given number of parameters.
}
\label{tab:cnn-mnist}
\def\ccAmin{97.07624763250351}
\def\ccAmax{99.32224780321121}
\def\ccA[#1]#2{\heatmapcell{#1}{#2}{\ccAmin}{\ccAmax}}
\def\ccBmin{98.97299781441689}
\def\ccBmax{99.499122351408}
\def\ccB[#1]#2{\heatmapcell{#1}{#2}{\ccBmin}{\ccBmax}}
\def\ccCmin{99.20724749565125}
\def\ccCmax{99.543496966362}
\def\ccC[#1]#2{\heatmapcell{#1}{#2}{\ccCmin}{\ccCmax}}
\def\ccDmin{99.51274693012238}
\def\ccDmax{99.61549758911133}
\def\ccD[#1]#2{\heatmapcell{#1}{#2}{\ccDmin}{\ccDmax}}
\centerline{
\begin{tabular}{lrrrr}
\toprule
 & \multicolumn{4}{c}{Test Accuracy (\%) by \textnumero{} Params.} \\
\cmidrule(l){2-5}
Activation function                                     & \multicolumn{1}{c}{$\sim\! 15\,\text{k}$} & \multicolumn{1}{c}{$\sim\! 63\,\text{k}$} & \multicolumn{1}{c}{$\sim\! 260\,\text{k}$} & \multicolumn{1}{c}{$\sim\! 4\,\text{M}$} \\
\toprule
$\opn{ReLU}$                                             & \ccA[97.0762]{$97.08\wpm{ 0.12}$} & \ccB[98.9730]{$98.97\wpm{ 0.02}$} & \ccC[99.4005]{$99.40\wpm{ 0.01}$} & \ccD[99.5717]{$99.57\wpm{ 0.01}$} \\
\midrule
SiLU                                                     & \ccA[97.0762]{$97.08\wpm{ 0.14}$} & \ccB[99.0735]{$99.07\wpm{ 0.01}$} & \ccC[99.3935]{$99.39\wpm{ 0.01}$} & \ccD[99.5967]{$99.60\wpm{ 0.01}$} \\
\midrule
$\opn{Max}$                                              & \ccA[99.3222]{$\mbf{\mbns{99.32}}\wpm{ 0.01}$} & \ccB[99.4991]{$\mbf{\mbns{99.50}}\wpm{ 0.01}$} & \ccC[99.5435]{$\mbf{\mbns{99.54}}\wpm{ 0.01}$} & \ccD[99.6155]{$\mbf{\mbns{99.62}}\wpm{ 0.00}$} \\
$\opn{Max},\opn{Min}$ (d)                                & \ccA[98.8953]{$98.90\wpm{ 0.02}$} & \ccB[99.4189]{$99.42\wpm{ 0.01}$} & \ccC[99.5087]{$\mbs{99.51}\wpm{ 0.01}$} & \ccD[99.6045]{$\mbns{99.60}\wpm{ 0.01}$} \\
\midrule
SignedGeomean                                            & \ccA[98.2782]{$98.28\wpm{ 0.03}$} & \ccB[99.1167]{$99.12\wpm{ 0.01}$} & \ccC[99.2072]{$99.21\wpm{ 0.01}$} & \ccD[99.5127]{$99.51\wpm{ 0.01}$} \\
\midrule
$\opn{XNOR_{AIL}}$                                       & \ccA[98.5065]{$98.51\wpm{ 0.02}$} & \ccB[99.2280]{$99.23\wpm{ 0.01}$} & \ccC[99.3600]{$99.36\wpm{ 0.01}$} & \ccD[99.5800]{$99.58\wpm{ 0.00}$} \\
$\opn{OR_{AIL}}$                                         & \ccA[98.8285]{$98.83\wpm{ 0.02}$} & \ccB[99.3582]{$99.36\wpm{ 0.01}$} & \ccC[99.4555]{$99.46\wpm{ 0.01}$} & \ccD[99.5735]{$99.57\wpm{ 0.01}$} \\
$\opn{OR}, \opn{AND_{AIL}}$ (d)                          & \ccA[99.2012]{$\mbs{99.20}\wpm{ 0.01}$} & \ccB[99.4817]{$\mbs{\mbns{99.48}}\wpm{ 0.01}$} & \ccC[99.4832]{$99.48\wpm{ 0.01}$} & \ccD[99.5452]{$99.55\wpm{ 0.01}$} \\
$\opn{OR}, \opn{XNOR_{AIL}}$ (p)                         & \ccA[98.3372]{$98.34\wpm{ 0.04}$} & \ccB[99.2532]{$99.25\wpm{ 0.01}$} & \ccC[99.4102]{$99.41\wpm{ 0.01}$} & \ccD[99.6092]{$\mbs{\mbns{99.61}}\wpm{ 0.01}$} \\
$\opn{OR}, \opn{XNOR_{AIL}}$ (d)                         & \ccA[99.1455]{$99.15\wpm{ 0.01}$} & \ccB[99.4815]{$\mbs{\mbns{99.48}}\wpm{ 0.01}$} & \ccC[99.5115]{$\mbs{99.51}\wpm{ 0.01}$} & \ccD[99.6122]{$\mbf{\mbns{99.61}}\wpm{ 0.01}$} \\
$\opn{OR}, \opn{AND}, \opn{XNOR_{AIL}}$ (p)              & \ccA[98.8720]{$98.87\wpm{ 0.02}$} & \ccB[99.3862]{$99.39\wpm{ 0.01}$} & \ccC[99.4775]{$99.48\wpm{ 0.01}$} & \ccD[99.5927]{$99.59\wpm{ 0.01}$} \\
$\opn{OR}, \opn{AND}, \opn{XNOR_{AIL}}$ (d)              & \ccA[97.8797]{$97.88\wpm{ 0.05}$} & \ccB[99.1062]{$99.11\wpm{ 0.01}$} & \ccC[99.3587]{$99.36\wpm{ 0.01}$} & \ccD[99.6110]{$\mbf{\mbns{99.61}}\wpm{ 0.01}$} \\
\bottomrule
\end{tabular}
}
\end{table}

\subsection{ResNet50 on CIFAR-10/100}
\label{a:cifar}

\revision{In this experiment, we explored the impact of the choice of activation functions on the performance of a pre-activation ResNet50 model~\citep{resnet,resnetv2} applied to the CIFAR-10 and -100 datasets.
Our base network was a pre-activation ResNet50 with 4 layers comprised of $[3, 4, 6, 3]$ bottleneck residual blocks \citep{resnetv2} with an expansion factor of 4.
As described in the main text, we exchanged all ReLU activation functions in the network to a candidate activation function while maintaining the size of the pass-through embedding, with the exception of the activation function between the first convolutional layer (stem) and the first residual block which was ReLU throughout.
We did not use any activation on the identity/pass-through branches of the residual blocks.
The base widths for the residual blocks were $[64, 128, 256, 512]$, respectively.
We experimented with changing the width of the network, scaling up the embedding space and all hidden layers by a common factor, $w$.
We used width factors of 0.25, 0.5, 1, 2, and 4 for $1\!\to \!1$ activation functions (ReLU, etc), widths of 0.375, 0.75, 1.5, 3, and 6 for $2\!\to \!1$ activation functions (Max, etc), and widths of 0.2, 0.4, 0.75, 1.5, and 3 for \{$\opn{OR_{AIL}}$, $\opn{AND_{AIL}}$, $\opn{XNOR_{AIL}}$ (d)\}.}
The stem width was held constant at 64 channels throughout.

For this experiment we trained a ResNet50 model for 100 epochs on CIFAR-10 and CIFAR-100 using ADAM optimizer, one-cycle learning rate schedule, cross entropy loss, and augmentations derived for CIFAR-10 by AutoAugment \citep{autoaugment}, with batch size of 128. The hyperparameters for the optimizer and scheduler were determined through a random search of the hyper parameter space on the CIFAR-100 dataset \revision{for a fixed width factor $w=2$ only}.
We used the same set of hyperparameters from this search for both the CIFAR-10 and CIFAR-100 experiments.
The hyperparameter search was tuned against a partition of 10\% of the CIFAR-100 training samples. Our hyperparameter search follows a similar approach to the one used for our MLP and CNN models on MNIST (\aref{a:mnist}), but due to compute constraints with our larger models here we only trained 100 seeds each iteration, and only re-examined the top 10 seeds in the final round.

For PReLU, we did not perform a new hyperparameter search and simply re-used the same hyperparameters as discovered in the search with ReLU.

\begin{figure}[htb]
    \centering
    \begin{subfigure}[b]{0.49\linewidth} 
    \centering
        \includegraphics[scale=0.4]{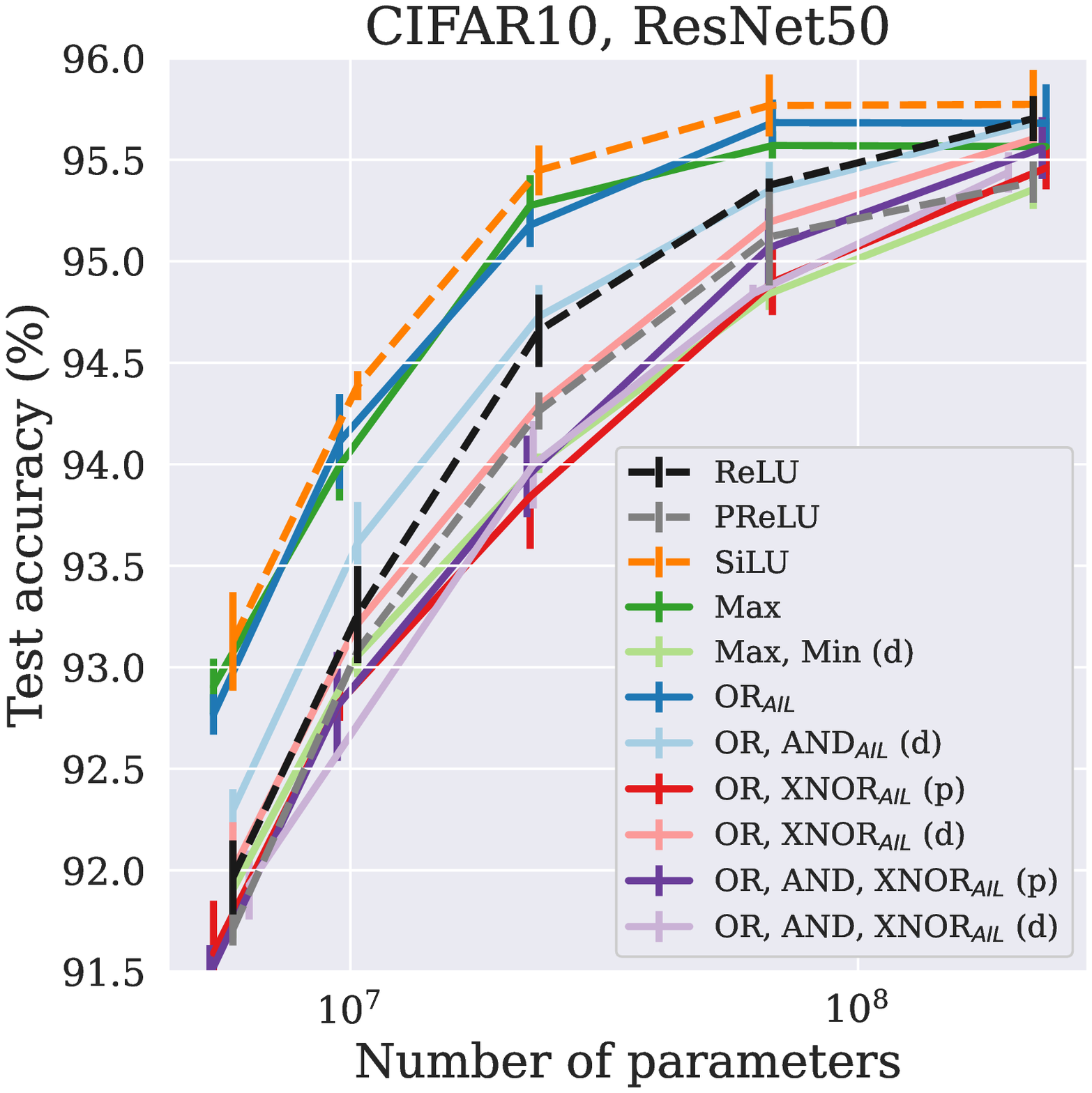} 
    \end{subfigure}%
    \hfill{}
    \begin{subfigure}[b]{0.49\linewidth}
    \centering
        \includegraphics[scale=0.4]{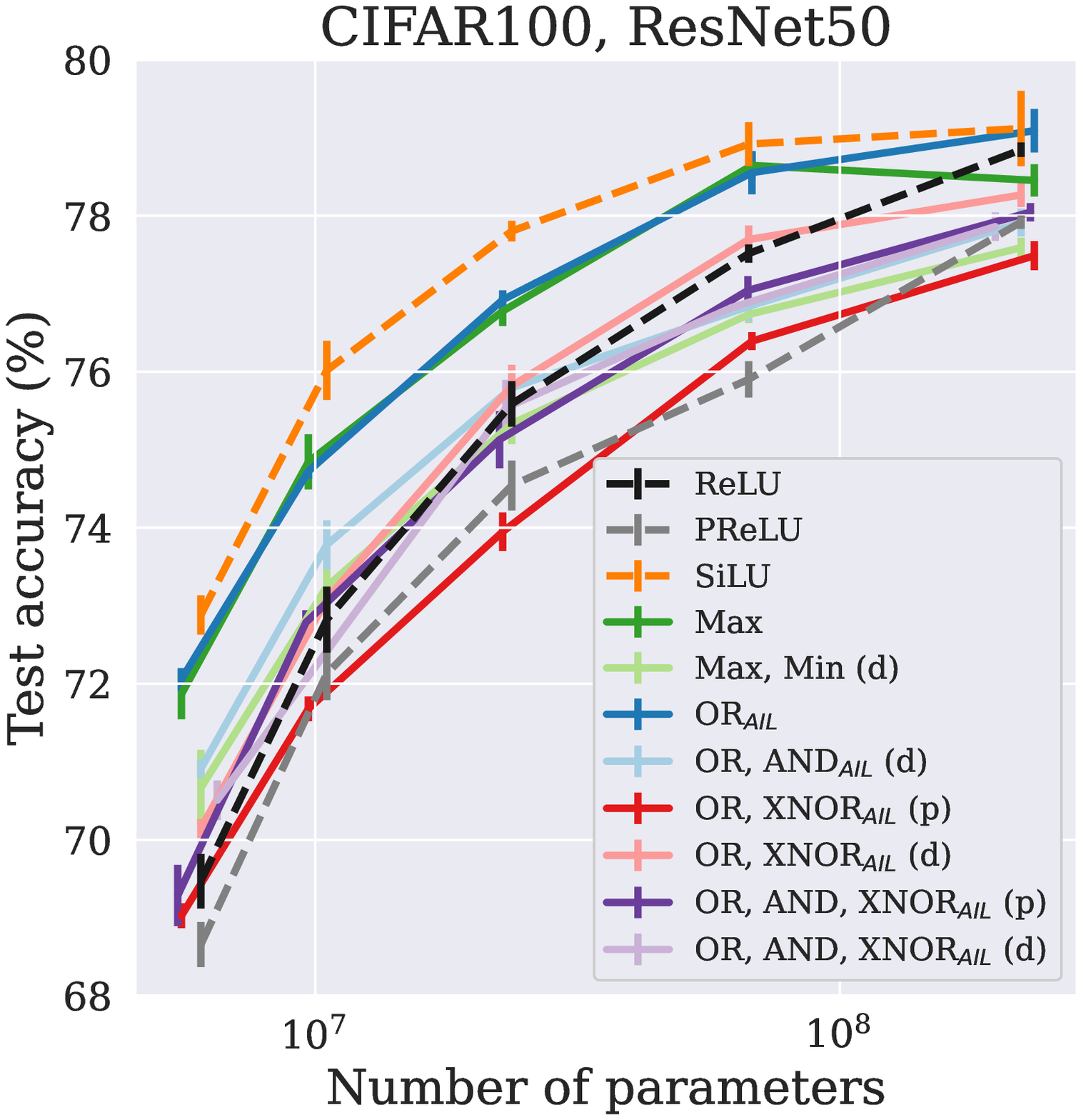}
    \end{subfigure}%
\caption{
ResNet50 on CIFAR-10/100, varying the activation function used through the network.
The width was varied to explore a range of network sizes (see text).
Trained for 100~ep. with ADAM, using hyperparams as determined by random search on CIFAR-100 with width factor $w=2$.
Mean (bars: std dev) of $n\!=\!4$ weight inits.
}
\label{fig:cifar-extra}
\end{figure}

Our abridged results are plotted in \autoref{fig:cifar}, plotted in full with results for more activation functions shown in \autoref{fig:cifar-extra}.
Performance measurements from five slices of total parameter count are tabulated in \autoref{tab:resnet50-cifar10} and \autoref{tab:resnet50-cifar100}.

\begin{table}[tbhp]
\small
  \centering
  \caption{%
ResNet50 on CIFAR-10, by number of parameters.
Mean (std. error) of $n\!=\!4$ random initializations.
Bold: best.
Underlined: top two.
Italic: no sig. diff. from best (two-sided Student's $t$-test, $p\!>\!0.05$).
Background: linear color scale from second-lowest (white) to best (black) with a given number of parameters.
}
\label{tab:resnet50-cifar10}
\def\ccAmin{91.48749709129333}
\def\ccAmax{93.12749803066254}
\def\ccA[#1]#2{\heatmapcell{#1}{#2}{\ccAmin}{\ccAmax}}
\def\ccBmin{92.80749708414078}
\def\ccBmax{94.38749701247559}
\def\ccB[#1]#2{\heatmapcell{#1}{#2}{\ccBmin}{\ccBmax}}
\def\ccCmin{93.83999705314636}
\def\ccCmax{95.44749856573867}
\def\ccC[#1]#2{\heatmapcell{#1}{#2}{\ccCmin}{\ccCmax}}
\def\ccDmin{94.83999758958817}
\def\ccDmax{95.76749801635742}
\def\ccD[#1]#2{\heatmapcell{#1}{#2}{\ccDmin}{\ccDmax}}
\def\ccEmin{95.35333116849264}
\def\ccEmax{95.77249735593796}
\def\ccE[#1]#2{\heatmapcell{#1}{#2}{\ccEmin}{\ccEmax}}
\centerline{
\begin{tabular}{lrrrrr}
\toprule
 & \multicolumn{5}{c}{Test Accuracy (\%) by \textnumero{} Params.} \\
\cmidrule(l){2-6}
Activation function                                     & \multicolumn{1}{c}{$\sim\! 6\,\text{M}$} & \multicolumn{1}{c}{$\sim\! 10\,\text{M}$} & \multicolumn{1}{c}{$\sim\! 23\,\text{M}$} & \multicolumn{1}{c}{$\sim\! 67\,\text{M}$} & \multicolumn{1}{c}{$\sim\! 225\,\text{M}$} \\
\toprule
$\opn{ReLU}$                                             & \ccA[91.9650]{$91.96\wpm{ 0.09}$} & \ccB[93.2600]{$93.26\wpm{ 0.12}$} & \ccC[94.6575]{$94.66\wpm{ 0.09}$} & \ccD[95.3750]{$95.37\wpm{ 0.02}$} & \ccE[95.7025]{$\mbs{\mbns{95.70}}\wpm{ 0.06}$} \\
PReLU                                                    & \ccA[91.7100]{$91.71\wpm{ 0.04}$} & \ccB[93.0825]{$93.08\wpm{ 0.03}$} & \ccC[94.2625]{$94.26\wpm{ 0.05}$} & \ccD[95.1200]{$95.12\wpm{ 0.12}$} & \ccE[95.3900]{$95.39\wpm{ 0.05}$} \\
\midrule
SiLU                                                     & \ccA[93.1275]{$\mbf{\mbns{93.13}}\wpm{ 0.12}$} & \ccB[94.3875]{$\mbf{\mbns{94.39}}\wpm{ 0.04}$} & \ccC[95.4475]{$\mbf{\mbns{95.45}}\wpm{ 0.06}$} & \ccD[95.7675]{$\mbf{\mbns{95.77}}\wpm{ 0.08}$} & \ccE[95.7725]{$\mbf{\mbns{95.77}}\wpm{ 0.09}$} \\
\midrule
$\opn{Max}$                                              & \ccA[92.9050]{$\mbs{\mbns{92.90}}\wpm{ 0.07}$} & \ccB[93.9925]{$93.99\wpm{ 0.09}$} & \ccC[95.2750]{$\mbs{\mbns{95.27}}\wpm{ 0.07}$} & \ccD[95.5700]{$\mbns{95.57}\wpm{ 0.04}$} & \ccE[95.5650]{$\mbns{95.56}\wpm{ 0.07}$} \\
$\opn{Max},\opn{Min}$ (d)                                & \ccA[91.9050]{$91.90\wpm{ 0.12}$} & \ccB[93.0580]{$93.06\wpm{ 0.05}$} & \ccC[94.0050]{$94.00\wpm{ 0.02}$} & \ccD[94.8400]{$94.84\wpm{ 0.04}$} & \ccE[95.3533]{$95.35\wpm{ 0.06}$} \\
\midrule
$\opn{XNOR_{AIL}}$                                       & \ccA[88.9700]{$88.97\wpm{ 0.11}$} & \ccB[88.9325]{$88.93\wpm{ 0.20}$} & \ccC[90.0175]{$90.02\wpm{ 0.29}$} & \ccD[90.7725]{$90.77\wpm{ 0.53}$} & \ccE[89.9650]{$89.96\wpm{ 0.34}$} \\
$\opn{OR_{AIL}}$                                         & \ccA[92.7675]{$\mbns{92.77}\wpm{ 0.05}$} & \ccB[94.1125]{$\mbs{\mbns{94.11}}\wpm{ 0.12}$} & \ccC[95.1775]{$95.18\wpm{ 0.05}$} & \ccD[95.6825]{$\mbs{\mbns{95.68}}\wpm{ 0.06}$} & \ccE[95.6800]{$\mbns{95.68}\wpm{ 0.10}$} \\
$\opn{OR}, \opn{AND_{AIL}}$ (d)                          & \ccA[92.2975]{$92.30\wpm{ 0.05}$} & \ccB[93.6125]{$93.61\wpm{ 0.10}$} & \ccC[94.7275]{$94.73\wpm{ 0.08}$} & \ccD[95.3475]{$95.35\wpm{ 0.07}$} & \ccE[95.6800]{$\mbns{95.68}\wpm{ 0.05}$} \\
$\opn{OR}, \opn{XNOR_{AIL}}$ (p)                         & \ccA[91.5925]{$91.59\wpm{ 0.13}$} & \ccB[92.8275]{$92.83\wpm{ 0.05}$} & \ccC[93.8400]{$93.84\wpm{ 0.13}$} & \ccD[94.8975]{$94.90\wpm{ 0.08}$} & \ccE[95.4600]{$95.46\wpm{ 0.05}$} \\
$\opn{OR}, \opn{XNOR_{AIL}}$ (d)                         & \ccA[92.0000]{$92.00\wpm{ 0.12}$} & \ccB[93.2150]{$93.21\wpm{ 0.04}$} & \ccC[94.2900]{$94.29\wpm{ 0.02}$} & \ccD[95.1925]{$95.19\wpm{ 0.04}$} & \ccE[95.6025]{$\mbns{95.60}\wpm{ 0.03}$} \\
$\opn{OR}, \opn{AND}, \opn{XNOR_{AIL}}$ (p)              & \ccA[91.4875]{$91.49\wpm{ 0.07}$} & \ccB[92.8075]{$92.81\wpm{ 0.13}$} & \ccC[93.9400]{$93.94\wpm{ 0.10}$} & \ccD[95.0650]{$95.06\wpm{ 0.10}$} & \ccE[95.5575]{$\mbns{95.56}\wpm{ 0.08}$} \\
$\opn{OR}, \opn{AND}, \opn{XNOR_{AIL}}$ (d)              & \ccA[91.9275]{$91.93\wpm{ 0.08}$} & \ccB[92.9100]{$92.91\wpm{ 0.07}$} & \ccC[93.9975]{$94.00\wpm{ 0.11}$} & \ccD[94.8400]{$94.84\wpm{ 0.02}$} & \ccE[95.4375]{$95.44\wpm{ 0.05}$} \\
\bottomrule
\end{tabular}
}
\end{table}

\begin{table}[tbhp]
\small
  \centering
  \caption{%
ResNet50 on CIFAR-100, by number of parameters.
Mean (std. error) of $n\!=\!4$ random initializations.
Bold: best.
Underlined: top two.
Italic: no sig. diff. from best (two-sided Student's $t$-test, $p\!>\!0.05$).
Background: linear color scale from second-lowest (white) to best (black) with a given number of parameters.
}
\label{tab:resnet50-cifar100}
\def\ccAmin{68.65749955177307}
\def\ccAmax{72.88249731063843}
\def\ccA[#1]#2{\heatmapcell{#1}{#2}{\ccAmin}{\ccAmax}}
\def\ccBmin{71.67999893426895}
\def\ccBmax{76.01749747991562}
\def\ccB[#1]#2{\heatmapcell{#1}{#2}{\ccBmin}{\ccBmax}}
\def\ccCmin{73.94999712705612}
\def\ccCmax{77.80249863862991}
\def\ccC[#1]#2{\heatmapcell{#1}{#2}{\ccCmin}{\ccCmax}}
\def\ccDmin{75.9024977684021}
\def\ccDmax{78.91749888658524}
\def\ccD[#1]#2{\heatmapcell{#1}{#2}{\ccDmin}{\ccDmax}}
\def\ccEmin{77.48749703168869}
\def\ccEmax{79.11749631166458}
\def\ccE[#1]#2{\heatmapcell{#1}{#2}{\ccEmin}{\ccEmax}}
\centerline{
\begin{tabular}{lrrrrr}
\toprule
 & \multicolumn{5}{c}{Test Accuracy (\%) by \textnumero{} Params.} \\
\cmidrule(l){2-6}
Activation function                                     & \multicolumn{1}{c}{$\sim\! 6\,\text{M}$} & \multicolumn{1}{c}{$\sim\! 10\,\text{M}$} & \multicolumn{1}{c}{$\sim\! 23\,\text{M}$} & \multicolumn{1}{c}{$\sim\! 67\,\text{M}$} & \multicolumn{1}{c}{$\sim\! 225\,\text{M}$} \\
\toprule
$\opn{ReLU}$                                             & \ccA[69.4700]{$69.47\wpm{ 0.17}$} & \ccB[72.8225]{$72.82\wpm{ 0.21}$} & \ccC[75.5875]{$75.59\wpm{ 0.15}$} & \ccD[77.5150]{$77.51\wpm{ 0.06}$} & \ccE[78.8450]{$\mbns{78.84}\wpm{ 0.05}$} \\
PReLU                                                    & \ccA[68.6575]{$68.66\wpm{ 0.14}$} & \ccB[72.1300]{$72.13\wpm{ 0.17}$} & \ccC[74.5420]{$74.54\wpm{ 0.14}$} & \ccD[75.9025]{$75.90\wpm{ 0.12}$} & \ccE[77.9175]{$77.92\wpm{ 0.04}$} \\
\midrule
SiLU                                                     & \ccA[72.8825]{$\mbf{\mbns{72.88}}\wpm{ 0.13}$} & \ccB[76.0175]{$\mbf{\mbns{76.02}}\wpm{ 0.19}$} & \ccC[77.8025]{$\mbf{\mbns{77.80}}\wpm{ 0.07}$} & \ccD[78.9175]{$\mbf{\mbns{78.92}}\wpm{ 0.14}$} & \ccE[79.1175]{$\mbf{\mbns{79.12}}\wpm{ 0.24}$} \\
\midrule
$\opn{Max}$                                              & \ccA[71.8675]{$71.87\wpm{ 0.16}$} & \ccB[74.8450]{$\mbs{74.84}\wpm{ 0.18}$} & \ccC[76.7775]{$76.78\wpm{ 0.09}$} & \ccD[78.6450]{$\mbs{\mbns{78.64}}\wpm{ 0.09}$} & \ccE[78.4525]{$\mbns{78.45}\wpm{ 0.10}$} \\
$\opn{Max},\opn{Min}$ (d)                                & \ccA[70.6650]{$70.66\wpm{ 0.24}$} & \ccB[73.2450]{$73.24\wpm{ 0.12}$} & \ccC[75.3340]{$75.33\wpm{ 0.12}$} & \ccD[76.7340]{$76.73\wpm{ 0.05}$} & \ccE[77.5900]{$77.59\wpm{ 0.06}$} \\
\midrule
$\opn{XNOR_{AIL}}$                                       & \ccA[62.0825]{$62.08\wpm{ 0.35}$} & \ccB[63.7750]{$63.77\wpm{ 0.35}$} & \ccC[64.3600]{$64.36\wpm{ 0.31}$} & \ccD[66.3000]{$66.30\wpm{ 0.24}$} & \ccE[68.6050]{$68.60\wpm{ 0.31}$} \\
$\opn{OR_{AIL}}$                                         & \ccA[72.0575]{$\mbs{72.06}\wpm{ 0.07}$} & \ccB[74.7250]{$74.72\wpm{ 0.05}$} & \ccC[76.9225]{$\mbs{76.92}\wpm{ 0.06}$} & \ccD[78.5525]{$\mbns{78.55}\wpm{ 0.12}$} & \ccE[79.0900]{$\mbs{\mbns{79.09}}\wpm{ 0.12}$} \\
$\opn{OR}, \opn{AND_{AIL}}$ (d)                          & \ccA[70.9125]{$70.91\wpm{ 0.06}$} & \ccB[73.7820]{$73.78\wpm{ 0.14}$} & \ccC[75.7940]{$75.79\wpm{ 0.10}$} & \ccD[76.8380]{$76.84\wpm{ 0.09}$} & \ccE[77.9120]{$77.91\wpm{ 0.08}$} \\
$\opn{OR}, \opn{XNOR_{AIL}}$ (p)                         & \ccA[69.0275]{$69.03\wpm{ 0.08}$} & \ccB[71.6800]{$71.68\wpm{ 0.08}$} & \ccC[73.9500]{$73.95\wpm{ 0.12}$} & \ccD[76.3925]{$76.39\wpm{ 0.06}$} & \ccE[77.4875]{$77.49\wpm{ 0.09}$} \\
$\opn{OR}, \opn{XNOR_{AIL}}$ (d)                         & \ccA[70.1450]{$70.14\wpm{ 0.06}$} & \ccB[73.0860]{$73.09\wpm{ 0.06}$} & \ccC[75.8180]{$75.82\wpm{ 0.12}$} & \ccD[77.6950]{$77.69\wpm{ 0.09}$} & \ccE[78.2660]{$78.27\wpm{ 0.07}$} \\
$\opn{OR}, \opn{AND}, \opn{XNOR_{AIL}}$ (p)              & \ccA[69.2875]{$69.29\wpm{ 0.20}$} & \ccB[72.7875]{$72.79\wpm{ 0.08}$} & \ccC[75.1300]{$75.13\wpm{ 0.18}$} & \ccD[77.0425]{$77.04\wpm{ 0.09}$} & \ccE[78.0450]{$78.04\wpm{ 0.06}$} \\
$\opn{OR}, \opn{AND}, \opn{XNOR_{AIL}}$ (d)              & \ccA[70.5075]{$70.51\wpm{ 0.13}$} & \ccB[73.1720]{$73.17\wpm{ 0.06}$} & \ccC[75.5378]{$75.54\wpm{ 0.15}$} & \ccD[76.8233]{$76.82\wpm{ 0.05}$} & \ccE[77.8600]{$77.86\wpm{ 0.11}$} \\
\bottomrule
\end{tabular}
}
\end{table}

\FloatBarrier

\subsection{Transfer learning}
\label{a:transfer}

We used a pretrained ResNet50 model taken from the PyTorch hub.
The 2-layer MLP head was trained for 25 epochs on each dataset.
Since we are interested in transfer learning in a data-limited regime for these experiments, we used only a small amount of data augmentation.
We applied horizontal flip ($p\!=\!0.5$), scaling (0.7 to 1), aspect ratio stretching ($\nicefrac{3}{4}$ to $\nicefrac{4}{3}$), and colour jitter (intensity 0.4) only.
For the SVHN dataset, horizontal flip was not performed.
Pixel intensity normalization was done against the ImageNet mean and standard deviation.

The MLP head was optimized using SGD, momentum 0.9.
We used a batch size of 128, maximum learning rate 0.01, and weight decay \num{1e-4}.
Before training began, we passed one epoch worth of inputs through the network without updating the weights in order to refresh the batch normalization statistics to the new dataset.
We also performed one epoch of training with a warmup learning rate of \num{1e-5} before commencing training at the maximum learning rate of 0.01.
The learning rate was decayed with a cosine annealing schedule over 24 epochs.
We report the performance of the final model at the end of the 25 epochs.

We used a pre-activation width of 512 neurons for ReLU and other and activation functions which map $1\!\to \!1$.
To approximately match up the number of parameters, we used a width of 650 for activation functions which map $2\!\to \!1$, and 438 for \{$\opn{OR}, \opn{AND}, \opn{XNOR_{AIL}}$ (d)\} which maps features $3\!\to \!2$.
These values control the total number of parameters in the head to be the same for datasets with 100 classes (the median number of classes in the datasets we considered).
Our results with widths 438/512/650 and $\sim$600k parameters are shown in the main results, \autoref{tab:transfer-head-big}.
To investigate a comprehensive set of baselines, we compared against every activation function implemented in PyTorch~v1.10.
This was too many comparisons to include in the main results, so for brevity we narrowed the activation functions down by omitting some activation functions where there were several similar functions and kept the better performing one (SiLU vs Mish, LeakyReLU vs PReLU, etc).
The full set of results is shown in \autoref{tab:transfer-extra}.
Additionally, the precise widths and number of trainable parameters used for these experiments are shown \autoref{tab:transfer-nparam-big}.

\begin{table}[htb]
\small
  \centering
  \caption{%
Transfer learning from a frozen ResNet-18 architecture pretrained on ImageNet-1k to other computer vision datasets.
Mean (std. error) of $n\!=\!5$ random initializations of the MLP (same pretrained encoder).
Bold: best.
Underlined: top two.
Italic: no sig. diff. from best (two-sided Student's $t$-test, $p\!>\!0.05$).
Background: linear color scale from ReLU baseline (white) to best (black).
}
\label{tab:transfer-extra}
\def\ccAmin{86.5801886792453}
\def\ccAmax{88.34905660377358}
\def\ccA[#1]#2{\heatmapcell{#1}{#2}{\ccAmin}{\ccAmax}}
\def\ccBmin{81.63399999999999}
\def\ccBmax{82.88000000000001}
\def\ccB[#1]#2{\heatmapcell{#1}{#2}{\ccBmin}{\ccBmax}}
\def\ccCmin{58.044000000000004}
\def\ccCmax{60.778}
\def\ccC[#1]#2{\heatmapcell{#1}{#2}{\ccCmin}{\ccCmax}}
\def\ccDmin{90.70731696896436}
\def\ccDmax{93.25785683392981}
\def\ccD[#1]#2{\heatmapcell{#1}{#2}{\ccDmin}{\ccDmax}}
\def\ccEmin{30.969666852908368}
\def\ccEmax{39.835902544398145}
\def\ccE[#1]#2{\heatmapcell{#1}{#2}{\ccEmin}{\ccEmax}}
\def\ccFmin{94.62249999999999}
\def\ccFmax{94.94500000000001}
\def\ccF[#1]#2{\heatmapcell{#1}{#2}{\ccFmin}{\ccFmax}}
\def\ccGmin{53.25829748424489}
\def\ccGmax{55.342655196890334}
\def\ccG[#1]#2{\heatmapcell{#1}{#2}{\ccGmin}{\ccGmax}}
\centerline{
  \scalebox{0.835}{
\begin{tabular}{lrrrrrrr}
\toprule
 & \multicolumn{7}{c}{Test Accuracy (\%)} \\
\cmidrule(l){2-8}
Activation function                                     & \multicolumn{1}{c}{Cal101} & \multicolumn{1}{c}{CIFAR10} & \multicolumn{1}{c}{CIFAR100} & \multicolumn{1}{c}{Flowers} & \multicolumn{1}{c}{StfCars} & \multicolumn{1}{c}{STL-10} & \multicolumn{1}{c}{SVHN} \\
\toprule
Linear layer only                                        & \ccA[88.3491]{$\mbf{\mbns{88.35}}\wpm{ 0.15}$} & \ccB[78.5640]{$78.56\wpm{ 0.09}$} & \ccC[57.3920]{$57.39\wpm{ 0.09}$} & \ccD[92.3171]{$92.32\wpm{ 0.20}$} & \ccE[33.5132]{$33.51\wpm{ 0.06}$} & \ccF[94.6825]{$94.68\wpm{ 0.02}$} & \ccG[46.6011]{$46.60\wpm{ 0.14}$} \\
\midrule
$\opn{ReLU}$                                             & \ccA[86.5802]{$86.58\wpm{ 0.17}$} & \ccB[81.6340]{$81.63\wpm{ 0.05}$} & \ccC[58.0440]{$58.04\wpm{ 0.11}$} & \ccD[90.7073]{$90.71\wpm{ 0.26}$} & \ccE[30.9697]{$30.97\wpm{ 0.26}$} & \ccF[94.6225]{$94.62\wpm{ 0.06}$} & \ccG[53.2583]{$53.26\wpm{ 0.08}$} \\
LeakyReLU                                                & \ccA[86.6038]{$86.60\wpm{ 0.13}$} & \ccB[81.6660]{$81.67\wpm{ 0.11}$} & \ccC[58.0080]{$58.01\wpm{ 0.09}$} & \ccD[90.7317]{$90.73\wpm{ 0.32}$} & \ccE[31.0890]{$31.09\wpm{ 0.24}$} & \ccF[94.6150]{$94.61\wpm{ 0.05}$} & \ccG[53.2383]{$53.24\wpm{ 0.10}$} \\
PReLU                                                    & \ccA[87.8302]{$\mbns{87.83}\wpm{ 0.21}$} & \ccB[81.0300]{$81.03\wpm{ 0.13}$} & \ccC[58.8980]{$58.90\wpm{ 0.18}$} & \ccD[93.1707]{$\mbs{\mbns{93.17}}\wpm{ 0.19}$} & \ccE[39.8359]{$\mbf{\mbns{39.84}}\wpm{ 0.18}$} & \ccF[94.5375]{$94.54\wpm{ 0.05}$} & \ccG[53.4750]{$53.47\wpm{ 0.08}$} \\
Softplus                                                 & \ccA[86.1557]{$86.16\wpm{ 0.18}$} & \ccB[79.1300]{$79.13\wpm{ 0.08}$} & \ccC[56.5840]{$56.58\wpm{ 0.07}$} & \ccD[89.3902]{$89.39\wpm{ 0.29}$} & \ccE[21.2332]{$21.23\wpm{ 0.13}$} & \ccF[94.6300]{$94.63\wpm{ 0.03}$} & \ccG[48.7892]{$48.79\wpm{ 0.08}$} \\
\midrule
ELU                                                      & \ccA[87.1816]{$87.18\wpm{ 0.09}$} & \ccB[80.4400]{$80.44\wpm{ 0.08}$} & \ccC[58.0840]{$58.08\wpm{ 0.10}$} & \ccD[91.7073]{$91.71\wpm{ 0.14}$} & \ccE[34.7041]{$34.70\wpm{ 0.06}$} & \ccF[94.5500]{$94.55\wpm{ 0.05}$} & \ccG[51.8869]{$51.89\wpm{ 0.09}$} \\
CELU                                                     & \ccA[87.1816]{$87.18\wpm{ 0.09}$} & \ccB[80.4400]{$80.44\wpm{ 0.08}$} & \ccC[58.0840]{$58.08\wpm{ 0.10}$} & \ccD[91.7073]{$91.71\wpm{ 0.14}$} & \ccE[34.7041]{$34.70\wpm{ 0.06}$} & \ccF[94.5500]{$94.55\wpm{ 0.05}$} & \ccG[51.8869]{$51.89\wpm{ 0.09}$} \\
SELU                                                     & \ccA[87.7358]{$87.74\wpm{ 0.09}$} & \ccB[79.9340]{$79.93\wpm{ 0.13}$} & \ccC[58.2380]{$58.24\wpm{ 0.06}$} & \ccD[92.2683]{$92.27\wpm{ 0.13}$} & \ccE[37.5112]{$37.51\wpm{ 0.17}$} & \ccF[94.5300]{$94.53\wpm{ 0.07}$} & \ccG[50.9365]{$50.94\wpm{ 0.12}$} \\
GELU                                                     & \ccA[87.0991]{$87.10\wpm{ 0.15}$} & \ccB[81.3860]{$81.39\wpm{ 0.09}$} & \ccC[58.5060]{$58.51\wpm{ 0.13}$} & \ccD[91.5122]{$91.51\wpm{ 0.15}$} & \ccE[33.4286]{$33.43\wpm{ 0.15}$} & \ccF[94.6250]{$94.62\wpm{ 0.06}$} & \ccG[53.4250]{$53.43\wpm{ 0.23}$} \\
SiLU                                                     & \ccA[86.9104]{$86.91\wpm{ 0.11}$} & \ccB[80.5280]{$80.53\wpm{ 0.11}$} & \ccC[58.1440]{$58.14\wpm{ 0.12}$} & \ccD[91.3659]{$91.37\wpm{ 0.18}$} & \ccE[32.1482]{$32.15\wpm{ 0.17}$} & \ccF[94.5875]{$94.59\wpm{ 0.05}$} & \ccG[52.3548]{$52.35\wpm{ 0.16}$} \\
Hardswish                                                & \ccA[87.1226]{$87.12\wpm{ 0.12}$} & \ccB[80.1040]{$80.10\wpm{ 0.10}$} & \ccC[58.2540]{$58.25\wpm{ 0.10}$} & \ccD[91.5610]{$91.56\wpm{ 0.25}$} & \ccE[33.1726]{$33.17\wpm{ 0.23}$} & \ccF[94.6225]{$94.62\wpm{ 0.05}$} & \ccG[51.8231]{$51.82\wpm{ 0.13}$} \\
Mish                                                     & \ccA[87.1108]{$87.11\wpm{ 0.12}$} & \ccB[81.0880]{$81.09\wpm{ 0.11}$} & \ccC[58.3720]{$58.37\wpm{ 0.10}$} & \ccD[91.6098]{$91.61\wpm{ 0.15}$} & \ccE[33.7469]{$33.75\wpm{ 0.14}$} & \ccF[94.6075]{$94.61\wpm{ 0.05}$} & \ccG[53.0539]{$53.05\wpm{ 0.12}$} \\
\midrule
Softsign                                                 & \ccA[81.4741]{$81.47\wpm{ 0.18}$} & \ccB[80.0340]{$80.03\wpm{ 0.09}$} & \ccC[54.8420]{$54.84\wpm{ 0.09}$} & \ccD[82.3415]{$82.34\wpm{ 0.22}$} & \ccE[17.3322]{$17.33\wpm{ 0.10}$} & \ccF[94.7050]{$94.70\wpm{ 0.03}$} & \ccG[51.2500]{$51.25\wpm{ 0.08}$} \\
Tanh                                                     & \ccA[87.4764]{$87.48\wpm{ 0.06}$} & \ccB[80.5560]{$80.56\wpm{ 0.07}$} & \ccC[57.3480]{$57.35\wpm{ 0.08}$} & \ccD[90.3171]{$90.32\wpm{ 0.20}$} & \ccE[29.5077]{$29.51\wpm{ 0.12}$} & \ccF[94.6275]{$94.63\wpm{ 0.07}$} & \ccG[51.8600]{$51.86\wpm{ 0.05}$} \\
\midrule
GLU                                                      & \ccA[86.7099]{$86.71\wpm{ 0.31}$} & \ccB[79.1920]{$79.19\wpm{ 0.07}$} & \ccC[57.6380]{$57.64\wpm{ 0.10}$} & \ccD[90.3415]{$90.34\wpm{ 0.19}$} & \ccE[27.0363]{$27.04\wpm{ 0.12}$} & \ccF[94.5700]{$94.57\wpm{ 0.03}$} & \ccG[50.1160]{$50.12\wpm{ 0.19}$} \\
\midrule
$\opn{Max}$                                              & \ccA[86.9575]{$86.96\wpm{ 0.20}$} & \ccB[81.7560]{$81.76\wpm{ 0.14}$} & \ccC[58.5980]{$58.60\wpm{ 0.12}$} & \ccD[90.9756]{$90.98\wpm{ 0.18}$} & \ccE[33.3690]{$33.37\wpm{ 0.15}$} & \ccF[94.7025]{$94.70\wpm{ 0.06}$} & \ccG[53.5257]{$53.53\wpm{ 0.16}$} \\
$\opn{Max},\opn{Min}$ (d)                                & \ccA[87.2288]{$87.23\wpm{ 0.13}$} & \ccB[82.3060]{$82.31\wpm{ 0.10}$} & \ccC[59.0520]{$59.05\wpm{ 0.10}$} & \ccD[91.6829]{$91.68\wpm{ 0.18}$} & \ccE[34.9105]{$34.91\wpm{ 0.12}$} & \ccF[94.6400]{$94.64\wpm{ 0.04}$} & \ccG[53.9136]{$53.91\wpm{ 0.13}$} \\
\midrule
SignedGeomean                                            & \ccA[87.0324]{$87.03\wpm{ 0.23}$} & \ccB[51.4520]{$\mbns{51.45}\wpm{16.92}$} & \ccC[11.7980]{$11.80\wpm{10.80}$} & \ccD[91.3447]{$91.34\wpm{ 0.34}$} & \ccE[26.3748]{$\mbns{26.37}\wpm{ 6.46}$} & \ccF[94.6825]{$94.68\wpm{ 0.06}$} & \ccG[37.1635]{$37.16\wpm{ 7.18}$} \\
\midrule
$\opn{XNOR_{IL}}$                                        & \ccA[85.0118]{$85.01\wpm{ 0.17}$} & \ccB[79.6180]{$79.62\wpm{ 0.09}$} & \ccC[57.1360]{$57.14\wpm{ 0.07}$} & \ccD[84.7561]{$84.76\wpm{ 0.43}$} & \ccE[ 1.3401]{$ 1.34\wpm{ 0.11}$} & \ccF[94.5125]{$94.51\wpm{ 0.03}$} & \ccG[51.9914]{$51.99\wpm{ 0.16}$} \\
$\opn{OR_{IL}}$                                          & \ccA[87.1108]{$87.11\wpm{ 0.08}$} & \ccB[79.7480]{$79.75\wpm{ 0.05}$} & \ccC[58.0700]{$58.07\wpm{ 0.11}$} & \ccD[91.1220]{$91.12\wpm{ 0.36}$} & \ccE[33.1179]{$33.12\wpm{ 0.12}$} & \ccF[94.6000]{$94.60\wpm{ 0.03}$} & \ccG[51.2077]{$51.21\wpm{ 0.17}$} \\
$\opn{OR}, \opn{AND_{IL}}$ (d)                           & \ccA[87.2170]{$87.22\wpm{ 0.14}$} & \ccB[79.5620]{$79.56\wpm{ 0.06}$} & \ccC[57.9100]{$57.91\wpm{ 0.12}$} & \ccD[91.4390]{$91.44\wpm{ 0.18}$} & \ccE[36.0443]{$36.04\wpm{ 0.07}$} & \ccF[94.4400]{$94.44\wpm{ 0.05}$} & \ccG[50.4172]{$50.42\wpm{ 0.25}$} \\
$\opn{OR}, \opn{XNOR_{IL}}$ (p)                          & \ccA[86.3127]{$86.31\wpm{ 0.24}$} & \ccB[80.2860]{$80.29\wpm{ 0.10}$} & \ccC[58.1340]{$58.13\wpm{ 0.08}$} & \ccD[90.5134]{$90.51\wpm{ 0.20}$} & \ccE[31.4140]{$31.41\wpm{ 0.11}$} & \ccF[94.7625]{$94.76\wpm{ 0.02}$} & \ccG[52.7774]{$52.78\wpm{ 0.13}$} \\
$\opn{OR}, \opn{XNOR_{IL}}$ (d)                          & \ccA[86.7571]{$86.76\wpm{ 0.12}$} & \ccB[80.8260]{$80.83\wpm{ 0.08}$} & \ccC[58.5480]{$58.55\wpm{ 0.10}$} & \ccD[90.7561]{$90.76\wpm{ 0.07}$} & \ccE[33.5579]{$33.56\wpm{ 0.14}$} & \ccF[94.4700]{$94.47\wpm{ 0.07}$} & \ccG[52.5415]{$52.54\wpm{ 0.09}$} \\
$\opn{OR}, \opn{AND}, \opn{XNOR_{IL}}$ (p)               & \ccA[86.6077]{$86.61\wpm{ 0.10}$} & \ccB[80.3260]{$80.33\wpm{ 0.06}$} & \ccC[58.1240]{$58.12\wpm{ 0.04}$} & \ccD[90.8802]{$90.88\wpm{ 0.31}$} & \ccE[32.4039]{$32.40\wpm{ 0.12}$} & \ccF[94.7375]{$94.74\wpm{ 0.04}$} & \ccG[52.6905]{$52.69\wpm{ 0.10}$} \\
$\opn{OR}, \opn{AND}, \opn{XNOR_{IL}}$ (d)               & \ccA[86.7807]{$86.78\wpm{ 0.10}$} & \ccB[80.1760]{$80.18\wpm{ 0.02}$} & \ccC[58.7020]{$58.70\wpm{ 0.17}$} & \ccD[91.6829]{$91.68\wpm{ 0.07}$} & \ccE[34.8160]{$34.82\wpm{ 0.14}$} & \ccF[94.4875]{$94.49\wpm{ 0.04}$} & \ccG[52.0874]{$52.09\wpm{ 0.18}$} \\
\midrule
$\opn{XNOR_{NIL}}$                                       & \ccA[87.2477]{$87.25\wpm{ 0.22}$} & \ccB[82.8800]{$\mbf{\mbns{82.88}}\wpm{ 0.08}$} & \ccC[60.7780]{$\mbf{\mbns{60.78}}\wpm{ 0.08}$} & \ccD[93.2579]{$\mbf{\mbns{93.26}}\wpm{ 0.26}$} & \ccE[39.4747]{$\mbns{39.47}\wpm{ 0.20}$} & \ccF[94.8300]{$\mbns{94.83}\wpm{ 0.06}$} & \ccG[55.3427]{$\mbf{\mbns{55.34}}\wpm{ 0.19}$} \\
$\opn{OR_{NIL}}$                                         & \ccA[87.1873]{$87.19\wpm{ 0.16}$} & \ccB[79.6100]{$79.61\wpm{ 0.05}$} & \ccC[58.4380]{$58.44\wpm{ 0.10}$} & \ccD[91.6458]{$91.65\wpm{ 0.29}$} & \ccE[35.8205]{$35.82\wpm{ 0.04}$} & \ccF[94.5775]{$94.58\wpm{ 0.03}$} & \ccG[50.9488]{$50.95\wpm{ 0.15}$} \\
$\opn{OR}, \opn{AND_{NIL}}$ (d)                          & \ccA[86.8201]{$86.82\wpm{ 0.18}$} & \ccB[80.0900]{$80.09\wpm{ 0.09}$} & \ccC[58.5980]{$58.60\wpm{ 0.07}$} & \ccD[91.4425]{$91.44\wpm{ 0.20}$} & \ccE[37.0252]{$37.03\wpm{ 0.11}$} & \ccF[94.6500]{$94.65\wpm{ 0.05}$} & \ccG[52.4915]{$52.49\wpm{ 0.09}$} \\
$\opn{OR}, \opn{XNOR_{NIL}}$ (p)                         & \ccA[87.6460]{$87.65\wpm{ 0.21}$} & \ccB[82.6700]{$\mbns{82.67}\wpm{ 0.08}$} & \ccC[60.2380]{$60.24\wpm{ 0.18}$} & \ccD[92.5183]{$\mbns{92.52}\wpm{ 0.20}$} & \ccE[38.8559]{$38.86\wpm{ 0.29}$} & \ccF[94.7825]{$94.78\wpm{ 0.02}$} & \ccG[55.2482]{$\mbs{\mbns{55.25}}\wpm{ 0.11}$} \\
$\opn{OR}, \opn{XNOR_{NIL}}$ (d)                         & \ccA[87.8184]{$\mbns{87.82}\wpm{ 0.19}$} & \ccB[82.6720]{$\mbns{82.67}\wpm{ 0.05}$} & \ccC[60.6000]{$\mbs{\mbns{60.60}}\wpm{ 0.11}$} & \ccD[92.9268]{$\mbns{92.93}\wpm{ 0.12}$} & \ccE[39.2190]{$39.22\wpm{ 0.17}$} & \ccF[94.6275]{$94.63\wpm{ 0.06}$} & \ccG[54.8748]{$\mbns{54.87}\wpm{ 0.09}$} \\
$\opn{OR}, \opn{AND}, \opn{XNOR_{NIL}}$ (p)              & \ccA[87.5044]{$87.50\wpm{ 0.17}$} & \ccB[82.3300]{$82.33\wpm{ 0.11}$} & \ccC[59.9860]{$59.99\wpm{ 0.06}$} & \ccD[92.6161]{$\mbns{92.62}\wpm{ 0.35}$} & \ccE[38.4131]{$38.41\wpm{ 0.10}$} & \ccF[94.8150]{$94.81\wpm{ 0.03}$} & \ccG[54.9063]{$\mbns{54.91}\wpm{ 0.08}$} \\
$\opn{OR}, \opn{AND}, \opn{XNOR_{NIL}}$ (d)              & \ccA[87.4114]{$87.41\wpm{ 0.27}$} & \ccB[82.8360]{$\mbs{\mbns{82.84}}\wpm{ 0.06}$} & \ccC[60.3760]{$60.38\wpm{ 0.10}$} & \ccD[92.9756]{$\mbns{92.98}\wpm{ 0.17}$} & \ccE[39.4166]{$\mbns{39.42}\wpm{ 0.18}$} & \ccF[94.7125]{$94.71\wpm{ 0.03}$} & \ccG[55.1137]{$\mbns{55.11}\wpm{ 0.12}$} \\
\midrule
$\opn{XNOR_{AIL}}$                                       & \ccA[86.9693]{$86.97\wpm{ 0.18}$} & \ccB[81.8330]{$81.83\wpm{ 0.06}$} & \ccC[58.4610]{$58.46\wpm{ 0.10}$} & \ccD[90.9268]{$90.93\wpm{ 0.15}$} & \ccE[32.5559]{$32.56\wpm{ 0.10}$} & \ccF[94.7050]{$94.71\wpm{ 0.06}$} & \ccG[53.7485]{$53.75\wpm{ 0.14}$} \\
$\opn{OR_{AIL}}$                                         & \ccA[87.4528]{$87.45\wpm{ 0.14}$} & \ccB[81.8800]{$81.88\wpm{ 0.07}$} & \ccC[59.0960]{$59.10\wpm{ 0.09}$} & \ccD[92.0000]{$92.00\wpm{ 0.15}$} & \ccE[36.0094]{$36.01\wpm{ 0.12}$} & \ccF[94.6875]{$94.69\wpm{ 0.04}$} & \ccG[53.6809]{$53.68\wpm{ 0.14}$} \\
$\opn{OR}, \opn{AND_{AIL}}$ (d)                          & \ccA[87.4292]{$87.43\wpm{ 0.11}$} & \ccB[82.3800]{$82.38\wpm{ 0.06}$} & \ccC[59.8980]{$59.90\wpm{ 0.08}$} & \ccD[92.0732]{$92.07\wpm{ 0.18}$} & \ccE[37.1631]{$37.16\wpm{ 0.15}$} & \ccF[94.5500]{$94.55\wpm{ 0.05}$} & \ccG[54.0542]{$54.05\wpm{ 0.07}$} \\
$\opn{OR}, \opn{XNOR_{AIL}}$ (p)                         & \ccA[87.0560]{$87.06\wpm{ 0.19}$} & \ccB[82.0780]{$82.08\wpm{ 0.03}$} & \ccC[59.2940]{$59.29\wpm{ 0.10}$} & \ccD[91.4670]{$91.47\wpm{ 0.35}$} & \ccE[35.0653]{$35.07\wpm{ 0.18}$} & \ccF[94.8075]{$94.81\wpm{ 0.03}$} & \ccG[54.4092]{$54.41\wpm{ 0.19}$} \\
$\opn{OR}, \opn{XNOR_{AIL}}$ (d)                         & \ccA[87.0873]{$87.09\wpm{ 0.21}$} & \ccB[82.1960]{$82.20\wpm{ 0.04}$} & \ccC[59.4360]{$59.44\wpm{ 0.07}$} & \ccD[91.9024]{$91.90\wpm{ 0.10}$} & \ccE[36.8772]{$36.88\wpm{ 0.10}$} & \ccF[94.6850]{$94.69\wpm{ 0.06}$} & \ccG[54.0258]{$54.03\wpm{ 0.11}$} \\
$\opn{OR}, \opn{AND}, \opn{XNOR_{AIL}}$ (p)              & \ccA[87.2212]{$87.22\wpm{ 0.12}$} & \ccB[82.1020]{$82.10\wpm{ 0.03}$} & \ccC[59.3420]{$59.34\wpm{ 0.13}$} & \ccD[91.7115]{$91.71\wpm{ 0.19}$} & \ccE[35.6523]{$35.65\wpm{ 0.11}$} & \ccF[94.7675]{$94.77\wpm{ 0.05}$} & \ccG[54.3185]{$54.32\wpm{ 0.12}$} \\
$\opn{OR}, \opn{AND}, \opn{XNOR_{AIL}}$ (d)              & \ccA[87.4882]{$87.49\wpm{ 0.11}$} & \ccB[82.4980]{$82.50\wpm{ 0.08}$} & \ccC[59.8320]{$59.83\wpm{ 0.12}$} & \ccD[92.3659]{$92.37\wpm{ 0.08}$} & \ccE[37.5982]{$37.60\wpm{ 0.20}$} & \ccF[94.7150]{$94.72\wpm{ 0.02}$} & \ccG[54.5513]{$54.55\wpm{ 0.09}$} \\
\midrule
$\opn{XNOR_{NAIL}}$                                      & \ccA[87.6135]{$87.61\wpm{ 0.23}$} & \ccB[82.3840]{$82.38\wpm{ 0.07}$} & \ccC[59.7680]{$59.77\wpm{ 0.13}$} & \ccD[93.0732]{$\mbns{93.07}\wpm{ 0.20}$} & \ccE[39.7737]{$\mbs{\mbns{39.77}}\wpm{ 0.04}$} & \ccF[94.8100]{$94.81\wpm{ 0.03}$} & \ccG[53.9136]{$53.91\wpm{ 0.05}$} \\
$\opn{OR_{NAIL}}$                                        & \ccA[87.1934]{$87.19\wpm{ 0.16}$} & \ccB[81.7920]{$81.79\wpm{ 0.09}$} & \ccC[59.4040]{$59.40\wpm{ 0.09}$} & \ccD[92.1220]{$92.12\wpm{ 0.12}$} & \ccE[37.3247]{$37.32\wpm{ 0.17}$} & \ccF[94.6525]{$94.65\wpm{ 0.04}$} & \ccG[53.8199]{$53.82\wpm{ 0.21}$} \\
$\opn{OR}, \opn{AND_{NAIL}}$ (d)                         & \ccA[87.6179]{$87.62\wpm{ 0.11}$} & \ccB[82.2800]{$82.28\wpm{ 0.10}$} & \ccC[59.7140]{$59.71\wpm{ 0.05}$} & \ccD[92.0976]{$92.10\wpm{ 0.20}$} & \ccE[37.6977]{$37.70\wpm{ 0.12}$} & \ccF[94.6125]{$94.61\wpm{ 0.08}$} & \ccG[53.8614]{$53.86\wpm{ 0.10}$} \\
$\opn{OR}, \opn{XNOR_{NAIL}}$ (p)                        & \ccA[87.6106]{$87.61\wpm{ 0.15}$} & \ccB[82.3680]{$82.37\wpm{ 0.11}$} & \ccC[59.9520]{$59.95\wpm{ 0.13}$} & \ccD[92.6650]{$\mbns{92.67}\wpm{ 0.10}$} & \ccE[38.6047]{$38.60\wpm{ 0.08}$} & \ccF[94.8825]{$\mbs{\mbns{94.88}}\wpm{ 0.03}$} & \ccG[54.4852]{$54.49\wpm{ 0.13}$} \\
$\opn{OR}, \opn{XNOR_{NAIL}}$ (d)                        & \ccA[87.8538]{$\mbs{\mbns{87.85}}\wpm{ 0.22}$} & \ccB[82.5240]{$82.52\wpm{ 0.11}$} & \ccC[60.0180]{$60.02\wpm{ 0.10}$} & \ccD[93.1220]{$\mbns{93.12}\wpm{ 0.13}$} & \ccE[39.6370]{$\mbns{39.64}\wpm{ 0.09}$} & \ccF[94.7450]{$94.75\wpm{ 0.03}$} & \ccG[54.1334]{$54.13\wpm{ 0.05}$} \\
$\opn{OR}, \opn{AND}, \opn{XNOR_{NAIL}}$ (p)             & \ccA[87.4808]{$87.48\wpm{ 0.14}$} & \ccB[82.2660]{$82.27\wpm{ 0.08}$} & \ccC[59.8720]{$59.87\wpm{ 0.05}$} & \ccD[92.1027]{$92.10\wpm{ 0.29}$} & \ccE[38.2589]{$38.26\wpm{ 0.12}$} & \ccF[94.9450]{$\mbf{\mbns{94.95}}\wpm{ 0.02}$} & \ccG[54.4199]{$54.42\wpm{ 0.13}$} \\
$\opn{OR}, \opn{AND}, \opn{XNOR_{NAIL}}$ (d)             & \ccA[87.7830]{$87.78\wpm{ 0.14}$} & \ccB[82.6680]{$\mbns{82.67}\wpm{ 0.06}$} & \ccC[60.0100]{$60.01\wpm{ 0.21}$} & \ccD[93.1220]{$\mbns{93.12}\wpm{ 0.21}$} & \ccE[39.6544]{$\mbns{39.65}\wpm{ 0.14}$} & \ccF[94.7750]{$94.78\wpm{ 0.03}$} & \ccG[54.5790]{$54.58\wpm{ 0.12}$} \\
\bottomrule
\end{tabular}
  }
}
\end{table}

\begin{table}[htb]
\small
  \centering
  \caption{%
Hidden pre-activation widths and number of trainable parameters used in transfer learning experiments (corresponding to results shown in \autoref{tab:transfer-head-big}).
}
\label{tab:transfer-nparam-big}
\centerline{
  \scalebox{0.835}{
\begin{tabular}{lrrrrrrrrr}
\toprule
 & & & \multicolumn{7}{c}{\textnumero{} Trainable Parameters} \\
\cmidrule(l){4-10}
Activation function                                     & Map & Width & \multicolumn{1}{c}{Cal101} & \multicolumn{1}{c}{CIFAR10} & \multicolumn{1}{c}{CIFAR100} & \multicolumn{1}{c}{Flowers} & \multicolumn{1}{c}{StfCars} & \multicolumn{1}{c}{STL-10} & \multicolumn{1}{c}{SVHN} \\
\toprule
Linear layer only                                       & &  & 52k & 5k & 51k & 52k & 101k & 5k & 5k \\
\midrule
$\opn{ReLU}$                                            & $1\!\to\! 1$& 512 & 577k & 530k & 577k & 578k & 626k & 530k & 530k \\
LeakyReLU                                               & $1\!\to\! 1$& 512 & 577k & 530k & 577k & 578k & 626k & 530k & 530k \\
PReLU                                                   & $1\!\to\! 1$& 512 & 577k & 530k & 577k & 578k & 626k & 530k & 530k \\
Softplus                                                & $1\!\to\! 1$& 512 & 577k & 530k & 577k & 578k & 626k & 530k & 530k \\
\midrule
ELU                                                     & $1\!\to\! 1$& 512 & 577k & 530k & 577k & 578k & 626k & 530k & 530k \\
CELU                                                    & $1\!\to\! 1$& 512 & 577k & 530k & 577k & 578k & 626k & 530k & 530k \\
SELU                                                    & $1\!\to\! 1$& 512 & 577k & 530k & 577k & 578k & 626k & 530k & 530k \\
GELU                                                    & $1\!\to\! 1$& 512 & 577k & 530k & 577k & 578k & 626k & 530k & 530k \\
SiLU                                                    & $1\!\to\! 1$& 512 & 577k & 530k & 577k & 578k & 626k & 530k & 530k \\
Hardswish                                               & $1\!\to\! 1$& 512 & 577k & 530k & 577k & 578k & 626k & 530k & 530k \\
Mish                                                    & $1\!\to\! 1$& 512 & 577k & 530k & 577k & 578k & 626k & 530k & 530k \\
\midrule
Softsign                                                & $1\!\to\! 1$& 512 & 577k & 530k & 577k & 578k & 626k & 530k & 530k \\
Tanh                                                    & $1\!\to\! 1$& 512 & 577k & 530k & 577k & 578k & 626k & 530k & 530k \\
\midrule
GLU                                                     & $2\!\to\! 1$& 650 & 578k & 549k & 578k & 579k & 609k & 549k & 549k \\
\midrule
$\opn{Max}$                                             & $2\!\to\! 1$& 650 & 578k & 549k & 578k & 579k & 609k & 549k & 549k \\
$\opn{Max},\opn{Min}$ (d)                               & $2\!\to\! 2$& 512 & 577k & 530k & 577k & 578k & 626k & 530k & 530k \\
\midrule
SignedGeomean                                           & $2\!\to\! 1$& 650 & 578k & 549k & 578k & 579k & 609k & 549k & 549k \\
\midrule
$\opn{XNOR_{IL}}$                                       & $2\!\to\! 1$& 650 & 578k & 549k & 578k & 579k & 609k & 549k & 549k \\
$\opn{OR_{IL}}$                                         & $2\!\to\! 1$& 650 & 578k & 549k & 578k & 579k & 609k & 549k & 549k \\
$\opn{OR}, \opn{AND_{IL}}$ (d)                          & $2\!\to\! 2$& 512 & 577k & 530k & 577k & 578k & 626k & 530k & 530k \\
$\opn{OR}, \opn{XNOR_{IL}}$ (p)                         & $2\!\to\! 1$& 648 & 576k & 546k & 576k & 576k & 607k & 546k & 546k \\
$\opn{OR}, \opn{XNOR_{IL}}$ (d)                         & $2\!\to\! 2$& 512 & 577k & 530k & 577k & 578k & 626k & 530k & 530k \\
$\opn{OR}, \opn{AND}, \opn{XNOR_{IL}}$ (p)              & $2\!\to\! 1$& 648 & 576k & 546k & 576k & 576k & 607k & 546k & 546k \\
$\opn{OR}, \opn{AND}, \opn{XNOR_{IL}}$ (d)              & $2\!\to\! 3$& 438 & 579k & 519k & 579k & 580k & 642k & 519k & 519k \\
\midrule
$\opn{XNOR_{NIL}}$                                      & $2\!\to\! 1$& 650 & 578k & 549k & 578k & 579k & 609k & 549k & 549k \\
$\opn{OR_{NIL}}$                                        & $2\!\to\! 1$& 650 & 578k & 549k & 578k & 579k & 609k & 549k & 549k \\
$\opn{OR}, \opn{AND_{NIL}}$ (d)                         & $2\!\to\! 2$& 512 & 577k & 530k & 577k & 578k & 626k & 530k & 530k \\
$\opn{OR}, \opn{XNOR_{NIL}}$ (p)                        & $2\!\to\! 1$& 648 & 576k & 546k & 576k & 576k & 607k & 546k & 546k \\
$\opn{OR}, \opn{XNOR_{NIL}}$ (d)                        & $2\!\to\! 2$& 512 & 577k & 530k & 577k & 578k & 626k & 530k & 530k \\
$\opn{OR}, \opn{AND}, \opn{XNOR_{NIL}}$ (p)             & $2\!\to\! 1$& 648 & 576k & 546k & 576k & 576k & 607k & 546k & 546k \\
$\opn{OR}, \opn{AND}, \opn{XNOR_{NIL}}$ (d)             & $2\!\to\! 3$& 438 & 579k & 519k & 579k & 580k & 642k & 519k & 519k \\
\midrule
$\opn{XNOR_{AIL}}$                                      & $2\!\to\! 1$& 650 & 578k & 549k & 578k & 579k & 609k & 549k & 549k \\
$\opn{OR_{AIL}}$                                        & $2\!\to\! 1$& 650 & 578k & 549k & 578k & 579k & 609k & 549k & 549k \\
$\opn{OR}, \opn{AND_{AIL}}$ (d)                         & $2\!\to\! 2$& 512 & 577k & 530k & 577k & 578k & 626k & 530k & 530k \\
$\opn{OR}, \opn{XNOR_{AIL}}$ (p)                        & $2\!\to\! 1$& 648 & 576k & 546k & 576k & 576k & 607k & 546k & 546k \\
$\opn{OR}, \opn{XNOR_{AIL}}$ (d)                        & $2\!\to\! 2$& 512 & 577k & 530k & 577k & 578k & 626k & 530k & 530k \\
$\opn{OR}, \opn{AND}, \opn{XNOR_{AIL}}$ (p)             & $2\!\to\! 1$& 648 & 576k & 546k & 576k & 576k & 607k & 546k & 546k \\
$\opn{OR}, \opn{AND}, \opn{XNOR_{AIL}}$ (d)             & $2\!\to\! 3$& 438 & 579k & 519k & 579k & 580k & 642k & 519k & 519k \\
\midrule
$\opn{XNOR_{NAIL}}$                                     & $2\!\to\! 1$& 650 & 578k & 549k & 578k & 579k & 609k & 549k & 549k \\
$\opn{OR_{NAIL}}$                                       & $2\!\to\! 1$& 650 & 578k & 549k & 578k & 579k & 609k & 549k & 549k \\
$\opn{OR}, \opn{AND_{NAIL}}$ (d)                        & $2\!\to\! 2$& 512 & 577k & 530k & 577k & 578k & 626k & 530k & 530k \\
$\opn{OR}, \opn{XNOR_{NAIL}}$ (p)                       & $2\!\to\! 1$& 648 & 576k & 546k & 576k & 576k & 607k & 546k & 546k \\
$\opn{OR}, \opn{XNOR_{NAIL}}$ (d)                       & $2\!\to\! 2$& 512 & 577k & 530k & 577k & 578k & 626k & 530k & 530k \\
$\opn{OR}, \opn{AND}, \opn{XNOR_{NAIL}}$ (p)            & $2\!\to\! 1$& 648 & 576k & 546k & 576k & 576k & 607k & 546k & 546k \\
$\opn{OR}, \opn{AND}, \opn{XNOR_{NAIL}}$ (d)            & $2\!\to\! 3$& 438 & 579k & 519k & 579k & 580k & 642k & 519k & 519k \\
\bottomrule
\end{tabular}
  }
}
\end{table}

Since the transfer learning experiments were performed without retraining the base network, we found (\autoref{tab:transfer-head-big}) the performance is limited by the features which the base network produces, whose relevance depends on the overlap between the domain of the pretraining task (ImageNet) and the new task.
This information bottleneck, and variation in its utility, means the variation between performance on different datasets is much larger than the variation for different activation functions within the same dataset.

For transfer tasks which involve coarse-grained discrimination on images which are similar to ImageNet, the embedding generated by the pretrained model is already sufficient to separate the classes in the new dataset. Examples of this are Caltech101, where linear layer beats out using additional layers with non-linearities; and STL-10, where our best model is on-par with the linear model.
The results on these two transfer learning tasks are too simple to be of interest to study. The performance is limited by the information retained by the pretrained embedding, but it appears that what task-relevant information is there is readily available with a linear layer without needing additional logic to interpret it.

Other transfer learning tasks we attempted are less trivial and show a larger difference in performance across the activation functions, and a gap from the linear layer readout model. The Stanford Cars dataset involves fine-grained discrimination between different car models, for which features generated by a model pretrained on ImageNet are not effective. The SVHN dataset, which contains images of house numbers, is coarse-grained but uses images which are outside the domain of ImageNet, which makes the transfer-learning task more difficult.

As shown in \autoref{tab:transfer-extra}, we found the duplication ensemble strategy consistently outperformed the partition strategy on the transfer learning experiments.
Compared to using $\opn{OR_{NAIL}}$, there was no benefit to using an ensemble with the partition strategy.
However, it was beneficial to ensemble $\opn{OR_{NAIL}}$ with $\opn{AND_{NAIL}}$, and better still with $\opn{XNOR_{NAIL}}$, under the duplication strategy.
The performance of duplication-ensembles using $\opn{XNOR_{NAIL}}$, i.e. \{$\opn{OR_{NAIL}}$, $\opn{AND_{NAIL}}$, $\opn{XNOR_{NAIL}}$ (d)\} and \{$\opn{OR_{NAIL}}$, $\opn{AND_{NAIL}}$, $\opn{XNOR_{NAIL}}$ (d)\}, was comparable to using $\opn{XNOR_{NAIL}}$ alone, so in this case an ensemble was neither beneficial nor detrimental.

Additionally, we ran the experiment again with a pre-activation width of $w=512$ for all activation functions.
The results with constant pre-activation width $w=512$ are shown in \autoref{tab:transfer-head-512}.
By using a constant pre-activation width, the number of parameters used is not consistent across activation functions.
The number of trainable parameters in each experiment is shown in \autoref{tab:transfer-nparam-512}.
Note: some experiments on SVHN are missing due to an issue with the job scheduler which was not noticed in time to run the missing experiments.

\begin{table}[htb]
\small
  \centering
  \caption{%
Transfer learning from a frozen ResNet-18 architecture pretrained on ImageNet-1k to other computer vision datasets.
As with \autoref{tab:transfer}, but in this case we show results where the MLP has a pre-activation width of $w\!=\!512$.
Note that although the pre-activation width is constant, the number of parameters in the network is not consistent between experiments.
The number of parameters is shown in \autoref{tab:transfer-nparam-512}.
Mean (standard error) of $n\!=\!5$ inits of the MLP (same pretrained network).
}
\label{tab:transfer-head-512}
\def\ccAmin{86.5801886792453}
\def\ccAmax{88.34905660377358}
\def\ccA[#1]#2{\heatmapcell{#1}{#2}{\ccAmin}{\ccAmax}}
\def\ccBmin{81.63399999999999}
\def\ccBmax{83.106}
\def\ccB[#1]#2{\heatmapcell{#1}{#2}{\ccBmin}{\ccBmax}}
\def\ccCmin{58.044000000000004}
\def\ccCmax{60.898}
\def\ccC[#1]#2{\heatmapcell{#1}{#2}{\ccCmin}{\ccCmax}}
\def\ccDmin{90.70731696896436}
\def\ccDmax{93.1707316031107}
\def\ccD[#1]#2{\heatmapcell{#1}{#2}{\ccDmin}{\ccDmax}}
\def\ccEmin{30.969666852908368}
\def\ccEmax{39.835902544398145}
\def\ccE[#1]#2{\heatmapcell{#1}{#2}{\ccEmin}{\ccEmax}}
\def\ccFmin{94.62249999999999}
\def\ccFmax{94.88}
\def\ccF[#1]#2{\heatmapcell{#1}{#2}{\ccFmin}{\ccFmax}}
\def\ccGmin{53.25829748424489}
\def\ccGmax{55.76521204999422}
\def\ccG[#1]#2{\heatmapcell{#1}{#2}{\ccGmin}{\ccGmax}}
\centerline{
  \scalebox{0.835}{
\begin{tabular}{lrrrrrrr}
\toprule
 & \multicolumn{7}{c}{Test Accuracy (\%)} \\
\cmidrule(l){2-8}
Activation function                                     & \multicolumn{1}{c}{Cal101} & \multicolumn{1}{c}{CIFAR10} & \multicolumn{1}{c}{CIFAR100} & \multicolumn{1}{c}{Flowers} & \multicolumn{1}{c}{StfCars} & \multicolumn{1}{c}{STL-10} & \multicolumn{1}{c}{SVHN} \\
\toprule
Linear layer only                                        & \ccA[88.3491]{$\mbf{\mbns{88.35}}\wpm{ 0.15}$} & \ccB[78.5640]{$78.56\wpm{ 0.09}$} & \ccC[57.3920]{$57.39\wpm{ 0.09}$} & \ccD[92.3171]{$92.32\wpm{ 0.20}$} & \ccE[33.5132]{$33.51\wpm{ 0.06}$} & \ccF[94.6825]{$94.68\wpm{ 0.02}$} & \ccG[46.6011]{$46.60\wpm{ 0.14}$} \\
\midrule
$\opn{ReLU}$                                             & \ccA[86.5802]{$86.58\wpm{ 0.17}$} & \ccB[81.6340]{$81.63\wpm{ 0.05}$} & \ccC[58.0440]{$58.04\wpm{ 0.11}$} & \ccD[90.7073]{$90.71\wpm{ 0.26}$} & \ccE[30.9697]{$30.97\wpm{ 0.26}$} & \ccF[94.6225]{$94.62\wpm{ 0.06}$} & \ccG[53.2583]{$53.26\wpm{ 0.08}$} \\
LeakyReLU                                                & \ccA[86.6038]{$86.60\wpm{ 0.13}$} & \ccB[81.6660]{$81.67\wpm{ 0.11}$} & \ccC[58.0080]{$58.01\wpm{ 0.09}$} & \ccD[90.7317]{$90.73\wpm{ 0.32}$} & \ccE[31.0890]{$31.09\wpm{ 0.24}$} & \ccF[94.6150]{$94.61\wpm{ 0.05}$} & \ccG[53.2383]{$53.24\wpm{ 0.10}$} \\
PReLU                                                    & \ccA[87.8302]{$\mbns{87.83}\wpm{ 0.21}$} & \ccB[81.0300]{$81.03\wpm{ 0.13}$} & \ccC[58.8980]{$58.90\wpm{ 0.18}$} & \ccD[93.1707]{$\mbf{\mbns{93.17}}\wpm{ 0.19}$} & \ccE[39.8359]{$\mbf{\mbns{39.84}}\wpm{ 0.18}$} & \ccF[94.5375]{$94.54\wpm{ 0.05}$} & \ccG[53.4750]{$53.47\wpm{ 0.08}$} \\
Softplus                                                 & \ccA[86.1557]{$86.16\wpm{ 0.18}$} & \ccB[79.1300]{$79.13\wpm{ 0.08}$} & \ccC[56.5840]{$56.58\wpm{ 0.07}$} & \ccD[89.3902]{$89.39\wpm{ 0.29}$} & \ccE[21.2332]{$21.23\wpm{ 0.13}$} & \ccF[94.6300]{$94.63\wpm{ 0.03}$} & \ccG[48.7892]{$48.79\wpm{ 0.08}$} \\
\midrule
ELU                                                      & \ccA[87.1816]{$87.18\wpm{ 0.09}$} & \ccB[80.4400]{$80.44\wpm{ 0.08}$} & \ccC[58.0840]{$58.08\wpm{ 0.10}$} & \ccD[91.7073]{$91.71\wpm{ 0.14}$} & \ccE[34.7041]{$34.70\wpm{ 0.06}$} & \ccF[94.5500]{$94.55\wpm{ 0.05}$} & \ccG[51.8869]{$51.89\wpm{ 0.09}$} \\
CELU                                                     & \ccA[87.1816]{$87.18\wpm{ 0.09}$} & \ccB[80.4400]{$80.44\wpm{ 0.08}$} & \ccC[58.0840]{$58.08\wpm{ 0.10}$} & \ccD[91.7073]{$91.71\wpm{ 0.14}$} & \ccE[34.7041]{$34.70\wpm{ 0.06}$} & \ccF[94.5500]{$94.55\wpm{ 0.05}$} & \ccG[51.8869]{$51.89\wpm{ 0.09}$} \\
SELU                                                     & \ccA[87.7358]{$87.74\wpm{ 0.09}$} & \ccB[79.9340]{$79.93\wpm{ 0.13}$} & \ccC[58.2380]{$58.24\wpm{ 0.06}$} & \ccD[92.2683]{$92.27\wpm{ 0.13}$} & \ccE[37.5112]{$37.51\wpm{ 0.17}$} & \ccF[94.5300]{$94.53\wpm{ 0.07}$} & \ccG[50.9365]{$50.94\wpm{ 0.12}$} \\
GELU                                                     & \ccA[87.0991]{$87.10\wpm{ 0.15}$} & \ccB[81.3860]{$81.39\wpm{ 0.09}$} & \ccC[58.5060]{$58.51\wpm{ 0.13}$} & \ccD[91.5122]{$91.51\wpm{ 0.15}$} & \ccE[33.4286]{$33.43\wpm{ 0.15}$} & \ccF[94.6250]{$94.62\wpm{ 0.06}$} & \ccG[53.4250]{$53.43\wpm{ 0.23}$} \\
SiLU                                                     & \ccA[86.9104]{$86.91\wpm{ 0.11}$} & \ccB[80.5280]{$80.53\wpm{ 0.11}$} & \ccC[58.1440]{$58.14\wpm{ 0.12}$} & \ccD[91.3659]{$91.37\wpm{ 0.18}$} & \ccE[32.1482]{$32.15\wpm{ 0.17}$} & \ccF[94.5875]{$94.59\wpm{ 0.05}$} & \ccG[52.3548]{$52.35\wpm{ 0.16}$} \\
Hardswish                                                & \ccA[87.1226]{$87.12\wpm{ 0.12}$} & \ccB[80.1040]{$80.10\wpm{ 0.10}$} & \ccC[58.2540]{$58.25\wpm{ 0.10}$} & \ccD[91.5610]{$91.56\wpm{ 0.25}$} & \ccE[33.1726]{$33.17\wpm{ 0.23}$} & \ccF[94.6225]{$94.62\wpm{ 0.05}$} & \ccG[51.8231]{$51.82\wpm{ 0.13}$} \\
Mish                                                     & \ccA[87.1108]{$87.11\wpm{ 0.12}$} & \ccB[81.0880]{$81.09\wpm{ 0.11}$} & \ccC[58.3720]{$58.37\wpm{ 0.10}$} & \ccD[91.6098]{$91.61\wpm{ 0.15}$} & \ccE[33.7469]{$33.75\wpm{ 0.14}$} & \ccF[94.6075]{$94.61\wpm{ 0.05}$} & \ccG[53.0539]{$53.05\wpm{ 0.12}$} \\
\midrule
Softsign                                                 & \ccA[81.4741]{$81.47\wpm{ 0.18}$} & \ccB[80.0340]{$80.03\wpm{ 0.09}$} & \ccC[54.8420]{$54.84\wpm{ 0.09}$} & \ccD[82.3415]{$82.34\wpm{ 0.22}$} & \ccE[17.3322]{$17.33\wpm{ 0.10}$} & \ccF[94.7050]{$94.70\wpm{ 0.03}$} & \ccG[51.2500]{$51.25\wpm{ 0.08}$} \\
Tanh                                                     & \ccA[87.4764]{$87.48\wpm{ 0.06}$} & \ccB[80.5560]{$80.56\wpm{ 0.07}$} & \ccC[57.3480]{$57.35\wpm{ 0.08}$} & \ccD[90.3171]{$90.32\wpm{ 0.20}$} & \ccE[29.5077]{$29.51\wpm{ 0.12}$} & \ccF[94.6275]{$94.63\wpm{ 0.07}$} & \ccG[51.8600]{$51.86\wpm{ 0.05}$} \\
\midrule
GLU                                                      & \ccA[86.3443]{$86.34\wpm{ 0.16}$} & \ccB[79.3460]{$79.35\wpm{ 0.05}$} & \ccC[57.2200]{$57.22\wpm{ 0.12}$} & \ccD[90.0976]{$90.10\wpm{ 0.20}$} & \ccE[26.7205]{$26.72\wpm{ 0.13}$} & \ccF[94.6325]{$94.63\wpm{ 0.05}$} & \ccG[50.6369]{$50.64\wpm{ 0.10}$} \\
\midrule
$\opn{Max}$                                              & \ccA[86.8632]{$86.86\wpm{ 0.11}$} & \ccB[81.5580]{$81.56\wpm{ 0.06}$} & \ccC[58.1180]{$58.12\wpm{ 0.10}$} & \ccD[90.5854]{$90.59\wpm{ 0.25}$} & \ccE[32.8046]{$32.80\wpm{ 0.04}$} & \ccF[94.6500]{$94.65\wpm{ 0.05}$} & \ccG[53.5533]{$53.55\wpm{ 0.09}$} \\
$\opn{Max},\opn{Min}$ (d)                                & \ccA[87.2288]{$87.23\wpm{ 0.13}$} & \ccB[82.3060]{$82.31\wpm{ 0.10}$} & \ccC[59.0520]{$59.05\wpm{ 0.10}$} & \ccD[91.6829]{$91.68\wpm{ 0.18}$} & \ccE[34.9105]{$34.91\wpm{ 0.12}$} & \ccF[94.6400]{$94.64\wpm{ 0.04}$} & \ccG[53.9136]{$53.91\wpm{ 0.13}$} \\
\midrule
SignedGeomean                                            & \ccA[86.7375]{$86.74\wpm{ 0.33}$} & \ccB[23.7940]{$23.79\wpm{13.79}$} & \ccC[11.7520]{$11.75\wpm{10.75}$} & \ccD[90.9780]{$90.98\wpm{ 0.21}$} & \ccE[32.4537]{$32.45\wpm{ 0.15}$} & \ccF[94.5875]{$94.59\wpm{ 0.04}$} & \ccG[19.5874]{$19.59\wpm{ 0.00}$} \\
\midrule
$\opn{XNOR_{IL}}$                                        & \ccA[85.1061]{$85.11\wpm{ 0.10}$} & \ccB[79.7700]{$79.77\wpm{ 0.08}$} & \ccC[56.9180]{$56.92\wpm{ 0.09}$} & \ccD[84.2927]{$84.29\wpm{ 0.29}$} & \ccE[ 1.7752]{$ 1.78\wpm{ 0.21}$} & \ccF[94.5625]{$94.56\wpm{ 0.05}$} & \ccG[52.2657]{$52.27\wpm{ 0.08}$} \\
$\opn{OR_{IL}}$                                          & \ccA[87.0150]{$87.01\wpm{ 0.21}$} & \ccB[79.8840]{$79.88\wpm{ 0.04}$} & \ccC[57.6820]{$57.68\wpm{ 0.08}$} & \ccD[91.1220]{$91.12\wpm{ 0.10}$} & \ccE[32.5485]{$32.55\wpm{ 0.04}$} & \ccF[94.5675]{$94.57\wpm{ 0.06}$} & \ccG[51.5435]{$51.54\wpm{ 0.07}$} \\
$\opn{OR}, \opn{AND_{IL}}$ (d)                           & \ccA[87.2170]{$87.22\wpm{ 0.14}$} & \ccB[79.5620]{$79.56\wpm{ 0.06}$} & \ccC[57.9100]{$57.91\wpm{ 0.12}$} & \ccD[91.4390]{$91.44\wpm{ 0.18}$} & \ccE[36.0443]{$36.04\wpm{ 0.07}$} & \ccF[94.4400]{$94.44\wpm{ 0.05}$} & \ccG[50.4172]{$50.42\wpm{ 0.25}$} \\
$\opn{OR}, \opn{XNOR_{IL}}$ (p)                          & \ccA[86.7375]{$86.74\wpm{ 0.21}$} & \ccB[80.4140]{$80.41\wpm{ 0.06}$} & \ccC[57.8520]{$57.85\wpm{ 0.08}$} & \ccD[90.3912]{$90.39\wpm{ 0.25}$} & \ccE[30.9787]{$30.98\wpm{ 0.17}$} & \ccF[94.6275]{$94.63\wpm{ 0.04}$} & \ccG[52.7651]{$52.77\wpm{ 0.04}$} \\
$\opn{OR}, \opn{XNOR_{IL}}$ (d)                          & \ccA[86.7571]{$86.76\wpm{ 0.12}$} & \ccB[80.8260]{$80.83\wpm{ 0.08}$} & \ccC[58.5480]{$58.55\wpm{ 0.10}$} & \ccD[90.7561]{$90.76\wpm{ 0.07}$} & \ccE[33.5579]{$33.56\wpm{ 0.14}$} & \ccF[94.4700]{$94.47\wpm{ 0.07}$} & \ccG[52.5415]{$52.54\wpm{ 0.09}$} \\
$\opn{OR}, \opn{AND}, \opn{XNOR_{IL}}$ (p)               & \ccA[86.5841]{$86.58\wpm{ 0.15}$} & \ccB[80.3120]{$80.31\wpm{ 0.08}$} & \ccC[57.9920]{$57.99\wpm{ 0.10}$} & \ccD[90.3912]{$90.39\wpm{ 0.28}$} & \ccE[31.7597]{$31.76\wpm{ 0.16}$} & \ccF[94.6550]{$94.66\wpm{ 0.04}$} & \ccG[52.5154]{$52.52\wpm{ 0.06}$} \\
$\opn{OR}, \opn{AND}, \opn{XNOR_{IL}}$ (d)               & \ccA[86.9796]{$86.98\wpm{ 0.24}$} & \ccB[80.2220]{$80.22\wpm{ 0.10}$} & \ccC[58.8940]{$58.89\wpm{ 0.10}$} & \ccD[91.5565]{$91.56\wpm{ 0.26}$} & \ccE[35.2662]{$35.27\wpm{ 0.08}$} & \ccF[94.4925]{$94.49\wpm{ 0.05}$} & \ccG[52.5300]{$52.53\wpm{ 0.04}$} \\
\midrule
$\opn{XNOR_{NIL}}$                                       & \ccA[87.9410]{$\mbs{87.94}\wpm{ 0.08}$} & \ccB[82.8880]{$\mbs{\mbns{82.89}}\wpm{ 0.09}$} & \ccC[60.2160]{$60.22\wpm{ 0.12}$} & \ccD[93.1051]{$\mbns{93.11}\wpm{ 0.22}$} & \ccE[38.9330]{$38.93\wpm{ 0.12}$} & \ccF[94.7825]{$94.78\wpm{ 0.03}$} & \ccG[55.2259]{$\mbs{55.23}\wpm{ 0.11}$} \\
$\opn{OR_{NIL}}$                                         & \ccA[87.2330]{$87.23\wpm{ 0.14}$} & \ccB[79.7620]{$79.76\wpm{ 0.05}$} & \ccC[58.6160]{$58.62\wpm{ 0.02}$} & \ccD[91.4670]{$91.47\wpm{ 0.07}$} & \ccE[35.5702]{$35.57\wpm{ 0.11}$} & \ccF[94.6625]{$94.66\wpm{ 0.07}$} & \ccG[51.1832]{$51.18\wpm{ 0.05}$} \\
$\opn{OR}, \opn{AND_{NIL}}$ (d)                          & \ccA[86.8201]{$86.82\wpm{ 0.18}$} & \ccB[80.0900]{$80.09\wpm{ 0.09}$} & \ccC[58.5980]{$58.60\wpm{ 0.07}$} & \ccD[91.4425]{$91.44\wpm{ 0.20}$} & \ccE[37.0252]{$37.03\wpm{ 0.11}$} & \ccF[94.6500]{$94.65\wpm{ 0.05}$} & \ccG[52.4915]{$52.49\wpm{ 0.09}$} \\
$\opn{OR}, \opn{XNOR_{NIL}}$ (p)                         & \ccA[87.5988]{$87.60\wpm{ 0.16}$} & \ccB[82.6120]{$82.61\wpm{ 0.05}$} & \ccC[59.8960]{$59.90\wpm{ 0.16}$} & \ccD[92.3227]{$92.32\wpm{ 0.29}$} & \ccE[38.3062]{$38.31\wpm{ 0.14}$} & \ccF[94.6925]{$94.69\wpm{ 0.03}$} & \ccG[55.0092]{$55.01\wpm{ 0.06}$} \\
$\opn{OR}, \opn{XNOR_{NIL}}$ (d)                         & \ccA[87.8184]{$\mbns{87.82}\wpm{ 0.19}$} & \ccB[82.6720]{$82.67\wpm{ 0.05}$} & \ccC[60.6000]{$\mbs{\mbns{60.60}}\wpm{ 0.11}$} & \ccD[92.9268]{$\mbns{92.93}\wpm{ 0.12}$} & \ccE[39.2190]{$39.22\wpm{ 0.17}$} & \ccF[94.6275]{$94.63\wpm{ 0.06}$} & \ccG[54.8748]{$54.87\wpm{ 0.09}$} \\
$\opn{OR}, \opn{AND}, \opn{XNOR_{NIL}}$ (p)              & \ccA[87.2330]{$87.23\wpm{ 0.14}$} & \ccB[82.3620]{$82.36\wpm{ 0.09}$} & \ccC[59.4180]{$59.42\wpm{ 0.12}$} & \ccD[92.3227]{$92.32\wpm{ 0.02}$} & \ccE[37.9231]{$37.92\wpm{ 0.20}$} & \ccF[94.7675]{$94.77\wpm{ 0.04}$} & \ccG[54.6028]{$54.60\wpm{ 0.08}$} \\
$\opn{OR}, \opn{AND}, \opn{XNOR_{NIL}}$ (d)              & \ccA[87.3274]{$87.33\wpm{ 0.26}$} & \ccB[83.1060]{$\mbf{\mbns{83.11}}\wpm{ 0.10}$} & \ccC[60.8980]{$\mbf{\mbns{60.90}}\wpm{ 0.16}$} & \ccD[93.1540]{$\mbs{\mbns{93.15}}\wpm{ 0.28}$} & \ccE[39.2663]{$\mbns{39.27}\wpm{ 0.18}$} & \ccF[94.7100]{$\mbns{94.71}\wpm{ 0.09}$} & \ccG[55.7652]{$\mbf{\mbns{55.77}}\wpm{ 0.13}$} \\
\midrule
$\opn{XNOR_{AIL}}$                                       & \ccA[86.4210]{$86.42\wpm{ 0.13}$} & \ccB[81.7370]{$81.74\wpm{ 0.05}$} & \ccC[57.8790]{$57.88\wpm{ 0.09}$} & \ccD[90.5000]{$90.50\wpm{ 0.17}$} & \ccE[31.5490]{$31.55\wpm{ 0.11}$} & \ccF[94.7787]{$\mbns{94.78}\wpm{ 0.04}$} & \ccG[53.7700]{$53.77\wpm{ 0.04}$} \\
$\opn{OR_{AIL}}$                                         & \ccA[87.2759]{$87.28\wpm{ 0.09}$} & \ccB[81.8140]{$81.81\wpm{ 0.02}$} & \ccC[58.6780]{$58.68\wpm{ 0.05}$} & \ccD[91.7805]{$91.78\wpm{ 0.23}$} & \ccE[35.2760]{$35.28\wpm{ 0.09}$} & \ccF[94.6700]{$94.67\wpm{ 0.07}$} & \ccG[53.6755]{$53.68\wpm{ 0.03}$} \\
$\opn{OR}, \opn{AND_{AIL}}$ (d)                          & \ccA[87.4292]{$87.43\wpm{ 0.11}$} & \ccB[82.3800]{$82.38\wpm{ 0.06}$} & \ccC[59.8980]{$59.90\wpm{ 0.08}$} & \ccD[92.0732]{$92.07\wpm{ 0.18}$} & \ccE[37.1631]{$37.16\wpm{ 0.15}$} & \ccF[94.5500]{$94.55\wpm{ 0.05}$} & \ccG[54.0542]{$54.05\wpm{ 0.07}$} \\
$\opn{OR}, \opn{XNOR_{AIL}}$ (p)                         & \ccA[87.2448]{$87.24\wpm{ 0.22}$} & \ccB[81.9360]{$81.94\wpm{ 0.03}$} & \ccC[58.5720]{$58.57\wpm{ 0.07}$} & \ccD[90.9291]{$90.93\wpm{ 0.17}$} & \ccE[34.5853]{$34.59\wpm{ 0.18}$} & \ccF[94.7675]{$\mbns{94.77}\wpm{ 0.06}$} & \ccG[53.8891]{$53.89\wpm{ 0.11}$} \\
$\opn{OR}, \opn{XNOR_{AIL}}$ (d)                         & \ccA[87.0873]{$87.09\wpm{ 0.21}$} & \ccB[82.1960]{$82.20\wpm{ 0.04}$} & \ccC[59.4360]{$59.44\wpm{ 0.07}$} & \ccD[91.9024]{$91.90\wpm{ 0.10}$} & \ccE[36.8772]{$36.88\wpm{ 0.10}$} & \ccF[94.6850]{$94.69\wpm{ 0.06}$} & \ccG[54.0258]{$54.03\wpm{ 0.11}$} \\
$\opn{OR}, \opn{AND}, \opn{XNOR_{AIL}}$ (p)              & \ccA[87.0206]{$87.02\wpm{ 0.11}$} & \ccB[81.8300]{$81.83\wpm{ 0.04}$} & \ccC[58.9260]{$58.93\wpm{ 0.16}$} & \ccD[91.1736]{$91.17\wpm{ 0.29}$} & \ccE[34.8564]{$34.86\wpm{ 0.15}$} & \ccF[94.7500]{$94.75\wpm{ 0.04}$} & \ccG[53.9206]{$53.92\wpm{ 0.09}$} \\
$\opn{OR}, \opn{AND}, \opn{XNOR_{AIL}}$ (d)              & \ccA[87.2642]{$87.26\wpm{ 0.15}$} & \ccB[82.4520]{$82.45\wpm{ 0.08}$} & \ccC[60.2040]{$60.20\wpm{ 0.11}$} & \ccD[92.4146]{$92.41\wpm{ 0.17}$} & \ccE[37.8891]{$37.89\wpm{ 0.14}$} & \ccF[94.6475]{$94.65\wpm{ 0.04}$} & \ccG[55.0453]{$55.05\wpm{ 0.09}$} \\
\midrule
$\opn{XNOR_{NAIL}}$                                      & \ccA[87.4690]{$87.47\wpm{ 0.12}$} & \ccB[82.3800]{$82.38\wpm{ 0.05}$} & \ccC[59.3420]{$59.34\wpm{ 0.08}$} & \ccD[92.8362]{$\mbns{92.84}\wpm{ 0.18}$} & \ccE[39.0325]{$39.03\wpm{ 0.05}$} & \ccF[94.8800]{$\mbf{\mbns{94.88}}\wpm{ 0.02}$} & \ccG[54.2256]{$54.23\wpm{ 0.15}$} \\
$\opn{OR_{NAIL}}$                                        & \ccA[87.3274]{$87.33\wpm{ 0.14}$} & \ccB[81.7760]{$81.78\wpm{ 0.03}$} & \ccC[59.2720]{$59.27\wpm{ 0.06}$} & \ccD[91.8337]{$91.83\wpm{ 0.14}$} & \ccE[36.7914]{$36.79\wpm{ 0.20}$} & \ccF[94.7775]{$\mbns{94.78}\wpm{ 0.07}$} & \ccG[53.7869]{$53.79\wpm{ 0.15}$} \\
$\opn{OR}, \opn{AND_{NAIL}}$ (d)                         & \ccA[87.6179]{$87.62\wpm{ 0.11}$} & \ccB[82.2800]{$82.28\wpm{ 0.10}$} & \ccC[59.7140]{$59.71\wpm{ 0.05}$} & \ccD[92.0976]{$92.10\wpm{ 0.20}$} & \ccE[37.6977]{$37.70\wpm{ 0.12}$} & \ccF[94.6125]{$94.61\wpm{ 0.08}$} & \ccG[53.8614]{$53.86\wpm{ 0.10}$} \\
$\opn{OR}, \opn{XNOR_{NAIL}}$ (p)                        & \ccA[87.7876]{$87.79\wpm{ 0.10}$} & \ccB[82.2580]{$82.26\wpm{ 0.05}$} & \ccC[59.2700]{$59.27\wpm{ 0.10}$} & \ccD[92.1760]{$92.18\wpm{ 0.19}$} & \ccE[38.1992]{$38.20\wpm{ 0.15}$} & \ccF[94.8200]{$\mbs{\mbns{94.82}}\wpm{ 0.05}$} & \ccG[54.0911]{$54.09\wpm{ 0.09}$} \\
$\opn{OR}, \opn{XNOR_{NAIL}}$ (d)                        & \ccA[87.8538]{$\mbns{87.85}\wpm{ 0.22}$} & \ccB[82.5240]{$82.52\wpm{ 0.11}$} & \ccC[60.0180]{$60.02\wpm{ 0.10}$} & \ccD[93.1220]{$\mbns{93.12}\wpm{ 0.13}$} & \ccE[39.6370]{$\mbs{\mbns{39.64}}\wpm{ 0.09}$} & \ccF[94.7450]{$94.75\wpm{ 0.03}$} & \ccG[54.1334]{$54.13\wpm{ 0.05}$} \\
$\opn{OR}, \opn{AND}, \opn{XNOR_{NAIL}}$ (p)             & \ccA[87.3392]{$87.34\wpm{ 0.26}$} & \ccB[82.1140]{$82.11\wpm{ 0.18}$} & \ccC[59.5580]{$59.56\wpm{ 0.10}$} & \ccD[92.2738]{$92.27\wpm{ 0.22}$} & \ccE[37.5700]{$37.57\wpm{ 0.10}$} & \ccF[94.7825]{$\mbns{94.78}\wpm{ 0.04}$} & \ccG[54.1180]{$54.12\wpm{ 0.08}$} \\
$\opn{OR}, \opn{AND}, \opn{XNOR_{NAIL}}$ (d)             & \ccA[87.4218]{$87.42\wpm{ 0.10}$} & \ccB[82.8300]{$82.83\wpm{ 0.05}$} & \ccC[60.5800]{$\mbns{60.58}\wpm{ 0.15}$} & \ccD[92.5428]{$\mbns{92.54}\wpm{ 0.20}$} & \ccE[39.4876]{$\mbns{39.49}\wpm{ 0.15}$} & \ccF[94.8075]{$\mbns{94.81}\wpm{ 0.11}$} & \ccG[55.0983]{$55.10\wpm{ 0.04}$} \\
\bottomrule
\end{tabular}
  }
}
\end{table}

\begin{table}[htb]
\small
  \centering
  \caption{%
Number of trainable parameters used in transfer learning experiments (corresponding to results shown in \autoref{tab:transfer-head-512}).
}
\label{tab:transfer-nparam-512}
\centerline{
  \scalebox{0.835}{

\begin{tabular}{lrrrrrrrrr}
\toprule
 & & & \multicolumn{7}{c}{\textnumero{} Trainable Parameters} \\
\cmidrule(l){4-10}
Activation function                                     & Map & Width & \multicolumn{1}{c}{Cal101} & \multicolumn{1}{c}{CIFAR10} & \multicolumn{1}{c}{CIFAR100} & \multicolumn{1}{c}{Flowers} & \multicolumn{1}{c}{StfCars} & \multicolumn{1}{c}{STL-10} & \multicolumn{1}{c}{SVHN} \\
\toprule
Linear layer only                                       & &  & 52k & 5k & 51k & 52k & 101k & 5k & 5k \\
\midrule
$\opn{ReLU}$                                            & $1\!\to\! 1$& 512 & 577k & 530k & 577k & 578k & 626k & 530k & 530k \\
LeakyReLU                                               & $1\!\to\! 1$& 512 & 577k & 530k & 577k & 578k & 626k & 530k & 530k \\
PReLU                                                   & $1\!\to\! 1$& 512 & 577k & 530k & 577k & 578k & 626k & 530k & 530k \\
Softplus                                                & $1\!\to\! 1$& 512 & 577k & 530k & 577k & 578k & 626k & 530k & 530k \\
\midrule
ELU                                                     & $1\!\to\! 1$& 512 & 577k & 530k & 577k & 578k & 626k & 530k & 530k \\
CELU                                                    & $1\!\to\! 1$& 512 & 577k & 530k & 577k & 578k & 626k & 530k & 530k \\
SELU                                                    & $1\!\to\! 1$& 512 & 577k & 530k & 577k & 578k & 626k & 530k & 530k \\
GELU                                                    & $1\!\to\! 1$& 512 & 577k & 530k & 577k & 578k & 626k & 530k & 530k \\
SiLU                                                    & $1\!\to\! 1$& 512 & 577k & 530k & 577k & 578k & 626k & 530k & 530k \\
Hardswish                                               & $1\!\to\! 1$& 512 & 577k & 530k & 577k & 578k & 626k & 530k & 530k \\
Mish                                                    & $1\!\to\! 1$& 512 & 577k & 530k & 577k & 578k & 626k & 530k & 530k \\
\midrule
Softsign                                                & $1\!\to\! 1$& 512 & 577k & 530k & 577k & 578k & 626k & 530k & 530k \\
Tanh                                                    & $1\!\to\! 1$& 512 & 577k & 530k & 577k & 578k & 626k & 530k & 530k \\
\midrule
GLU                                                     & $2\!\to\! 1$& 512 & 420k & 397k & 420k & 420k & 445k & 397k & 397k \\
\midrule
$\opn{Max}$                                             & $2\!\to\! 1$& 512 & 420k & 397k & 420k & 420k & 445k & 397k & 397k \\
$\opn{Max},\opn{Min}$ (d)                               & $2\!\to\! 2$& 512 & 577k & 530k & 577k & 578k & 626k & 530k & 530k \\
\midrule
SignedGeomean                                           & $2\!\to\! 1$& 512 & 420k & 397k & 420k & 420k & 445k & 397k & 397k \\
\midrule
$\opn{XNOR_{IL}}$                                       & $2\!\to\! 1$& 512 & 420k & 397k & 420k & 420k & 445k & 397k & 397k \\
$\opn{OR_{IL}}$                                         & $2\!\to\! 1$& 512 & 420k & 397k & 420k & 420k & 445k & 397k & 397k \\
$\opn{OR}, \opn{AND_{IL}}$ (d)                          & $2\!\to\! 2$& 512 & 577k & 530k & 577k & 578k & 626k & 530k & 530k \\
$\opn{OR}, \opn{XNOR_{IL}}$ (p)                         & $2\!\to\! 1$& 512 & 420k & 397k & 420k & 420k & 445k & 397k & 397k \\
$\opn{OR}, \opn{XNOR_{IL}}$ (d)                         & $2\!\to\! 2$& 512 & 577k & 530k & 577k & 578k & 626k & 530k & 530k \\
$\opn{OR}, \opn{AND}, \opn{XNOR_{IL}}$ (p)              & $2\!\to\! 1$& 510 & 418k & 395k & 418k & 418k & 442k & 395k & 395k \\
$\opn{OR}, \opn{AND}, \opn{XNOR_{IL}}$ (d)              & $2\!\to\! 3$& 512 & 734k & 664k & 733k & 735k & 807k & 664k & 664k \\
\midrule
$\opn{XNOR_{NIL}}$                                      & $2\!\to\! 1$& 512 & 420k & 397k & 420k & 420k & 445k & 397k & 397k \\
$\opn{OR_{NIL}}$                                        & $2\!\to\! 1$& 512 & 420k & 397k & 420k & 420k & 445k & 397k & 397k \\
$\opn{OR}, \opn{AND_{NIL}}$ (d)                         & $2\!\to\! 2$& 512 & 577k & 530k & 577k & 578k & 626k & 530k & 530k \\
$\opn{OR}, \opn{XNOR_{NIL}}$ (p)                        & $2\!\to\! 1$& 512 & 420k & 397k & 420k & 420k & 445k & 397k & 397k \\
$\opn{OR}, \opn{XNOR_{NIL}}$ (d)                        & $2\!\to\! 2$& 512 & 577k & 530k & 577k & 578k & 626k & 530k & 530k \\
$\opn{OR}, \opn{AND}, \opn{XNOR_{NIL}}$ (p)             & $2\!\to\! 1$& 510 & 418k & 395k & 418k & 418k & 442k & 395k & 395k \\
$\opn{OR}, \opn{AND}, \opn{XNOR_{NIL}}$ (d)             & $2\!\to\! 3$& 512 & 734k & 664k & 733k & 735k & 807k & 664k & 664k \\
\midrule
$\opn{XNOR_{AIL}}$                                      & $2\!\to\! 1$& 512 & 420k & 397k & 420k & 420k & 445k & 397k & 397k \\
$\opn{OR_{AIL}}$                                        & $2\!\to\! 1$& 512 & 420k & 397k & 420k & 420k & 445k & 397k & 397k \\
$\opn{OR}, \opn{AND_{AIL}}$ (d)                         & $2\!\to\! 2$& 512 & 577k & 530k & 577k & 578k & 626k & 530k & 530k \\
$\opn{OR}, \opn{XNOR_{AIL}}$ (p)                        & $2\!\to\! 1$& 512 & 420k & 397k & 420k & 420k & 445k & 397k & 397k \\
$\opn{OR}, \opn{XNOR_{AIL}}$ (d)                        & $2\!\to\! 2$& 512 & 577k & 530k & 577k & 578k & 626k & 530k & 530k \\
$\opn{OR}, \opn{AND}, \opn{XNOR_{AIL}}$ (p)             & $2\!\to\! 1$& 510 & 418k & 395k & 418k & 418k & 442k & 395k & 395k \\
$\opn{OR}, \opn{AND}, \opn{XNOR_{AIL}}$ (d)             & $2\!\to\! 3$& 512 & 734k & 664k & 733k & 735k & 807k & 664k & 664k \\
\midrule
$\opn{XNOR_{NAIL}}$                                     & $2\!\to\! 1$& 512 & 420k & 397k & 420k & 420k & 445k & 397k & 397k \\
$\opn{OR_{NAIL}}$                                       & $2\!\to\! 1$& 512 & 420k & 397k & 420k & 420k & 445k & 397k & 397k \\
$\opn{OR}, \opn{AND_{NAIL}}$ (d)                        & $2\!\to\! 2$& 512 & 577k & 530k & 577k & 578k & 626k & 530k & 530k \\
$\opn{OR}, \opn{XNOR_{NAIL}}$ (p)                       & $2\!\to\! 1$& 512 & 420k & 397k & 420k & 420k & 445k & 397k & 397k \\
$\opn{OR}, \opn{XNOR_{NAIL}}$ (d)                       & $2\!\to\! 2$& 512 & 577k & 530k & 577k & 578k & 626k & 530k & 530k \\
$\opn{OR}, \opn{AND}, \opn{XNOR_{NAIL}}$ (p)            & $2\!\to\! 1$& 510 & 418k & 395k & 418k & 418k & 442k & 395k & 395k \\
$\opn{OR}, \opn{AND}, \opn{XNOR_{NAIL}}$ (d)            & $2\!\to\! 3$& 512 & 734k & 664k & 733k & 735k & 807k & 664k & 664k \\
\bottomrule
\end{tabular}
  }
}
\end{table}

We found that reducing the width from 650 to 512 (and hence reducing total number of trainable parameters) reduced the performance of $2\!\to\! 1$ activation functions $\opn{XNOR_{AIL}}$, $\opn{OR_{AIL}}$, \{$\opn{OR}, \opn{XNOR_{AIL}}$ (p)\}, and \{$\opn{OR}, \opn{AND}, \opn{XNOR_{AIL}}$ (p)\}.
Increasing the width from 438 to 512 increased the performance of \{$\opn{OR}, \opn{AND}, \opn{XNOR_{AIL}}$ (d)\}.


\FloatBarrier

\subsection{Abstract reasoning}
\label{s:abstractreasoning}
\label{a:abstractreasoning}


Abstract reasoning is challenging for neural networks to learn because their structure and objective function are more effective for tasks which come instinctively to humans \citep[``System~1'' of][]{Kahneman2011}, such as object recognition, as opposed to demanding 
(``System~2'') logic tasks.

In recent years, several challenges have been proposed to evaluate the ability of neural networks to perform abstract reasoning.
The Raven's Progressive Matrices \citep{Raven1938} are a long-standing IQ test, which have been emulated by \citet{pgm} with Procedurally Generated Matrices (PGM), and then the RAVEN task by \citep{raven}.
Recently \citet{sran} and \citet{ravenfair} have improved on that task with I-RAVEN and RAVEN-FAIR, respectively, both aiming to make the task more balanced.
Other abstract reasoning tasks have also been proposed \citep{svrt,clevr,pgm}.

We considered the application of AIL activation functions in the context of the I-RAVEN task, by adapting the Stratified Rule-Aware Network (SRAN) of \citet{sran} to include our AIL activation functions.
We first added LayerNorm to the network, and then swapped out the seven ReLU activations in the gating module.
The architecture for the three ResNet-18 \citep{resnet} base models were unchanged.
Where necessary, the number of units per layer was modified to facilitate the change in dimensionality caused by our activation function.
The networks were trained using the same procedure as described by \citet{sran}.

\begin{table}[h]
\small
  \centering
  \caption{%
Performance of SRAN-based models on the I-RAVEN dataset \citep{sran}.
Bold: best.
Underlined: second best.
Background: color scale from worst in to best, linear with accuracy value.
}
\label{tab:sran}
\def\ccAmin{53.5143}
\def\ccAmax{64.3071}
\def\ccA[#1]#2{\heatmapcell{#1}{#2}{\ccAmin}{\ccAmax}}
\def\ccBmin{67.25}
\def\ccBmax{84.4}
\def\ccB[#1]#2{\heatmapcell{#1}{#2}{\ccBmin}{\ccBmax}}
\def\ccCmin{43.75}
\def\ccCmax{50.05}
\def\ccC[#1]#2{\heatmapcell{#1}{#2}{\ccCmin}{\ccCmax}}
\def\ccDmin{39.7}
\def\ccDmax{44.0}
\def\ccD[#1]#2{\heatmapcell{#1}{#2}{\ccDmin}{\ccDmax}}
\def\ccEmin{62.75}
\def\ccEmax{71.55}
\def\ccE[#1]#2{\heatmapcell{#1}{#2}{\ccEmin}{\ccEmax}}
\def\ccFmin{44.5}
\def\ccFmax{49.55}
\def\ccF[#1]#2{\heatmapcell{#1}{#2}{\ccFmin}{\ccFmax}}
\def\ccGmin{57.45}
\def\ccGmax{76.55}
\def\ccG[#1]#2{\heatmapcell{#1}{#2}{\ccGmin}{\ccGmax}}
\def\ccHmin{57.3}
\def\ccHmax{77.0}
\def\ccH[#1]#2{\heatmapcell{#1}{#2}{\ccHmin}{\ccHmax}}
\centerline{
  \scalebox{0.92}{
\begin{tabular}{lrrrrrrrrr}
\toprule
 & & \multicolumn{8}{c}{I-RAVEN Test Acc (\%)} \\
\cmidrule(l){3-10}
Activation function & Params & Acc & Center & 2×2G & 3×3G & O-IC & O-IG & L-R & U-D \\
\toprule
ReLU, Base SRAN \citep{sran}                            & 44.0M & \ccA[60.8]{$60.8$} & \ccB[78.2]{$78.2$} &\ccC[50.1]{$\mbf{50.1}$}& \ccD[42.4]{$42.4$} & \ccE[68.2]{$68.2$} & \ccF[46.3]{$46.3$} & \ccG[70.1]{$70.1$} & \ccH[70.3]{$70.3$} \\
ReLU, SRAN+LayerNorm                                    & 45.6M & \ccA[62.95]{$\mbs{63.0}$} & \ccB[84.05]{$\mbs{84.0}$} & \ccC[50.05]{$\mbf{50.0}$} & \ccD[42.55]{$42.5$} & \ccE[69.9]{$\mbs{69.9}$} & \ccF[48.3]{$48.3$} & \ccG[73.5]{$73.5$} & \ccH[72.3]{$\mbs{72.3}$} \\
PReLU                                                   & 45.6M & \ccA[54.5429]{$54.5$} & \ccB[69.2]{$69.2$} & \ccC[45.4]{$45.4$} & \ccD[40.45]{$40.5$} & \ccE[64.4]{$64.4$} & \ccF[47.6]{$47.6$} & \ccG[57.45]{$57.5$} & \ccH[57.3]{$57.3$} \\
\midrule
CELU                                                    & 45.6M & \ccA[56.6071]{$56.6$} & \ccB[73.55]{$73.5$} & \ccC[46.5]{$46.5$} & \ccD[40.95]{$41.0$} & \ccE[66.6]{$66.6$} & \ccF[46.15]{$46.1$} & \ccG[62.8]{$62.8$} & \ccH[59.7]{$59.7$} \\
SELU                                                    & 45.6M & \ccA[53.5143]{$53.5$} & \ccB[67.25]{$67.2$} & \ccC[43.75]{$43.8$} & \ccD[40.15]{$40.1$} & \ccE[62.75]{$62.8$} & \ccF[44.5]{$44.5$} & \ccG[58.05]{$58.0$} & \ccH[58.15]{$58.1$} \\
GELU                                                    & 45.6M & \ccA[61.3643]{$61.4$} & \ccB[80.8]{$80.8$} & \ccC[49.25]{$49.2$} & \ccD[42.55]{$42.5$} & \ccE[69.15]{$69.2$} & \ccF[48.3]{$48.3$} & \ccG[70.3]{$70.3$} & \ccH[69.2]{$69.2$} \\
SiLU                                                    & 45.6M & \ccA[59.4286]{$59.4$} & \ccB[78.05]{$78.0$} & \ccC[47.2]{$47.2$} & \ccD[41.35]{$41.4$} & \ccE[66.3]{$66.3$} & \ccF[47.75]{$47.8$} & \ccG[69.1]{$69.1$} & \ccH[66.25]{$66.2$} \\
\midrule
$\opn{Max}$                                             & 44.6M & \ccA[57.8]{$57.8$} & \ccB[76.3]{$76.3$} & \ccC[45.2]{$45.2$} & \ccD[39.7]{$39.7$} & \ccE[65.35]{$65.3$} & \ccF[48.7]{$\mbs{48.7}$} & \ccG[64.6]{$64.6$} & \ccH[64.7]{$64.7$} \\
$\opn{Max},\opn{Min}$ (d)                               & 45.6M & \ccA[60.1857]{$60.2$} & \ccB[80.25]{$80.2$} & \ccC[49.55]{$\mbs{49.5}$} & \ccD[41.85]{$41.9$} & \ccE[66.45]{$66.5$} & \ccF[47.0]{$47.0$} & \ccG[68.55]{$68.5$} & \ccH[67.65]{$67.7$} \\
\midrule
$\opn{XNOR_{AIL}}$                                      & 44.6M & \ccA[57.6857]{$57.7$} & \ccB[74.65]{$74.7$} & \ccC[45.95]{$46.0$} & \ccD[39.95]{$40.0$} & \ccE[66.8]{$66.8$} & \ccF[47.7]{$47.7$} & \ccG[65.6]{$65.6$} & \ccH[63.2]{$63.2$} \\
$\opn{OR_{AIL}}$                                        & 44.6M & \ccA[64.3071]{$\mbf{64.3}$} & \ccB[84.4]{$\mbf{84.4}$} & \ccC[49.55]{$\mbs{49.5}$} & \ccD[44.0]{$\mbf{44.0}$} & \ccE[71.55]{$\mbf{71.5}$} & \ccF[47.1]{$47.1$} & \ccG[76.55]{$\mbf{76.5}$} & \ccH[77.0]{$\mbf{77.0}$} \\
$\opn{OR}, \opn{AND_{AIL}}$ (d)                         & 45.6M & \ccA[57.4857]{$57.5$} & \ccB[74.65]{$74.7$} & \ccC[45.6]{$45.6$} & \ccD[41.3]{$41.3$} & \ccE[68.35]{$68.3$} & \ccF[44.9]{$44.9$} & \ccG[64.45]{$64.5$} & \ccH[63.05]{$63.0$} \\
$\opn{OR}, \opn{XNOR_{AIL}}$ (p)                        & 44.6M & \ccA[59.8]{$59.8$} & \ccB[80.55]{$80.5$} & \ccC[45.75]{$45.8$} & \ccD[41.4]{$41.4$} & \ccE[67.15]{$67.2$} & \ccF[48.2]{$48.2$} & \ccG[67.45]{$67.5$} & \ccH[68.1]{$68.1$} \\
$\opn{OR}, \opn{XNOR_{AIL}}$ (d)                        & 45.6M & \ccA[62.8429]{$62.8$} & \ccB[83.7]{$83.7$} & \ccC[49.1]{$49.1$} & \ccD[43.3]{$\mbs{43.3}$} & \ccE[68.1]{$68.1$} & \ccF[49.55]{$\mbf{49.5}$} & \ccG[73.85]{$\mbs{73.8}$} & \ccH[72.25]{$72.2$} \\
$\opn{OR}, \opn{AND}, \opn{XNOR_{AIL}}$ (p)             & 44.6M & \ccA[55.0]{$55.0$} & \ccB[68.25]{$68.2$} & \ccC[47.2]{$47.2$} & \ccD[41.1]{$41.1$} & \ccE[65.05]{$65.0$} & \ccF[45.0]{$45.0$} & \ccG[60.95]{$61.0$} & \ccH[57.45]{$57.5$} \\
\bottomrule
\end{tabular}
  }
}
\end{table}

We found the network using $\opn{OR_{AIL}}$ activation function performed best overall, and across most of the subtasks.
The second best performing activation functions were \{$\opn{OR_{AIL}}$, $\opn{XNOR_{AIL}}$ (d)\} and ReLU.

\FloatBarrier

\subsection{Compositional Zero-Shot Learning}
\label{s:compositional}
\label{s:czsl}
\label{a:compositional}
\label{a:czsl}

Zero-shot learning encompasses all problems which involve completing novel tasks which the subject has never seen before.
The subject must infer both the task and its solution based on their previous experiences (a meta-learning task).
Compositional zero-shot learning is a subset of zero-shot learning which involves combining knowledge about multiple stimulus properties together in novel pairings.
For instance, if the network has been trained on ``sliced bread,'' ``sliced pear,'' and ``caramelized pear,'' is it able to classify images of ``caramelized bread'' despite having never seen an example of this before?

We based our experiments on the Task-driven Modular Networks (TMN) proposed by \citet{taskdriven}.
We used the code shared by the authors, but were unable to replicate the results they reported in the paper when using a different random seed (see \autoref{tab:tmn-og}).
We adapted this network by changing out all the ReLU activation functions in the gate and module networks with a different activation function.
Because the modules each terminate with an activation function, we needed to double the size of the hidden layer for some of our networks in order to maintain the dimensionality of the output.
Consequently, some experiments had around 50\% more parameters in total than others.

Experiments were performed on the MIT-States dataset \citep{mit-states}.
We trained and tested on the corresponding partitions of the dataset as introduced by \citet{taskdriven}.
We used the same paradigm as \citet{taskdriven}: ADAM \citep{adam} with a learning rate of 0.001 for the module network and 0.01 for the gating network, momentum 0.9, batch size 256, weight decay \num{5e-5}.
We evaluated the network using the AUC between seen and unseen samples \citep{taskdriven}.
The network was trained until the validation AUC had plateaued, determined by not increasing for 5 epochs.
We selected the model from the epoch with highest validation AUC to apply to the test set.
The best performance was typically attained after around 5 epochs.

As shown in \autoref{tab:tmn-og}, we find that $\opn{OR_{AIL}}$ performs best, but uses more parameters, and is very closely followed by \{$\opn{OR}, \opn{XNOR_{AIL}}$ (d)\},.

\begin{table}[htp]
\small
  \centering
  \caption{%
Performance of TMN-based networks at compositional zero-shot learning (CZSL) on the MIT-States dataset.
Mean (standard error) of $n\!=\!5$ random initializations. Bold: best. Underlined: top two. Background: linear color scale from ReLU baseline (white) to best (black).
}
\label{tab:tmn-og}
\def\ccAmin{2.47}%
\def\ccAmax{2.65}%
\def\ccA[#1]#2{\heatmapcell{#1}{#2}{\ccAmin}{\ccAmax}}%
\def\ccBmin{6.42}%
\def\ccBmax{6.80}%
\def\ccB[#1]#2{\heatmapcell{#1}{#2}{\ccBmin}{\ccBmax}}%
\def\ccCmin{10.27}%
\def\ccCmax{10.78}%
\def\ccC[#1]#2{\heatmapcell{#1}{#2}{\ccCmin}{\ccCmax}}%
  \scalebox{0.96}{
\begin{tabular}{lccrrr}
\toprule
 & & & \multicolumn{3}{c}{MIT-States Test AUC (\%)} \\
\cmidrule(l){4-6}
Activation function                                     & Mapping      & Params & Top-1           & Top-2           & Top-3           \\
\toprule
TMN \citep{taskdriven}                      & $1\!\to\! 1$ & & \multicolumn{1}{c}{$2.9$}    & \multicolumn{1}{c}{$7.1$} & \multicolumn{1}{c}{$11.5$} \\
TMN repro. ($\opn{ReLU}$)                   & $1\!\to\! 1$ & 438\,k & \ccA[ 2.47]{$ 2.47\wpm{0.07}$}      & \ccB[ 6.42]{$ 6.42\wpm{0.09}$}       & \ccC[10.27]{$10.27\wpm{0.11}$} \\
\midrule
$\opn{Max}$                                 & $2\!\to\! 1$ & 650\,k & \ccA[ 2.42]{$ 2.42\wpm{0.07}$}      & \ccB[ 6.37]{$ 6.37\wpm{0.08}$}       & \ccC[10.28]{$10.28\wpm{0.06}$} \\
$\opn{Max},\opn{Min}$ (d)                   & $1\!\to\! 1$ & 438\,k & \ccA[ 2.53]{$ 2.53\wpm{0.06}$}      & \ccB[ 6.69]{$ 6.69\wpm{0.11}$}       & \ccC[10.61]{$10.61\wpm{0.14}$} \\
\midrule
$\opn{XNOR_{AIL}}$                          & $2\!\to\! 1$ & 650\,k & \ccA[ 1.22]{$ 1.22\wpm{0.05}$}      & \ccB[ 3.47]{$ 3.47\wpm{0.12}$}       & \ccC[ 5.82]{$ 5.82\wpm{0.19}$} \\
$\opn{OR_{AIL}}$                            & $2\!\to\! 1$ & 650\,k & \ccA[ 2.65]{$\mbf{2.65}\wpm{0.05}$} & \ccB[ 6.80]{$\mbf{6.80}\wpm{0.08}$}  & \ccC[10.78]{$\mbf{10.78}\wpm{0.13}$} \\
$\opn{OR}, \opn{AND_{AIL}}$ (d)             & $2\!\to\! 2$ & 438\,k & \ccA[ 2.61]{$\mbs{2.61}\wpm{0.05}$} & \ccB[ 6.73]{$\mbs{6.73}\wpm{0.05}$}  & \ccC[10.77]{$\mbs{10.77}\wpm{0.12}$} \\
$\opn{OR}, \opn{XNOR_{AIL}}$ (d)            & $2\!\to\! 2$ & 438\,k & \ccA[ 1.89]{$ 1.89\wpm{0.13}$}      & \ccB[ 5.12]{$ 5.12\wpm{0.21}$}       & \ccC[ 8.35]{$ 8.35\wpm{0.24}$} \\
\bottomrule
\end{tabular}
  }
\end{table}

Since we could not flexibly scale the total number of parameters in the network with the original architecture, we modified the TMN architecture by adding an additional linear layer to the end of each module which projects from the activation function to the embedding space.
The modules would otherwise terminate with an activation function, which makes it difficult to handle activation functions which map $2\!\to\! 1$.
We dub the modified version of the network ``TMNx.''

Comparing the TMN results in \autoref{tab:tmn-og} to TMNx in \autoref{tab:tmnx}, we can see that adding the extra linear layer improved performance of the network in of itself.
Intuitively, this makes sense since the output of the TMN modules are weighted with the output of the gating network and then summed, and this weighting and summing of evidence is best performed with on logits instead of the truncated output of ReLU units.
But also, performance may have improved just because the model became larger.

We performed a hyperparameter search across the training parameters for the new network against the validation set using the ReLU activation function only.
We adopted the hyperparameters discovered for ReLU for all other activation functions.
The batch size was reduced to 128 due to the increase in size of the model.
The discovered hyperparameters were a learning rate of \num{3e-3} for both the module and gating network, and a weight decay of \num{1e-5}.
Other training hyperparameters, such as the ratio of negative samples to present, were left unchanged.

In order to match the number of parameters in the network, we used the original TMN hidden width of $64$ for the module and gater networks with activation functions which map $1\!\to \!1$, and increased this to $70$ for activation function which map $2\!\to \!1$, to maintain the total number of trainable parameters.
To investigate a comprehensive set of baselines, we compared against every activation function implemented in PyTorch~v1.10.
The results are shown in \autoref{tab:tmnx}.

\begin{table}[h]
\small
  \centering
  \caption{%
Performance of TMNx networks at compositional zero-shot learning (CZSL) on the MIT-States dataset.
Pre-activation width of 60, 64, 70, or 72 (depending on activation function, to keep the number of parameters approximately constant).
Mean (standard error) of $n\!=\!5$ random initializations.
Bold: best.
Underlined: top two.
Italic: no significant difference from best (two-sided Student's $t$-test, $p\!>\!0.05$).
Background: color scale from second-worst in column to best, linear with accuracy value.
}
\label{tab:tmnx}
  \scalebox{0.96}{
\def\ccAmin{1.7823763506491868}%
\def\ccAmax{3.173191639676634}%
\def\ccA[#1]#2{\heatmapcell{#1}{#2}{\ccAmin}{\ccAmax}}%
\def\ccBmin{4.840848213588182}%
\def\ccBmax{7.805374189785475}%
\def\ccB[#1]#2{\heatmapcell{#1}{#2}{\ccBmin}{\ccBmax}}%
\def\ccCmin{7.950284499804558}%
\def\ccCmax{12.366833171749144}%
\def\ccC[#1]#2{\heatmapcell{#1}{#2}{\ccCmin}{\ccCmax}}%
\begin{tabular}{lcccrrr}
\toprule
 & & & \multicolumn{3}{c}{MIT-States Test Accuracy (\%)} \\
\cmidrule(r){4-6}
Activation function                                     & Mapping & Width & Params & \multicolumn{1}{c}{Top-1} & \multicolumn{1}{c}{Top-2} \\
\toprule
$\opn{ReLU}$                                            & $1\!\to\! 1$ & 64 & 1.15M & \ccA[ 2.7594]{$ 2.76\wpm{0.08}$} & \ccB[ 7.1055]{$ 7.11\wpm{0.08}$} & \ccC[11.2569]{$11.26\wpm{0.09}$} \\
LeakyReLU                                               & $1\!\to\! 1$ & 64 & 1.15M & \ccA[ 2.7165]{$ 2.72\wpm{0.09}$} & \ccB[ 6.9855]{$ 6.99\wpm{0.17}$} & \ccC[10.9468]{$10.95\wpm{0.20}$} \\
PReLU                                                   & $1\!\to\! 1$ & 64 & 1.15M & \ccA[ 2.8280]{$ 2.83\wpm{0.06}$} & \ccB[ 7.2549]{$ 7.25\wpm{0.14}$} & \ccC[11.4415]{$11.44\wpm{0.17}$} \\
Softplus                                                & $1\!\to\! 1$ & 64 & 1.15M & \ccA[ 2.9215]{$\mbns{ 2.92}\wpm{0.11}$} & \ccB[ 7.4650]{$\mbns{ 7.46}\wpm{0.19}$} & \ccC[11.7499]{$11.75\wpm{0.18}$} \\
\midrule
ELU                                                     & $1\!\to\! 1$ & 64 & 1.15M & \ccA[ 3.1261]{$\mbs{\mbns{ 3.13}}\wpm{0.05}$} & \ccB[ 7.7834]{$\mbs{\mbns{ 7.78}}\wpm{0.14}$} & \ccC[12.0398]{$\mbs{\mbns{12.04}}\wpm{0.20}$} \\
CELU                                                    & $1\!\to\! 1$ & 64 & 1.15M & \ccA[ 3.1732]{$\mbf{\mbns{ 3.17}}\wpm{0.06}$} & \ccB[ 7.8054]{$\mbf{\mbns{ 7.81}}\wpm{0.15}$} & \ccC[12.0013]{$\mbns{12.00}\wpm{0.19}$} \\
SELU                                                    & $1\!\to\! 1$ & 64 & 1.15M & \ccA[ 3.0508]{$\mbns{ 3.05}\wpm{0.04}$} & \ccB[ 7.7724]{$\mbns{ 7.77}\wpm{0.10}$} & \ccC[12.3668]{$\mbf{\mbns{12.37}}\wpm{0.16}$} \\
GELU                                                    & $1\!\to\! 1$ & 64 & 1.15M & \ccA[ 2.8089]{$ 2.81\wpm{0.06}$} & \ccB[ 7.1900]{$ 7.19\wpm{0.11}$} & \ccC[11.3412]{$11.34\wpm{0.15}$} \\
SiLU                                                    & $1\!\to\! 1$ & 64 & 1.15M & \ccA[ 2.8714]{$ 2.87\wpm{0.08}$} & \ccB[ 7.3431]{$\mbns{ 7.34}\wpm{0.14}$} & \ccC[11.6002]{$11.60\wpm{0.18}$} \\
Hardswish                                               & $1\!\to\! 1$ & 64 & 1.15M & \ccA[ 2.8420]{$ 2.84\wpm{0.08}$} & \ccB[ 7.2270]{$ 7.23\wpm{0.11}$} & \ccC[11.4838]{$11.48\wpm{0.11}$} \\
Mish                                                    & $1\!\to\! 1$ & 64 & 1.15M & \ccA[ 2.8995]{$ 2.90\wpm{0.08}$} & \ccB[ 7.3722]{$\mbns{ 7.37}\wpm{0.17}$} & \ccC[11.5904]{$11.59\wpm{0.13}$} \\
\midrule
Softsign                                                & $1\!\to\! 1$ & 64 & 1.15M & \ccA[ 2.8978]{$ 2.90\wpm{0.04}$} & \ccB[ 7.2768]{$ 7.28\wpm{0.09}$} & \ccC[11.5322]{$11.53\wpm{0.18}$} \\
Tanh                                                    & $1\!\to\! 1$ & 64 & 1.15M & \ccA[ 2.9276]{$ 2.93\wpm{0.09}$} & \ccB[ 7.4695]{$\mbns{ 7.47}\wpm{0.16}$} & \ccC[11.6358]{$11.64\wpm{0.20}$} \\
\midrule
GLU                                                     & $2\!\to\! 1$ & 70 & 1.16M & \ccA[ 2.5867]{$ 2.59\wpm{0.06}$} & \ccB[ 6.8578]{$ 6.86\wpm{0.15}$} & \ccC[10.9114]{$10.91\wpm{0.24}$} \\
\midrule
$\opn{Max}$                                             & $2\!\to\! 1$ & 70 & 1.16M & \ccA[ 2.4544]{$ 2.45\wpm{0.08}$} & \ccB[ 6.5181]{$ 6.52\wpm{0.18}$} & \ccC[10.5336]{$10.53\wpm{0.29}$} \\
$\opn{Max},\opn{Min}$ (d)                               & $2\!\to\! 2$ & 64 & 1.15M & \ccA[ 2.5334]{$ 2.53\wpm{0.10}$} & \ccB[ 6.6133]{$ 6.61\wpm{0.15}$} & \ccC[10.6482]{$10.65\wpm{0.15}$} \\
\midrule
SignedGeomean                                           & $2\!\to\! 1$ & 70 & 1.16M & \ccA[ 1.6482]{$ 1.65\wpm{0.04}$} & \ccB[ 4.4034]{$ 4.40\wpm{0.08}$} & \ccC[ 7.1752]{$ 7.18\wpm{0.07}$} \\
\midrule
$\opn{XNOR_{IL}}$                                       & $2\!\to\! 1$ & 70 & 1.16M & \ccA[ 2.0276]{$ 2.03\wpm{0.04}$} & \ccB[ 5.4294]{$ 5.43\wpm{0.02}$} & \ccC[ 8.8167]{$ 8.82\wpm{0.05}$} \\
$\opn{OR_{IL}}$                                         & $2\!\to\! 1$ & 70 & 1.16M & \ccA[ 2.8978]{$ 2.90\wpm{0.05}$} & \ccB[ 7.3965]{$\mbns{ 7.40}\wpm{0.11}$} & \ccC[11.6757]{$11.68\wpm{0.14}$} \\
$\opn{OR}, \opn{AND_{IL}}$ (d)                          & $2\!\to\! 2$ & 64 & 1.15M & \ccA[ 2.6808]{$ 2.68\wpm{0.06}$} & \ccB[ 7.0494]{$ 7.05\wpm{0.08}$} & \ccC[11.1687]{$11.17\wpm{0.12}$} \\
$\opn{OR}, \opn{XNOR_{IL}}$ (p)                         & $2\!\to\! 1$ & 72 & 1.19M & \ccA[ 2.4150]{$ 2.41\wpm{0.09}$} & \ccB[ 6.1987]{$ 6.20\wpm{0.13}$} & \ccC[ 9.8223]{$ 9.82\wpm{0.14}$} \\
$\opn{OR}, \opn{XNOR_{IL}}$ (d)                         & $2\!\to\! 2$ & 64 & 1.15M & \ccA[ 2.5995]{$ 2.60\wpm{0.04}$} & \ccB[ 6.6811]{$ 6.68\wpm{0.12}$} & \ccC[10.7050]{$10.71\wpm{0.17}$} \\
$\opn{OR}, \opn{AND}, \opn{XNOR_{IL}}$ (p)              & $2\!\to\! 1$ & 72 & 1.19M & \ccA[ 2.4957]{$ 2.50\wpm{0.04}$} & \ccB[ 6.3403]{$ 6.34\wpm{0.16}$} & \ccC[10.0535]{$10.05\wpm{0.24}$} \\
$\opn{OR}, \opn{AND}, \opn{XNOR_{IL}}$ (d)              & $2\!\to\! 3$ & 60 & 1.15M & \ccA[ 2.3901]{$ 2.39\wpm{0.05}$} & \ccB[ 6.3057]{$ 6.31\wpm{0.07}$} & \ccC[10.1745]{$10.17\wpm{0.12}$} \\
\midrule
$\opn{XNOR_{NIL}}$                                      & $2\!\to\! 1$ & 70 & 1.16M & \ccA[ 2.0773]{$ 2.08\wpm{0.07}$} & \ccB[ 5.5461]{$ 5.55\wpm{0.09}$} & \ccC[ 9.1088]{$ 9.11\wpm{0.10}$} \\
$\opn{OR_{NIL}}$                                        & $2\!\to\! 1$ & 70 & 1.16M & \ccA[ 2.9194]{$ 2.92\wpm{0.07}$} & \ccB[ 7.4156]{$\mbns{ 7.42}\wpm{0.13}$} & \ccC[11.7442]{$11.74\wpm{0.14}$} \\
$\opn{OR}, \opn{AND_{NIL}}$ (d)                         & $2\!\to\! 2$ & 64 & 1.15M & \ccA[ 2.8190]{$ 2.82\wpm{0.11}$} & \ccB[ 7.4345]{$\mbns{ 7.43}\wpm{0.23}$} & \ccC[11.6912]{$11.69\wpm{0.19}$} \\
$\opn{OR}, \opn{XNOR_{NIL}}$ (p)                        & $2\!\to\! 1$ & 72 & 1.19M & \ccA[ 2.4975]{$ 2.50\wpm{0.08}$} & \ccB[ 6.5075]{$ 6.51\wpm{0.14}$} & \ccC[10.4027]{$10.40\wpm{0.17}$} \\
$\opn{OR}, \opn{XNOR_{NIL}}$ (d)                        & $2\!\to\! 2$ & 64 & 1.15M & \ccA[ 2.4792]{$ 2.48\wpm{0.03}$} & \ccB[ 6.5568]{$ 6.56\wpm{0.11}$} & \ccC[10.6346]{$10.63\wpm{0.12}$} \\
$\opn{OR}, \opn{AND}, \opn{XNOR_{NIL}}$ (p)             & $2\!\to\! 1$ & 72 & 1.19M & \ccA[ 2.4836]{$ 2.48\wpm{0.03}$} & \ccB[ 6.5401]{$ 6.54\wpm{0.07}$} & \ccC[10.3754]{$10.38\wpm{0.08}$} \\
$\opn{OR}, \opn{AND}, \opn{XNOR_{NIL}}$ (d)             & $2\!\to\! 3$ & 60 & 1.15M & \ccA[ 2.5370]{$ 2.54\wpm{0.05}$} & \ccB[ 6.7393]{$ 6.74\wpm{0.11}$} & \ccC[10.7164]{$10.72\wpm{0.11}$} \\
\midrule
$\opn{XNOR_{AIL}}$                                      & $2\!\to\! 1$ & 70 & 1.16M & \ccA[ 1.8818]{$ 1.88\wpm{0.06}$} & \ccB[ 4.9391]{$ 4.94\wpm{0.15}$} & \ccC[ 7.9910]{$ 7.99\wpm{0.19}$} \\
$\opn{OR_{AIL}}$                                        & $2\!\to\! 1$ & 70 & 1.16M & \ccA[ 2.6113]{$ 2.61\wpm{0.08}$} & \ccB[ 6.8346]{$ 6.83\wpm{0.12}$} & \ccC[10.9730]{$10.97\wpm{0.20}$} \\
$\opn{OR}, \opn{AND_{AIL}}$ (d)                         & $2\!\to\! 2$ & 64 & 1.15M & \ccA[ 2.8054]{$ 2.81\wpm{0.06}$} & \ccB[ 7.1158]{$ 7.12\wpm{0.10}$} & \ccC[11.2698]{$11.27\wpm{0.14}$} \\
$\opn{OR}, \opn{XNOR_{AIL}}$ (p)                        & $2\!\to\! 1$ & 72 & 1.19M & \ccA[ 2.4955]{$ 2.50\wpm{0.06}$} & \ccB[ 6.3405]{$ 6.34\wpm{0.14}$} & \ccC[10.0732]{$10.07\wpm{0.21}$} \\
$\opn{OR}, \opn{XNOR_{AIL}}$ (d)                        & $2\!\to\! 2$ & 64 & 1.15M & \ccA[ 2.4614]{$ 2.46\wpm{0.04}$} & \ccB[ 6.5118]{$ 6.51\wpm{0.11}$} & \ccC[10.4505]{$10.45\wpm{0.17}$} \\
$\opn{OR}, \opn{AND}, \opn{XNOR_{AIL}}$ (p)             & $2\!\to\! 1$ & 72 & 1.19M & \ccA[ 2.4478]{$ 2.45\wpm{0.12}$} & \ccB[ 6.3331]{$ 6.33\wpm{0.21}$} & \ccC[10.0897]{$10.09\wpm{0.23}$} \\
$\opn{OR}, \opn{AND}, \opn{XNOR_{AIL}}$ (d)             & $2\!\to\! 3$ & 60 & 1.15M & \ccA[ 2.5656]{$ 2.57\wpm{0.03}$} & \ccB[ 6.8311]{$ 6.83\wpm{0.11}$} & \ccC[10.9724]{$10.97\wpm{0.17}$} \\
\midrule
$\opn{XNOR_{NAIL}}$                                     & $2\!\to\! 1$ & 70 & 1.16M & \ccA[ 1.7824]{$ 1.78\wpm{0.04}$} & \ccB[ 4.8408]{$ 4.84\wpm{0.11}$} & \ccC[ 7.9503]{$ 7.95\wpm{0.19}$} \\
$\opn{OR_{NAIL}}$                                       & $2\!\to\! 1$ & 70 & 1.16M & \ccA[ 2.5110]{$ 2.51\wpm{0.03}$} & \ccB[ 6.5843]{$ 6.58\wpm{0.12}$} & \ccC[10.5557]{$10.56\wpm{0.24}$} \\
$\opn{OR}, \opn{AND_{NAIL}}$ (d)                        & $2\!\to\! 2$ & 64 & 1.15M & \ccA[ 2.6745]{$ 2.67\wpm{0.04}$} & \ccB[ 6.9810]{$ 6.98\wpm{0.06}$} & \ccC[11.0236]{$11.02\wpm{0.16}$} \\
$\opn{OR}, \opn{XNOR_{NAIL}}$ (p)                       & $2\!\to\! 1$ & 72 & 1.19M & \ccA[ 2.2920]{$ 2.29\wpm{0.10}$} & \ccB[ 6.1344]{$ 6.13\wpm{0.19}$} & \ccC[ 9.9465]{$ 9.95\wpm{0.17}$} \\
$\opn{OR}, \opn{XNOR_{NAIL}}$ (d)                       & $2\!\to\! 2$ & 64 & 1.15M & \ccA[ 2.4304]{$ 2.43\wpm{0.06}$} & \ccB[ 6.3182]{$ 6.32\wpm{0.09}$} & \ccC[10.2828]{$10.28\wpm{0.18}$} \\
$\opn{OR}, \opn{AND}, \opn{XNOR_{NAIL}}$ (p)            & $2\!\to\! 1$ & 72 & 1.19M & \ccA[ 2.4494]{$ 2.45\wpm{0.07}$} & \ccB[ 6.3426]{$ 6.34\wpm{0.14}$} & \ccC[10.1063]{$10.11\wpm{0.12}$} \\
$\opn{OR}, \opn{AND}, \opn{XNOR_{NAIL}}$ (d)            & $2\!\to\! 3$ & 60 & 1.15M & \ccA[ 2.5783]{$ 2.58\wpm{0.09}$} & \ccB[ 6.8522]{$ 6.85\wpm{0.18}$} & \ccC[10.9399]{$10.94\wpm{0.24}$} \\
\bottomrule
\end{tabular}
  }
\end{table}

We found that the ELU family of activations (ELU, CELU, SELU) performed best across all three top-k values.
Our best activation functions were $\opn{OR_{(N)IL}}$ and \{$\opn{OR}, \opn{AND_{(N)IL}}$ (d)\}, which came next alongside Softplus and Tanh, outperforming ReLU.
Activation functions using $\opn{XNOR}$ and our AIL approximations performed less well on this task, and were did not surpass ReLU.
These results suggests this is a domain where logical activation functions are less well suited.

\end{document}